\documentclass[11pt]{article}
\usepackage{geometry}
\usepackage{verbatim}
\usepackage{graphicx}
\usepackage{titletoc}
\usepackage{amsmath, amsfonts, bm, mathrsfs}
\usepackage{amsmath, amsfonts, amssymb, bm, mathrsfs}
\usepackage{amsthm,graphicx}
\usepackage[dvipsnames]{xcolor}
\usepackage[unicode=true,
 bookmarks=false,
 breaklinks=false,
 pdfborder={0 0 1},colorlinks=false
 ]{hyperref}
\hypersetup{colorlinks,citecolor=OliveGreen,filecolor=blue,linkcolor=blue,urlcolor=blue}
\allowdisplaybreaks
\usepackage{float}
\usepackage{multirow}
\usepackage{footnote}
\usepackage{dsfont}
\usepackage{booktabs}
\usepackage{enumitem}
\usepackage{tikz}
\usetikzlibrary{arrows.meta}
\usepackage{pgfplots}
\pgfplotsset{compat=1.17}

\usepackage[linesnumbered,ruled,vlined]{algorithm2e}
\usepackage{algorithmic}

\definecolor{cc}{RGB}{0,127,0}

\usepackage{dsfont}

\title{Phase Transitions for Feature Learning in Neural Networks \vspace{0.5em}}

\author{\vspace{0em}Andrea Montanari\thanks{Department of Statistics and Department of Mathematics, Stanford University}\and Zihao Wang\thanks{Department of Mathematics, Stanford University}}

\date{\vspace{0em}\today}

\newcommand{\Tr}{\operatorname{Tr}}
\newcommand{\de}{\mathrm{d}}
\newcommand{\E}{\mathbb{E}}

\newcommand{\opnorm}[1]{{\left\vert\kern-0.25ex\left\vert\kern-0.25ex\left\vert #1 
\right\vert\kern-0.25ex\right\vert\kern-0.25ex\right\vert}}

\newcommand{\ba}{\boldsymbol{a}}

\newcommand{\bO}{\boldsymbol{O}}

\newcommand{\bb}{\boldsymbol{b}}

\newcommand{\norm}[1]{\left\|{#1}\right\|}

\newcommand{\abs}[1]{\left\vert#1\right\vert}

\def\naturals{{\mathbb N}}
\def\reals{{\mathbb R}}

\DeclareMathOperator*{\plim}{p-lim}

\DeclareMathOperator*{\argmax}{arg\,max}

\newtheoremstyle{myexample}
    {\topsep}           
    {\topsep}                
    {\rmfamily\small }                
    {}                  
    {\bfseries }                
    {.5em}                  
    {}

\newtheoremstyle{myremark}
    {\topsep}             
    {\topsep}                       
    {\rmfamily}
    {}
    {\bfseries}              
    {.}                          
    {.5em}        
    {}

\newtheorem{claim}{Claim}[section]
\newtheorem{lemma}[claim]{Lemma}

\newtheorem{assumption}{Assumption}

\newtheorem{theorem}{Theorem}

\newtheorem{definition}[claim]{Definition}

\theoremstyle{myremark}
\newtheorem{remark}{Remark}[section]

\theoremstyle{myremark}

\theoremstyle{myexample}

\definecolor{darkblue}{RGB}{25,25,200}

\def\<{\langle}
\def\>{\rangle}

\def\ssim{\mbox{\tiny\rm sim}}

\def\GP{{\sf GP}}
\def\GeLU{{\sf GeLU}}

\def\bzero{\boldsymbol{0}}

\def\eps{{\varepsilon}}

\def\bI{{\boldsymbol I}}

\def\sT{{\sf T}}

\def\bTheta{{\boldsymbol \Theta}}

\def\bQ{{\boldsymbol Q}}

\def\det{{\rm det}}

\def\upper{\mathbb{H}}

\def\hhF{\widehat{F}}
\def\rE{\mathrm{E}}
\def\rH{\mathrm{H}}
\def\indep{{\perp\!\!\!\perp}}

\def\normal{{\sf N}}

\def\cT{{\cal T}}

\def\S{\mathcal{S}}

\def\ssp{\mathsf{sp}}
\def\sopt{\mathsf{opt}}
\def\Quad{\mathsf{Quad}}

\def\cF{{\cal F}}

\def\bG{{\boldsymbol G}}

\def\ed{\stackrel{{\rm d}}{=}}

\def\salg{\sf alg}
\def\sIT{\sf IT}
\def\sNN{\sf NN}

\def\cuU{\mathscr{U}}

\def\Uhard{{\mathfrak U}_{\mathrm H}}
\def\Ueasy{{\mathfrak U}_{\mathrm E}}

\def\cK{{\mathcal{K}}}

\def\cO{{\mathcal{O}}}

\def\bA{{\boldsymbol A}}

\def\bB{{\boldsymbol B}}
\def\bC{{\boldsymbol C}}
\def\bH{{\boldsymbol H}}
\def\bW{{\boldsymbol W}}
\def\bR{{\boldsymbol R}}
\def\bU{{\boldsymbol U}}

\def\bK{{\boldsymbol K}}
\def\bM{{\boldsymbol M}}
\def\bD{{\boldsymbol D}}

\def\bT{{\boldsymbol T}}
\def\bX{{\boldsymbol X}}
\def\bY{{\boldsymbol Y}}
\def\bZ{{\boldsymbol Z}}

\def\be{{\boldsymbol e}}

\def\bu{{\boldsymbol u}}

\def\bx{{\boldsymbol x}}

\def\by{{\boldsymbol y}}

\def\btheta{{\boldsymbol \theta}}
\def\bxi{{\boldsymbol \xi}}

\def\bP{{\boldsymbol P}}

\def\bF{{\boldsymbol F}}

\def\bf{{\boldsymbol f}}
\def\bS{{\boldsymbol S}}

\def\hbtheta{\widehat{\boldsymbol \theta}}

\def\beps{{\boldsymbol \varepsilon}}
\def\br{{\boldsymbol r}}

\def\spn{{\rm span}}

\def\prob{{\mathbb P}}

\def\Treg[#1]{T^{{\rm reg},#1}}
\def\GW[#1]{{\rm GW}(#1)}
\def\MGW[#1]{{\rm MGW}(#1)}

\def\hard{{\mathrm H}}
\def\easy{{\mathrm E}}

\def\Unif{{\sf Unif}}
\def\Risk{{\sf Risk}}

\def\supp{{\rm supp}}
\def\dim{{\rm dim}}
\def\ker{{\rm ker}}
\def\rank{{\rm rank}}

\renewcommand{\Re}{\mathrm{Re}}
\renewcommand{\Im}{\mathrm{Im}}

\renewcommand{\leq}{\leqslant}

\newcommand{\cD}{\mathcal{D}}

\begin{document}

\maketitle

\begin{abstract}
According to a popular viewpoint, neural networks learn from data by first identifying
effective low-dimensional representations, and subsequently fitting the 
best model in this low-dimensional space. A  sequence of recent works provides
a rigorous formalization of this phenomenon when learning multi-index models.
In this setting, we are given $n$ i.i.d. pairs of covariate vectors and responses $(\bx_i,y_i)$, 
where the vectors $\bx_i\in\mathbb{R}^d$ are isotropic, and responses $y_i$ only depend on $\bx_i$
through a $k$-dimensional projection $\bTheta_*^{\sT}\bx_i$. Feature
learning amounts to learning the latent space spanned by $\bTheta_*$.

In this context, we study the gradient descent dynamics of two-layer neural networks
under the proportional asymptotics $n,d\to\infty$, $n/d\to\delta\in (0,\infty)$,
while the dimension of the latent space $k$ and the number of hidden neurons $m$ 
are kept fixed. Earlier work establishes that feature learning via polynomial-time algorithms
is possible if $\delta> \delta_{\salg}$, for $\delta_{\salg}$ a threshold depending on the data
distribution, and is impossible (within a certain class of efficient algorithms) below $\delta_{\salg}$.
Here we derive an analogous threshold $\delta_{\sNN}$ for a class of two-layer networks.
Our characterization of $\delta_{\sNN}$ is sufficiently explicit to open
the way to study the dependence of learning dynamics 
on the target function, loss, activation, network width, and initialization.

The threshold $\delta_{\sNN}$ is determined by the following training scenario. 
Gradient descent first visits points for which the gradient of the empirical risk
is large and learns the 
directions spanned by these gradients. Then the gradient becomes smaller and the 
dynamics becomes dominated by negative directions of the Hessian. 
The threshold $\delta_{\sNN}$ corresponds to a phase transition in the spectrum of the Hessian
after the first phase of training, and provides a quantitative explanation for empirical
phenomena such as grokking.
\end{abstract}

\setcounter{tocdepth}{2}
\tableofcontents

\section{Introduction}

Pressed to explain the success of modern deep learning systems, many practitioners might mention the ability
of these systems to `learn low-dimensional representation of the data,' during training. 
This is often contrasted to linear methods (e.g. kernel methods or neural networks under `lazy training')
which work in a fixed representation that is determined by a human's choice of the kernel or featurization map.

An important line of recent work has provided a formalization of this
intuition for the problem of learning multi-index models, see \cite{ba2022highdimensionalasymptoticsfeaturelearning,damian2022neuralnetworkslearnrepresentations,abbe2023sgdlearningneuralnetworks,berthier2024learning,zhang2025neuralnetworkslearngeneric} for a few pointers to this literature.
This is a classical supervised learning problem.
We observe $n$ i.i.d. samples $(\bx_i,y_i)$,  ${i\le n}$ with 
$d$-dimensional covariates $\bx_i$ which we assume don't have low-dimensional structure:
$\bx_i\sim \normal(\boldsymbol{0}, \bI_d)$.
The responses $y_i$ depend on a $k$-dimensional projection
of the covariates:
\begin{align} 
\label{eq:FirstMultiIndex}
y_i = h\big(\bTheta_*^\sT\bx_i,\,\eps_i\big),\qquad \eps_i\sim \normal(0,1),\quad i\in[n]\, .
\end{align}
Here $\bTheta_*\in\reals^{d\times k}$ has orthonormal columns, noise $\eps_i$ is independent of $\bx_i$, and $h:\reals^{k+1}\to\reals$ is a link function. 

We will attempt to learn a predictive model using a two-layer neural network:
\begin{align}
f_{\bTheta}(\bx) := \frac{1}{m}\sum_{j=1}^m a_j\sigma\big(\btheta_j^\sT\bx+b_j\big),\qquad \bTheta:=\begin{bmatrix}\btheta_1&\dots&\btheta_m\end{bmatrix}\in\reals^{d\times m}\, .\label{eq:FirstNeuralModel}
\end{align}
In order to focus on the feature-learning aspects
of this problem and avoid inessential technicalities, we will only train the first-layer 
weights $\bTheta$ and and regard $(a_j,b_j)_{j=1}^m$ as fixed.
We also emphasize that the number of neurons $m$ is in general unrelated to the latent dimension $k$. Throughout, we will regard $k$ as constants independent of $n,d\to\infty$. We will eventually let $m\to\infty$ \emph{after} $n,d\to\infty$.

Given a loss function $\ell(\cdot,\cdot):\mathbb{R}\times\mathbb{R}\to\mathbb{R}$, we 
are interested in the dynamics of gradient descent (GD) on the empirical risk
\begin{align}
\Risk(\bTheta):=\frac{1}{n}\sum_{i=1}^n \ell\big(y_i, f_{\bTheta}(\bx_i)\big). \label{eq:EmpiricalRisk}
\end{align}

The neural network \eqref{eq:FirstNeuralModel} is a natural
model for the multi-index target distribution. Indeed any distribution \eqref{eq:FirstMultiIndex}
(more precisely, the conditional expectation $\E[y|\bx]=h_{\mathsf E}(\bTheta_*^\sT\bx)$) can be 
approximated by a network of the form \eqref{eq:FirstNeuralModel}, with $\|\ba\|_1/m\le R_0$ (for $R_0$ a
sufficiently large constant), resulting in generalization error bounded by $C R_0\sqrt{d/n}$
\cite{bartlett1996valid,bach2017breaking}. Vice versa, any such neural network is well approximated by
a $k$-index model with $k=O(R_0\vee 1)$. 

Since the generalization error is $O(R_0\sqrt{d/n})$, learning by empirical risk minimization succeeds with $n=O(d)$ samples.
Hence, the key question is whether GD (with respect to the empirical risk \eqref{eq:EmpiricalRisk})
can achieve the same and whether it is optimal among algorithms with comparable complexity.
Despite a copious literature on this model, no sharp results exist about this question.

In order to explain the issue, consider for the moment 
the case of a single-index model, $k=1$, and
assume $h(z)=h(z,\eps)$ is a generic even\footnote{If $h$ is not even, then for any $\delta>0$ the latent direction
$\btheta_0$ can be learned with non-trivial accuracy by  linear regression, after preprocessing the $y_i$'s.} function of $z$. 
Letting $n,d\to \infty$ with $\delta_n:=n/d\rightarrow \delta$, then 
there exists an estimator $\hbtheta=\hbtheta(\by,\bX)$ with $|\<\hbtheta,\btheta_0\>|/\|\hbtheta\|$
bounded away from zero (i.e. it achieves `weak recovery') for any $\delta>\delta_{\sIT}$, 
while such an estimator does not exist for $\delta<\delta_{\sIT}$.
The `information theoretic' (or `statistical') $\delta_{\sIT} = \delta_{\sIT}(h)$ was characterized in \cite{Barbier_2019}.

\begin{figure}[!htb]
    \centering
    \includegraphics[width=1.0\textwidth]{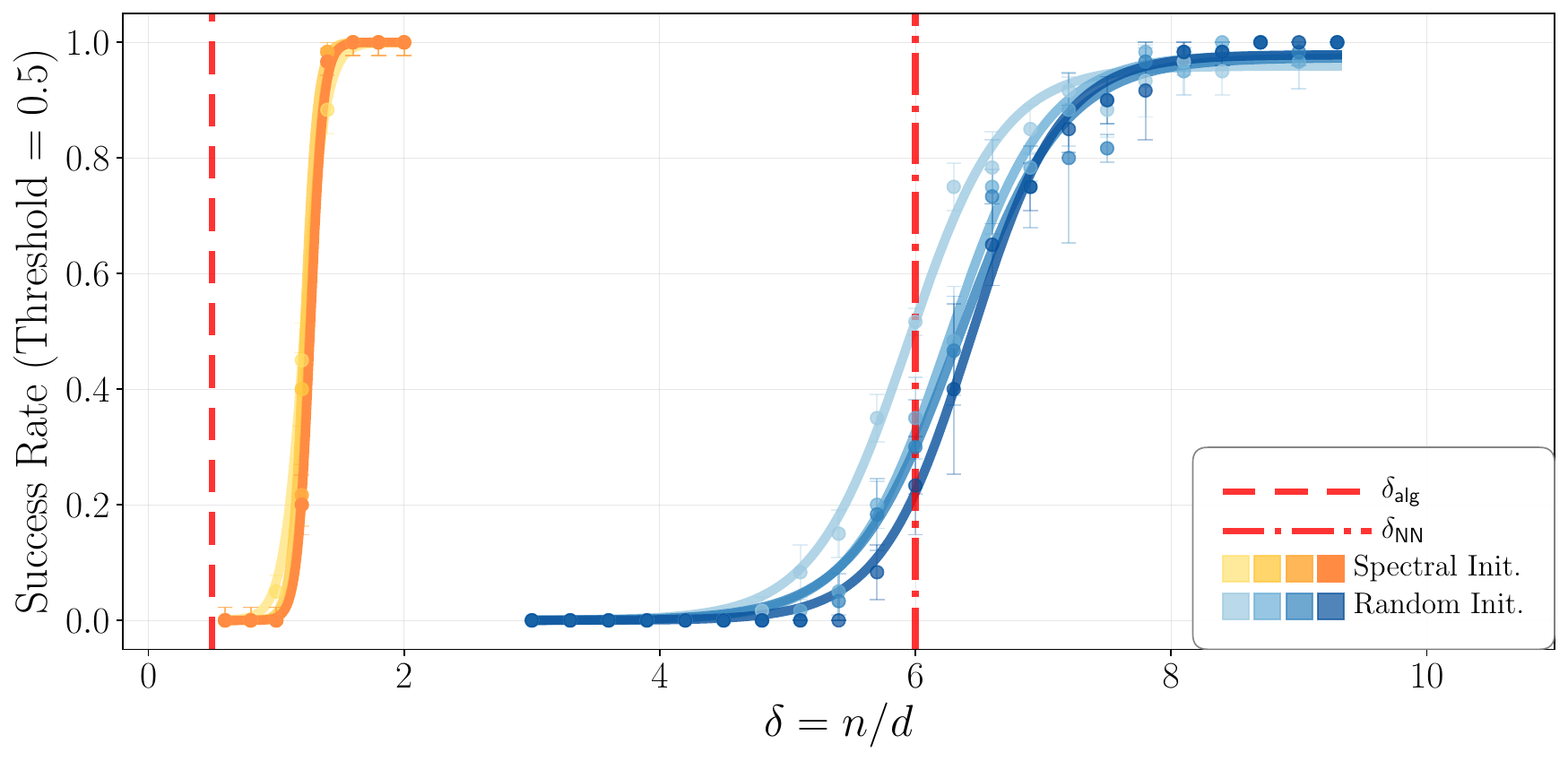}
    
    \caption{Empirical success probabilities and predicted thresholds for learning a single-index model with link function $h(z,\eps)=z^2$ (noiseless phase retrieval). The model is a single-neuron network with $\GeLU$ activation, trained via gradient descent; we declare success when the correlation between the learned feature and the target direction exceeds $1/2$. Orange curves correspond to optimal spectral initialization, blue curves to random initialization; within each color family, lighter to darker shades indicate increasing dimensions $d=1000, 2000, 3000, 4000$. Dots are success rates over $60$ independent trials (error bars show the standard deviation), and solid lines are sigmoid fits.}
    \label{fig:combined_success_rate}
\end{figure}

On the other hand, there exists a threshold $\delta_{\salg}=\delta_{\salg}(h)\in (0,\infty]$ 
and a spectral algorithm\footnote{Existence of a spectral algorithm that matches the algorithmic threshold $\delta_{\salg}$ is an open problem for $k>1$.} that achieves weak recovery for any $\delta>\delta_{\salg}$ \cite{mondelli2018fundamentallimitsweakrecovery},
and it is conjectured that no polynomial time algorithm achieves weak recovery for $\delta<\delta_{\salg}$. In general, the two 
algorithmic and information theoretic thresholds do
not coincide, i.e. $\delta_{\sIT}(h)\le \delta_{\salg}(h)$
always, and it is quite easy to come up with functions $h$ 
for which $\delta_{\sIT}(h)<\delta_{\salg}(h)$ strictly.

Finally (still assuming $k=1$ for now), it is possible to construct an empirical risk of the form \eqref{eq:EmpiricalRisk},
with $f_{\btheta}(\bx) = \<\btheta,\bx\>$ and a gradient-based algorithm that achieves weak recovery for any
$\delta>\delta_{\salg}$. This construction essentially mimics power iteration for the ideal spectral method. It is also known that for generic multi-index models (with `generative exponent' at most two), a standard two-layer neural network can achieve
weak recovery from $n=\widetilde{O}(d)$ samples\footnote{We use $\widetilde{O}$
to hide logarithmic factors.} \cite{dandi2024benefitsreusingbatchesgradient,zhang2025neuralnetworkslearngeneric}.

These results leaves open two crucial questions:
\begin{enumerate}
\item[{\sf Q1.}] Do standard neural networks of the form \eqref{eq:FirstNeuralModel} achieve weak recovery
for all $\delta$ above the algorithmic threshold $\delta_{\salg}$? If not, what is the gap? How does it depend on architecture (activation function, width, \dots) and algorithm (loss function, initialization, stepsize, \dots)?
\item[{\sf Q2.}] What is the learning mechanism implemented by gradient descent for the two-layer network? Is
it an approximation of a spectral method? If so, which spectral method?
\end{enumerate}

\begin{figure}[!htb]
    \centering
    \includegraphics[width=0.7\textwidth]{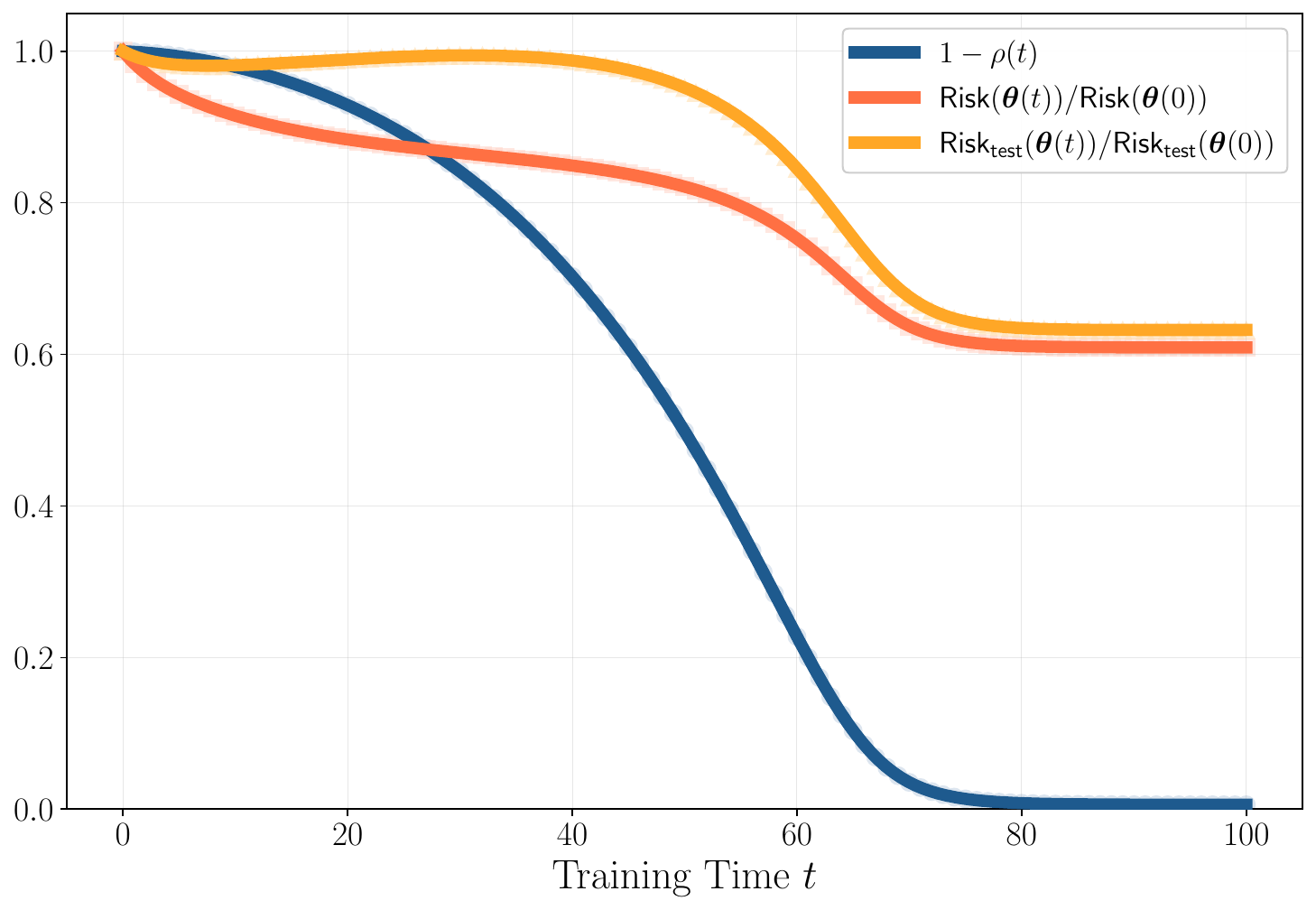}
    \vspace{-0.1cm}
    \caption{Grokking phenomenon for a $\GeLU$ neuron learning the noiseless phase retrieval problem ($d=5000$, $\delta=n/d=17.5$, single run). Here $\rho(t) := |\<\btheta(t),\btheta_*\>|/\|\btheta(t)\|$ denotes the cosine similarity between the learned and target parameters. Light-colored markers show data points; solid lines connect them for visual clarity.
    Initially, training loss decreases while test loss stays near its initial value and misalignment remains high (overfitting phase). Subsequently, the misalignment drops to zero---indicating successful learning of the signal direction---while both losses plateau at around $0.6$ and the generalization error drops to zero (grokking phase).}
    \label{fig:grokking_dynamics_first}
  \end{figure}

As a preview of our results, Fig.~\ref{fig:combined_success_rate} compares the success probabilities
of two algorithms for the case $k=1$ and $h(z,\eps) = z^2$ (noiseless phase retrieval), whereby we perform GD with respect to 
the risk of Eq.~\eqref{eq:EmpiricalRisk}, with $m=1$, Huber loss $\ell$ and $f_{\btheta}(\bx)=\GeLU (\<\btheta,\bx\>)$.
The two sets of curves correspond to different initializations: $(i)$~The optimal spectral initialization given in \cite{mondelli2018fundamentallimitsweakrecovery}; $(ii)$~Random initialization. (See Section \ref{sec:examples_and_numerical} for further details.) Success is declared when GD achieves $|\<\btheta,\btheta_*\>|/\|\btheta\|\ge 1/2$.
We also report the algorithmic threshold for weak recovery, which is $\delta_{\salg}=1/2$ in this case. 
We observe a factor-$5$ gap between the optimal initialization and the random initialization, and a factor-$12$ difference between the weak recovery threshold $\delta_{\salg}$ and the observed threshold for the one-neuron network. 
Our prediction for the latter is $\delta_{\sNN}\approx 6.0$ and matches well the empirical results.

The threshold at $\delta_{\sNN}$ also predicts the emergence of grokking as
demonstrated in Fig.~\ref{fig:grokking_dynamics_first} (same setting of Fig.~\ref{fig:combined_success_rate}, at $\delta=17.5$).
`Grokking' refers to the widely reported empirical phenomenon according to which 
generalization error (the gap between test and train error) 
first increases to reach a plateau and then (often abruptly) decreases, 
with a corresponding decrease in test error \cite{power2022grokking}. This phenomenology is clarified by our
theoretical analysis:
\begin{enumerate}
\item We observe grokking for $\delta>\delta_{\sNN}$. 
Indeed, feature learning (which takes place only for $\delta>\delta_{\sNN}$) drives the abrupt
decrease of generalization error in the late phase of training. 
\item The training time required for the drop in generalization error increases as $\delta\downarrow \delta_{\sNN}$,
(making the drop difficult to observe for $\delta$ close to $\delta_{\sNN}$). This is because feature learning is 
driven by a spectral phase transition in the Hessian of the empirical risk, and the spectral gap vanishes as $\delta\downarrow \delta_{\sNN}$.
\item The maximum of generalization error (the generalization error at the plateau) 
decreases for  $\delta\gg \delta_{\sNN}$, in agreement with uniform convergence arguments. 
\end{enumerate}

We will provide a conceptually simple quantitative explanation of the
threshold $\delta_{\sNN}$ for neural networks, and its gap with respect to
the optimal algorithmic threshold $\delta_{\salg}$.
Namely, we will study the spectrum of the Hessian $\nabla^2\Risk(\bTheta(t))$
along the GD trajectory $\bTheta(t)$ and determine 
a threshold $\delta_{\sNN}(t)$ for the emergence of a descent direction 
(an eigenvector with negative eigenvalue) aligned with the latent subspace. 
Our prediction for the feature learning phase transition is $\delta_{\sNN} = \lim_{t\to\infty}\delta_{\sNN}(t)$. Here, the limit $t\to\infty$  is taken after $n,d\to\infty$. In other words, we are studying the Hessian along the 
GD trajectory on time horizons of order $O(1)$.
This picture provides a conceptual explanation of the gap
between optimal algorithmic threshold $\delta_{\salg}$ and neural network $\delta_{\sNN}$, since the latter corresponds to
a spectral phase transition for a `sub-optimal' preprocessing of the data. 

One closely related explanation was provided in \cite{sarao2020complex,bonnaire2025role}
which analyzed gradient descent in the (real) phase retrieval problem using
non-rigorous statistical physics techniques (see Section \ref{subsec:related} for further 
comparison). 

We emphasize that our determination of the spectral threshold
$\delta_{\sNN}$ (i.e. the study of the Hessian along the
GD trajectory for $t=O(1)$) is fully rigorous.
On the other hand, its connection to the GD behavior on logarithmic (in $d$) 
time scales is a hypothesis based on a heuristic argument. Numerical simulations support this hypothesis but its mathematical proof is an open problem.

\subsection{Main results: An informal overview}\label{subsec:overview}

\subsubsection{Background}
\label{sec:Background}

We train the model \eqref{eq:FirstNeuralModel} using full-batch GD with step size $\eta>0$:
\begin{align}\label{eq:GD-First}
\bTheta(t+1)-\bTheta(t) = -\eta\nabla_{\bTheta}\Risk\big(\bTheta(t)\big)\, , 
\end{align}
and initialization $(\btheta_j)_{j\le m}\stackrel{i.i.d.}{\sim} \Unif(\S^{d-1})$, where $\S^{d-1}$ denotes the unit sphere in $d$
dimensions (a constant scaling of $\btheta_j$ can be absorbed in the activation function).
Throughout, we will focus on the proportional asymptotics
\begin{align}
n,d\to\infty\, ,\;\; n/d:=\delta_n\to \delta\in (0,\infty)\, .\label{eq:FirstProportional}
\end{align}

We begin by decomposing  the latent space $\spn(\bTheta_*)$
into a subspace that (as we will see) can be learned by $O(1)$ 
gradient evaluations (`easy') and the complementary subspace (`hard').

For $z\sim\normal(0,I_k)$ independent of $\eps\sim\normal(0,1)$,
let $y:=h(z,\eps)$.
In the case $k=1$, there is an obvious difference between data
distributions such that $\E\left[\cT(y)z\right]\neq 0$ for some function $\cT:\reals\to\reals$, and others, for which $\E\left[\cT(y)z\right]= 0$ 
for all such $\cT$. In the former case $\hbtheta(\by,\bX) = n^{-1} \sum_{i=1}^n\cT(y_i)\bx_i$ achieves weak recovery for all $\delta>0$. In the latter, weak recovery is information theoretically impossible for $\delta$
small enough.

For $k>1$, the two situations can coexist for different subspaces of $\spn(\bTheta_*)$.
For $k'\le k$, let $\cO(k,k')\subseteq \reals^{k\times k'}$ be
the set of matrices with with orthonormal columns ($\cO(k,k')$ is also known as Stiefel
manifold).
 For any $U\in\cO(k,k')\subseteq \reals^{k\times k'}$, define the projectors $P_{U}:=U U^\sT$ and $P_{U}^\perp:=I_k-P_{U}$,
and the linear subspace spanned by the columns of $U$, $\spn(U)$.

\begin{definition}\label{def:hard_subspace}
Let the collection of \emph{hard orthonormal sets} be 
\begin{align}
\label{eq:HardFrames}
\cuU_{\hard}
 := \Big\{U\in \cO(k,k'): \E\big[\cT\big(y,P_{U}^\perp z,\xi\big)P_{U} z\big]=0,\ \forall\ \text{measurable }\cT:\reals^{k+2}\to\reals\Big\},
\end{align}
where $\xi\sim\normal(0,1)$ is independent of $(x,z,\eps)$.

We then define the \emph{hard subspace} as the span of the maximal hard orthonormal set 
\begin{align}
\Uhard &:= \spn(U^*_{\hard})\, ,\;\;\;\;
 U^*_{\hard} 
 \in \argmax_{U\in\cuU_{\hard}} \rank(U)\, .
\end{align}
With an abuse of terminology, we also refer to the linear subspace  
$\{\bTheta_* u\, :\; u \in 
\Uhard\}\subseteq\reals^d$ as to the hard subspace.
We denote by $r=r_{\hard}:=\max_{U\in\cuU_{\hard}} \dim(U)$ the dimension of the hard subspace.

We refer to the orthogonal complement $\Ueasy:=\Uhard^{\perp}$ as to the \emph{easy subspace}.
\end{definition}

As an example, consider the case $k=2$ and $y=h(z_1,z_2)$, for a generic function $h$ 
that is symmetric under permutations of its arguments $h(z_1,z_2) = h(z_2,z_1)$.
Then $\Uhard = \spn((+1,-1))$ and $\Ueasy =\spn((+1,+1))$. 
It is immediate to show that the definition of $\Uhard$ is unique, because the union of hard subspaces is a hard subspace. 

\begin{remark}[Related definitions]
Related decompositions of the latent space were introduced, among others, in 
\cite{abbe2024mergedstaircasepropertynecessarynearly,abbe2023sgdlearningneuralnetworks,troiani2025fundamentalcomputationallimitsweak,damian2025generativeleapsharpsample}, with the last two being most closely related
to the present one. 
In the terminology of 
\cite{damian2025generativeleapsharpsample}, $\Uhard$ is the maximal subspace with 
generative index greater or equal to two (denoted by $k(\Uhard^{\perp})$ in \cite{damian2025generativeleapsharpsample}). 
The relation with \cite{troiani2025fundamentalcomputationallimitsweak} is a bit more subtle,
in the sense that our `easy' subspace is equal to the union of `trivial' subspaces of the `grand-staircase'
of \cite{troiani2025fundamentalcomputationallimitsweak} if the last construction does not gave easy directions,
and in general is a subspace of the latter.

We emphasize that the term `hard' here is not to be interpreted in its meaning in computational
complexity theory.
\end{remark}

With an abuse of terminology, we also refer to any orthonormal basis  $U_{\hard}^*$ of
$\Uhard$ as to the hard subspace: there will be no risk of confusion because all statements will be invariant under change of basis. 
Similarly, we denote by $U_\mathrm{E}^*$ an (arbitrary) orthonormal basis of $\Ueasy$.

In alternative to Eq.~\eqref{eq:HardFrames}, we use the equivalent definition
\begin{equation}
 U\in \cuU_{\hard}\quad \text{if and only if}\quad \E\big[P_{U} z \big|y,P_{U}^\perp z\big]=0.
\end{equation}
In words, along a hard subspace, the expectation of $P_{U} z$ carries no information 
even after conditioning on both the target and the orthogonal complement.

As shown in \cite{dandi2024benefitsreusingbatchesgradient}, the 
network (weakly) learns the easy subspace $\Ueasy$ within $O(1)$ iterations of gradient descent
(for generic $a_i$'s) as long as $n/d\to \alpha>0$. Namely, there exists $\eps_0>0$ and $t=O(1)$ such that,
for $m$ enough, and under certain conditions on the $m$, the step size $\eta$, 
and the activation function $\sigma$:
\begin{align}
\lim_{n,d\to\infty} \prob\left(s_{\min}(\bTheta(t)^{\sT}\bTheta_*U^*_{\easy}) \ge\eps_0\right)=1\, .
\end{align}
where $s_{\min}(\cdot)$ denotes the minimum singular value. On the other hand, as stated in the next section, after any constant (in $n,d$) number of gradient steps,
the weights $\bTheta(t)$ remain roughly orthogonal to the subspace of hard directions. 
A related result was proved in \cite[Theorem 3.2]{dandi2024benefitsreusingbatchesgradient} for the case in 
which all latent directions are hard (i.e. $\Uhard=\reals^k$).

Our main results are summarized in Section \ref{sec:Hessian_Preview}.
We characterize a spectral phase transition in the Hessian of the empirical risk, along the 
GD trajectory. Namely, we analyze the Hessian $\nabla^2\Risk(\bTheta(t))$ after $t=O(1)$ iterations, and show that,
for $\delta$ above a certain threshold $\delta_{\sNN}(t)$, it develops negative outlier eigenvalues, with the corresponding eigenvectors aligned with the hard subspace. 
This phase transition drives feature learning of the hard directions.

\subsubsection{Hard directions are not learned within $O(1)$ steps}\label{sec:HardDirections}

Our first result establishes that hard directions are not
learned within $O(1)$ steps of gradient descent. In particular
for any fixed $t$, $\bTheta(t)$ is asymptotically orthogonal to $\mathfrak{U}_{\hard}$.

We use an exact characterization of gradient descent in the limit 
of $n,d\to\infty$, for any constant number of iterations $t$, which goes  under the name of dynamical mean field 
theory\footnote{As shown in \cite{celentano2020estimation} this is in fact a consequence of the state evolution analysis of approximate message passing algorithms.} (DMFT)~\cite{sompolinsky1981dynamic,sompolinsky1982relaxational,celentano2020estimation,celentano2021high}. 
For any fixed $t$, DMFT construct a sequence of
random vectors $\Theta_*\in\reals^k$, $\Theta(0),\dots, \Theta(t)\in\reals^m$ such that, for any suitably regular test function $\phi$:
\begin{align}
    \plim_{n,d\to\infty}\frac{1}{d}\sum_{i=1}^d \phi\left(\sqrt{d}\bTheta(0)_i, \dots, 
    \sqrt{d}\bTheta(t)_i, \sqrt{d}\bTheta_{*i}\right) = \E[\phi(\Theta(0),\dots,\Theta(t),\Theta_*)]\, .
\end{align}
where $\bTheta(t)_i$ is the $i$-th row of $\bTheta(t)$ and same for $\bTheta_{*i}$. (Throughout the paper, we use boldface to denote pre-limit vectors or matrices and normal font to denote their asymptotic counterpart.)

Using this characterization, we prove that, for any 
bounded continuous function $\cT: \reals^{m(t+1)}\to \reals$ (see Lemma \ref{lemma_amp_feature}):
\begin{align}
\plim_{n,d\to\infty}\widehat{\bTheta}^{\sT}\bTheta_* P_{U_\mathrm{H}^*} = 0\, , \;\;\;\;\;\;
    \widehat{\bTheta}_i := \frac{1}{\sqrt{d}}\cT\left(\sqrt{d}\bTheta(0)_i, \dots, 
    \sqrt{d}\bTheta(t)_i\right) 
\end{align}
In words, any estimator $\widehat{\bTheta}$ computed in constant number of gradient steps and possibly postprocessing  remains uncorrelated with the hard directions 
$\bTheta_* P_{U_\mathrm{H}^*}$.

We note that the DMFT construction of the random vectors $(\Theta(0),\dots,\Theta(t),\Theta_*)$ 
(as well as the random vectors $(V(0),\dots,V(t),V_*,\eps)$ which characterize the limit of
$(\bTheta(0)^{\sT}\bx_i,\dots,\bTheta(t)^{\sT}\bx_i,\bTheta_*^{\sT}\bx_i,\eps_i)$)
is explicit and can be implemented numerically, see Section \ref{sec:MainDMFT}. This allows us to study the Hessian at
the current position
$\bTheta(t)$, which is our main goal, as we discuss next.

\subsubsection{Hessian phase transitions after gradient descent}
\label{sec:Hessian_Preview}

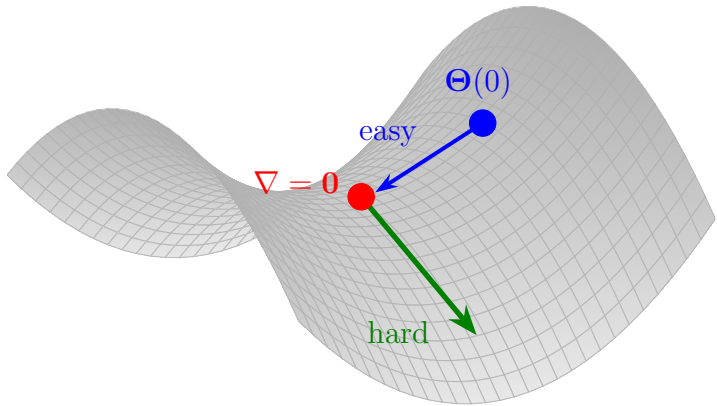
\begin{figure}[htb]
\centering
\vspace{0.3cm}
\begin{tikzpicture}
\begin{axis}[
    view={-35}{30},
    axis lines=none,
    domain=-1.2:1.2,
    y domain=-1.2:1.2,
    samples=30,
    z buffer=sort,
    width=11cm,
    height=9cm,
]
\addplot3[
    surf,
    shader=faceted interp,
    colormap={depthgray}{color=(gray!20) color=(gray!45) color=(gray!70)},
    opacity=0.9,
    faceted color=gray!60,
] {x^2 - y^2};

\addplot3[only marks, mark=*, mark size=5pt, blue] coordinates {(0.7, 0, 0.49)};
\node[blue, font=\large] at (axis cs:0.85, 0.25, 0.7) {$\bTheta(0)$};

\addplot3[only marks, mark=*, mark size=5pt, red] coordinates {(0, 0, 0)};
\node[red, font=\large] at (axis cs:-0.22, 0.22, 0.12) {$\nabla = \bzero$};

\addplot3[<-, blue, line width=1.5pt, >={Stealth[length=3mm, width=2mm]}] coordinates {(0.7, 0, 0.49) (0.08, 0, 0.01)};
\node[blue, font=\large] at (axis cs:0.4, 0.35, 0.25) {easy};

\addplot3[->, green!50!black, line width=2pt, >={Stealth[length=4mm, width=3mm]}] coordinates {(0, 0, 0) (0, -0.95, -0.90)};
\node[green!50!black, font=\large] at (axis cs:-0.45, -0.95, -0.65) {hard};

\end{axis}
\end{tikzpicture}
\caption{Cartoon of feature learning via GD on the empirical risk landscape. Within $O(1)$ steps, the network learns  easy directions (blue arrow) to a saddle point. Escaping this point requires moving along hard directions (green arrow), governed by the Hessian's negative eigenvalues.}
\label{fig:saddle_cartoon}
\end{figure}

In order to understand the dynamics beyond time $t=O(1)$, we need to consider deviation from the DMFT 
trajectory. These deviations are of order\footnote{This can be proven using the finite sample analysis of AMP of \cite{rush2018finite}
and the mapping between AMP and generalized first order methods of \cite{celentano2020estimation}.} $1/\sqrt{d}$, but their dynamics depends on 
the structure of the Hessian at $\bTheta(t)$. It will be convenient to  study the rescaled Hessian
\begin{equation}
 \bH(t) := m\nabla^2\Risk(\bTheta(t))\, .
\end{equation}
Our main result is a sharp characterization asymptotic
of the smallest eigenvalues of $\bH(t)$ and associated eigenvectors under the proportional asymptotics
$n,d\to\infty$, $n/d\to \delta$, in two limit cases $m=1$ and $m\gg 1$. Surprisingly, these
cases can be treated in unified fashion from a mathematical perspective.

\noindent\emph{Case $m=1$.} In this case $\bH(t)\in\reals^{d\times d}$ takes the form
\begin{align}
\bH(t)&:=\frac{1}{n}\sum_{i=1}^n g(y_i,\btheta(t)^{\sT}\bx_i)\bx_i\bx_i^\sT\, ,\\
g(y,z)&:=a^2\ell'\big(y,a\sigma(z+b)\big)\sigma''(z+b)+a\ell''\big(y,a\sigma(z+b)\big)\sigma'(z+b)^2.
\end{align}

\noindent\emph{Case $m\gg 1$.} In this case $\bH(t)\in\reals^{md\times md}$ has a block structure ($m^2$ blocks
of dimensions $d\times d$).
Lemma~\ref{lemma:eigenvalue_diagonal} below shows that, for a convex loss $\ell(y,\,\cdot\,)$, 
the smallest eigenvalues and corresponding eigenvectors of $m\nabla^2\Risk(\bTheta)$ are well-approximated by a block diagonal 
matrix $\bH_{\mathrm{diag}}(t)$, up to an $O(1/m)$ correction. 
(Here we are considering large width $m$ \emph{after} the limit $n,d\to\infty$). 

The $j$-th block of $\bH_{\mathrm{diag}}(t)$ is $a_j \bH_{j}(t)$ where
\begin{align}
 \bH_j(t)&:=\frac{1}{n}\sum_{i=1}^n  g(y_i,\bTheta(t)^{\sT}\bx_i;j)\bx_i\bx_i^\sT\, ,\\
 g(y,\bTheta^\sT \bx;j)&:=\ell'\big(y,f_{\bTheta}(\bx)\big)\sigma''(\btheta_j^{\sT}\bx+b_j)\, .
\end{align}
Hence, characterizing the low-lying eigenvalues of $\bH(t)$ amounts to 
analyzing a random matrix with similar structure in the two cases $m=1$ and $m\gg 1$.
From the mathematical viewpoint, the challenge lies in the fact that the parameters $\bTheta(t)$
depend in an intricate way on the random data $\{\bx_i,y_i\}_{i=1}^n$ through gradient descent updates.

Our main results can be summarized as follows (we refer for concreteness to the case $m\gg 1$,
and the block $\bH_j(t)$, and assume throughout the proportional asymptotics 
\eqref{eq:FirstProportional}).
\begin{description}
\item[Asymptotics of the empirical spectral distribution (ESD).] For any $t\ge 0$, let $\mu_{n,d}^j(t)$ denote the empirical spectral distribution of block $j$.
We then determine a distribution  $\mu_{\infty}^j(t)$ such that 
\begin{align}
\mu_{n,d}^j(t) \Rightarrow \mu_{\infty}^j(t)\, .
\end{align}
(We use $\Rightarrow$ to denote weak convergence and the convergence is understood to hold in probability).
Under certain regularity assumptions on the loss and activation functions, the support of
$\mu_{\infty}^j(t)$ is bounded below, and we denote by $c_j(t)$ its infimum.
\item[Thresholds for emergence of informative eigenvectors.]
 We characterize thresholds $\delta^*_j$ for each block $j$
such that, for $\delta>\delta^*_j$,  the following happens: 
\begin{itemize}
    \item There exists $t_0<\infty$ such that, for all $t\ge t_0$, the spectrum of $\bH_j(t)$
    contains a negative eigenvalue $z_{j}^{n,d}(t)$ below the infimum of the support of $\mu^j_{\infty}$.
    Namely, for some fixed $\Delta>0$
    \begin{align}
    \plim_{n,d\to\infty} z_{j}^{n,d}(t) = z_j^*(t) \le  \min(c_j(t),0)- \Delta\, 
    \end{align}
    \item The corresponding eigenvector $\boldsymbol{\xi}_{j}^{n,d}(t)$ has a projection on the hard subspace
    with magnitude bounded away from zero.
\end{itemize}
\item[No informative eigenvector below spectral thresholds.] Viceversa, 
for $\delta<\delta^*_j$ and for all $t$ large enough, no eigenvector of $\bH_j(t)$ has projection onto the hard subspace which is asymptotically bounded away from zero.
\end{description}

This leads us to define the overall feature learning threshold as $\delta_{\sNN}:=\min_{j\in [m]}\delta_j^*$.
In summary, when the sample size is large enough, the Hessian has a negative isolated
eigenvalue whose corresponding eigenvector is correlated with the hard directions. When the sample size is small, the Hessian has no eigenvector correlated with the hard directions. 
Our results provide explicit formulas for these thresholds, enabling sharp and numerically computable predictions for this phenomenon.

We note that the threshold $\delta_j^*$ is dependent on  the details of the learning procedure, including the loss function, learning rate, initialization, and activation function.
Our results allow to study the dependence of feature learning with respect to these 
factors that can be designed to minimize the sample size for learning.

\paragraph{Outline of the paper.}
The remainder of the paper is organized as follows. Section \ref{subsec:related} reviews related literature, and Section \ref{subsec:notation} summarizes our notation. Section \ref{sec:results} presents our main results
and Section \ref{sec:Consequences} discusses some of their consequences. Section \ref{sec:examples_and_numerical} illustrates our theoretical predictions through numerical experiments. Section \ref{sec:proof_sketch} provides an outline of the proofs. Section \ref{sec:discussion} concludes with a discussion of implications and future directions.
Some technical preliminaries are provided in Appendix \ref{appendix:tools}. Detailed proofs are provided in Appendices \ref{sec_Proofs} and \ref{sec:proof_result_iii}, supporting technical lemmas are collected in Appendix \ref{sec:technical_lemmas}, and additional numerical experiments are presented in Appendix \ref{sec:additional_simulation}.

\subsection{Related literature}\label{subsec:related}

\paragraph{Optimal feature learning.} 

The literature on multi-index models is quite extensive.
Here we limit ourselves to results relevant to the specific setting of our paper, referring to
\cite{bruna2025surveyalgorithmsmultiindexmodels} for a recent review. In particular we assume $k$
fixed, isotropic covariates and focus on the weak recovery problem, i.e., estimating the latent directions with non-vanishing correlation. We generically refer to weak recovery of the latent space as `feature learning'.

Information-theoretically, the sample size required to learn a $k$-index model is linear in the 
dimension $d$
\cite{Barbier_2019,Aubin_2019}, as expected from counting degrees of freedom. 
A significant amount of work has been devoted to understanding whether this can be achieved by 
polynomial-time algorithms \cite{Barbier_2019,mondelli2018fundamentallimitsweakrecovery,lu2019phasetransitionsspectralinitialization,chen2020learningpolynomialsrelevantdimensions,damian2024computationalstatisticalgapsgaussiansingleindex,damian2025generativeleapsharpsample,troiani2025fundamentalcomputationallimitsweak,kovačević2025spectralestimatorsmultiindexmodels,defilippis2025optimalspectraltransitionshighdimensional}. We will refer to this as `efficient feature-learning'.

The situation is particularly simple in the case $k=1$. If 
$\E[z|y]$ is not identically zero (in the notation of Section \ref{sec:Background}), then 
weak recovery of $\btheta_*$
is possible for any $\delta>0$. If on the other hand $\E[z|y]=0$ (this is the case if $y$ only depends on
$z$ via $|z|$),
\cite{mondelli2018fundamentallimitsweakrecovery} showed that efficient feature learning
is possible when $\delta$ crosses a certain threshold $\delta_{\salg}$
by using an optimal spectral method. It has been conjectured that no polynomial-time algorithm 
succeeds for $\delta<\delta_{\salg}$ \cite{mondelli2018fundamentallimitsweakrecovery,Barbier_2019}. 
While it is possible to construct examples (functions $h$) for which $\delta_{\salg}=\infty$,
this situation is non-generic in the sense that it requires the function $h$ to
satisfy an additional equality constraints.

For $k>1$, the simplest case is the one in which all directions are easy, in our language  
$\Ueasy=\reals^k$. Starting with \cite{abbe2023sgdlearningneuralnetworks}, several results
established that efficient feature learning is possible from $n=O(d)$ samples for certain 
classes of distributions of this type.
On the other hand, for several interesting problems $\Ueasy\subsetneq \reals^k$,
e.g. when $h(z,\eps)$ is invariant under permutations of some of the entries 
of $z$.
Recently, \cite{damian2025generativeleapsharpsample} 
proved that sequential spectral method achieve efficient feature learning from linearly many samples
even when hard directions are present, generically. More generally, 
they related the scaling of sample size to a property of the link function $h$,
known as the `generative exponent'. The generic situation (covering common examples) 
corresponds to generative exponent at most two.

The results of \cite{damian2025generativeleapsharpsample} do not yield a sharp threshold 
$\delta_{\salg}$ for polynomial-time feature learning.  
For certain classes of functions $h(z,\eps)$  \cite{troiani2025fundamentalcomputationallimitsweak}
analyze the precise asymptotics of AMP in the presence of
a certain amount of `side information' in the latent direction. 
In the limit of vanishing side information, this yields a 
phase transition behavior at a threshold $\delta_{\salg}$.
Proving that this threshold is predictive of the actual behavior in absence of side information remains an outstanding problem.

Optimal spectral methods are developed in \cite{kovačević2025spectralestimatorsmultiindexmodels,defilippis2025optimalspectraltransitionshighdimensional}.
These works establish that weak recovery is feasible for  $\delta>\delta_{\salg}$
but do not provide evidence for hardness for $\delta<\delta_{\salg}$.

\paragraph{Feature learning in neural networks.}
Empirically, gradient-based training of neural networks efficiently finds hidden representations aligned with low-dimensional structure in the data. 
A large number of theoretical works studies how and when neural networks can efficiently perform feature learning in single-index, multi-index, and other models with planted structure
\cite{arous2021onlinestochasticgradientdescent,
damian2022neuralnetworkslearnrepresentations,abbe2024mergedstaircasepropertynecessarynearly,abbe2023sgdlearningneuralnetworks,ba2022highdimensionalasymptoticsfeaturelearning,bietti2023learninggaussianmultiindexmodels,wang2023learninghierarchicalpolynomialsthreelayer,berthier2024learning,arnaboldi2024repetita,lee2024neuralnetworklearnslowdimensional,dandi2024benefitsreusingbatchesgradient,fu2024learninghierarchicalpolynomialsmultiple,arnaboldi2025repetitaiuvantdatarepetition,montanari2025dynamicaldecouplinggeneralizationoverfitting,zhang2025neuralnetworkslearngeneric,fan2026highdimensionallearningdynamicsmultipass}.

For single-index models, \cite{arnaboldi2024repetita,lee2024neuralnetworklearnslowdimensional} showed that neural networks can efficiently learn single-index polynomial targets with $\widetilde{O}(d)$ samples through sample reuse, and \cite{montanari2025dynamicaldecouplinggeneralizationoverfitting} provided a precise picture of the dynamics of learning single-index models (with information exponent one, within $O(1)$ time) using wide networks under proportional scaling, demonstrating that overfitting occurs at a slower timescale than feature learning.

Early results 
\cite{damian2022neuralnetworkslearnrepresentations,abbe2024mergedstaircasepropertynecessarynearly,abbe2023sgdlearningneuralnetworks} established that neural networks 
can achieve feature learning in multi-index models, albeit with 
suboptimal sample complexity.
Next,
\cite{dandi2024benefitsreusingbatchesgradient} proved that $n=O(d)$ samples are sufficient for 
efficient feature learning using  neural networks and batch gradient descent 
when the generative exponent is one. 
In other words, gradient descent can match the ideal sample size scaling as predicted by the generative exponent. 
In concurrent work, 
\cite{zhang2025neuralnetworkslearngeneric}  showed that neural networks trained via layer-wise gradient descent can efficiently learn generative-exponent-two, non-generative-staircase multi-index models from $\widetilde{O}(d)$ samples. 

Despite such a substantial amount of work, sharp thresholds
for feature learning in neural networks remain unknown. In addition, the precise interplay between activation functions, data structure, loss functions, and initialization in determining the
feature learning behavior is not well understood.

\paragraph{Spiked Models.}

Spiked random matrix models provide a mathematical framework for understanding when and how informative directions can be detected from noisy data. In a standard spiked model, the spectrum of the bulk follows the generalized Marchenko-Pastur law \cite{MarchenkoPastur,BaiYin,BaiSilverstein}. When low-rank signals are introduced, the seminal work of \cite{baik2004phasetransitionlargesteigenvalue} identified the Baik–Ben Arous–Péché (BBP) transition, which precisely characterizes the critical signal-to-noise ratio at which outlier eigenvalues emerge from the bulk spectrum. Above this threshold, the top eigenvectors become correlated with the latent signal, enabling spectral estimators to recover the hidden directions.
Subsequent work has refined this picture through local laws and improved deformation analyses \cite{benaychgeorges2010eigenvalueseigenvectorsfinitelow,
knowles2011isotropic,knowles2016anisotropiclocallawsrandom}.  
Spiked models naturally emerge when designing and analyzing spectral methods used as initialization strategies for nonconvex problems such as phase retrieval \cite{chen2017solving,mondelli2018fundamentallimitsweakrecovery}. 
The structure of the Hessian under  generalized linear models also takes the same form and was
further
studied (for $m=k=1$) in \cite{liao2021hessianeigenspectrarealisticnonlinear}. However, this type of analysis only provides information
on the Hessian at a point $\btheta(0)$ independent of the data itself.

A key technical challenge that we face is the one of analyzing the Hessian at a point 
$\bTheta(t)$ that depend on the data $\bX,\by$. The recent paper 
\cite{arous2025localgeometryhighdimensionalmixture} analyzes the Hessian
after a certain number of stochastic gradient descent (SGD) steps. However they assume that the Hessian is
computed on fresh data, independent of the data used for SGD, hence sidestepping the 
challenge of dependencies between $\bTheta(t)$ and the data.

In order to capture these dependencies \cite{sarao2020complex,bonnaire2025role}
use a statistical mechanics approach, in the context of real phase retrieval
(see also \cite{annesi2025overparametrization} for some generalization). 
Namely, they assume that gradient descent converges to `marginal states', and use 
the replica method together with
random matrix theory to study spectrum of the Hessian in these states. 
While providing the correct picture, we point out that this type of analysis of gradient descent is not ---in general--- asymptotically exact (besides being non-rigorous).

Considering the setting in which the same data is queried
multiple times (and in particular Hessian is computed on the same data used for GD)
poses technical challenges but is unavoidable to capture several crucial learning 
phenomena\footnote{For instance, the generalization error vanishes if data is not re-used, and 
the required sample size can be significantly larger than with multiple-pass GD \cite{dandi2024benefitsreusingbatchesgradient}.}.

\subsection{Summary of notations}\label{subsec:notation}

We use $n, d, m, k$ to denote the sample size, input dimension, network width, and number of target features, respectively. We use $c,C,c_i,C_i,\dots$ to denote constants that may change 
from line to line; these constants are independent of $n,d,m$ but may depend on $k$
and on the constants that appear in the regularity assumptions on $h, \sigma, \ell$. 
We use boldface to denote pre-limit vectors and matrices, and normal font to denote their asymptotic limiting counterparts. We use $\|\cdot\|$ to denote the Euclidean norm when applied to vectors and the operator norm when applied to matrices. We denote by $I_p$ the $p\times p$ identity matrix. 

A function $\phi: \mathbb{R}^p \rightarrow \mathbb{R}$ is called pseudo-Lipschitz of order $s$ (for $s \geq 1$), denoted $\phi \in \operatorname{PL}_s$, if there exists a constant $L>0$ such that for all $x,y \in \mathbb{R}^p$, we have
\begin{equation}
|\phi(x)-\phi(y)| \leq L\left(1+\|x\|^{s-1}+\|y\|^{s-1}\right)\|x-y\|.
\end{equation}

We use $\lesssim$ and $\gtrsim$ to hide constants as above. We write $o_P(1)$ for a quantity that converges to zero in probability as $n,d\to\infty$. A sequence $X_n = O_P(1)$ means that $\{X_n\}$ is tight: for every $\varepsilon>0$, there exists $M<\infty$ such that $\sup_n \mathbb{P}(|X_n|>M)\le \varepsilon$. We denote by $W_p(\mu,\nu)$ the Wasserstein $p$-distance; when $p=2$, we drop the subscript. We use $\lambda_{\text{min}}(\cdot)$ to denote the minimal eigenvalue of a real symmetric matrix and $\S^{d-1}$ to denote the unit sphere in $\mathbb{R}^d$.

We use $\Rightarrow$ to denote weak convergence and $\xrightarrow{~p~}$ or
$\plim_{n,d\to\infty}$ 
to indicate convergence in probability. We use $\mathbb{H}:=\{z\in \mathbb{C}: \Im(z)>0\}$ to denote the upper half plane.
Given an infinite symmetric positive semidefinite matrix $S = (S(t,s): t,s\ge 0)$ 
and a sequence $a= (a(t): t\in \naturals)$,
we write $Z\sim \GP(a,S)$ if $Z$ is a Gaussian process with mean $a$ and covariance $S$.

\section{Results}\label{sec:results}

\subsection{Problem setup and assumptions}

This subsection introduces the notation for the gradient descent (GD) updates and states the regularity assumptions that will be used in our analysis. Throughout, we use $\ell'$ and $\ell''$ to denote the derivatives of $\ell$ with respect to its second argument.

To express the update of GD in a compact form, we introduce the following notation. Define $F_j(\bTheta^\sT\bx_i,\bTheta^{\sT}_* \bx_i,\varepsilon_i):=(\eta/m)a_j\ell'(y_i,f_{\bTheta}(\bx_i))\sigma'(\btheta_j^\sT\bx_i+b_j)$ and $\bX^\sT=\begin{bmatrix}
    \bx_1&\dots&\bx_n
\end{bmatrix}\in \reals^{d\times n}$. Then, the GD update reads
\begin{equation}\label{eq:gd_compact}
\bTheta(t+1)-\bTheta(t)=-\frac{1}{n}\bX^\sT\begin{bmatrix}
    \bF_1(\bX \bTheta(t),\bX \bTheta_*,\boldsymbol{\varepsilon}) & \dots & \bF_m(\bX \bTheta(t),\bX \bTheta_*,\boldsymbol{\varepsilon})
\end{bmatrix}.
\end{equation}
where $\bF_j(\cdot)\in \reals^n$ denotes the coordinate-wise application of $F_j(\cdot)$. We also define $F^\sT:=\begin{bmatrix}
    F_1 & \dots & F_m
\end{bmatrix}$. 

We now state our regularity assumptions.
\begin{assumption}\label{assumption_regularity}
We assume the following.
\begin{enumerate}[label=(\alph*)]
\item The data $\{(\bx_i, y_i)\}_{i=1}^n$ are i.i.d.\ with $\bx_i \sim \normal(\boldsymbol{0}, \bI_d)$ and $y_i = h(\bTheta_*^\sT \bx_i, \eps_i)$, where $\eps_i \sim \normal(0,1)$ is independent of $\bx_i$ and $\bTheta_* \in \reals^{d\times k}$ has orthonormal columns.
\item $h(\cdot)$ is locally Lipschitz continuous.
\item $\sigma(\cdot) \in C^2(\reals)$, $\|\sigma'\|_{\infty}\leq C$, $\|\sigma''\|_{\infty}\leq C$.
\item $0\leq a_j \leq C$, $\abs{b_j}\leq C$.
\item For each $y\in \reals$, $\ell(y,\,\cdot\,)\in C^2(\reals)$. Further, $\ell(y,y)=0$ for all $y$ and 
$\|\ell''\|_{\infty}\leq C$, $\|\ell'\|_{\infty}\leq C$.
\item The hard subspace has dimension at least one.
\end{enumerate}
\end{assumption}

\begin{remark}
By rotation invariance of the distribution of the covariates $\bx_i$, 
the learning dynamics is unchanged if we replace $\bTheta_*$, $\bTheta(t)$
by  $\bQ\bTheta_*$, $\bQ\bTheta(t)$, for $\bQ$ an arbitrary orthogonal matrix.
By choosing a uniformly random $\bQ$, there is no loss of generality in assuming the following.
For each $\phi\in \operatorname{PL}_2$, we have 
\begin{equation}\label{eq:PL2_convergence}
\frac{1}{d}\sum_{i=1}^d \phi\left(\sqrt{d}\bTheta_{*i}\right) \rightarrow \E\left[\phi(\Theta_*)\right]
\end{equation}
where $\Theta_*\sim\normal(0,I_k)$, and $\bTheta_{*i}$ denotes the $i$-th row of $\bTheta_*$. 
\end{remark}

We note that the  assumptions $\|\ell'\|_{\infty}\leq C$ and $\|\ell''\|_{\infty}\leq C$ can be relaxed to polynomial growth conditions $\abs{\ell'(y,z)}\leq C(1+\abs{y}+\abs{z})^C$ and $\abs{\ell''(y,z)}\leq C(1+\abs{y}+\abs{z})^C$ if we additionally assume $\abs{y}\leq C$ and $\|\sigma\|_{\infty}\leq C$. This relaxation allows us to include the square loss.

\subsection{Preliminaries: Discrete time DMFT}
\label{sec:MainDMFT}
To characterize the gradient descent dynamics in the high-dimensional limit, we employ the discrete time dynamical mean field theory (DMFT) framework, which has a rich history in physics and mathematics, and has been studied extensively in the literature. DMFT provides a low-dimensional characterization of training dynamics, whose complexity does not scale with $n,d$. We briefly review the main setup below, which is mostly a restatement of results from \cite{celentano2021high}.

The DMFT asymptotics are determined by a set of
self-consistent systems of stochastic processes. We define the DMFT process via 
\begin{equation}\label{eq:dmft_system}
\begin{aligned}
V(t) &= - \frac{1}{\delta} \sum_{s=0}^{t-1} R_\theta(t, s)
F(V(s),W_*,\eps) +W(t),\phantom{AAAA}  W \sim \GP(0,C_\theta)\, ,&\\
\Theta(t+1) &= \Theta(t) - \eta \sum_{s=0}^{t} R_\ell(t, s) \Theta(s)-\eta R_{\ell}(t,*) \Theta_*+ \eta Q(t) ,\quad Q\sim 
\GP(0,C_\ell/\delta)\,.&
\end{aligned}
\end{equation}
Here, $(W_*,W(t);t\in \naturals)$ and $(Q(t);t\in\naturals)$ are Gaussian processes taking values in $\reals^m$ with covariances determined by kernels $C_\theta$ and $C_\ell$:
\begin{equation}\label{eq:dmft_gauss_W}
    \E\left[W_* W_*^\sT\right] = I_k, \quad  \E\left[W(t) W(s)^\sT\right] = C_\theta(t, s), \quad \E\left[W(t) W_*^\sT\right] = C_\theta(t, *), \quad t,s\in\naturals,
    \end{equation}
and
\begin{equation}\label{eq:dmft_gauss_Q}
    \E\left[Q(t) Q(s)^\sT\right] = \frac{1}{\delta} C_\ell(t, s), \quad t,s\in \naturals.
    \end{equation}
The kernels $C_{\theta},C_{\ell},R_{\theta},R_{\ell}$ are determined as the unique solution of the following self-consistency equations
\begin{equation}\label{eq:dmft_kernels}
\begin{aligned}
C_\theta(t, s) &= \E\left[\Theta(t)\Theta(s)^{\sT}\right], \qquad
C_\ell(t, s) = \frac{1}{\eta^2}\E\left[F(V(t),W_*,\eps)F(V(s),W_*,\eps)^\sT\right], \\
C_{\theta}(t, *) &= \E\left[\Theta(t)\Theta_*^{\sT}\right], \qquad
R_\theta(t, s) = \frac{1}{\eta}\E\left[\frac{\partial \Theta(t)}{\partial Q(s)}\right], \\
R_\ell(t, s) &= \frac{1}{\eta}\E\left[\frac{\partial F(V(t),W_*,\eps)}{\partial W(s)}\right], \qquad
R_\ell(t, *) = \frac{1}{\eta}\E\left[\frac{\partial F(V(t),W_*,\eps)}{\partial W_*}\right].
\end{aligned}
\end{equation}
We also denote $V_*:=W_*$, and initialize the DMFT process with $\Theta(0)\sim \normal(0,I_m)$ independent of $\Theta_*$.
When computing partial derivatives with respect to $W(s)$, $W_*$ above, it is understood that $V(t)$
is a function of $W_*, (W(s):s\le t)$ via Eq.~\eqref{eq:dmft_system}, and similarly for the partial derivative of $\Theta(t)$
with respect to $Q(s)$.
It is also useful to recall at this point that $F(V,W,\eps)$ is of order $\eta$ and the above quantities have
well defined limits as $\eta\to 0^+$ and $\eta t = \Theta(1)$, which is described in \cite{celentano2021high}.

 For any fixed time horizon $T$, the coupled system \eqref{eq:dmft_system}--\eqref{eq:dmft_kernels} can be solved forward in time: the kernels $\{C_\theta,R_\theta,C_\ell,R_\ell\}$ determine the Gaussian processes $(W,Q)$, which in turn determine $(V(t),\Theta(t))_{t\le T}$ via \eqref{eq:dmft_system}; plugging these processes back into \eqref{eq:dmft_kernels} yields the same kernels but for the next time step. Hence the DMFT process and kernels are uniquely determined by 
Eqs.~\eqref{eq:dmft_system}--\eqref{eq:dmft_kernels} .

\subsection{Result I: Hard directions cannot be learned in $O(1)$ time}\label{sec:MainResults_I}

This subsection analyzes the gradient descent dynamics over a finite number of steps. We establish that hard directions cannot be learned in $O(1)$ time.

The results are summarized in the following two lemmas (the first parts of these lemmas are
restatements of a result from \cite{celentano2021high}).
\begin{lemma}\label{lemma_amp_data}
    Let $\phi$ be a test function which is locally Lipschitz continuous and has at most quadratic growth.
    Under proportional scaling $n/d:=\delta_n \rightarrow \delta \in (0,+\infty)$ and Assumption~\ref{assumption_regularity}, we have 
    \begin{equation}\label{eq:data_convergence_main}
    \frac{1}{n}\sum_{i=1}^n \phi\left(\bTheta(0)^\sT \bx_i, \dots, \bTheta(t)^\sT \bx_i, \bTheta^{\sT}_*\bx_i,\varepsilon_i\right) \xrightarrow{~p~} \mathbb{E}[\phi(V(0),\dots,V(t),V_*,\varepsilon)]
    \end{equation}
    where $(V(0),\dots,V(t),V_*,\varepsilon)$ is the discrete-time DMFT process.
    
Furthermore, the joint distribution satisfies
\begin{equation}\label{eq:hard_orthogonality}
\mathbb{E}\left[\psi\left(V(0),\dots,V(t),\varepsilon,h(V_*,\varepsilon),P^{\perp}_{U_{\mathrm{H}}^*}V_*\right)P_{U_\mathrm{H}^*}V_*\right]=0
\end{equation}
for every measurable $\psi(\cdot)$ such that the expectation makes sense.
\end{lemma}

\begin{lemma}\label{lemma_amp_feature}
    Let $\phi$ be a test function which is locally Lipschitz continuous and has at most quadratic growth.
    Under proportional scaling $n/d:=\delta_n \rightarrow \delta \in (0,+\infty)$ and Assumption~\ref{assumption_regularity}, 
    we have
    \begin{equation}\label{eq:feature_convergence}
    \frac{1}{d}\sum_{i=1}^d \phi\left(\sqrt{d}\bTheta(0)_i,\dots,\sqrt{d}\bTheta(t)_i,\sqrt{d}\bTheta^*_i\right) \xrightarrow{~p~} \mathbb{E}\left[\phi(\Theta(0),\dots,\Theta(t),\Theta^*)\right]
    \end{equation}
    where $(\Theta(0),\dots,\Theta(t),\Theta^*)$ is the discrete-time DMFT process.

    Furthermore, we have \begin{equation}\label{eq:hard_orthogonality_B}
    \mathbb{E}\left[\cT(\Theta(0),\dots,\Theta(t))P_{U_\mathrm{H}^*}\Theta^*\right]=0
    \end{equation}for each $t$ and 
    for every measurable  $\cT:\mathbb{R}^{m(t+1)}\to\mathbb{R}$  such that the expectation makes sense.
\end{lemma}

Collectively, these lemmas establish that hard directions cannot be learned in constant time. Specifically, for each fixed time $t$, the iterates remain orthogonal to the hard directions.
The proofs of Lemma \ref{lemma_amp_data} and Lemma \ref{lemma_amp_feature} are deferred to Appendix \ref{sec_Proofs}. 
Intuitively, the proof of the second part of Lemma \ref{lemma_amp_data} is based on the fact that we can represent $V(t)$ as a deterministic function of $h(V_*,\varepsilon)$ and another random vector independent of $P_{U_\mathrm{H}^*}V_*$. We point out that
related statements were proven in \cite{dandi2024benefitsreusingbatchesgradient}, but only for the case in which all latent
directions are hard (i.e. $\Uhard=\reals^k$).

We also note that the joint distributions of $(V(0),\dots,V(t),V_*,\varepsilon)$ and $(\Theta(0),\dots,\Theta(t),\Theta^*)$ can
be computed explicitly via the discrete time DMFT formulas  for each finite $t$ (see also in Appendix \ref{sec_Proofs}). Importantly, the complexity
of these formulas scales only with $m$ and $t$, rather than with $n$ and $d$. These evolution formulas also provide a precise description of how the easy directions are learned in constant time.

Finally, we emphasize an important point.
 Even if all directions are hard, i.e., $\Uhard=\reals^k$, and hence GD does not develop positive correlation with the latent space $\bTheta_*$
 for any $t=O(1)$, the GD dynamics remains highly non-trivial on the same time horizon.
 Indeed,  under the proportional asymptotics $n = O(d)$,
 the parameters $\bTheta(t)$ evolve to overfit the training data to a certain extent.

\subsection{Result II: Bulk of the Hessian}\label{sec:MainResults_II}

This subsection characterizes the bulk spectrum of the Hessian, along the GD trajectory. 
As explained in Section \ref{subsec:overview}, the motivation of this analysis is that we expect the learning 
of the hard subspace to be driven by a spectral phase transition in the Hessian after easy directions have been learned.

We first recall the form of the Hessian of our empirical loss for parameter $\bTheta(t)$.
Under Assumption~\ref{assumption_regularity}, we have
\begin{equation}\label{eq:hessian}
\nabla^2 \Risk(\bTheta)=\frac{1}{n}\sum_{i=1}^n\left(\ell''(y_i,f_{\bTheta}(\bx_i))\nabla f_{\bTheta}(\bx_i) \nabla f_{\bTheta}(\bx_i)^\sT+ \ell'(y_i,f_{\bTheta}(\bx_i))\nabla^2 f_{\bTheta}(\bx_i)\right)
\end{equation}
where
\begin{align}
	\nabla f_{\bTheta}(\bx_i) &= \frac{1}{m}\begin{bmatrix}
		 a_1\sigma'(\btheta_1^\sT \bx_i+b_1) \bx_i \\
		\vdots \\
		 a_m\sigma'(\btheta_m^\sT \bx_i+b_m) \bx_i
	\end{bmatrix}, \label{eq:grad_network}\\
	\nabla^2 f_{\bTheta}(\bx_i) &= \frac{1}{m}\begin{bmatrix}
	 a_1\sigma''(\btheta_1^\sT \bx_i+b_1) \bx_i \bx_i^\sT & & \\
	& \ddots& \\
	& &  a_m\sigma''(\btheta_m^\sT \bx_i+b_m) \bx_i \bx_i^\sT
	\end{bmatrix}\label{eq:hess_network}
\end{align}
for each $\bTheta$. 
We then denote by $\bH_{\text{diag}}(\bTheta):=(m/n)\sum_{i=1}^n \ell'(y_i,f_{\bTheta}(\bx_i))\nabla^2 f_{\bTheta}(\bx_i)$ 
the rescaled block diagonal part of the Hessian.

\begin{lemma}\label{lemma:eigenvalue_diagonal}
Under Assumption~\ref{assumption_regularity} and additionally $\ell''(y,z)\geq 0$ for all $y,z$, there exists a constant $C$
independent of $m,d,n$ such that,  with high probability,
    \begin{equation}\label{eq:eigenvalue_sandwich}
    \lambda_{\mathrm{min}}\left(\bH_{\mathrm{diag}}(\bTheta)\right) \leq \lambda_{\mathrm{min}}\left(m\nabla^2 \Risk (\bTheta) \right) \leq \lambda_{\mathrm{min}}\left(\bH_{\mathrm{diag}}(\bTheta)\right) + \frac{C}{m}.
    \end{equation}

    Furthermore, let $\boldsymbol{\xi}$ with $\norm{\boldsymbol{\xi}}=1$ be a unit eigenvector associated with the minimum eigenvalue, i.e., $\nabla^2 \Risk (\bTheta)\boldsymbol{\xi}=\lambda_{\min}\left(\nabla^2 \Risk (\bTheta)\right)\boldsymbol{\xi}$. Then 
    there exists a constant $C$
independent of $m,d,n$ such that,  with high probability,
    \begin{equation}\label{eq:eigenvector_bound}
    \<\boldsymbol{\xi},\bH_{\mathrm{diag}}(\bTheta)\boldsymbol{\xi} \>\leq \lambda_{\mathrm{min}}\left(\bH_{\mathrm{diag}}(\bTheta)\right) + \frac{C}{m}.
    \end{equation}
\end{lemma}
This establishes that the block diagonal term $\bH_{\text{diag}}(\bTheta)$ plays a dominant role in determining the minimum eigenvalue. 
This motivates us to  analyze the spectrum of the following matrix:
\begin{equation}\label{eq:hessian_block}
\bH_{j}(t) :=\frac{1}{n}\sum_{i=1}^n \ell'(y_i,f_{\bTheta(t)}(\bx_i))\sigma''(\btheta_j(t)^\sT\bx_i+b_j)\bx_i\bx_i^\sT:=\frac{1}{n}\sum_{i=1}^n g(\bTheta(t)^\sT \bx_i,y_i;j) \bx_i\bx_i^\sT\, .
\end{equation}
We note  that $a_j\bH_j(t)$ is precisely the $j$-th block of $m\bH_{\textrm{diag}}(\bTheta(t))$.

\begin{lemma}\label{lemma:bulk}
    Let $\mu_{n,d}^j(t)$ denote the empirical spectral distribution (ESD) measure of $\bH_j(t)$.

    Under Assumption~\ref{assumption_regularity} and the proportional scaling $n/d:=\delta_n \rightarrow \delta \in (0,+\infty)$, we have $\mu_{n,d}^j (t)\Rightarrow \mu_{\infty}^j(t)$ in probability. The limiting measure $\mu_{\infty}^j(t)$ is characterized by its Stieltjes transform $\alpha_t^j(z)$ for $z\in \mathbb{H}$, which is the unique solution $\alpha_t^j\in \upper$ of
    \begin{equation}\label{eq:bulk_stieltjes}
    z+\frac{1}{\alpha_t^j}=\delta\cdot \mathbb{E}\left[\frac{G_t^j}{\delta + G_t^j \alpha_t^j}\right]
    \end{equation}
    where $G_t^j := g(V(t),h(V_{*},\varepsilon);j)$.
\end{lemma}
This lemma characterizes the bulk of the spectrum:
$\mu_{\infty}^j(t)$ follows a generalized Marchenko-Pastur law.
In order to provide some intuition of this result, note that that $\bH_j = \bX^{\sT}\bD_j\bX$, for a diagonal matrix
$\bD_j$ (with $(\bD_j)_{ii}:=g(\bTheta(t)^\sT \bx_i,y_i;j)/n$). If
$\bD_j$ were independent of $\bX$, then Lemma \ref{lemma:bulk} would follow from 
textbook theory \cite{BaiSilverstein}. The content of Lemma \ref{lemma:bulk} is that the bulk of the spectrum 
remains the same despite the dependencies between $\bTheta(t)^{\sT}\bx_i$, $y_i$ and $\bX$ (which 
arise in particular because $\bTheta(t)$ depends on $\bX$).

The proof, deferred to Appendix \ref{sec_Proofs}, relies on Gaussian conditioning 
to show that these dependencies can be accounted (for any  $t=o(d)$)  by a low rank 
deformation of $\bX$. As a consequence (by Cauchy interlacing theorem), the spectrum of the Hessian is only modified by the emergence of outliers, which we analyze next.

\subsection{Main results: Outliers of the Hessian and phase transitions}\label{sec:MainResults_III}

This subsection characterizes the isolated eigenvalues (outliers) that detach from the bulk, as well as the alignment of the associated eigenvectors with the hard subspace. We establish a phase transition at a computable threshold $\delta^*_j(t)$.

Recall that $\mu_{\infty}^j(t)$ is the limiting empirical spectral distribution of $\bH_j(t)$ with Stieltjes transform $\alpha_t^j(z)$, and let $c_j(t)$ be its left edge:
\begin{equation}\label{eq:edge}
c_j(t) := \inf \supp\; \mu_{\infty}^j(t)  \, .
\end{equation}
This can also be computed via the following edge equation which follows e.g. from \cite[Section 4]{silverstein1995analysis}
(here we use the convention $1/0=+\infty$):
\begin{equation}\label{eq:edge-variational}
c_j(t) := \sup 
\left\{-\frac{1}{\alpha}+\delta\cdot \mathbb{E}\left[\frac{G_t^j}{\delta + G_t^j \alpha}\right]:\;\;
\alpha\in \big(0, A_t^j\big)
\right\}\, ,\;\;\;\; A^j_t :=\frac{\delta}{(-\inf(\supp(\textrm{Law}(G_t^j)))\vee 0}\, .
\end{equation}

Define $\bTheta_{*\mathrm{H}}:=\bTheta_* U^*_{\mathrm{H}}\in \mathbb{R}^{d\times r}$ which is a basis of the 
hard subspace.

We impose one more regularity assumption. To state it, we denote $O_{t}^{\sT}:=\begin{bmatrix}
    \Theta_*^\sT & \Theta(0)^\sT & \dots & \Theta(t)^\sT
\end{bmatrix}$ and $\widetilde F_{t}^\sT:=\begin{bmatrix}
    F(V(0),V_*,\varepsilon)^\sT & \dots & F(V(t),V_*,\varepsilon)^\sT
\end{bmatrix}$, where we recall $F^\sT = \begin{bmatrix}
    F_1,\dots,F_m
\end{bmatrix}$.

\begin{assumption}\label{assumption:reg_1}
    For each $t$, the matrices $\mathbb{E}\left[O_tO_t^\sT\right]$ and $\mathbb{E}\left[\widetilde F_t \widetilde F_t^\sT\right]$ are positive definite.
\end{assumption}
This assumption requires that no two features are perfectly parallel during training, and no two gradients are perfectly parallel as well. This assumption holds generically and may be removed via a more refined analysis.
In the next subsections, we will present our results about outlier eigenvalues/eigenvectors,
first for a fixed time $t$, and then in the large-time limit as $t\to\infty$. (The limit $t\to\infty$
is taken after $n,d\to\infty$.)

\subsubsection{Results at fixed time $t$}

For real $z<\min(0,c_j(t))$, consider the equation (recall that $r_{\textrm{H}} = \dim(\Uhard)$, 
and $U_\mathrm{H}\in\reals^{k\times r_{\textrm{H}}}$)
\begin{equation}\label{equa_main}
\det\left(-z I_{r_{\textrm{H}}} +\mathbb{E}\left[\frac{\delta G_t^j }{\delta +G_t^j\alpha_t^j(z)}U_\mathrm{H}^{*\sT}V_{* }\left(U_\mathrm{H}^{*\sT}V_{*}\right)^\sT\right]\right)=0
\end{equation}
where we recall $G_t^j=g(V(t),h(V_{*},\varepsilon);j)$. Define the threshold (as usual, $\delta^*_j(t)=\infty$
if the condition is never verified)
\begin{equation}\label{eq:threshold_fixed_time}
\delta^*_j(t):=\inf \big\{\delta\in (0,\infty)\,  \big| \text{ Eq. \eqref{equa_main} has a real solution }z<\min(0,c_j(t)) \big\}.
\end{equation}

\begin{theorem}\label{theorem:main_result}
Under proportional scaling $n/d:=\delta_n \rightarrow \delta \in (0,+\infty)$, Assumption~\ref{assumption_regularity}, and Assumption~\ref{assumption:reg_1}, we have the following:
    \begin{enumerate}[label=(\alph*)]
        \item When $\delta>\delta^*_j(t)$, equation \eqref{equa_main} has at least one real solution $z^*<\min(0,c_j(t))$.

        In this case, denote the distinct real solutions as $z^*_1,\dots,z^*_p<\min(0,c_j(t))$ with multiplicities $k_1,\dots,k_p$. 
        Then, for each $i=1,\dots,p$, there exist eigenvalues  $z^{(i)}_{n,1},\dots,z^{(i)}_{n,k_i}\in\reals$
        and (orthonormal) eigenvectors $\bxi^{(i)}_{n,1},\dots,\bxi^{(i)}_{n,k_i}$
        such that 
        \begin{equation}\label{eq:eigenvalue_convergence}
        \bH_j(t)\bxi^{(i)}_{n,q}= z^{(i)}_{n,q}\bxi^{(i)}_{n,q}
        \quad \text{and} \quad \max_{q=1,\dots,k_i}\abs{z_{n,q}^{(i)}-z_i^*} \xrightarrow{~p~} 0\, .
        \end{equation}
        Furthermore, for each $i=1,\dots,p$, let $S^i_t\in \mathbb{R}^{r\times k_i}$ be a full-rank matrix whose columns are orthonormal and form a basis of 
    \begin{equation}\label{eq:kernel_spike}
    \ker\left(-z^*_i I_{r_{\mathrm{H}}} +\mathbb{E}\left[\frac{\delta G_t^j }{\delta +G_t^j\alpha_t^j(z^*_i)}U_\mathrm{H}^{*\sT}V_{* }\left(U_\mathrm{H}^{*\sT}V_{*}\right)^\sT\right]\right).
    \end{equation}
We have the following convergence\begin{align}\label{eq:eigenvector_alignment}
        &\plim_{n,d\to\infty}\sum_{q=1}^{k_i}\norm{\bTheta_{*\mathrm{H}}^{\sT}\bxi^{(i)}_{n,q}}^2 = \Omega^{(i)}_t\, , \\
        &\Omega^{(i)}_t:= \Tr\left(\left(S^{i\sT}_t\left(I + \mathbb{E}\left[\frac{\delta (G^j_t)^2\alpha_t^{j'}(z_i^*)}{(\delta + G^j_t\alpha_t^j(z_i^*))^{2}}U^{*\sT}_\mathrm{H}V_*\left(U^{*\sT}_\mathrm{H}V_*\right)^\sT\right]\right)S^i_t\right)^{-1}\right).
       \end{align}
        \item When $\delta <\delta^*_j(t)$, equation \eqref{equa_main} has no negative solution for $z<c_j(t)$.
        
        Let $\bxi_p$ be the unit norm eigenvector corresponding to the $p$-th smallest eigenvalue of $\bH_j(t)$. 
        Then, for every  fixed $p$, $\plim_{n,d\to\infty}\bTheta_{*\mathrm{H}}^{\sT}\bxi_p=0$.
    \end{enumerate}
\end{theorem}

\subsubsection{Results for  large time $t$}

The next assumption states that the discrete-time DMFT process converges to a stationary point.

\begin{assumption}\label{assumption:saddle_2}
The sequence of random vectors $\{V(t)\}_{t\ge 0}$ defined by the DMFT process has a limit in $L^2$,
namely $\lim_{t\to\infty}\E\left[\|V(t)-V_{\infty}\|^2\right]=0$ for some random vector $V_{\infty}$.

Equivalently,  $\{V(t)\}_{t\ge 0}$  is Cauchy in $L^2$, i.e. $\lim_{t,s\to\infty}\E\left[\|V(t)-V(s)\|^2\right]=0$.
\end{assumption}
Note that the second form of this assumption can be checked using the DMFT equations, which give access 
to the quantity $\E\left[\|V(t)-V(s)\|^2\right]$.

Define $G_\infty^j:=g(V_{\infty},h(V_*,\varepsilon);j)$ and the limiting measure $\mu^j_{\infty}(\infty)$ via its Stieltjes transform $\alpha_\infty^j(z)$, 
$z\in\upper$ as 
the unique solution $\alpha^j_{\infty}\in\upper$ (uniqueness from Theorem 4.3 in \cite{BaiSilverstein}):
\begin{equation}\label{equa_stieltjes_transform_limit}
    z+\frac{1}{\alpha_\infty^j}=\delta\cdot \mathbb{E}\left[\frac{G_\infty^j}{\delta + G_\infty^j \alpha_\infty^j(z)}\right].
\end{equation}
We define $c_j(\infty)$ to be  the left edge of the limiting measure $\mu^j_{\infty}(\infty)$:
    \begin{equation}\label{equa_edge_limit}
    c_j(\infty) := \inf \supp\; \mu^j_{\infty}(\infty)\, .
\end{equation}
This can be computed via an analogous of Eq.~\eqref{eq:edge-variational}.
The following lemma establishes the convergence of the spectral measures.
\begin{lemma}\label{lemma:limiting_measure}
    Under Assumption~\ref{assumption_regularity} and \ref{assumption:saddle_2}, we have
    $ \mu_{\infty}^j(t)\Rightarrow\mu_{\infty}^j(\infty)$ as $t\to\infty$
    and $c_j(t)\to c_j(\infty)$.

\end{lemma}

For $z<\min(0,c_j(\infty))$, consider the following equation:
\begin{equation}\label{equa_main_limit}
\operatorname{det}\left(-z I_{r_{\mathrm{H}}} +\mathbb{E}\left[\frac{\delta G_\infty^j }{\delta +G_\infty^j\alpha_\infty^j(z)}U_\mathrm{H}^{*\sT}V_{* }\left(U_\mathrm{H}^{*\sT}V_{*}\right)^\sT\right]\right)=0
\end{equation}
Define the threshold (as usual, $\delta^*_j(\infty)=\infty$ if the condition is never verified)
\begin{equation}\label{equa_threshold_limit}
\delta^*_j(\infty):=\inf \left\{\delta\in (0,\infty)\,  \big| \text{ Eq. \eqref{equa_main_limit} has a real solution }z<c_j(\infty) \right\}.
\end{equation}
\begin{theorem}\label{theorem:main_result_limit}
    Under proportional scaling $n/d:=\delta_n \rightarrow \delta \in (0,+\infty)$ and Assumption~\ref{assumption_regularity}, \ref{assumption:reg_1}, and \ref{assumption:saddle_2}, we have the following:
    \begin{enumerate}[label=(\alph*)]
        \item When $\delta>\delta^*_j(\infty)$, equation \eqref{equa_main_limit} has at least one real solution $z^*<c_j(\infty)$.

        In this case, denote the distinct real solutions as $z_1^*,\dots,z_p^*<c_j(\infty)$ with multiplicities $k_1,\dots,k_p$. 
        Then, for each $t$ and each $i=1,\dots,p$, there exist eigenvalues $z_{n,1}^{(i)},\dots,z_{n,k_i}^{(i)}\in\reals$ and (orthonormal) eigenvectors $\bxi_{n,1}^{(i)},\dots,\bxi_{n,k_i}^{(i)}$ such that 
        \begin{equation}\label{eq:eigenvalue_convergence_limit}
        \bH_j(t)\bxi_{n,q}^{(i)}=z_{n,q}^{(i)}\bxi_{n,q}^{(i)} \quad\text{and} \quad \lim_{t\rightarrow\infty}\plim_{n,d\rightarrow +\infty}\max_{q=1,\dots,k_i}\abs{z_{n,q}^{(i)}-z_i^*} = 0\, .
        \end{equation}
        Furthermore, for each $i=1,\dots,p$, let $S_{\infty}^i\in \mathbb{R}^{r\times k_i}$ be a full-rank matrix whose columns are orthonormal and form a basis of 
    \begin{equation}\label{eq:kernel_spike_limit}
    \ker\left(-z^*_i I_{r_{\mathrm{H}}} +\mathbb{E}\left[\frac{\delta G_\infty^j }{\delta +G_\infty^j\alpha_\infty^j(z^*_i)}U_\mathrm{H}^{*\sT}V_{* }\left(U_\mathrm{H}^{*\sT}V_{*}\right)^\sT\right]\right).
    \end{equation}
We have the following convergence\begin{align}\label{eq:eigenvector_alignment_limit}
        &\lim_{t\rightarrow\infty}\plim_{n,d\rightarrow +\infty}\sum_{q=1}^{k_i}\norm{\bTheta_{*\mathrm{H}}^{\sT}\bxi^{(i)}_{n,q}}^2 = \Omega^{(i)}_\infty\, , \\
        &\Omega^{(i)}_\infty:= \Tr\left(\left(S_{\infty}^{i\sT}\left(I + \mathbb{E}\left[\frac{\delta (G^j_\infty)^2\alpha_\infty^{j'}(z_i^*)}{(\delta + G^j_\infty\alpha_\infty^j(z_i^*))^{2}}U^{*\sT}_\mathrm{H}V_*\left(U^{*\sT}_\mathrm{H}V_*\right)^\sT\right]\right)S_{\infty}^i\right)^{-1}\right).
       \end{align}

        \item When $\delta <\delta^*_j(\infty)$, equation \eqref{equa_main_limit} has no negative solution for $z<c_j(\infty)$.
        
        Let $\bxi_p$ be the unit norm eigenvector corresponding to the $p$-th smallest eigenvalue of $\bH_j(t)$. 
        Then, there exists a finite $t^*_j$ such that, for every $t>t_j^*$, we have  $\plim_{n,d\to\infty}\bTheta_{*\mathrm{H}}^{\sT}\bxi_p=0$ for every fixed $p$.
    \end{enumerate}
\end{theorem}
We defer all proofs to Appendix \ref{sec:proof_result_iii}.

\begin{remark}
Assumptions~\ref{assumption_regularity}.$(c)$, $(e)$ are essential for Theorem \ref{theorem:main_result} and
 Theorem \ref{theorem:main_result_limit} 
in that they ensure that $\ell'(y_i,f_{\bTheta(t)}(\bx_i))$ and $\sigma''(\, \cdot\, )$
are uniformly lower bounded and therefore $\bH_j=\bX^{\sT}\bD_j\bX$ where 
$(\bD_j)_{ii}\ge -C/n$. Without such bounds, the bulk spectrum of 
$\bH_j(t)$ can have unbounded support on the left, preventing any meaningful outlier/bulk decomposition and making the sharp  phase transition in Theorems~\ref{theorem:main_result} and \ref{theorem:main_result_limit} impossible.
\end{remark}

\section{Consequences and discussion of main results}
\label{sec:Consequences}
\subsection{The single-neuron case}

Our general theory applies to the $m=1$ (single-neuron) case, which is already interesting. Consider the case in which
the target is a single-index model where the only latent direction is hard.
Further assume $a=1$ with bias $b \in \mathbb{R}$. The Hessian with respect to $\btheta(t)$ is
 \begin{equation}\label{eq:hessian_single_neuron}
\nabla^2 \Risk(\btheta(t))=\frac{1}{n}\sum_{i=1}^n\Big(\ell''\big(y_i,f_{\btheta(t)}(\bx_i)\big)\sigma'(\btheta(t)^\sT\bx_i+b)^2+\ell'\big(y_i,f_{\btheta(t)}(\bx_i)\big)\sigma''(\btheta(t)^\sT\bx_i+b)\Big)\bx_i\bx_i^\sT.
 \end{equation}
 Equivalently, the diagonal entries become $\widehat g= \ell''\,\sigma'^2+\ell'\sigma''$. Define
 \begin{equation}\label{eq:single_neuron_weights}
 \widehat G_t:=\widehat g\big(V(t), h(V_*,\varepsilon)\big),\quad \widehat G_\infty:=\widehat g\big(V_\infty, h(V_*,\varepsilon)\big).
 \end{equation}
Since $m=1$, we can drop the subscript $j$ and
all results in the previous section hold with $G_t^j$ replaced by 
$\widehat G_t$; we denote the corresponding Stieltjes transform by $\widehat \alpha_t(z)$. In particular, the outlier equation \eqref{equa_main} at fixed time reads
 \begin{equation}\label{eq:spike_single_neuron}
 \det\!\left(-z I+\mathbb{E}\left[\frac{\delta\widehat G_t}{\delta+\widehat G_t\widehat \alpha_t(z)}U_\mathrm{H}^{*\sT}V_*\left(U_\mathrm{H}^{*\sT}V_*\right)^\sT\right]\right)=0,
 \end{equation}
 with the large-time analogue obtained by replacing $(\widehat G_t,\widehat \alpha_t)$ by $(\widehat G_\infty,\widehat \alpha_\infty)$. The eigenvector-alignment and threshold statements follow by the same substitutions: replace $G_t^j$ (resp. $G_\infty^j$) by $\widehat G_t$ (resp. $\widehat G_\infty$) and $\alpha_t^j$ (resp. $\alpha_\infty^j$) by $\widehat \alpha_t$ (resp. $\widehat \alpha_\infty$) in \eqref{equa_main}--\eqref{equa_threshold_limit} and in the corresponding alignment formulas.

\subsection{Grokking}
\label{sec:Grokking}

 As discussed in the introduction, our results suggest the emergence of
 grokking for $\delta>\delta_{\sNN}$. (See Figure  \ref{fig:grokking_dynamics_first}).
 \begin{enumerate}
     \item During the first training stage ($t=O(1)$), the network cannot learn the hard directions, and therefore the parameters overfit slightly to the data whenever $\delta$ is finite (even when there is no easy directions). 
     \item During the second training stage ($t\gg 1$), as long as $\delta>\delta_{\sNN}$, 
     the emergence  of negative directions of the Hessian drives
     the learning of hard features. Consequently the generalization error decreases. 
 \end{enumerate}
 This phenomenon becomes less pronounced for $\delta\gg \delta_{\sNN}$
 because the generalization error decreases by uniform convergence arguments.
 We elaborate on this phenomenon further in the numerical experiments of Section 
 \ref{sec:examples_and_numerical}.

\subsection{Interpretation of the outliers equation}

Useful intuition into the outlier equations \eqref{equa_main} and \eqref{equa_main_limit} 
is gained by recognizing that they coincide with the outlier equations of
a simpler random matrix model. We will describe this interpretation for $t$
fixed (Theorem \ref{theorem:main_result}); the analogous interpretation for $t=\infty$
(Theorem \ref{theorem:main_result_limit}) is immediate.

In order to motivate the simplified random matrix model, note that, by the DMFT formulas, there exists a measurable map $\varphi_t$ and an auxiliary vector $\zeta$ independent of $U_{\mathrm{H}}^{*\sT}V_*$ such that
\begin{equation}\label{eq:state_evolution_map}
V(t) \ed \varphi_t\big(h(V_*,\varepsilon),\, \zeta\big).
\end{equation}
Therefore, writing $G_t^j=g(V(t), h(V_*,\varepsilon);j)$, there exists a map $\Gamma_t^j$ (obtained by composing $g$ with $\varphi_t$) such that, jointly with $U_\mathrm{H}^{*\sT}V_*$,
\begin{equation}\label{eq:gamma_map}
\big(G_t^j, U_\mathrm{H}^{*\sT}V_*\big) \ed
\big(\Gamma_t^j(h(V_*,\varepsilon),\zeta), U_\mathrm{H}^{*\sT}V_*\big), \quad \zeta \perp\!\!\!\perp U_\mathrm{H}^{*\sT}V_*.
\end{equation}

Motivated by these remarks, we define the random matrix model
\begin{align}\label{eq:surrogate_hessian}
&\bH_j^{\ssim}(t) := \frac{1}{n}\sum_{i=1}^n \Gamma_t^j\big(h(\bTheta_{*\mathrm{H}}^\sT\bx_i,v_i,\varepsilon_i),\zeta_i\big)\, \bx_i \bx_i^\sT, \\ 
&(v_i,\varepsilon_i,\zeta_i)_{i\le n} \stackrel{\text{i.i.d.}}{\sim} {\rm Law}(U_{\mathrm{E}}^{*\sT}V_*,\varepsilon,\zeta),\ \ (v_i,\varepsilon_i,\zeta_i) \perp\!\!\perp \bx_i.\nonumber
\end{align}
Note that this differs from the actual distribution of $\bH_j(t)$
in two ways: $(i)$~We replaced the diagonal terms $g(\bTheta(t)^{\sT}\bx_i, h(\bTheta_*^{\sT}\bx_i,\eps_i);j)$
(which can be written in the form $\Gamma_{n,t}^j\big(h(\bTheta_{*\mathrm{H}}^\sT \bx_i ,v_{n,i},\varepsilon_{n,i}),\, \zeta_{n,i}\big)$ for some $\Gamma_{n,t}^j$ and $(v_{n,i},\varepsilon_{n,i},\zeta_{n,i})$)
by its asymptotic counterpart 
$\Gamma_t^j\big(h(\bTheta_{*\mathrm{H}}^\sT\bx_i,v_i,\varepsilon_i),\zeta_i\big)$; $(ii)$~We drop dependencies
between $(v_{n,i},\varepsilon_{n,i},\zeta_{n,i})$ and $\bx_i$. 

By construction, $\Gamma_t^j(h(V_*,\varepsilon),\zeta) \ed G_t^j$, so the bulk Stieltjes transform and edge equation of $\bH_j^{\ssim}(t)$ obey the same generalized Marchenko--Pastur law as in Lemma~\ref{lemma:bulk}. Moreover, the standard resolvent computation for outliers in this multi-index model yields the outlier condition
\begin{equation}\label{eq:surrogate_spike}
\operatorname{det}\left(-z I +\mathbb{E}\left[\frac{\delta \Gamma_t^j(h(V_*,\varepsilon),\zeta)}{\delta + \Gamma_t^j(h(V_*,\varepsilon),\zeta) \alpha_t^j(z)} U_\mathrm{H}^{*\sT}V_*\big(U_\mathrm{H}^{*\sT}V_*\big)^\sT\right]\right)=0
\end{equation}
which coincides with \eqref{equa_main} upon identifying $\Gamma_t^j \ed G_t^j$.
Hence, the critical aspect ratio $\delta_j^*(t)$ defined from \eqref{equa_main} is exactly the spectral (BBP-type) threshold for detecting the hard subspace in the multi-index spiked covariance model $\bH_j^{\ssim}(t)$ (with the additional requirement $z<0$)
(see \cite{kovačević2025spectralestimatorsmultiindexmodels}).

In words, we claim that the informative outlier eigenvalue of the Hessian block 
$\bH_j(t)$ are the same as the ones of the simplified model 
$\bH_j^{\ssim}(t)$ which retains only the dependencies between 
$\bTheta_{*\mathrm{H}}^{\sT}\bx_i$ and $\bx_i$.
We emphasize that this does not mean that 
full random matrices $\bH_j(t)$ and $\bH_j^{\ssim}(t)$ are equal in law,
or even that they have the same outliers asymptotically.
In particular, $\bH_j^{\ssim}(t)$ can have at most $r_{\mathrm{H}}$ outliers corresponding to the hard directions, whereas $\bH_j(t)$ might have additional outliers whose eigenvectors are not aligned with the hard directions.

\subsection{Suboptimality among polynomial-time algorithms}

From previous discussion, we see that feature learning via gradient descent can be interpreted as a two-stage polynomial-time algorithm:
\begin{enumerate}
    \item \emph{Preprocessing stage:} Run $O(1)$ steps of gradient descent, which learns the easy directions but cannot recover the hard subspace (Lemma~\ref{lemma_amp_data}).
    \item \emph{Feature learning stage:} Learn the hard directions via the spectral algorithm induced by the Hessian at the current iterate, as characterized in Section~\ref{sec:MainResults_III}.
\end{enumerate}
We note that the first stage also induces the data preprocessing for the spectral feature learning algorithm in the second stage, as seen in the simplified model \eqref{eq:surrogate_hessian}. This preprocessing is in general sub-optimal.

This can be more clearly illustrated in the case when all latent directions are hard ($\Uhard = \mathbb{R}^k$). In this setting, there are no easy directions to learn, and the simplified model \eqref{eq:surrogate_hessian} (with the same learning threshold) reduces to
\begin{equation}\label{eq:spectral_all_hard}
\bH_j^{\ssim}(t) = \frac{1}{n}\sum_{i=1}^n \Gamma_t^j(y_i, \zeta_i) \bx_i\bx_i^\sT, \quad \zeta_i \stackrel{\text{i.i.d.}}{\sim} \Unif([0,1]),\ \zeta_i \perp\!\!\!\perp (\bx_i, y_i).
\end{equation}
Here $\zeta_i$ is independent random noise that does not depend on the data $(\bx_i, y_i)$, and without loss of generality we can assume it $\Unif([0,1])$. The function $\Gamma^j_t$ depends on the network architecture $\sigma$, $a_j$, $b_j$), the loss function ($\ell$), the training hyperparameters ($\eta$).
A general spectral method uses instead an arbitrary preprocessing function $\Psi$
\cite{mondelli2018fundamentallimitsweakrecovery,kovačević2025spectralestimatorsmultiindexmodels}:
\begin{equation}\label{eq:spectral_estimator_form}
\bH_{\Psi,\xi} = \frac{1}{n}\sum_{i=1}^n \Psi(y_i, \zeta_i)\bx_i\bx_i^\sT, \quad \zeta_i \stackrel{\text{i.i.d.}}{\sim} \Unif([0,1]),\ \xi_i \perp\!\!\!\perp 
(\bx_i, y_i)\, .
\end{equation}
The corresponding spectral threshold $\delta_{\ssp}(\Psi)$ is typically minimized at a unique
$\Psi_{\sopt}$ (up to measure zero modifications) and hence in general
we have the strict inequality:
\begin{equation}\label{eq:suboptimal_ineq}
\delta_j^*(t) =\delta_{\ssp}(\Gamma_t^j)>\delta_{\sopt} :=
\inf_{\Psi}\delta_{\ssp}(\Psi). 
\end{equation}
and in fact, $\inf_{t\ge 0} \delta_j^*(t) >\delta_{\sopt}$.

As a simple illustration, consider a single-neuron network ($m=1$) learning a single-index model ($k=1$), with $a=1$ and $b=0$. Suppose the activation function $\sigma$ is locally quadratic near zero (i.e., $\sigma(0)=\sigma'(0) = 0$ and $\sigma''(0) \neq 0$), and the weight vector is initialized with $\|\btheta(0)\| = d^{-C}$ for some large constant $C > 0$. In this regime, the gradient remains negligible for any $O(1)$ steps, and the dynamics is dominated by the Hessian. Consequently, gradient descent effectively performs a power method on the Hessian at zero:
\begin{equation}\label{eq:hessian_example}
\nabla^2 \Risk(\mathbf{0}) \propto \frac{1}{n}\sum_{i=1}^n \ell'(y_i,0)\,\bx_i\bx_i^\sT.
\end{equation}
The feature learning threshold $\delta_{j}^*(t)$ coincides with
the spectral threshold $\delta_{\ssp}^*(\ell'(\cdot\,,0))$ for the preprocessing function $y\mapsto \ell'(y, 0)$. 
This is in general not optimal (consider for instance $h(z,\eps) = z^2$),
demonstrating that inequality \eqref{eq:suboptimal_ineq} can be strict.

\section{Examples and numerical illustrations}\label{sec:examples_and_numerical}

This section instantiates the theory on the following data distribution with 
$k=1$ (`real phase retrieval')
\begin{equation}\label{eq:phase_retrieval_model}
 y = h(z,\varepsilon) = z^2, \qquad z := \btheta_*^\sT\bx, \quad \bx\sim \normal(\boldsymbol{0},\bI_d), \quad \|\btheta_*\|_2=1.
\end{equation}
In this example, the information-theoretic threshold and the algorithmic threshold are known to be $\delta_{\sIT}=\delta_{\salg}=1/2$ \cite{mondelli2018fundamentallimitsweakrecovery}.

We derive explicit thresholds for this task and  compare them with numerical experiments with various activation functions. As we will see, the predicted thresholds precisely match empirical phase transitions and explain the observed grokking behavior when $\delta>\delta_{\sNN}$. 

We extend the numerical analysis to multi-neuron networks in Appendix \ref{sec:additional_simulation}.

\subsection{General setting}

In this setting, the hard subspace coincides with the entire latent space,
i.e. $r=k=1$.
Consequently, Lemmas~\ref{lemma_amp_data} and~\ref{lemma_amp_feature} imply that no correlation with $\btheta_*$ can be learned in finitely many gradient steps.

We use the standard Huber loss with parameter $M>0$, which ensures bounded diagonal terms in the Hessian and thus a finite left edge of the bulk spectrum:
\begin{equation}\label{eq:huber_loss}
 \ell(y,z)=\begin{cases}
 \tfrac12\,(z-y)^2, & |z-y|\le M,\\
 M|z-y| - \tfrac12 M^2, & |z-y|>M,
 \end{cases}
\end{equation}

For $m=1$ with $(a,b)=(1,0)$, the empirical risk and Hessian read
\begin{equation}\label{eq:pr_hessian_single}
\Risk\big(\btheta\big)=\frac{1}{n}\sum_{i=1}^n\ell(y_i,\sigma(\btheta^\sT\bx_i))\, ,\;\;\;
 \nabla^2 \Risk\big(\btheta\big)=\frac{1}{n}\sum_{i=1}^n
 \widehat g(\btheta^\sT\bx_i;y_i)\, \bx_i\bx_i^{\sT}\, ,
\end{equation}
where
\begin{equation}\label{eq:pr_weight_function}
 \widehat g(u;y):=\ell''\big(y,\sigma(u)\big)\big(\sigma'(u)\big)^2+\ell'\big(y,\sigma(u)\big)\sigma''(u), \qquad u\in\reals,\ y\in\reals.
\end{equation}
Using the discrete-time DMFT random variables $(V(t),V_*)$ from Lemma~\ref{lemma_amp_data} (here $y=h(V_*,\varepsilon)=V_*^2$), we define
\begin{equation}\label{eq:pr_state_evolution_weights}
 \widehat G_t:=\widehat g\big(V(t);V_*^2\big),\qquad \widehat G_\infty:=\widehat g\big(V_{\infty};V_*^2\big).
\end{equation}

\paragraph{Formulas for the asymptotic spectrum.} The limiting spectral distribution of the Hessian follows a generalized Marchenko--Pastur law with Stieltjes transform $\widehat\alpha_t(z)$ determined by (cf. Lemma~\ref{lemma:bulk}):
\begin{equation}\label{eq:pr_bulk}
 z+\frac{1}{\widehat \alpha_t(z)}=\delta\cdot\E\left[\frac{\widehat G_t}{\delta+\widehat G_t\widehat \alpha_t(z)}\right].
\end{equation}
we denote the left edge of the support of the asymptotic ESD by $c(t)$.
In the large-time limit, replace $(\widehat G_t,\widehat \alpha_t,c(t))$ by $(\widehat G_\infty,\widehat \alpha_\infty,c(\infty))$ as in \eqref{equa_stieltjes_transform_limit}–\eqref{equa_edge_limit}.

Since $r=1$, there is only one hard direction. The outlier equation \eqref{equa_main} specializes to the scalar equation
\begin{equation}\label{eq:pr_outlier_finite}
 S_t(z;\delta):=-z+\E\left[\frac{\delta\widehat G_t V_*^2}{\delta+\widehat G_t\widehat\alpha_t(z)}\right]=0,\qquad z<\min(c(t),0),
\end{equation}
and the finite-time weak-recovery threshold is
\begin{equation}\label{eq:pr_threshold_finite}
 \delta^*(t):=\inf\Big\{\delta\in(0,\infty]\ :\ \exists\ z<\min(c(t),0)\text{ solving Eq.~\eqref{eq:pr_outlier_finite}}\Big\}.
\end{equation}
The obvious analogous of these formulas hold for $t=\infty$.

\paragraph{Computing the threshold.}
Since $\widehat\alpha_t'(z)>0$ for $z<c(t)$ by standard properties of Stieltjes transforms, differentiating yields $\frac{\partial}{\partial z} S_t(z;\delta)<0$ on $(-\infty,\min(c(t),0))$. Hence, for each fixed $\delta$, $S_t(z;\delta)$ is strictly decreasing in $z$, and a root exists for some $z<\min(c(t),0)$ if and only if $S_t\big(z^{\dagger}(t);\delta\big)< 0$ at the boundary value
\begin{equation}\label{eq:pr_boundary}
z^{\dagger}(t):=\min\big(c(t),0\big),\qquad \widehat\alpha_t^{\dagger}:=\lim_{z\uparrow z^{\dagger}(t)}\widehat\alpha_t(z).
\end{equation}
Therefore, the finite-time threshold $\delta^*(t)$ solves the boundary equation
\begin{equation}\label{eq:thr_finite}
\E\!\left[\frac{\delta\widehat G_t V_*^2}{\delta+\widehat G_t\widehat\alpha_t^{\dagger}}\right] =z^{\dagger}(t)
\end{equation}
for $\delta\in(0,\infty]$. If \eqref{eq:thr_finite} has no positive solution, then $\delta^*(t)=+\infty$.
The large-time threshold $\delta^*(\infty)$ is computed analogously by replacing $(\widehat G_t,\widehat\alpha_t, c(t))$ with $(\widehat G_\infty,\widehat\alpha_\infty, c(\infty))$.

\begin{figure}[!t]
    \centering
    \includegraphics[width=0.48\textwidth]{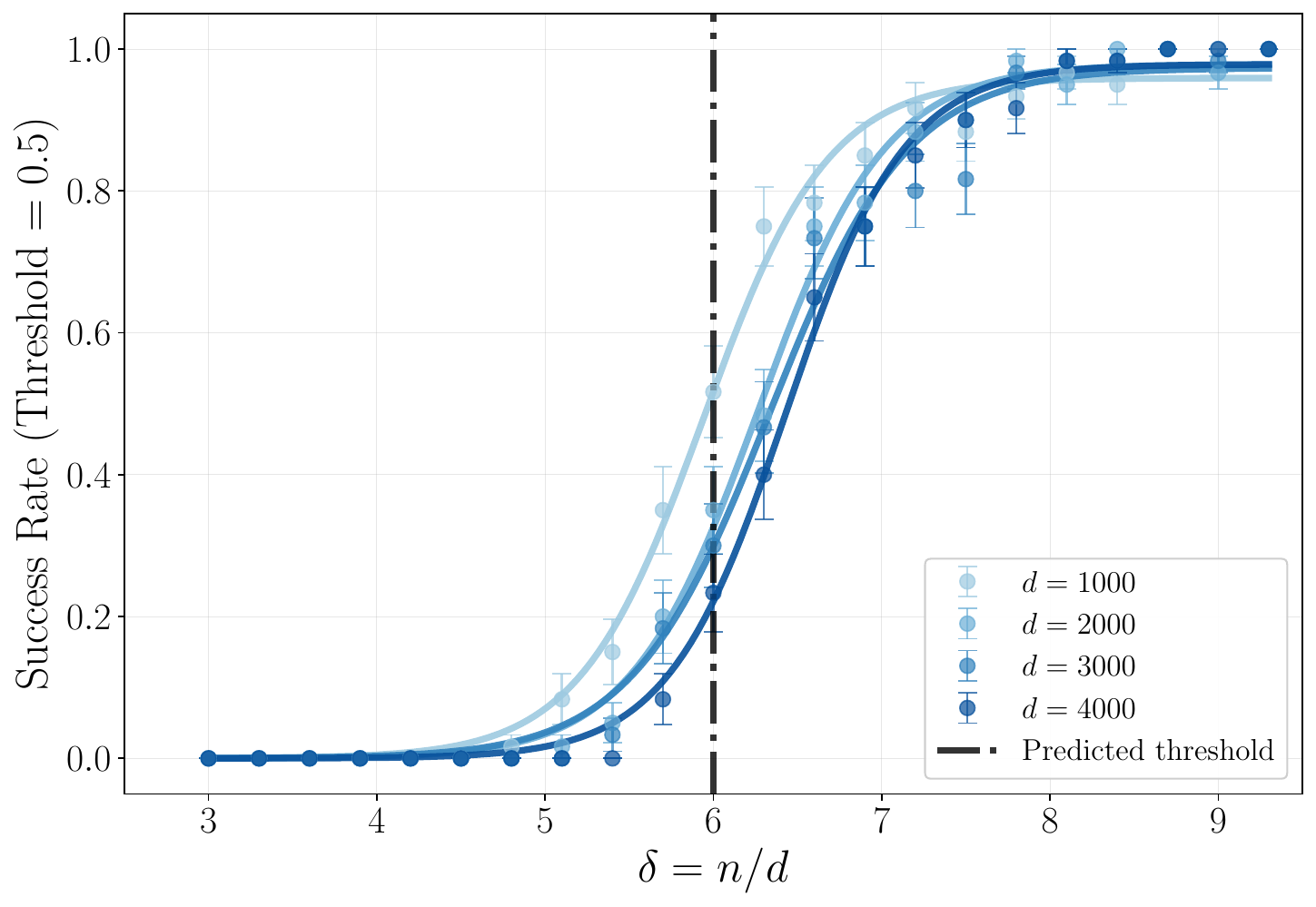}
    \includegraphics[width=0.48\textwidth]{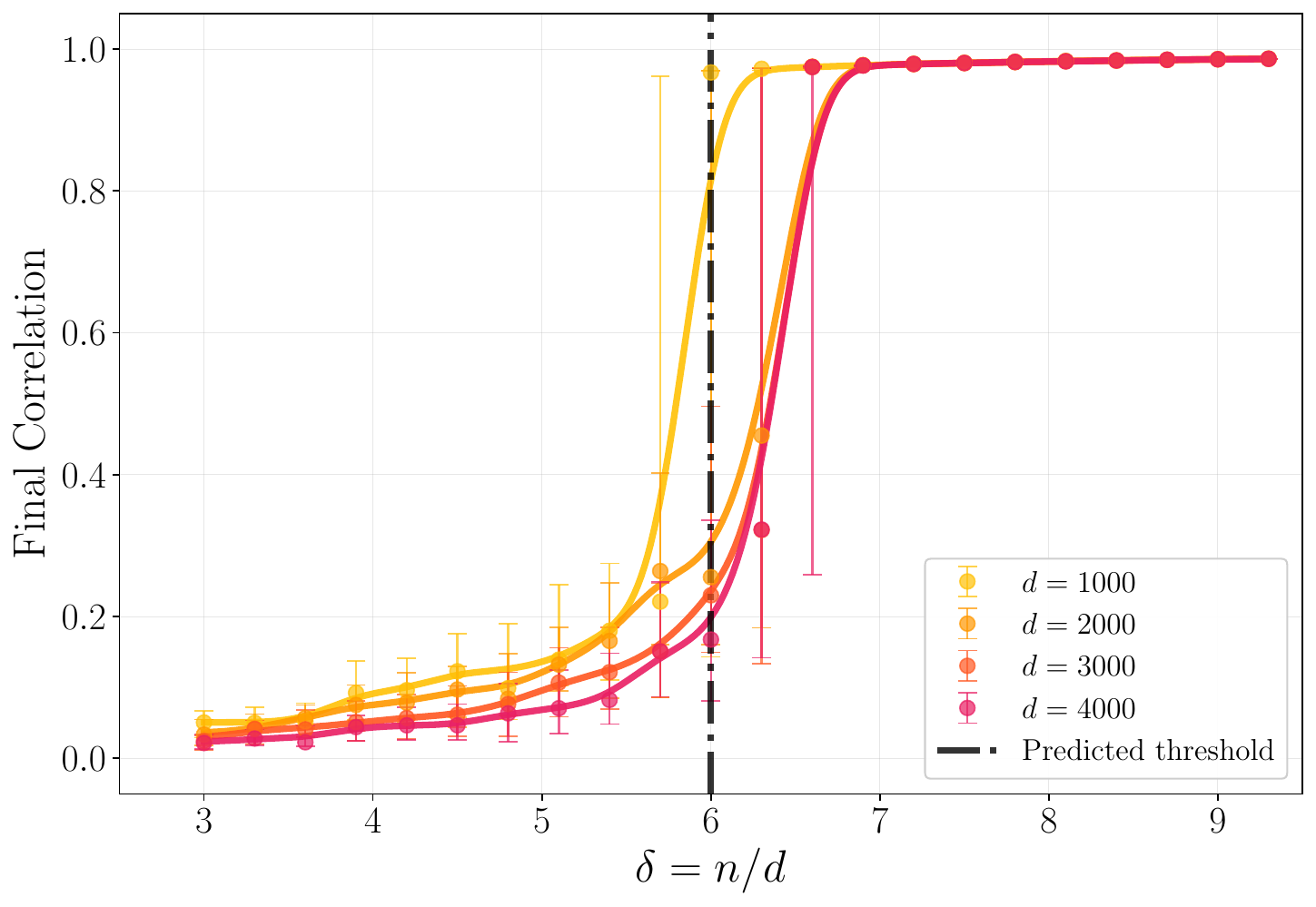}
    \vspace{-0.1cm}
    \caption{Phase transitions for $\GeLU$ with single-neuron learner. Left: Success rate with sigmoid fits; each data point shows mean $\pm$ standard deviation across 60 independent runs. Right: Final correlation with median and 30th/70th percentile error bars (smooth splines fitted). Both metrics exhibit sharp phase transitions matching the predicted threshold $\delta^*\approx 6.0$ (dash-dot line).}
    \label{fig:phase_transition_gelu}
    \end{figure}

    \begin{figure}[!t]
        \centering
        \includegraphics[width=0.48\textwidth]{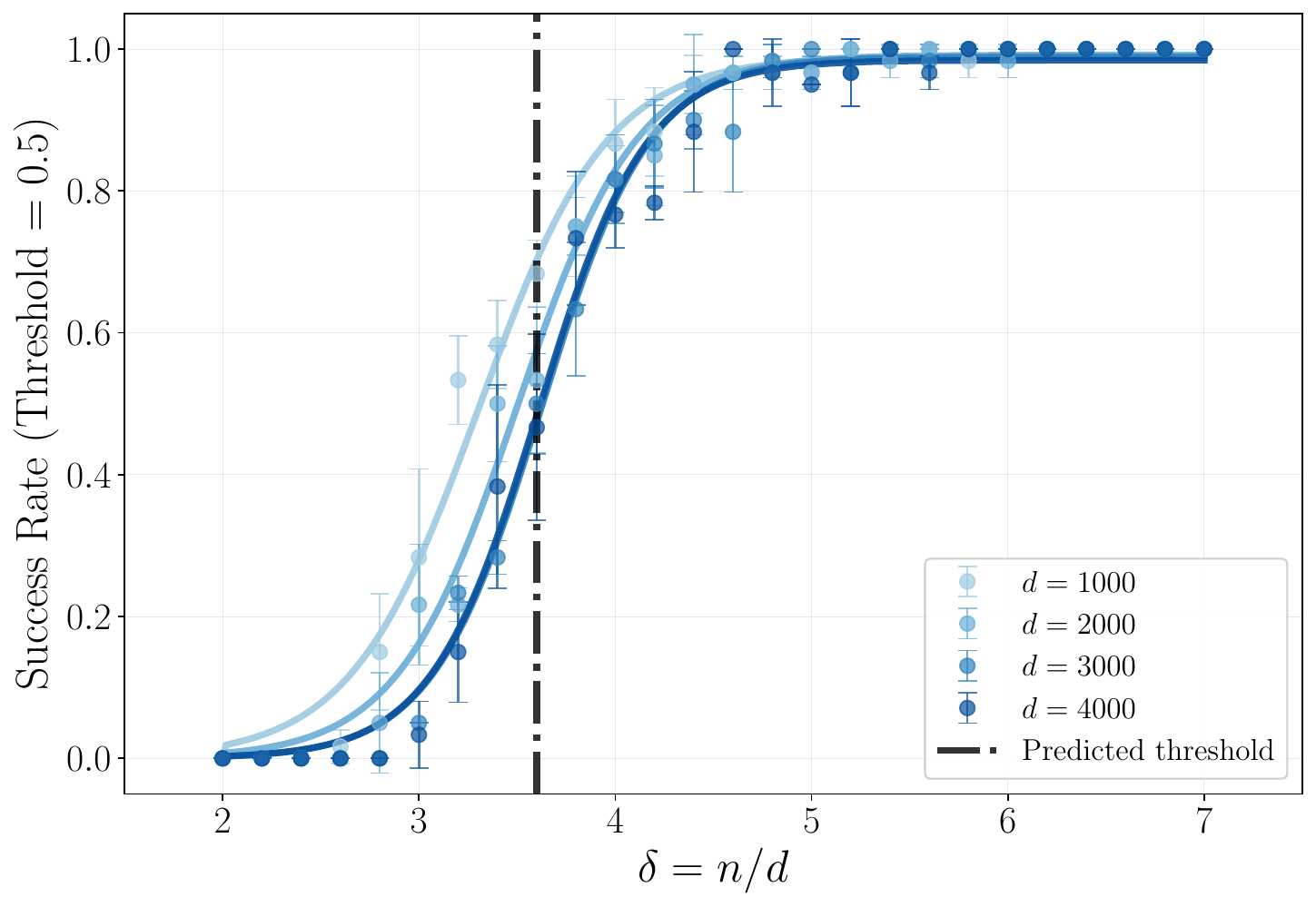}
        \includegraphics[width=0.48\textwidth]{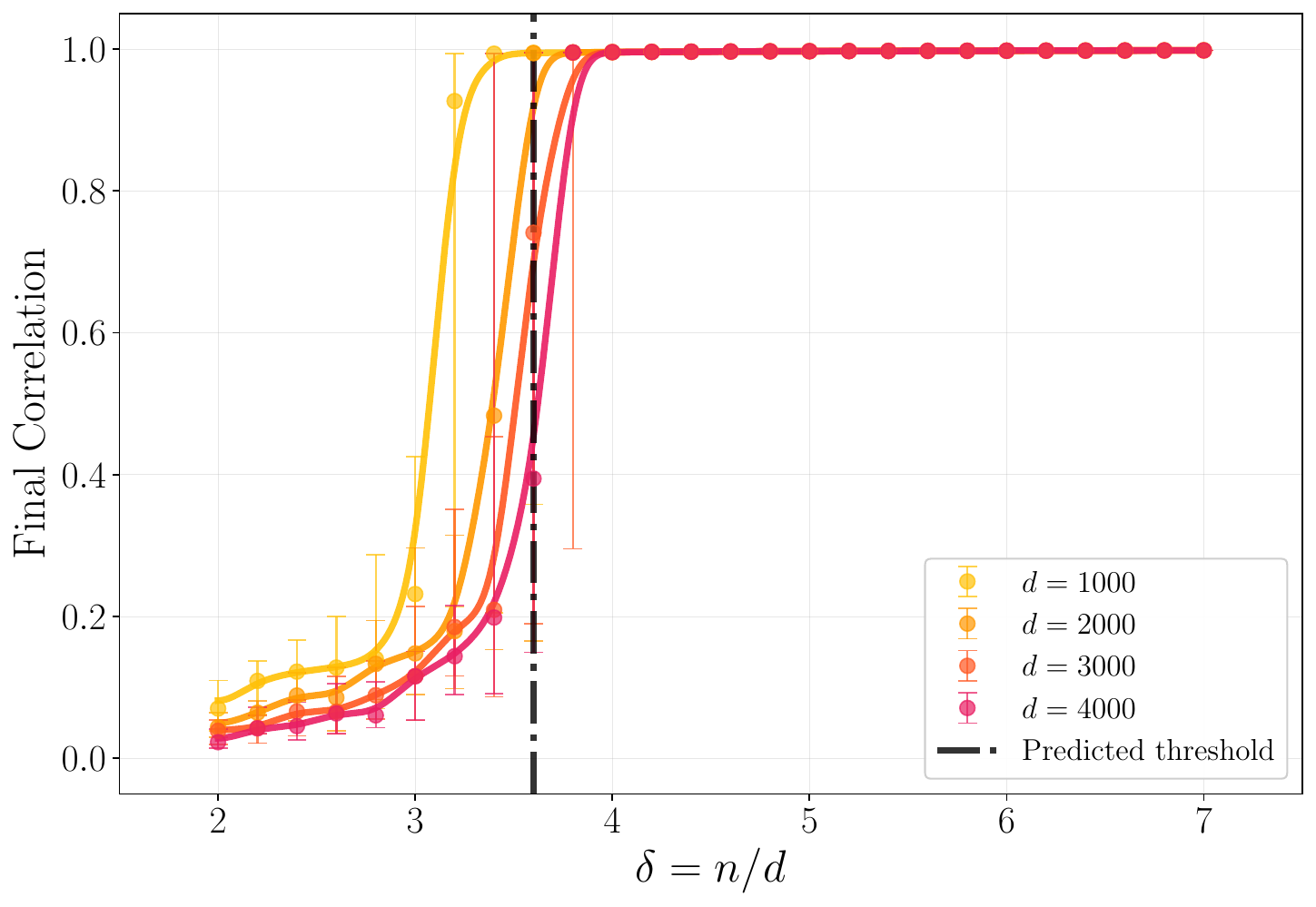}
        \vspace{-0.1cm}
        \caption{Phase transitions for $\mathsf{Quad}$ with single-neuron learner. Left: Success rate with sigmoid fits; each data point shows mean $\pm$ standard deviation across 60 independent runs. Right: Final correlation with median and 30th/70th percentile error bars (smooth splines fitted). Both metrics exhibit sharp phase transitions matching the predicted threshold $\delta^*\approx 3.6$ (dash-dot line).}
        \label{fig:phase_transition_quadratic}
      \end{figure}

\subsection{Results for $\GeLU$ and $\mathsf{Quad}$ activations}

We begin by reporting results for Huber loss with $M=1$ and $\GeLU$ activation
\begin{equation}\label{eq:num_gelu}
\sigma_{\GeLU}(z) := z\Phi(z)\, ,\;\;\;\;
\sigma_{\mathsf{Quad}}(z) = \alpha\left(1 - e^{-(\beta z)^2}\right),\qquad \alpha=9,\quad \beta=\tfrac{1}{3},
\end{equation}
where $\Phi(\cdot)$ is the CDF of the standard normal distribution.
For $\Quad$ activation we use $\alpha=9$, $\beta=1/3$,

We initialize $\btheta(0)$ as a uniformly random unit vector, and run full-batch gradient descent with learning rate $\eta=1.5$ for $\GeLU$ and $\eta=0.25$
for $\Quad$ (these were selected by optimizing accuracy over a grid of values of $\eta$) for at most 1000 iterations. We terminate early once the correlation reaches $\rho(t) \geq 0.9$, after which we run for 100 additional steps to ensure sufficient training. We also stop if the algorithm plateaus for too long, i.e., if the correlation changes by less than $0.01$ over 500 consecutive steps.

\paragraph{Simulation results.}

We  consider $d\in\{1000,2000,3000,4000\}$ with varying sample sizes $n$ to explore different sampling ratios $\delta = n/d$. For each configuration $(d,\delta)$, we perform 60 independent runs with different random seeds for $\btheta_*$, $\btheta(0)$, and the data $\{\bx_i,y_i\}_{i=1}^n$.

We measure the correlation between the learned parameter and the ground truth via the normalized inner product
\begin{equation}\label{eq:num_correlation}
\rho(t) := \left|\frac{\btheta(t)^\sT \btheta_*}{\|\btheta(t)\|_2}\right|
\end{equation}
evaluated at the final time $t=T$ (maximum 1000 steps). We report two key metrics: $(i)$ the \emph{success rate}, defined as the proportion of runs achieving $\rho(T)\geq 0.5$, and $(ii)$ the \emph{final correlation} $\rho(T)$ itself. For success rates, we fit sigmoid curves; for correlations, we report medians with 30th/70th percentile error bars.

Figures~\ref{fig:phase_transition_gelu} and \ref{fig:phase_transition_quadratic}
report simulation results for $\GeLU$ and $\Quad$ respectively.
We plot empirical success rate and final correlation versus $\delta = n/d$ across four dimensions. Both metrics exhibit a sharp phase transition around $\delta=6$
(for $\GeLU$) and $\delta=3.6$ (for $\mathsf{Quad}$) : the success rate (left panel) jumps from near zero to near one, while the final correlation (right panel) transitions from $\rho(T)\approx 0$ to $\rho(T)\approx 1$.

These behaviors match the predictions from high-dimensional asymptotics that are reported as vertical lines. We discuss the computation of these predictions next.

\begin{figure}[!t]
  \centering
  \includegraphics[width=0.48\textwidth]{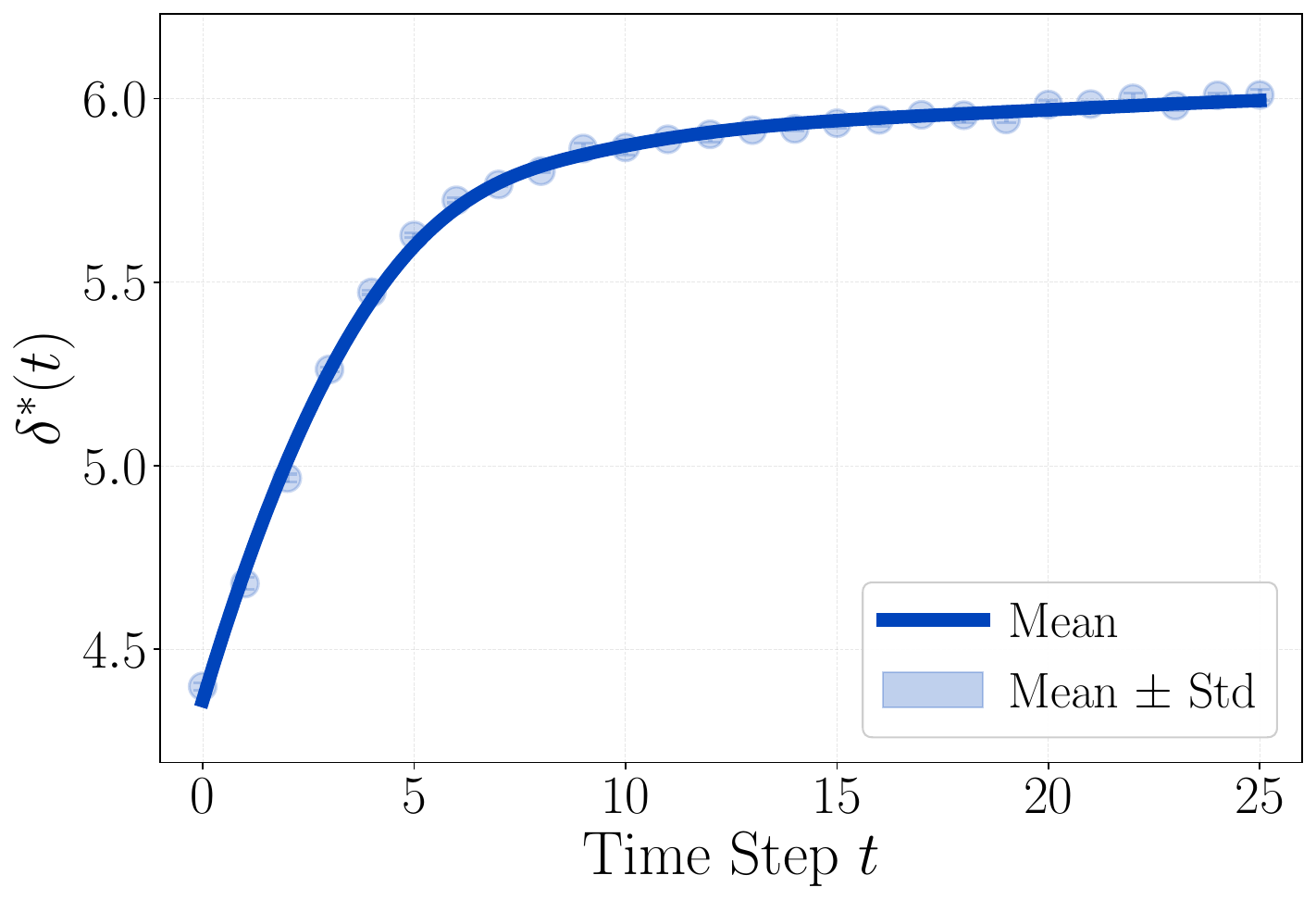}
  \includegraphics[width=0.48\textwidth]{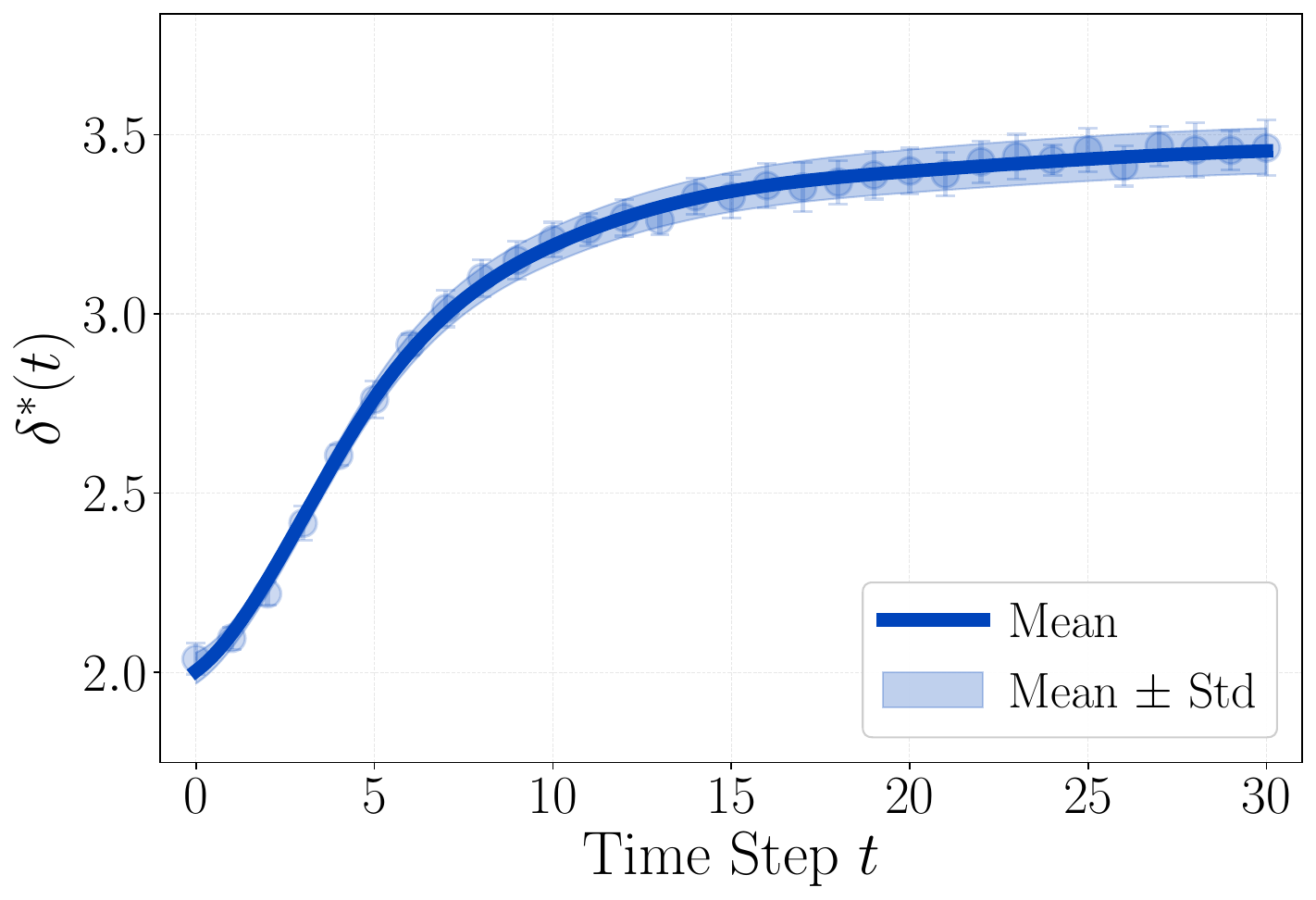}
  \vspace{-0.1cm}
  \caption{Predicted feature learning thresholds $\delta^*(t)$ with $m=1$.
  Left: $\GeLU$, $m=1$. Right: $\Quad$. Circles: mean over $3$ or $20$ seeds
  (respectively for $\GeLU$ 
  and $\mathsf{Quad}$); error bars: standard deviation. Solid line and shaded band: spline interpolations of the mean and mean $\pm$ standard deviation. }
  \label{fig:time_evolution_gelu_new}
\end{figure}

\paragraph{Theoretical predictions.}

We compute the predicted thresholds $\delta^*(t)$ for $t \in \{0, 1, \dots, 25\}$
(for $\GeLU$) or $t \in \{0, 1, \dots, 25\}$
(for $\mathsf{Quad}$) as follows. For each $t$ and each $\delta$, we run the discrete-time DMFT recursion (Section~\ref{sec:MainDMFT}) forward to time $t$, drawing fresh Monte Carlo samples (with sample size $N = 9 \times 10^5$) at each step to estimate the expectations in \eqref{eq:dmft_kernels} and update the covariance and response kernels $(C_\theta, C_\ell, R_\theta, R_\ell)$. We then check whether an outlier eigenvalue emerges to the left of the bulk spectrum by testing if the outlier equation \eqref{equa_main} admits a solution $z < c(t) - 0.01$, where $c(t)$ is the left bulk edge (cf.\ \eqref{eq:edge}); the gap of $0.01$ ensures numerical stability. We use binary search over $\delta$ (tolerance $0.01$) to locate the critical threshold, and average results over 3 independent random seeds. 

The resulting thresholds are plotted in Figure~\ref{fig:time_evolution_gelu_new}. 
For $\GeLU$, the threshold exhibits rapid initial growth from $\delta^*(0) \approx 4.35$ to $\delta^*(10)\approx 5.85$, and then appears to converge to a limiting value. To estimate this limit, we fit a degree-4 polynomial in $1/t$ to the data, obtaining $\delta^*(\infty) \approx 6.05$.

For $\mathsf{Quad}$, the threshold exhibits rapid initial growth from 
$\delta^*(0) \approx 2.0$ to $\delta^*(10) \approx 3.2$, and then appears to converge to a limiting value. To estimate this limit, we fit a degree-4 polynomial in $1/t$ to the data, obtaining $\delta^*(\infty) \approx 3.6$.

We report the predictions for $\delta^*(\infty)$ in 
Figures~\ref{fig:phase_transition_gelu} and~\ref{fig:phase_transition_quadratic} and observe that they closely match the empirical phase transition.

In summary, the empirical results confirm our expectation that,
as $n,d \to \infty$, the threshold $\delta^*(\infty)$ controls the learning of latent directions.
However, the separation between the first phase of training ($t=O(1)$),
which in this case corresponds to pure overfitting, and the second one 
 ($t=\Theta(\log d)$)\footnote{The time $\Theta(\log d)$ is known to be the saddle point escape time; see, e.g., \cite{arous2021onlinestochasticgradientdescent}.}, on which feature learning takes place, is relatively modest for 
 moderate $d$. Since $\delta^*(t)<\delta^*(\infty)$ in this case, it
 is possible that feature learning occurs at smaller $\delta$ 
 for moderate $d$.

\subsection{Training dynamics and grokking}

We already illustrated the grokking phenomenon in Figure 
\ref{fig:grokking_dynamics_first} of the introduction. We discussed its connection 
with the two stages learning dynamics in Section \ref{sec:Grokking}.

Figure~\ref{fig:grokking_multi_delta} illustrates how the learning dynamics change as a function of $\delta$, for a $\GeLU$ neuron with $d=5000$ and $\eta=0.5$ on noiseless phase retrieval. Recall that the predicted feature learning threshold is $\delta_{\sNN}\approx 6.0$. For $\delta=5$ (below threshold), no feature learning takes place: the correlation $\rho(t)$ remains near zero throughout training. At $\delta=10$, slightly above threshold, grokking emerges but requires extended training time ($t\approx 300$, compared to $t\approx 80$ in Fig.~\ref{fig:grokking_dynamics_first}). As $\delta$ increases to $15$, the transition accelerates (occurring around $t\approx 100$), while grokking can still be observed. Finally, at $\delta=50$, learning becomes rapid with negligible generalization gap, and the grokking phenomenon becomes less pronounced. In total, these observations confirm our theoretical predictions: $(i)$~no learning for $\delta<\delta_{\sNN}$; $(ii)$~grokking emerges for $\delta>\delta_{\sNN}$; $(iii)$~grokking time increases as $\delta\downarrow\delta_{\sNN}$; $(iv)$~generalization gap diminishes for $\delta\gg\delta_{\sNN}$.

The same scenario holds with other activation functions as shown in Appendix \ref{sec:additional_simulation}.

\begin{figure}[!t]
  \centering
  \includegraphics[width=1.0\textwidth]{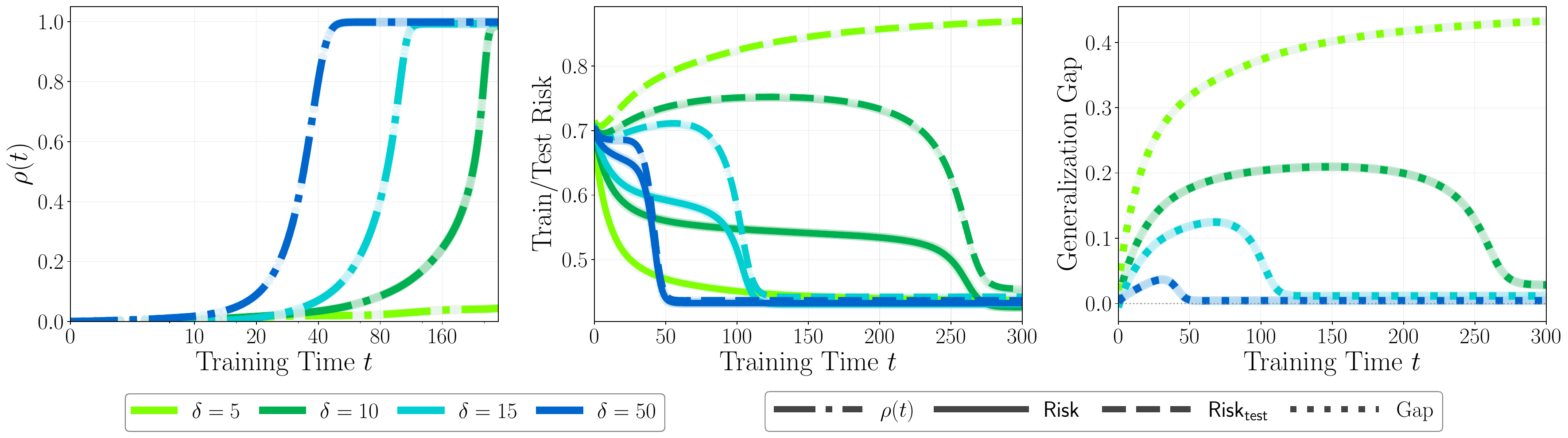}
  \caption{Learning dynamics of a $\GeLU$ neuron for varying sample ratios $\delta = n/d \in \{5, 10, 15, 50\}$, with $d=5000$, $\eta=0.5$, and noiseless phase retrieval target. Left: Correlation $\rho(t)$ between learned and target parameters (log time scale is used for better visualization). Center: Training risk (solid) and test risk (dashed). Right: Generalization gap $\mathsf{Risk}_{\mathsf{test}} - \mathsf{Risk}$. Grokking is observed when $\delta$ is moderately above the threshold: training risk first drops while test risk remains high before eventually decreasing. The generalization gap peaks during the overfitting phase and vanishes once feature learning completes.}
  \label{fig:grokking_multi_delta}
\end{figure}

\section{Outline of the proofs}\label{sec:proof_sketch}

This section outlines the proofs of  Lemma~\ref{lemma:bulk} and Theorem~\ref{theorem:main_result}.

We organize the arguments in three steps: Step 1: Gaussian conditioning; 
Step 2: Detecting outlier eigenvalues via resolvent expansion and concentration; Step 3: Verifying that eigenvectors corresponding to specific outliers align with hard directions. Full proofs are given in Appendices~\ref{sec_Proofs} and~\ref{sec:proof_result_iii}.

\subsection{Step 1: Gaussian conditioning}

The starting point is the $j$-th Hessian diagonal block of Eq.~\eqref{eq:hessian_block}, which we rewrite:
\begin{equation}\label{eq:outline_Hj_def}
\bH_j(t) = \frac{1}{n}\sum_{i=1}^ng(\bTheta(t)^\sT\bx_i,y_i;j)
 \bx_i \bx_i^\sT
= \bX^\sT \bG_j(t) \bX,
\end{equation}
where $\bG_j(t) = (1/n)\operatorname{diag}\big(g(\bTheta(t)^\sT\bx_1,y_1;j),\dots,g(\bTheta(t)^\sT\bx_n,y_n;j)\big)$ and we recall that
\begin{alignat}{2}
g(z,y;j) &:= \ell'\!\left(y,  \frac{1}{m}\sum_{i=1}^m a_i \sigma(z_i+b_i)\right)\sigma''\!\big(z_j\big) &\qquad& \text{for $m>1$,}\label{eq:Gdef_1}\\
g(z,y) &:= \ell''\!\big(y,  \sigma(z)\big)\big(\sigma'(z)\big)^2+\ell'\!\big(y,\sigma(z)\big)\sigma''(z) && \text{for $m=1$.}
\label{eq:Gdef_2}
\end{alignat}

\paragraph{Isolating the dependence on the GD trajectory.}

Although $\bH_j(t)$ depends on $\bX$ in a complicated nonlinear way (through  $\{\bTheta(t)^\sT\bx_i\}_{i\le n}$ and 
$\{y_i\}_{i\le n}$, where $\bTheta(t)$ itself depends on $\bX,\by$ via the GD update rule),
conditioning on a certain GD trajectory  is equivalent to condition on a certain number of 
\emph{linear} functions of $\bX$. Since $\bX$ is Gaussian, its conditional distribution 
given a set of linear observations can be written explicitly. In the present 
case $\bX$ decomposes into the sum of a low-rank component, which is measurable on the GD trajectory,
and a high-rank component which is (aside from a projection step) independent from the GD trajectory.
This decomposition is the foundation for both outlier detection (Step~2) and eigenvector analysis (Step~3).

\paragraph{Gaussian conditioning decomposition.} 

Define the parameter and gradient matrices
\begin{equation}\label{eq:outline_WQ_def}
\begin{aligned}
\bTheta(*;t)
&:= \big[\bTheta_*, \bTheta(0), \dots, \bTheta(t)\big] 
\in \reals^{d\times (k+m(t+1))},\\
\bF(0;t) 
&:= \frac{1}{\sqrt{n}}\big[\bF(\bX \bTheta(0), \bX \bTheta_*, \boldsymbol{\varepsilon}), \dots, \bF(\bX \bTheta(t-1), \bX \bTheta_*, \boldsymbol{\varepsilon})\big]
\in \reals^{n\times mt}\, .
\end{aligned}
\end{equation}
By the gradient descent update rule, we have $\bX^\sT \bF(0;t) \in \operatorname{span}(\bTheta(*;t))$. 
Let $\cF_t$ be the $\sigma$-algebra generated $\bTheta(*;t)$, $\bF(0;t)$.
Conditional on $\cF_t$, we have  the distributional equivalence\footnote{For two random variables $\bX$, $\bX'$
we write $\bX\big|_{\cF_t}\ed \bX'$ if $\E[\varphi(\bX) Z]= \E[\varphi(\bX') Z]$ for any bounded measurable function $\varphi$
and any bounded random variable $Z$ measurable on $\cF_t$.}
\begin{equation}\label{eq:outline_gauss_decomp}
\bX\big|_{\cF_t}
\ed
\bP_{\bF(0;t)}^{\perp} \bX_{\mathrm{new}} \bP_{\bTheta(*;t)}^{\perp}
+ \bX \bP_{\bTheta(*;t)},
\end{equation}
where $\bX_{\mathrm{new}} \in \reals^{n\times d}$ has i.i.d.\ $\normal(0,1)$ entries independent of all other randomness, and $\bP_{\bM}$ denotes orthogonal projection onto $\operatorname{span}(\bM)$. One of the crucial observations is that $\bP_{\bF(0;t)}\bX\bP_{\bTheta(*;t)}^{\perp} = \boldsymbol{0}$ due to the gradient descent structure. (Details in Appendix~\ref{sec_Proofs}, Lemma~\ref{lem:gauss_cond}.)

Substituting~\eqref{eq:outline_gauss_decomp} into~\eqref{eq:outline_Hj_def} and expanding yields
(using the shorthands $\bY(t) = \bP_{\bF(0;t)}^{\perp}\bX_{\text{new}}\bP_{\bTheta(*;t)}^{\perp}$, and $\bP_t:=\bP_{\bTheta(*;t)}$)
\begin{align}\nonumber
\bH_j(t)
 &\ed
\bY(t)^\sT \bG_j(t) \bY(t)+\bY(t)^\sT \bG_j(t) \bX \bP_{t}+ \bP_{t}\bX^\sT \bG_j(t) \bY(t)+
\bP_{t}\bX^\sT \bG_j(t)  \bX \bP_{t}\\
&= \bX_{\text{new}}^\sT \bG_j(t) \bX_{\text{new}} + \text{low-rank component},\label{eq:outline_Hj_decomp}
\end{align}
where the low-rank component has rank dependent uniquely on the number of GD steps $t$, the number of neurons $m$ and the number of target directions $k$, but not on $n,d$.

\paragraph{Bulk spectrum.} By Cauchy interlacing, the low-rank perturbation in~\eqref{eq:outline_Hj_decomp} do not affect the limiting spectral distribution. Since $\bX_{\text{new}}$ has i.i.d.\ Gaussian entries independent of $\bG_j(t)$, the asymptotic
empirical spectral distribution of $\bX_{\text{new}}^\sT \bG_j(t) \bX_{\text{new}}$ follows a general Marchenko--Pastur law. This asymptotic distribution depends on the aspect ratio $\delta =\lim_{n,d\to\infty} n/d$ and the asymptotic distribution of $\bG_j(t)$ (given by discrete-time DMFT in Lemma~\ref{lemma_amp_data}). This establishes Lemma~\ref{lemma:bulk}.

We refer to Appendix \ref{sec:appendix_bulk} for details omitted in this outline.

\subsection{Step 2: Outlier matrix and concentration to deterministic limit}

We now locate outlier eigenvalues, namely eigenvalues that do not converge to the support of
the asymptotic empirical spectral distribution $\mu^j_{\infty}(t)$ (which, as we saw, follows a general Marchenko-Pastur law)
that from the bulk. Applying the Woodbury matrix identity to decomposition~\eqref{eq:outline_Hj_decomp}, we express the resolvent as
\begin{equation}\label{eq:outline_resolvent}
\begin{aligned}
(\bH_j(t) - z \bI)^{-1}
&= \bR_0(z)
- \bR_0(z)
\begin{bmatrix} \bTheta(*;t) & \bY(t)^\sT \bG_j(t)\bZ(t) \end{bmatrix}
 \bM_j(z;t)^{-1}\\
&\quad\quad\quad\quad\times
\begin{bmatrix} \bTheta(*;t)^\sT \\ \bZ(t)^\sT \bG_j(t)\bY(t) \end{bmatrix}
\bR_0(z),
\end{aligned}
\end{equation}
where $\bR_0(z) := \big(\bY(t)^\sT \bG_j(t)\bY(t) - z\bI\big)^{-1}$ is the bulk resolvent, $\bZ(t) := \bX \bTheta(*;t)\big(\bTheta(*;t)^\sT \bTheta(*;t)\big)^{-1}$, and the matrix
$\bM_j(z;t)\in\mathbb{R}^{2(k+m(t+1))\times 2(k+m(t+1))}$ has the explicit form
\begin{equation}\label{eq:outline_outlier_matrix}
\bM_j(z;t)
:= \begin{bmatrix}
\bZ(t)^\sT \bG_j(t)\bZ(t) & \bI \\
\bI & \boldsymbol{0}
\end{bmatrix}^{-1}
+ \begin{bmatrix}
\bTheta(*;t)^\sT \\
\bZ(t)^\sT \bG_j(t)\bY(t)
\end{bmatrix}
\bR_0(z)
\begin{bmatrix}
\bTheta(*;t) & \bY(t)^\sT \bG_j(t)\bZ(t)
\end{bmatrix}.
\end{equation}
We refer to $\bM_j(z;t)$ (and its $n,d\to\infty$ limit defined below) as to the \emph{`outlier matrix.'}

Since all eigenvalues of $\bY(t)^\sT \bG_j(t)\bY(t)$ converge to the support of the general Marchenko-Pastur law, 
it follows from~\eqref{eq:outline_resolvent} that $z$ is an outlier eigenvalue of $\bH_j(t)$ 
only if $\det(\bM_j(z;t)) = 0$. Crucially, $\bM_j(z;t)$ is low-dimensional.
Namely, it has dimension $2(k+m+mt) \times 2(k+m+mt)$, independent of the ambient dimensions $n,d$. 
As a consequence, it concentrates (for $z$ in a suitable domain) around a deterministic matrix, 
and it is therefore sufficient to study the latter.  

\paragraph{Deterministic limit of the outlier matrix.} We next describe the deterministic limit of
$\bM_j(z;t)$. Without loss of generality, we consider the case where the hard directions are the first $r$ target directions.

\begin{lemma}[Restatement of Lemma~\ref{lemma:limiting_spike_matrix}]\label{lemma:outline_outlier_limit}
For $z<\min(0,c_j(t))$, under the assumptions of Theorem \ref{theorem:main_result}, the outlier matrix~\eqref{eq:outline_outlier_matrix} converges in probability to a deterministic limit
\begin{equation}\label{eq:outline_Minfty}
\plim_{n,d\to\infty}\bM_j(z;t) = M_j^\infty(z;t)
= \begin{bmatrix}
-\frac{1}{z} \E[O_t O_t^\sT] & I \\[4pt]
I & \left(\E[O_t O_t^\sT]\right)^{-1} S_j(z;t) \left(\mathbb{E}[O_t O_t^\sT]\right)^{-1}
\end{bmatrix},
\end{equation}
where $O_t^\sT := [\Theta_*^\sT, \Theta(0)^\sT, \dots, \Theta(t)^\sT]$ collects the state-evolution variables. 

Furthermore, denoting by $V_{*\leq}$ the first $r$ coordinates of $V_*$, the matrix $S_j(z;t)\in\mathbb{R}^{(k+m+mt)\times(k+m+mt)}$ is \emph{block-diagonal}:
\begin{equation}\label{eq:outline_S_block}
S_j(z;t)
= \begin{bmatrix}
- \E\!\left[\frac{\delta G_t^j}{\delta + G_t^j \alpha_t^j(z)}V_{*\leq}V_{*\leq}^\sT\right] & 0 \\[6pt]
0 & S_{j,\mathrm{R}}(z;t)
\end{bmatrix}.
\end{equation}
\end{lemma}
Lemma~\ref{lemma:outline_outlier_limit} reduces the problem of determining the 
outlier eigenvalues of the high-dimensional random matrix $\bH_j(t)\in\reals^{d\times d}$
to analyzing the deterministic matrix $M_j^\infty(z;t)$, whose dimension is independent of $n,d$. 
The proof proceeds in two stages. \emph{First,} we condition on the GD trajectory,
i.e. on the $\sigma$-algebra $\cF_t$. Using the decomposition \eqref{eq:outline_gauss_decomp}, 
we can think as $(\bX,\beps)$ fixed and treat the independent Gaussian matrix $\bX_{\mathrm{new}}$ from~\eqref{eq:outline_gauss_decomp} as the 
only source of randomness. We apply leave-one-out expansions and Hanson--Wright inequalities to show that 
each entry concentrates to a conditional limit involving the Stieltjes transform of the bulk 
distribution. \emph{Second,} we integrate over $(\bX,\boldsymbol\eps)$.
Namely, we use discrete-time DMFT  (Lemmas~\ref{lemma_amp_data} and~\ref{lemma_amp_feature}), replacing to show that quantities measurable over $\cF_t$ arising from the first step concentrate
around expectations over the state-evolution variables $(V_*,V(0),\dots,V(t),\varepsilon)$ and $(\Theta_*,\Theta(0),\dots,\Theta(t))$. (Full proof in Appendix~\ref{sec:proof_result_iii}, 
Sections~\ref{subsec:concentration}, \ref{subsec:convergence}.) 

\paragraph{Block-diagonal structure and the hard block equation.} The block-diagonal form~\eqref{eq:outline_S_block} turns out to be a consequence of the fact that $\Uhard$ is
the (maximal) hard subspace property and hence, by Lemma~\ref{lemma_amp_data}, hard 
directions $V_{*\leq}$ are uncorrelated with all quantities observable from training. This decouples the hard block (upper-left $r\times r$) from the rest. 
Lemma~\ref{lemma_amp_feature} establishes that $\E[O_tO_t^\sT]$ as a similar block-diagonal
structure. 

Using this block-diagonal structure, and applying the Schur complement formula to~\eqref{eq:outline_Minfty}, 
the equation $\det(M_j^\infty(z;t))=0$ factors as
\begin{equation}\label{eq:outline_det_factor}
\det\big(M_{j,\mathrm{H}}^\infty(z;t)\big)\cdot
\det\big(M_{j,\mathrm{R}}^\infty(z;t)\big)=0,
\end{equation}
for suitably defined matrices $M_{j,\mathrm{R}}^\infty(z;t)$, $M_{j,\mathrm{H}}^\infty(z;t)$. 
In particular, the equation $\det\big(M_{j,\mathrm{H}}^\infty(z;t)\big)$
reads
\begin{equation}\label{eq:outline_hard_eq}
\det\!\left(-z I_r + \mathbb{E}\!\left[\frac{\delta G_t^j}{\delta + G_t^j \alpha_t^j(z)}\, V_{*\leq}V_{*\leq}^\sT\right]\right) = 0.
\end{equation}
We will refer to this as to the `hard block equation'.
Note that this is precisely the  outlier equation~\eqref{equa_main} in Theorem~\ref{theorem:main_result}. 

Finally notice that outliers can converge not only to 
solutions of the hard block equation but also to
other solutions of Eq.~\eqref{eq:outline_det_factor},
namely to solutions of $\det\big(M_{j,\mathrm{R}}^\infty(z;t)\big)=0$.
However, in Step~3 we show that only solutions to~\eqref{eq:outline_hard_eq} yield eigenvectors aligned with hard directions. 

Finally, we use Rouch\'e's theorem  from complex analysis (Appendix~\ref{section:eigenvalues}) to show that the neighborhood
of solutions of $\det(M_j^\infty(z;t))$ must contain actual 
eigenvalues of the Hessian $\bH_j(t)$ (with matching multiplicity) converging to those solutions.

\subsection{Step 3: Eigenvector alignment via residue calculation}

Finally, we determine which outlier eigenvectors align with the  hard subspace $\bTheta_{*\mathrm{H}} = [\btheta_1^*,\dots,\btheta_r^*]$. For simplicity,
in this outline we assume that the zeros of $\det(M_j^\infty(z;t))=0$ are simple (multiplicity one). The general case follows analogously.

Let $z_*$ be a simple zero of $\det(M_j^\infty(z;t))$. By Rouch\'e's theorem (Appendix~\ref{section:eigenvalues}), there exists a unique eigenvalue $z_n$ of $\bH_j(t)$ satisfying $|z_n-z_*|\xrightarrow{~p~} 0$ with corresponding unit eigenvector $\boldsymbol{\xi}_n$. To access the eigenvector, we use the spectral projector
\begin{equation}\label{eq:outline_correlation}
\bP_*^{n,d} := -\frac{1}{2\pi i}\oint_\gamma (\bH_j(t)-z\bI)^{-1}\, \de z= \boldsymbol{\xi}_n\boldsymbol{\xi}_n^\sT,
\end{equation}
where $\gamma$ is a contour around $z_*$ enclosing 
no other zero of $\det(M_j^\infty(z;t))$ and no part of the support of the
general Marchenko-Pastur law $\mu_j^{\infty}(t)$. The squared correlation with hard directions is
\begin{equation}\label{eq:outline_correlation_trace}
\|\bTheta_{*\mathrm{H}}^{\sT}\boldsymbol{\xi}_n\|^2
= \operatorname{Tr}\!\left(\bTheta_{*\mathrm{H}}^{\sT}\boldsymbol{\xi}_n\boldsymbol{\xi}_n^\sT\bTheta_{*\mathrm{H}}\right)
= \operatorname{Tr}\!\left(\bTheta_{*\mathrm{H}}^{\sT}\bP_*^{n,d}\bTheta_{*\mathrm{H}}\right).
\end{equation}
The contour integral formulation converts the eigenvector alignment question into a residue calculation that can be analyzed using the limiting outlier matrix.

\paragraph{Residue formula and block-diagonal structure.} Substituting the resolvent expansion~\eqref{eq:outline_resolvent} into~\eqref{eq:outline_correlation}, we observe that the bulk resolvent $\bR_0(z)$ is analytic inside $\gamma$ and contributes no residue. The only pole arises from $\bM_j(z;t)^{-1}$, reducing the contour integral to
\begin{equation}\label{eq:outline_residue_integral}
\|\bTheta_{*\mathrm{H}}^{\sT}\boldsymbol{\xi}_n\|^2 = \frac{1}{2\pi i}\oint_\gamma \operatorname{Tr}\!\left(\bTheta_{*\mathrm{H}}^{\sT}\bR_0(z)
\begin{bmatrix}\bTheta(*;t) & \bY^\sT\bG\bZ\end{bmatrix}
\bM_j(z;t)^{-1}
\begin{bmatrix}\bTheta(*;t)^\sT \\ \bZ^\sT\bG\bY\end{bmatrix}
\bR_0(z)\bTheta_{*\mathrm{H}}\right)\de z.
\end{equation}
By Lemma~\ref{lemma:outline_outlier_limit}, we replace $\bM_j(z;t)^{-1}$ by its limit $M_j^\infty(z;t)^{-1}$. 

The matrix $\bTheta_{*\mathrm{H}}^{\sT}\bR_0(z)[\bTheta(*;t), \bY^\sT\bG\bZ]$ 
has constant dimension and concentrates around a deterministic limit 
as $n,d\to\infty$. First, by using $\bY\bTheta(*;t)= 0$  and Lemma~\ref{lemma_amp_feature},
we get
\begin{equation}\label{eq:outline_R0_theta}
\bTheta_{*\mathrm{H}}^{\sT} \bR_0(z) \bTheta(*;t) = -\frac{1}{z}\bTheta_{*\mathrm{H}}^{\sT}\bTheta(*;t) \xrightarrow{~p~} -\frac{1}{z}\Big[\; I_r  \; \Big\vert \; 0\; \Big] \in \reals^{r\times (k+m+mt)}.
\end{equation}
Second, using again $\bY\bTheta(*;t)=\bzero$, we get
$
\bTheta_{*\mathrm{H}}^{\sT} \bR_0(z) \bY^\sT\bG\bZ = 0$.
Summarizing, we proved
\begin{equation}\label{eq:outline_R0_structure}
\bTheta_{*\mathrm{H}}^{\sT}\bR_0(z)[\bTheta(*;t), \bY^\sT\bG\bZ]
\xrightarrow{~p~}
A_\infty(z)^{\sT}:= -\frac{1}{z}\Big[\; I_r  \; \Big\vert \; 0\; \Big]
\in \mathbb{R}^{r\times 2(k+m+mt) }.
\end{equation}

Consequently, using Lemma \ref{lemma:outline_outlier_limit},  the integrand in Eq.~\eqref{eq:outline_residue_integral} converges to
\begin{equation}\label{eq:outline_F_infty}
F_\infty(z) := A_\infty(z)^\sT M_j^\infty(z;t)^{-1} A_\infty(z).
\end{equation}
The only pole of $F_{\infty}(x)$ inside $\gamma$ arises because of the factor $M_j^\infty(z;t)^{-1}$ at
$z = z_*$. Let
\begin{equation}\label{eq:outline_kernel_vector}
v = \begin{bmatrix} v_{\leq} \\ v_{>} \end{bmatrix} \in \ker(M_j^\infty(z_*;t)),
\quad\text{where }v_{\leq} \in \mathbb{R}^{k+m+mt}\text{ is the upper block.}
\end{equation}
Since $z_*$ is a simple zero of $\det(M_j^\infty(z;t))$, the inverse $M_j^\infty(z;t)^{-1}$ has a simple pole at $z_*$ with residue
\begin{equation}\label{eq:outline_pole_residue}
\frac{v v^\sT}{v^\sT \partial_z M_j^{\infty}(z_*;t) v}.
\end{equation}
By the residue theorem, the contour integral equals
\begin{equation}\label{eq:outline_contour_equals}
\frac{1}{2\pi i}\oint_{\gamma}F_\infty(z)\,\de z
= \operatorname{Tr}\left(A_\infty(z_*)^\sT\,\frac{v v^\sT}{v^\sT\partial_z M_j^{\infty}(z_*;t) v}\,A_\infty(z_*)\right)
= 
 \frac{1}{(z_*)^2}\cdot\frac{\|v_{\mathrm{H}}\|^2}{v^\sT \partial_z M_j^{\infty}(z_*;t) v}.
\end{equation}
where $v_{\mathrm{H}}$ is the vector comprising the first $r$ coordinates of $v_{\le}$.

The derivative $\partial_zM_j^{\infty}(z_*;t)$ inherits a block structure from~\eqref{eq:outline_Minfty}, and direct calculation yields
\begin{equation}\label{eq:outline_derivative_identity}
v^\sT \partial_z M_j^{\infty}(z_*;t) v = \frac{1}{(z_*)^2} v_{\leq}^\sT\left(\mathbb{E}[O_tO_t^\sT] + \partial_zS(z_*;t)\right) v_{\leq}.
\end{equation}
Substituting into~\eqref{eq:outline_contour_equals}, the $(z_*)^2$ factors cancel, yielding the following compact formula for eigenvector correlation:
\begin{equation}\label{eq:outline_final_formula} 
\plim_{n,d\to\infty}\|\bTheta_{*\mathrm{H}}^{\sT}\boldsymbol{\xi}_n\|^2  =\frac{\|v_{\mathrm{H}}\|^2}{v_{\leq}^\sT\left(\mathbb{E}[O_tO_t^\sT] + \partial_zS(z_*;t)\right)v_{\leq}}.
\end{equation}

By the factorization formula~\eqref{eq:outline_det_factor},
$z_*$ can be either a zero of $\det(M_{j,\mathrm{R}}^\infty(z;t))$
or a zero of $\det(M_{j,\mathrm{H}}^\infty(z;t))$, 
In the first case, we will show that $v_{\mathrm{H}} = 0$, 
whence the expression in Eq.~\eqref{eq:outline_final_formula} vanishes: such eigenvectors are asymptotically
orthogonal to $\bTheta_{*\mathrm{H}}$. 

On the other hand, if $z_*$ is a zero of $\det(M_{j,\mathrm{H}}^\infty(z;t))$, 
 then $v_{\leq}^\sT=\begin{bmatrix}v_{\mathrm{H}}^\sT & 0^\sT\end{bmatrix}$, and the block-diagonal structures of $\mathbb{E}[O_tO_t^\sT]$ and $S'(z_*;t)$ simplify~\eqref{eq:outline_final_formula}  to an expression involving only the hard block, yielding strictly positive correlation. This completes the eigenvector alignment analysis (full details in Appendix~\ref{section:eigenvectors}).

 \section{Discussions and conclusions}\label{sec:discussion}

 We studied feature learning in $k$-index models.
 For each $t\ge 0$, we determine a  threshold $\delta_{\sNN}^*(t)$ in the number of samples per dimension, at which the spectrum of the Hessian of the empirical risk exhibits a qualitative change. For $n/d$ above $\delta_{\sNN}^*(t)$  the Hessian  after $t$ GD steps develops negative eigendirections which are correlated with `hard' direction  in the latent space. We expect the limit
 $\delta_{\sNN}^*(\infty) = \lim_{t\to\infty}\delta_{\sNN}^*(t)$ 
to be the threshold for feature learning under GD. 

Within the conjectured scenario, feature learning under GD is essentially equivalent to feature learning
in a spectral algorithm that computes the low-lying eigenvectors of a matrix of the form $\bX^{\sT}\bD\bX$.
The diagonal matrix $\bD$ depends on the response variables $y_i$ through a preprocessing 
function implicitly determined by GD. Since this preprocessing is constrained by the network architecture rather than optimized for detection, the threshold $\delta_{\sNN}$ is generally suboptimal compared to the best possible spectral algorithms.

This scenario also provides a quantitative explanation of grokking: for $\delta>\delta_{\sNN}$, the network first overfits to the training data, and subsequently learn the `hard' features from the negative curvature of the Hessian, leading to a delayed drop in generalization error.

Several problems remain open:
\begin{enumerate}
\item[$(i)$] Connect the spectral phase transition at $\delta_{\sNN}(\infty)$ with the GD dynamics. 
We expect that for $1\ll t\lesssim \log d$, GD dynamics will be closely related to the dynamics in a saddle with 
the spectral properties that we characterize. Making this rigorous requires studying the dynamics beyond $t=O(1)$.
\item[$(ii)$] Analysis for number of neurons $m>1$, $m=O(1)$. In this regime the reduction to the block diagonal 
form provided  by Lemma \ref{lemma:eigenvalue_diagonal} no longer holds and hence the analysis is more complicated.
\item[$(iii)$] Non-Gaussian and non-isotropic covariates. In particular, it would be important to understand
the interplay between the low-dimensional structure in the function $h$ and low-dimensional structures in the covariates distribution.
\end{enumerate}
 
\subsection*{Acknowledgments}

The authors thank Chen Cheng, Jason D. Lee and Basil Saeed for insightful discussions.
The authors also thank Fan Nie and Jiatong Yu for the help with numerical simulations.

This work was supported by the NSF through Award DMS-2031883, the Simons Foundation through Award 814639 for the Collaboration on the Theoretical Foundations of Deep Learning, and the NSF Award MFAI-2501597.

\newcommand{\etalchar}[1]{$^{#1}$}
\providecommand{\bysame}{\leavevmode\hbox to3em{\hrulefill}\thinspace}
\providecommand{\MR}{\relax\ifhmode\unskip\space\fi MR }
\providecommand{\MRhref}[2]{%
  \href{http://www.ams.org/mathscinet-getitem?mr=#1}{#2}
}
\providecommand{\href}[2]{#2}

\clearpage
\appendix

\section{Technical tools}\label{appendix:tools}
We begin with some technical tools that will be used in the proof of the main results. Most of these 
are standard or variations of standard results. 

\subsection{Rouch\'e's theorem for random functions}\label{subsec:rouche_theorem}
\begin{lemma}\label{lemma:rouche_zeros}
    Let $D \subset \mathbb{C}$ be a bounded connected open domain. For each $n \ge 1$, let $f_n: D \to \mathbb{C}$ be a random analytic function, and let $f: D \to \mathbb{C}$ be a nonzero deterministic analytic function. Assume for every compact $\cK \subset D$,
    \begin{equation}\label{eq:eig_fn_uniform}
    \sup_{z \in \cK}|f_n(z) - f(z)| \xrightarrow{~p~} 0.
    \end{equation}
    
    Fix $z^* \in D$ and suppose $f$ has a zero of order $k \ge 1$ at $z^*$; that is, $f(z) = (z - z^*)^k g(z)$ near $z^*$ with $g$ analytic and $g(z^*) \neq 0$.
    
    Then there exists $r_0 > 0$ such that for every $0 < r \le r_0$, the number of zeros $N_n(r)$ of $f_n$ in $\mathsf{D}(z^*, r) := \{z \in \mathbb{C}:\ |z - z^*| < r\}$ (counted with multiplicity) satisfies
    \begin{equation}\label{eq:eig_Nn_convergence}
    \mathbb{P}(N_n(r) = k) \to 1.
    \end{equation}
    Furthermore, there exist random zeros $z_{n,1}, \dots, z_{n,k} \in \mathsf{D}(z^*, r)$ of $f_n$ such that
    \begin{equation}\label{eq:eig_zeros_convergence}
    \max_{1 \le \ell \le k}|z_{n,\ell} - z^*| \xrightarrow{~p~} 0.
    \end{equation}
    \end{lemma}
    \begin{proof}
        Since $f$ is analytic and nonvanishing in $D$, its zeros are isolated. Choose $r_0 > 0$ such that $\overline{\mathsf{D}}(z^*, r_0) \subset D$, $z^*$ is the only zero of $f$ in $\mathsf{D}(z^*, r_0)$, and $f$ has no zeros on the circle $\Gamma_{r_0} := \{z \in \mathbb{C}:\ |z - z^*| = r_0\}$.
        
        For $0 < r \le r_0$, define
        \begin{equation}\label{eq:eig_mr}
        m_r := \min_{z \in \Gamma_r}|f(z)| > 0,
        \quad\text{where }\Gamma_r := \{z \in \mathbb{C}:\ |z - z^*| = r\}.
        \end{equation}
        Define the event
        \begin{equation}\label{eq:eig_event_A}
        A_{n,r} := \left\{\sup_{z \in \Gamma_r}|f_n(z) - f(z)| < m_r/2\right\}.
        \end{equation}
        By the uniform convergence \eqref{eq:eig_fn_uniform}, $\mathbb{P}(A_{n,r}) \to 1$.
        
        On $A_{n,r}$, we have $|f_n(z) - f(z)| < m_r/2 \le |f(z)|$ for all $z \in \Gamma_r$. By Rouch\'e's theorem, $f_n$ and $f$ have the same number of zeros (counted with multiplicity) in $\mathsf{D}(z^*, r)$. Since $f$ has exactly $k$ zeros in $\mathsf{D}(z^*, r)$ (all at $z^*$), we conclude $N_n(r) = k$ on $A_{n,r}$, proving \eqref{eq:eig_Nn_convergence}.
        
        To prove \eqref{eq:eig_zeros_convergence}, fix $\varepsilon > 0$ and choose $0 < \rho < r \le r_0$ with $\rho < \varepsilon$. Define
        \begin{equation}\label{eq:eig_events_B}
        A_{n,\rho} := \left\{\sup_{z \in \Gamma_\rho}|f_n(z) - f(z)| < m_\rho/2\right\},
        \quad
        A_{n,r} := \left\{\sup_{z \in \Gamma_r}|f_n(z) - f(z)| < m_r/2\right\}.
        \end{equation}
        Since $\mathbb{P}(A_{n,\rho}) \to 1$ and $\mathbb{P}(A_{n,r}) \to 1$, we have $\mathbb{P}(A_{n,\rho} \cap A_{n,r}) \to 1$.
        
        On $A_{n,\rho} \cap A_{n,r}$, Rouch\'e's theorem applied to $\Gamma_\rho$ and $\Gamma_r$ shows that $f_n$ has exactly $k$ zeros in $\mathsf{D}(z^*, \rho)$ and exactly $k$ zeros in $\mathsf{D}(z^*, r)$, with no zeros on $\Gamma_\rho$ or $\Gamma_r$. Therefore, all $k$ zeros in $\mathsf{D}(z^*, r)$ lie in $\mathsf{D}(z^*, \rho)$. Denoting these zeros as $z_{n,1}, \dots, z_{n,k}$, we have
        \begin{equation}\label{eq:eig_zeros_bound}
        \max_{1 \le \ell \le k}|z_{n,\ell} - z^*| \le \rho < \varepsilon
        \quad\text{on }A_{n,\rho} \cap A_{n,r}.
        \end{equation}
        Since $\mathbb{P}(A_{n,\rho} \cap A_{n,r}) \to 1$ and $\varepsilon$ is arbitrary, \eqref{eq:eig_zeros_convergence} follows.
        \end{proof}
We also need the following straightforward consequence of uniform convergence.
    \begin{lemma}\label{lem:rouche_no_zeros}
    Let $D \subset \mathbb{C}$ be a bounded connected open set, and let $f_n: D \to \mathbb{C}$ be random analytic functions and $f: D \to \mathbb{C}$ a deterministic analytic function satisfying \eqref{eq:eig_fn_uniform}. Suppose $f(z) \neq 0$ for all $z \in D$.
    
    Then for every compact $\cK \subset D$, there exists $c_\cK > 0$ such that
    \begin{equation}\label{eq:eig_no_zeros_bound}
    \mathbb{P}\left(\inf_{z \in \cK}|f_n(z)| \ge c_\cK\right) \to 1.
    \end{equation}
    \end{lemma}

\subsection{From pointwise convergence to uniform convergence}\label{subsec:proof_change_limit_first}

\begin{lemma}\label{lemma:change_limit_first}
    Let $\Gamma \subset \mathbb{C}$ be a compact set. Let $f_n:\Gamma \to \mathbb{C}$ be a sequence of random functions and $f:\Gamma \to \mathbb{C}$ a deterministic function, both continuous on $\Gamma$. 
    
    Suppose $f_n(z) \xrightarrow{~p~} f(z)$ for each $z \in \Gamma$. Assume there exist random variables $L_n$ satisfying $L_n = O_P(1)$ uniformly in $n$ such that
    \begin{equation}\label{eq:lipschitz_condition_lemma}
    |f_n(z) - f_n(z')| \le L_n|z - z'|
    \quad\text{for all}\quad z, z' \in \Gamma,\quad\text{almost surely}.
    \end{equation}
    Then $\sup_{z \in \Gamma}|f_n(z) - f(z)| \xrightarrow{~p~} 0$ and consequently $\int_\Gamma \! f_n(z)\, \de z \xrightarrow{~p~} \int_\Gamma \!f(z)\, \de z$.
    \end{lemma}

\begin{proof}
This is a special case of general uniform convergence theory \cite[Chapter 1.5]{van1996weak}.
The family of functions $\{f_n: n\ge 1\}$ is uniformly equicontinuous by \eqref{eq:lipschitz_condition_lemma},
and for each $z \in \Gamma$, the family $\{f_n(z): n\ge 1\}$ is tight because it converges in probability. Hence $\{f_n: n\ge 1\}$ is tight in sup norm
by \cite[Theorem 1.5.7]{van1996weak} and converges weakly to $f$ in sup norm by \cite[Theorem 1.5.4]{van1996weak}.
Since the limit is deterministic, this is equivalent to $\sup_{z \in \Gamma}|f_n(z) - f(z)| \xrightarrow{~p~} 0$.
\end{proof}

\subsection{Asymptotics of minimum eigenvalues}\label{subsec:proof_smallest_eigenvalue}

 The following gives asymptotics of the minimum eigenvalue for a certain class of random matrices.
\begin{lemma}\label{lemma:smallest_eigenvalue_concentration}
        Let $\bX \in \mathbb{R}^{n \times d}$ have i.i.d.\ $\normal(0,1)$ entries and $\bG = \operatorname{diag}(g_1, \ldots, g_n) \in \mathbb{R}^{n \times n}$ with uniformly bounded entries $|g_i| \le C$, independent of $\bX$. 
        Let $\delta_n:=n/d\to\delta\in(0,\infty)$. Assume the empirical distribution $\widehat{v}_g := \frac{1}{n}\sum_{i=1}^n \delta_{\cdot - g_i}$ converges weakly to a probability measure $v_g$ in probability. 
        Let $\mu_\star$ denote the limiting spectral distribution of $\bX^\sT \bG \bX / n$.
        
        Then, the minimal eigenvalue converges in probability to the left edge of the limiting spectrum:
        \begin{equation}\label{eq:spike_lambda_min_conv}
        \lambda_{\min}\left(\bX^\sT \bG \bX/n\right) \xrightarrow{~p~} \inf \operatorname{supp}(\mu_\star).
        \end{equation}
        \end{lemma}
We remark that this result is standard and expected; for instance, when $\bG$ is positive definite, it follows from \cite{knowles2016anisotropiclocallawsrandom}. We provide a brief proof here just for completeness, extending to the general case where $\bG$ may have negative entries.
    \begin{proof}[Proof of Lemma \ref{lemma:smallest_eigenvalue_concentration}]
    Denote
    \begin{equation}\label{eq:S_n_def}
    \bS_n := \bX^\sT\bG\bX / n \in \reals^{d\times d}.
    \end{equation}
    Let $a_{-}:=\inf\supp(\mu_\star)$ denote the left edge of the limiting spectral measure. The proof proceeds by establishing matching lower and upper bounds for $\lambda_{\min}(\bS_n)$.
    
    By standard random matrix theory, the empirical spectral distribution of $\bS_n$ converges in probability to a deterministic probability measure $\mu_\star$ on $\reals$, whose Stieltjes transform $\alpha: \mathbb{H}\to\mathbb{H}$ is characterized by
    \begin{equation}\label{eq:stieltjes_equation}
    z +\frac{1}{\alpha} =  \delta\cdot \E\!\left[\frac{G}{\delta+\alpha G}\right],
    \end{equation}
    where $G\sim v_g$.

    \paragraph{Sign-decomposition.}
    To analyze the extreme eigenvalues, we decompose $\bS_n$ into positive and negative parts. Define
    \begin{equation}\label{eq:sign_decomposition}
    \bS_n=\bA_{+}-\bA_{-},
    \quad\text{where}\quad
    \begin{aligned}
    \bA_{+}&:=\frac{1}{n}\sum_{i:g_i>0} g_i\bx_i\bx_i^\sT,\\
    \bA_{-}&:=\frac{1}{n}\sum_{i:g_i<0} (-g_i)\bx_i\bx_i^\sT,
    \end{aligned}
    \end{equation}
    and $\bx_i\in\mathbb{R}^d$ denotes the $i$-th row of $\bX$. 
    
    By Gaussian rotational invariance, there exist orthogonal matrices $\bO,\bQ\in\mathrm{O}(d)$ such that $\bO^\sT\bA_{+}\bO=\bB_{+}$ and $\bQ^\sT\bA_{-}\bQ=\bB_{-}$, with $\bB_{\pm}$ diagonal. Therefore,
    \begin{equation}\label{eq:free_representation}
    \bS_n \stackrel{d}{=} \bB_{+}-\bU\bB_{-}\bU^\sT,
    \end{equation}
    where $\bU$ is Haar-distributed on $\mathrm{O}(d)$ and independent of the pair $(\bB_{+},\bB_{-})$.
    
    \paragraph{Strong convergence.}
    Conditioning on $\bG$, we apply the anisotropic local law from \cite{knowles2016anisotropiclocallawsrandom} to the matrices $\bA_{\pm}$. Therefore, for any polynomial $p$,
    \begin{equation}\label{eq:strong_convergence_A}
    (1/d)\operatorname{Tr}(p(\bA_{\pm})) \to \int p\,d\mu_{\pm}
    \quad\text{and}\quad
    \|p(\bA_{\pm})\|_{\mathrm{op}}\to \sup_{x\in\supp(\mu_{\pm})}|p(x)|
    \quad\text{in probability},
    \end{equation}
    where $\mu_{\pm}$ are the limiting spectral distributions of $\bA_{\pm}$. Since $\bB_{\pm}$ are diagonal matrices with the same spectra as $\bA_{\pm}$, the same strong convergence holds for $\bB_{\pm}$.

    By \cite{collins2013strongasymptoticfreenesshaar}, since $\bU$ is Haar-distributed, the pair $(\bB_{+},\bU\bB_{-}\bU^\sT)$ converges strongly to a freely independent pair with marginal distributions $\mu_+$ and $\mu_-$. Consequently, the empirical spectral distribution of $\bS_n$ converges in probability to the free additive convolution $\mu_{+}\boxplus(-\mu_{-})=\mu_\star$.
    
    \paragraph{Lower bound.}
    Suppose that along some subsequence we have $\lambda_{\min}(\bS_n)\rightarrow a_{-}-\varepsilon$ for some fixed $\varepsilon>0$ for contradiction. Set $t:=a_{-}-\varepsilon$ and consider the resolvent function $f_t(x):=1/(x-t)$. Since $t<a_{-}=\inf\supp(\mu_\star)$, the function $f_t$ is uniformly bounded on $\supp(\mu_\star)$. By polynomial approximation (via the Weierstrass theorem) on a compact set $K\supset\supp(\mu_\star)$ such that $t\notin K$, there exist polynomials $p_m(x)$ satisfying $\sup_{x\in K}|p_m(x)-f_t(x)|\to 0$. \eqref{eq:strong_convergence_A} then implies
    \begin{equation}\label{eq:resolvent_convergence}
    \|(\bS_n-t\bI)^{-1}\|_{\mathrm{op}}
    = \sup_{x\in\sigma(\bS_n)}\frac{1}{|x-t|}
    \xrightarrow{~p~} \sup_{x\in\supp(\mu_\star)}\frac{1}{|x-t|}
    = \frac{1}{\varepsilon}.
    \end{equation}
    However, if $\lambda_{\min}(\bS_n)\rightarrow t=a_{-}-\varepsilon$, then $\|(\bS_n-t\bI)^{-1}\|_{\mathrm{op}}\to\infty$, contradicting \eqref{eq:resolvent_convergence}. Hence $\liminf_{n\to\infty}\lambda_{\min}(\bS_n)\geq a_{-}$ in probability.
    
    \paragraph{Upper bound.}
    Fix $\varepsilon>0$ and set $t':=a_{-}+\varepsilon$. Choose a smooth cutoff function $\phi\in C^\infty(\reals;[0,1])$ satisfying $\phi\equiv 1$ on $(-\infty,t']$ and $\supp(\phi)\subset(-\infty,t'+\varepsilon/2]$. By polynomial approximation and strong convergence,
    \begin{equation}\label{eq:trace_cutoff}
(1/d)\operatorname{Tr}(\phi(\bS_n))
    \xrightarrow{~p~} \int\phi\,d\mu_\star
    \geq \mu_\star((-\infty,a_{-}+\varepsilon])
    > 0.
    \end{equation}
    Since $\phi(\bS_n)$ has nonnegative eigenvalues and $\Tr(\phi(\bS_n))>0$ with high probability for all sufficiently large $n$, at least one eigenvalue of $\bS_n$ must lie in $(-\infty,t']$. Hence $\limsup_{n\to\infty}\lambda_{\min}(\bS_n)\leq a_{-}$ in probability.
    
    Combining the lower and upper bounds yields \eqref{eq:spike_lambda_min_conv}.
    \end{proof}
    
\subsection{Concentration inequalities}

First we recall the following concentration bound for quadratic forms of Gaussian vectors.
\begin{lemma}[Hanson--Wright]\label{lemma:hanson_wright_statement}
    Let $\bx \sim \normal(\boldsymbol{0},\bI_d)$ and $\bM\in \mathbb{R}^{d\times d}$ be deterministic. Then for any $t>0$ we have
    \begin{equation}\label{eq:conc_hanson_wright}
    \mathbb{P}\!\left(\abs{\bx^\sT \bM\bx - \operatorname{Tr}(\bM)} \ge t\right)
    \le 2\exp\!\left(-c\min\!\left(\frac{t^2}{\|\bM\|_F^2},\frac{t}{\|\bM\|_{\mathrm{op}}}\right)\right)
    \end{equation}
    for a universal constant $c>0$.
\end{lemma}

Second, we restate the following result from Lemma 21 of \cite{asgari2025localminimaempiricalrisk} for reference convenience.
Consider $\bX\in\mathbb{R}^{n\times d}$ with i.i.d.\ $\normal(0,1)$ entries and $\bG = \operatorname{diag}(g_1, \ldots, g_n) \in \mathbb{R}^{n \times n}$ with uniformly bounded entries $|g_i| \le C$, sampled independently of $\bX$.
Let $\delta_n:=n/d\to\delta\in(0,\infty)$. Assume $g_i$ is i.i.d.\ with distribution $v_g$.
Let $\mu_\star$ denote the limiting spectral distribution of $\bX^\sT \bG \bX / n$ and $\alpha_\star$ denote the Stieltjes transform of $\mu_\star$.

\begin{lemma}\label{lemma:local_law_adapted}
Let $\bQ := (\bX^\sT\bG\bX/n - (z+i\eta_n)\bI)^{-1}$ with $\eta_n \in \mathbb{R}^{+}$ and $\Im(z)\ge 0$.

Define the error quantity
\begin{equation}\label{eq:error_quantity}
\operatorname{E}_n 
:= C'
\frac{1+(\abs{z}^2+\eta_n^2)^2}{\eta_n^4}
\left(
L\sqrt{\frac{\log d}{n}} + \frac{1}{n\eta_n}
\right),
\end{equation}
where $C'$ depends only on $C$ and $L$ is a tunable parameter. 

Assume $\delta$ satisfies
\begin{equation}\label{eq:technical_assumption}
10\bigl(C^2+\abs{z}^2+\eta_n^2\bigr)\operatorname{E}_n(1+\delta\operatorname{E}_n)/\eta_n^2
\le \frac{1}{2\delta}.
\end{equation}
Then, with probability at least $1 - d^{-cL}$, we have
\begin{equation}\label{eq:local_law_bound}
\abs{\alpha_\star(z+i\eta_n) - \frac{1}{d}\operatorname{Tr}(\bQ)}
 \le
C''
\frac{1+(\abs{z}^2+\eta_n^2)^2}{\eta_n^5}
\operatorname{E}_n,
\end{equation}
where $C''>0$ depends only on $C$.
\end{lemma}

\section{Proofs for Sections \ref{sec:MainResults_I} and \ref{sec:MainResults_II}}\label{sec_Proofs}

Throughout this section, we assume without loss of generality $U_\mathrm{H}^* =\begin{bmatrix}
    e_1 & \dots & e_r
\end{bmatrix}$ since the input data distribution and learning algorithm are rotationally invariant.

\subsection{Proofs of Lemma \ref{lemma_amp_data} and Lemma \ref{lemma_amp_feature}: Discrete time DMFT}\label{sec_proof_amp}

The DMFT equations \eqref{eq:data_convergence_main} and \eqref{eq:feature_convergence}
follow immediately from  \cite{celentano2021high} (see, e.g. Lemma 6.1).

In order to prove Eqs.~\eqref{eq:hard_orthogonality} and \eqref{eq:hard_orthogonality_B}, we 
rewrite the DMFT equations \eqref{eq:dmft_system} as 
\begin{align}
V(t) &= - \frac{1}{\delta} \sum_{s=0}^{t-1} R_\theta(t, s)
\hhF(V(s),h(W_*,\eps)) +W(t),\phantom{AAAA}  W \sim \GP(0,C_\theta)\, ,&\label{eq:dmft_system_rewritten_A}\\
\Theta(t+1) &= \Theta(t) - \eta \sum_{s=0}^{t} R_\ell(t, s) \Theta(s)-\eta R_{\ell}(t,*) \Theta_*+ \eta Q(t) ,\quad Q\sim 
\GP(0,C_\ell/\delta)\,.&\label{eq:dmft_system_rewritten_B}
\end{align}
where we leverage the fact that the dependence of $F(\bTheta^{\sT}\bx_i,\bTheta^{\sT}_*\bx_i,\eps_i)$ on
$\bTheta^{\sT}_*\bx_i,\eps_i$ is through $y_i=h(\bTheta^{\sT}_*\bx_i,\eps_i)$, namely $F(\bTheta^{\sT}\bx_i,\bTheta^{\sT}_*\bx_i,\eps_i) =\hhF(\bTheta^{\sT}\bx_i,h(\bTheta^{\sT}_*\bx_i,\eps_i))$ for some $\hhF$.
 We also write $W_{*\rH}= W_*U_{\rH}$, $\Theta_{*\rH}= \Theta_*U_{\rH}$ and similarly for
  $W_{*\rE}= W_*U_{\rE}$, $\Theta_{*\rE}= \Theta_*U_{\rE}$.
  
We are going to prove the following claims via induction over $t$:
\begin{itemize}
\item[{\sf C1}$(t)$:]   $\E\big[\psi\big(W(0),\dots,W(t-1),\eps,h(W_*,\eps),W_{*\rE}\big) W_{*\rH}\big]=0$ for 
every bounded measurable $\psi$.
\item[{\sf C2}$(t)$:]    $R_{\ell}(s,*)U^*_{\mathrm{H}} = 0$, $\forall s\le t-1$.
\item[{\sf C3}$(t)$:] $\big(\Theta(0),\dots,\Theta(t)\big)\indep  \Theta_{*\mathrm{H}}$.
\item[{\sf C4}$(t)$:] $C_{\theta}(s,*)U^*_{\mathrm{H}} = 0$, $\forall s\le t$.
\end{itemize}
Note that {\sf C1}$(t+1)$ implies that \begin{equation}\E\big[\psi\big(V(0),\dots,V(t),\eps,h(W_*,\eps),W_{*\rE}\big) W_{*\rH}\big]=0\end{equation} for 
every bounded measurable $\psi$ which is equivalent to Eq.~\eqref{eq:hard_orthogonality}
(general $\psi$ can be obtained by taking monotone limits), and {\sf C3}$(t)$ is equivalent to
Eq.~\eqref{eq:hard_orthogonality_B} (by Gaussianity of $(\Theta_*,(\Theta(t):t\ge 0))$).

The claims are obviously true for $t=0$ (from the definition of $U_{\rH}^*$). We assume that it is true up to time $t$
and prove it up to time $t+1$. 

\vspace{0.2cm}

\noindent{\sf C1}$(t+1)$. Note that from the induction hypothesis {\sf C4}$(t)$,
$\left(W(0),\dots,W(t)\right)$ is independent of $W_{*\rH}$, since $W(\cdot)$ and $W_{*\rH}$ are jointly Gaussian with $\E\left[W(t)W_{*}^{\sT}\right] = C_{\theta}(t,*)$.
Hence, with the shorthand $W^{\le t}:=\left(W(0),\dots,W(t)\right)$, we can write $W(t) = g^W_t(W^{\le t-1},W_{*\rE},\xi)$ for some standard Gaussian vector $\xi$ independent of everything else.
Therefore
\begin{equation}
\begin{aligned}
\E\big[&\psi\big(W^{\le t},\eps,h(W_*,\eps),W_{*\rE}\big) W_{*\rH}\big]\\
&\quad\quad= \E\big[\psi\big(W^{\le t-1},g^W_t(W^{\le t-1},W_{*\rE},\xi),\eps,h(W_*,\eps),W_{*\rE}\big) W_{*\rH}\big]\\
&\quad\quad= \E\big[\widetilde{\psi}\big(W^{\le t-1},\eps,h(W_*,\eps),W_{*\rE}\big) W_{*\rH}\big] = 0\,
\end{aligned}
\end{equation}
thus proving the induction claim.

\vspace{0.2cm}

\noindent{\sf C2}$(t+1)$. By Stein's lemma
\begin{equation}
\begin{aligned}
\E&\left[\hhF(V(t),h(W_*,\eps))W_{*\rH}^{\sT}\right]\\ 
&=\E\left[\frac{\partial \hhF(V(t),h(W_*,\eps)) }{\partial W_*}\right]\, \E\left[W_*W_{*\rH}^{\sT}\right]
+\sum_{s=0}^{t}
\E\left[\frac{\partial \hhF(V(t),h(W_*,\eps)) }{\partial W(s)}\right]\, \E\left[W(s)W_{*\rH}^{\sT}\right]\\
& \stackrel{(a)}{=} \E\left[\frac{\partial \hhF(V(t),h(W_*,\eps)) }{\partial W_*}\right]\, U_{*\rH}
+\sum_{s=0}^{t}
\E\left[\frac{\partial \hhF(V(t),h(W_*,\eps)) }{\partial W(s)}\right]\, C_{\theta}(s,*)U_{*\rH}\\
& \stackrel{(b)}{=} \E\left[\frac{\partial \hhF(V(t),h(W_*,\eps)) }{\partial W_*}\right]\, U_{*\rH}
\end{aligned}
\end{equation}
where $(a)$ follows from  Eq.~\eqref{eq:dmft_gauss_W} and $(b)$ from {\sf C4}$(t)$.
We thus get
\begin{equation}
R_{\ell}(t,*)U^*_{\mathrm{H}} = \E\Big[\hhF(V(t),h(W_*,\eps))W_{*\rH}^{\sT}\Big] = 0\, ,
\end{equation}
where the last step follows from {\sf C1}$(t+1)$.

\vspace{0.2cm}

\noindent{\sf C3}$(t+1)$. Since $(\Theta(t):t\ge 0), \Theta_*$ are jointly Gaussian, 
it is sufficient to prove $\E\big[\Theta(t+1)\Theta_{*\rH}^{\sT}\big]=0$. Multiplying Eq.~\eqref{eq:dmft_system_rewritten_B}
with $\Theta_{*\rH}^{\sT}$ and taking expectation yields
\begin{equation}
\begin{aligned}
\E\big[\Theta(t+1)\Theta_{*\rH}^{\sT}\big] &=  \E\big[\Theta(t)\Theta_{*\rH}^{\sT}\big]- \eta \sum_{s=0}^{t} R_\ell(t, s)
\E\big[\Theta(s)\Theta_{*\rH}^{\sT}\big]-\eta R_{\ell}(t,*) \E\big[\Theta_*\Theta_{*\rH}^{\sT}\big]\\
& \stackrel{(a)}{=} -\eta R_{\ell}(t,*) U_{*\rH}U_{*\rH}^{\sT} \stackrel{(b)}{=} 0\, ,
\end{aligned}
\end{equation}
where $(a)$ uses induction hypothesis {\sf C3}$(t)$ and $\Theta_*\sim\normal(0,I_k)$ and $(b)$ uses {\sf C2}$(t+1)$.

\vspace{0.2cm}

\noindent{\sf C4}$(t+1)$. This follows immediately from {\sf C3}$(t+1)$.
\qed

\subsection{Proof of Lemma \ref{lemma:eigenvalue_diagonal}}
Define the cross term
\begin{equation}\label{eq:cross_def}
\bH_{\text{cross}}(\bTheta):=\frac{m}{n}\sum_{i=1}^n \ell''(y_i,f_{\bTheta}(\bx_i))\nabla f_{\bTheta}(\bx_i) \nabla f_{\bTheta}(\bx_i)^\sT.
\end{equation}

For the lower bound, we recall that $\ell''\geq 0$, so $\bH_{\text{cross}}(\bTheta)$ is positive semi-definite. Since $m\nabla^2 \Risk(\bTheta) = \bH_{\text{diag}}(\bTheta)+\bH_{\text{cross}}(\bTheta)$, we obtain
\begin{equation}\label{eq:lower_bd}
\lambda_{\min}(m\nabla^2 \Risk(\bTheta)) \geq \lambda_{\min}(\bH_{\text{diag}}(\bTheta)).
\end{equation}

For the upper bound, let $\bH_{*}(j,j)\in\mathbb{R}^{d\times d}$ denote the $(j,j)$ block of $\bH_*$ for $* \in \{\text{cross}, \text{diag}\}$. We have
\begin{equation}\label{eq:upper_bd_chain}
\begin{aligned}
\lambda_{\text{min}}\left(m\nabla^2 \Risk(\bTheta)\right) 
&= \min_{\norm{\bu}=1}\langle \bu,m\nabla^2 \Risk(\bTheta)\bu\rangle &&\text{(variational)} \\
&\leq  \min_{j\in [m]} \lambda_{\text{min}}\left(\bH_{\text{cross}}(j,j)+\bH_{\text{diag}}(j,j)\right) &&\\
&\leq \min_{j\in [m]} \lambda_{\text{min}}(\bH_{\text{diag}}(j,j))+ \max_{j\in [m]}\norm{\bH_{\text{cross}}(j,j)}_{\mathrm{op}} &&\text{(Weyl's inequality)} \\
&=\lambda_{\text{min}}(\bH_{\text{diag}})+\max_{j\in [m]}\norm{\bH_{\text{cross}}(j,j)}_{\mathrm{op}}. &&\text{(block-diagonal)}
\end{aligned}
\end{equation}
By Assumptions \ref{assumption_regularity}, the coefficients $\ell'', a_j, \sigma'$ are bounded, so
\begin{equation}\label{eq:cross_bd}
\max_{j\in [m]}\norm{\bH_{\text{cross}}(j,j)}_{\mathrm{op}} 
= \max_{j\in [m]} \frac{1}{m}\norm{\frac{1}{n}\sum_{i=1}^n \ell''(y_i,f_{\bTheta}(\bx_i))a_j^2 \sigma'(\btheta_j^\sT \bx_i+b_j)^2\bx_i\bx_i^\sT}_{\mathrm{op}}
\leq \frac{C}{m} \norm{\frac{1}{n}\sum_{i=1}^n \bx_i\bx_i^\sT}_{\mathrm{op}}.
\end{equation}
Combining \eqref{eq:upper_bd_chain} and \eqref{eq:cross_bd} yields
\begin{equation}\label{eq:upper_bd}
\lambda_{\text{min}}\left(m\nabla^2 \Risk(\bTheta)\right) \leq \lambda_{\text{min}}(\bH_{\text{diag}})+\frac{C}{m} \norm{\frac{1}{n}\sum_{i=1}^n \bx_i\bx_i^\sT}_{\mathrm{op}}.
\end{equation}

For the eigenvector bound, let $\boldsymbol{\xi}$ with $\norm{\boldsymbol{\xi}}=1$ satisfy $\nabla^2 \Risk(\bTheta) \boldsymbol{\xi} = \lambda_{\text{min}}(\nabla^2 \Risk(\bTheta)) \boldsymbol{\xi}$. Then
\begin{equation}\label{eq:eigvec_bd}
\begin{aligned}
\boldsymbol{\xi}^\sT \bH_{\text{diag}}\boldsymbol{\xi}
&\leq\boldsymbol{\xi}^\sT \left(\bH_{\text{diag}}+\bH_{\text{cross}}\right)\boldsymbol{\xi}&&\text{($\bH_{\text{cross}}$ positive semidefinite)} \\
&= m\boldsymbol{\xi}^\sT \nabla^2 \Risk(\bTheta) \boldsymbol{\xi} && \\
&= \lambda_{\text{min}}(m\nabla^2 \Risk(\bTheta)) &&\text{(definition)} \\
&\leq \lambda_{\text{min}}(\bH_{\text{diag}})+\frac{C}{m} \norm{\frac{1}{n}\sum_{i=1}^n \bx_i\bx_i^\sT}_{\mathrm{op}}. &&\text{(using \eqref{eq:upper_bd})}
\end{aligned}
\end{equation}
Inequalities \eqref{eq:lower_bd}, \eqref{eq:upper_bd}, and \eqref{eq:eigvec_bd} establish the lemma since $\norm{\frac{1}{n}\sum_{i=1}^n \bx_i\bx_i^\sT}_{\mathrm{op}}=C(1+o_P(1))$.\qed

\subsection{Proof of Lemma \ref{lemma:bulk}: Bulk of the Hessian}
\label{sec:appendix_bulk}

 Recall that the gradient descent dynamics is given by Eq.~\eqref{eq:gd_compact}.
We  define
\begin{equation}\label{eq:W_matrix}
\bTheta(*;t):=\begin{bmatrix}
    \bTheta_* & \bTheta(0)& \dots & \bTheta(t)
\end{bmatrix}\in \mathbb{R}^{d\times (k+m+mt)},
\end{equation}
and
\begin{equation}\label{eq:Q_matrix}
\bF(0;t)=\frac{1}{\sqrt{n}}\begin{bmatrix}
    \bF(\bX \bTheta(0),\bX \bTheta_*(0),\boldsymbol{\varepsilon}) & \dots &\bF(\bX \bTheta(t-1),\bX \bTheta_*(t-1),\boldsymbol{\varepsilon})
\end{bmatrix}\in \mathbb{R}^{n\times mt},
\end{equation}
where $\bF=\begin{bmatrix}
    \bF_1 & \dots & \bF_m
\end{bmatrix}$. Note that we add the $\frac{1}{\sqrt{n}}$ scaling in \eqref{eq:Q_matrix} for scaling consistency.

\paragraph{Step I. Gaussian conditioning.}
We decompose $\bX$ into orthogonal components via Gaussian conditioning. Define the $\sigma$-algebra
\begin{equation}\label{eq:filtration}
\begin{aligned}
\cF_t:=\sigma\big(&\boldsymbol{\varepsilon},\bTheta(0),\dots,\bTheta(t),\bX \bTheta_*,\bX \bTheta(0),\dots,\bX \bTheta(t),\\
&\bF(\bX \bTheta(0),\bX \bTheta_*(0),\boldsymbol{\varepsilon}),\dots,\bF(\bX \bTheta(t),\bX \bTheta_*(t),\boldsymbol{\varepsilon})\big).
\end{aligned}
\end{equation}

For two random variables $\bX$, $\bX'$
we write $\bX\big|_{\cF_t}\ed \bX'$ if $\E[\varphi(\bX) Z]= \E[\varphi(\bX') Z]$ for any bounded measurable function $\varphi$
and any bounded random variable $Z$ measurable on $\cF_t$. 
\begin{lemma}[Gaussian conditioning]\label{lem:gauss_cond}
We have
\begin{equation}\label{eq:X_cond_decomp}
\bX\big|_{\cF_t}
\ed
\bP_{\bF(0;t)}^{\perp}\bX_{\textrm{new}}\bP_{\bTheta(*;t)}^{\perp}
+ \bX\bP_{\bTheta(*;t)},
\end{equation}
where $\bP_\bM := \bM(\bM^\sT \bM)^{-1}\bM^\sT$ denotes the orthogonal projection onto $\operatorname{span}(\bM)$, and $\bX_{\textrm{new}} \in \mathbb{R}^{n\times d}$ is independent of $\bX$ and $\boldsymbol{\varepsilon}$, with entries drawn i.i.d. from $\normal(0,1)$.
\end{lemma}
\begin{proof}
We first establish an algebraic decomposition of $\bX$:
\begin{equation}\label{eq:X_algebraic_decomp}
\begin{aligned}
\bX
&= \bP_{\bF(0;t)}^{\perp}\bX\bP_{\bTheta(*;t)}^{\perp}
+ \bP_{\bF(0;t)}^{\perp}\bX\bP_{\bTheta(*;t)}
+ \bP_{\bF(0;t)}\bX\bP_{\bTheta(*;t)}^{\perp}
+ \bP_{\bF(0;t)}\bX\bP_{\bTheta(*;t)} &&\\
&=\bP_{\bF(0;t)}^{\perp}\bX\bP_{\bTheta(*;t)}^{\perp}
+ \bX\bP_{\bTheta(*;t)}. &&\text{(middle terms vanish)}
\end{aligned}
\end{equation}
The first equality is purely algebraic. The second equality follows from the fact that
$\bX^{\sT}\bF(0;t)$ has columns in $\operatorname{span}(\bTheta(*;t))$,
This is in turn a consequence of the fact  that by the gradient descent update \eqref{eq:gd_compact}, 
for each $i\le t-1$ we have
\begin{equation}\label{eq:update_expanded}
\bX^\sT\bF\big(\bX \bTheta(i),\bX \bTheta_*, \boldsymbol{\varepsilon}\big) = -n\big(\bTheta(i+1)-\bTheta(i)\big).
\end{equation}

It remains to show that, conditioned on $\cF_t$, we may replace
$\bP_{\bF(0;t)}^{\perp}\bX\bP_{\bTheta(*;t)}^{\perp}$  
by  
$\bP_{\bF(0;t)}^{\perp}\bX_{\text{new}}\bP_{\bTheta(*;t)}^{\perp}$.
This follows immediately from the fact that $\cF_t$ is generates by $\bX\bTheta(*;t)$, 
$\bX^{\sT}\bF(0;t)$, and $\bP_{\bF(0;t)}^{\perp}\bX\bP_{\bTheta(*;t)}^{\perp}$
is uncorrelated and hence independent of these quantities.
\end{proof}

\paragraph{Step II. Low-rank decomposition.}
Define
\begin{equation}\label{eq:YZ_def}
\bY(t):=\bP_{\bF(0;t)}^{\perp}\bX_{\text{new}}\bP_{\bTheta(*;t)}^{\perp}\in\mathbb{R}^{n\times d},
\quad
\bZ(t):=\bX \bTheta(*;t)(\bTheta(*;t)^\sT \bTheta(*;t))^{-1}\in\mathbb{R}^{n\times(k+mt)}.
\end{equation}
Using Lemma \ref{lem:gauss_cond} in the expression for the $j$-th block of the Hessian,
cf. Eq.~\eqref{eq:outline_Hj_def} we get
\begin{equation}\label{eq:Hj_expansion}
\begin{aligned}
\bH_j(t)
&\stackrel{d}{=}  \left(\bP_{\bF(0;t)}^{\perp}\bX_{\text{new}}\bP_{\bTheta(*;t)}^{\perp}+\bX\bP_{\bTheta(*;t)}\right)^\sT\bG_j(t)\left(\bP_{\bF(0;t)}^{\perp}\bX_{\text{new}}\bP_{\bTheta(*;t)}^{\perp}+\bX\bP_{\bTheta(*;t)}\right) && \\
&= \bY(t)^\sT\bG_j(t)\bY(t) +\begin{bmatrix}
    \bTheta(*;t) & \bY(t)^\sT\bG_j(t)\bZ(t)
\end{bmatrix}\begin{bmatrix}
    \bZ(t)^\sT \bG_j(t)\bZ(t) & \bI \\
    \bI & \boldsymbol{0}
\end{bmatrix}\begin{bmatrix}
    \bTheta(*;t)^\sT \\
    \bZ(t)^\sT \bG_j(t)\bY(t)
\end{bmatrix}. &&
\end{aligned}
\end{equation}
The second term has rank at most $2(k+mt)$.

\paragraph{Step III. Removing low-rank perturbations via Cauchy interlacing.}
 By Cauchy interlacing theorem, if $\bA$, $\bB$ are symmetric matrices with ordered eigenvalues
 $\lambda_i(\bA)$, $\lambda_i(\bB)$, $1\le i\le s$ and $\operatorname{rank}(\bB - \bA) \le s$, then
$\lambda_i(\bA) \le \lambda_i(\bB) \le \lambda_{i+s}(\bA)$ for all $1 \le i \le d - s$.
As a consequence the limiting spectral distributions of the two matrices coincide.
We state this as the following lemma (see e.g. \cite[Section 2.4.1]{tao2023topics} or \cite[Theorem A.44]{BaiSilverstein}).
\begin{lemma}[Interlacing for finite-rank perturbations]\label{lem:interlacing}
Let $\bH_d$ be a sequence of $d\times d$ real symmetric random matrices, and denote their empirical spectral distributions by
\begin{equation}\label{eq:ESD_H}
  \mu_{\bH_d} = \frac{1}{d}\sum_{i=1}^d\delta_{\cdot-\lambda_i(\bH_d)}.
\end{equation}
Suppose $\mu_{\bH_d}\Rightarrow \mu$ in probability, where $\mu$ is a deterministic probability measure.  

Let $\bW_d$ be a sequence of real symmetric matrices with $\operatorname{rank}(\bW_d)\le s=O(1)$.  
Define the perturbed matrices
\begin{equation}\label{eq:perturbed_H}
\widetilde \bH_d := \bH_d + \bW_d,
\end{equation}
and their empirical distributions
\begin{equation}\label{eq:ESD_Htilde}
  \mu_{\widetilde \bH_d} = \frac{1}{d}\sum_{i=1}^d\delta_{\cdot -\lambda_i(\widetilde \bH_d)}.
\end{equation}
Then we have
\begin{equation}\label{eq:interlacing_result}
  \mu_{\widetilde \bH_d} \Rightarrow \mu \quad\text{in probability as }d\to\infty.
\end{equation}
\end{lemma}

\paragraph{Step IV. Applying Marchenko-Pastur theory.}
By Lemma \ref{lem:interlacing} applied to the decomposition \eqref{eq:Hj_expansion}, the  limiting
spectral distribution LSD of $\bH_j(t)$ equals the LSD of
\begin{equation}\label{eq:bulk_Hj}
\widetilde \bH_j(t):=\bY(t)^\sT\bG_j(t)\bY(t),
\end{equation}
since the rank-$O(1)$ perturbation does not affect the limit.

We now further simplify $\widetilde \bH_j(t)$ by removing the projection operators. Recall from \eqref{eq:YZ_def} that $\bY(t)=\bP_{\bF(0;t)}^{\perp}\bX_{\text{new}}\bP_{\bTheta(*;t)}^{\perp}$. Since $\bF(0;t)\in\mathbb{R}^{n\times mt}$ and $\bTheta(*;t)\in\mathbb{R}^{d\times(k+mt)}$ have $O(1)$ columns (by \eqref{eq:W_matrix} and \eqref{eq:Q_matrix}), the projections $\bP_{\bF(0;t)}^{\perp}$ and $\bP_{\bTheta(*;t)}^{\perp}$ differ from the identity by rank-$O(1)$ matrices. Therefore, $\bY(t)^\sT\bG_j(t)\bY(t)$ differs from $\bX_{\text{new}}^\sT\bG_j(t)\bX_{\text{new}}$ by a rank-$O(1)$ perturbation. Applying Lemma \ref{lem:interlacing} again, the LSD of $\widetilde \bH_j(t)$ equals the LSD of
\begin{equation}\label{eq:simplified_bulk}
\widehat \bH_j(t):=\bX_{\text{new}}^\sT\bG_j(t)\bX_{\text{new}},
\end{equation}
where $\bX_{\text{new}} \in \mathbb{R}^{n\times d}$ has i.i.d. $\normal(0,1)$ entries.

We now characterize the LSD of $\widehat \bH_j(t)$. Note that $\bG_j(t)$ is diagonal and independent of $\bX_{\text{new}}$. By Lemma \ref{lemma_amp_data}, the empirical spectral distribution of $\bG_j(t)$ converges in probability to the law of
\begin{equation}\label{eq:Gj_limit}
G_t^j := g(V(t),h(V_*,\varepsilon);j)
\end{equation}
in distribution and in Wasserstein-2 distance.

This falls into the framework of generalized Marchenko-Pastur laws for random matrices of the form $\bX^\sT\bG\bX$, where $\bX$ has i.i.d. Gaussian entries and $\bG$ is independent of $\bX$ and diagonal with a limiting spectral measure. The classical Marchenko-Pastur result states that the LSD $\mu_{\infty}^j(t)$ of $\widehat \bH_j(t)$ is uniquely characterized by its Stieltjes transform $\alpha_t^j(z)$, which satisfies the fixed-point equation
\begin{equation}\label{eq:MP_fixed_point}
z+\frac{1}{\alpha_t^j(z)}=\delta\cdot \mathbb{E}\left[\frac{G_t^j}{\delta + G_t^j \alpha_t^j(z)}\right].
\end{equation}
This completes the proof of Lemma \ref{lemma:bulk}.\qed

\section{Proofs of Theorems \ref{theorem:main_result} and~\ref{theorem:main_result_limit}}\label{sec:proof_result_iii}

This section proves Theorems~\ref{theorem:main_result} and~\ref{theorem:main_result_limit}, establishing the existence of a sharp phase transition for detecting hard directions via Hessian outliers. 
We refer to Section \ref{sec:proof_sketch} for a roadmap. 

Throughout this appendix, as in the previous one, we assume without loss of generality $U_\mathrm{H}^* =\begin{bmatrix}
    e_1 & \dots & e_r
\end{bmatrix}$.

\subsection{Outliers via resolvent analysis}\label{subsec:spike_basic}

We fix $j\in[m]$ and time $t\ge 0$ throughout. 
Recall that
\begin{equation}\label{eq:spike_Hj_def}
\bH_j(t) =\bX^{\sT}\bG_j(t)\bX \, ,\;\;\;\;
\bG_j(t)= \frac{1}{n}\operatorname{diag}\!\left(g(\bTheta(t)^\sT\bx_1,y_1;j),\,\dots,\,g(\bTheta(t)^\sT\bx_n,y_n;j)\right)
\in \mathbb{R}^{n\times n},
\end{equation}
where $g$ is defined by Eqs.~\eqref{eq:Gdef_1}, \eqref{eq:Gdef_2}.

We also recall the augmented parameter and gradient matrices defined in Eq.~\eqref{eq:outline_WQ_def}
\begin{align}\label{eq:spike_W_def}
&\bTheta(*;t) 
:= \begin{bmatrix}
\bTheta_* & \bTheta(0) & \dots & \bTheta(t)
\end{bmatrix}
\in \mathbb{R}^{d\times (k+m+mt)},\\
\label{eq:spike_Q_def}
&\bF(0;t)
:= \frac{1}{\sqrt{n}}\begin{bmatrix}
\bF(\bX\bTheta(0),\bX\bTheta_*,\boldsymbol{\varepsilon}) & \dots & \bF(\bX\bTheta(t-1),\bX\bTheta_*,\boldsymbol{\varepsilon})
\end{bmatrix}
\in \mathbb{R}^{n\times mt}\, .
\end{align}

 In what follows, $j\in[m]$, $t\ge 0$ will be fixed and hence 
 we drop these in $\bTheta(*;t),\bF(0;t),\bG_j(t),\bH_j(t)$ whenever clear, writing $\bTheta_{*;t},\bF_{0;t},\bG,\bH$.

By Lemma~\ref{lem:gauss_cond} in Appendix~\ref{sec_Proofs}, conditioned on the $\sigma$-algebra $\cF_t$, the data matrix admits the representation $\bX \big|_{\cF_t}
\ed \bP_{\bF_{0;t}}^{\perp}\bX_{\text{new}}\bP_{\bTheta_{*;t}}^{\perp} 
+ \bX\bP_{\bTheta_{*;t}}$.
We also recall the notations
\begin{align}
\label{eq:spike_YZ_def}
&\bY := \bP_{\bF_{0;t}}^{\perp}\bX_{\text{new}}\bP_{\bTheta_{*;t}}^{\perp}
\in \mathbb{R}^{n\times d},
\quad
\bZ := \bX\bTheta_{*;t}(\bTheta_{*;t}^\sT\bTheta_{*;t})^{-1}
\in \mathbb{R}^{n\times (k+m+mt)}.
\end{align}

Let $z\in\mathbb{C}$. We seek a formula for the resolvent $(\bH-z\bI)^{-1}$ that isolates outlier contributions. Applying the Woodbury matrix identity to the resolvent, we obtain the representation \eqref{eq:outline_resolvent},
which we copy here
\begin{equation}\label{eq:spike_resolvent_identity}
\begin{aligned}
(\bH-z\bI)^{-1}
&= \bR_0(z)
 - \bR_0(z)
\begin{bmatrix}
\bTheta_{*;t} & \bY^\sT\bG\bZ
\end{bmatrix}
\bM(z)^{-1}
\begin{bmatrix}
\bTheta_{*;t}^\sT \\
\bZ^\sT\bG\bY
\end{bmatrix}
\bR_0(z),
\end{aligned}
\end{equation}
where $\bR_0(z) := \big(\bY^\sT \bG\bY - z\bI\big)^{-1}$ and the \emph{outlier matrix} is defined by
\begin{equation}\label{eq:spike_Mj_def}
\bM(z) 
:= \begin{bmatrix}
\bZ^\sT\bG\bZ & \bI \\
\bI & \boldsymbol{0}
\end{bmatrix}^{-1}
+ \begin{bmatrix}
\bTheta_{*;t}^\sT \\
\bZ^\sT\bG\bY
\end{bmatrix}
\bR_0(z)
\begin{bmatrix}
\bTheta_{*;t} & \bY^\sT\bG\bZ
\end{bmatrix}
\in \mathbb{C}^{(2k+2m+2mt)\times (2k+2m+2mt)}.
\end{equation}

By Lemma~\ref{lemma:bulk} and Lemma \ref{lemma:smallest_eigenvalue_concentration}, 
for $z<0$ in the left of the bulk of $\mu_{\infty}(t)$, the matrix $\bY^\sT\bG\bY-z\bI$ is invertible with probability tending to one as $n,d\to\infty$. 
From \eqref{eq:spike_resolvent_identity}, $z$ is an outlier eigenvalue of $\bH$ if and only if $\operatorname{det}(\bM(z))=0$. Thus outlier eigenvalues satisfy
\begin{equation}\label{eq:spike_det_condition}
\operatorname{det}\!\big(\bM(z)\big) = 0 
\quad\text{for }z<c(t).
\end{equation}
where we recall that $c(t)$ denotes the left edge of $\mu_\infty(t)$.

Let $\mathcal{D} := \{z\in \mathbb{C}:\ \Re(z) < \min(c(t),0),\ |\Im(z)| \le 1,\ |z| \le C\}$ for $C$ a sufficiently large constant. Further, let $V_{*\leq} \in \mathbb{R}^r$ denotes the first $r$ coordinates of $V_*$. We now analyze the asymptotic behavior of $\bM(z)$ as $n,d\to\infty$, which is the main result of this section.

\begin{lemma}\label{lemma:limiting_spike_matrix}
Under the assumptions of Theorem~\ref{theorem:main_result}, for $z \in \mathcal{D}$, we have
\begin{equation}\label{eq:spike_main_convergence}
\bM_j(z;t) \xrightarrow{~p~} M_j^\infty(z;t),
\end{equation}
where the limiting outlier matrix is given by
\begin{equation}\label{eq:spike_Mj_infinity}
M_j^\infty(z;t)
= \begin{bmatrix}
-\dfrac{1}{z}\mathbb{E}[O_tO_t^\sT] & I \\[8pt]
I & \left(\mathbb{E}[O_tO_t^\sT]\right)^{-1}\begin{bmatrix}
\mathbb{E}\left[-\dfrac{\delta G_t^j}{\delta + G_t^j\alpha_t^j(z)} V_{*\leq}V_{*\leq}^\sT\right] & 0 \\[6pt]
0 & S_{\mathrm{R}}^j(z;t)
\end{bmatrix}\left(\mathbb{E}[O_tO_t^\sT]\right)^{-1}
\end{bmatrix}.
\end{equation}
Here $S_{\mathrm{R}}^j(z;t) \in \mathbb{R}^{(k-r+m+mt)\times(k-r+m+mt)}$ is a deterministic matrix obtained by taking expectations of certain matrix-valued functions of the discrete-time DMFT random variables (up to time $t$).
\end{lemma}

The proof of Lemma~\ref{lemma:limiting_spike_matrix} occupies the remainder of this section. We proceed by analyzing each block of $\bM(z)$ separately.

\paragraph{Step I. Simple blocks of outlier matrix.}

We now analyze the block structure of $\bM(z)$ and establish concentration for its components. The first term in \eqref{eq:spike_Mj_def} has explicit inverse:
\begin{equation}\label{eq:spike_block_inv}
\begin{bmatrix}
\bZ^\sT\bG\bZ & \bI \\
\bI & \boldsymbol{0}
\end{bmatrix}^{-1}
= \begin{bmatrix}
\boldsymbol{0} & \bI \\
\bI & -\bZ^\sT\bG\bZ
\end{bmatrix}.
\end{equation}

For later use, we introduce concatenated vectors
\begin{equation}\label{eq:spike_V_concat}
V_{*:t}
:= \begin{bmatrix}
V_* \\
V(0) \\
\vdots \\
V(t)
\end{bmatrix}
\in \mathbb{R}^{(k+m+mt)},\;\;\;\;\;
O_t
:= \begin{bmatrix}
\Theta_* \\
\Theta(0) \\
\vdots \\
\Theta(t)
\end{bmatrix}
\in \mathbb{R}^{(k+m+mt)}\, .
\end{equation}
Define the scalar random variable
\begin{equation}\label{eq:spike_Gt_def}
G_t := g(V(t), h(V_*,\varepsilon)),
\end{equation}
where $(V(t), V_*, \varepsilon)$ are the state evolution random variables from Lemma~\ref{lemma_amp_data} and $h(V_*,\varepsilon)$ is the output function. In what follows we drop the subscript $t$ and write $G$ when the context is clear.
\begin{lemma}\label{lem:spike_ZGZ_concentration}
 Under the assumptions of Theorem~\ref{theorem:main_result}, we have
\begin{equation}\label{eq:spike_ZGZ_limit}
\bZ^\sT\bG\bZ 
\xrightarrow{~p~} 
\big(\mathbb{E}[O_tO_t^\sT]\big)^{-1}
\mathbb{E}\!\left[G V_{*:t}V_{*:t}^\sT\right]
\big(\mathbb{E}[O_tO_t^\sT]\big)^{-1},
\end{equation}
and consequently
\begin{equation}\label{eq:spike_block_inv_limit}
\begin{bmatrix}
\bZ^\sT\bG\bZ & \bI \\
\bI & \boldsymbol{0}
\end{bmatrix}^{-1}
\xrightarrow{~p~}
\begin{bmatrix}
0& I \\
I & -\big(\mathbb{E}[O_tO_t^\sT]\big)^{-1}
\mathbb{E}\!\left[G V_{*:t}V_{*:t}^\sT\right]
\big(\mathbb{E}[O_tO_t^\sT]\big)^{-1}
\end{bmatrix}.
\end{equation}
\end{lemma}

\begin{proof}
Expand the quadratic form:
\begin{equation}\label{eq:spike_ZGZ_expand}
\begin{aligned}
\bZ^\sT\bG\bZ 
&= \big(\bTheta_{*;t}^\sT\bTheta_{*;t}\big)^{-1} 
\bTheta_{*;t}^\sT\bX^\sT\bG\bX\bTheta_{*;t} 
\big(\bTheta_{*;t}^\sT\bTheta_{*;t}\big)^{-1}
&&\text{(definition \eqref{eq:spike_YZ_def})} \\
&= \big(\bTheta_{*;t}^\sT\bTheta_{*;t}\big)^{-1} 
\left(\frac{1}{n}\sum_{i=1}^n g(\bTheta^\sT\bx_i,y_i)\, \bTheta_{*;t}^\sT\bx_i(\bTheta_{*;t}^\sT\bx_i)^\sT\right)
\big(\bTheta_{*;t}^\sT\bTheta_{*;t}\big)^{-1}.
&&\text{(expanding $\bG$)}
\end{aligned}
\end{equation}
By Assumption \ref{assumption_regularity}, the function $g(\cdot)$ is uniformly bounded (by $\|\ell'\|_\infty\|\sigma''\|_\infty<\infty$) and Lipschitz continuous. 

Therefore, the test function
\begin{equation}\label{eq:spike_test_func}
\phi(\bTheta(0)^\sT\bx_i,\bTheta(1)^\sT\bx_i,\dots,\bTheta(t)^\sT\bx_i, \bTheta^{\sT}_*\bx_i, \varepsilon_i)
:= g(\bTheta(t)^\sT\bx_i,y_i)\, \bTheta_{*;t}^\sT\bx_i(\bTheta_{*;t}^\sT\bx_i)^\sT
\end{equation}
has at most quadratic growth and is locally Lipschitz. Applying Lemma~\ref{lemma_amp_data} yields
\begin{equation}\label{eq:spike_inner_conv}
\frac{1}{n}\sum_{i=1}^n g(\bTheta(t)^\sT\bx_i,y_i)\, \bTheta_{*;t}^\sT\bx_i(\bTheta_{*;t}^\sT\bx_i)^\sT
\xrightarrow{~p~}
\mathbb{E}\!\left[G V_{*:t}V_{*:t}^\sT\right],
\end{equation}
where $G$ is defined in \eqref{eq:spike_Gt_def}.

Similarly, Lemma~\ref{lemma_amp_feature} gives $\bTheta_{*;t}^\sT\bTheta_{*;t} \xrightarrow{~p~} \mathbb{E}[O_tO_t^\sT]$, which is positive definite under Assumption~\ref{assumption:reg_1}. By continuous mapping,
\begin{equation}\label{eq:spike_WW_inv_conv}
\big(\bTheta_{*;t}^\sT\bTheta_{*;t}\big)^{-1} \xrightarrow{~p~} \big(\mathbb{E}[O_tO_t^\sT]\big)^{-1}.
\end{equation}
Combining \eqref{eq:spike_ZGZ_expand}, \eqref{eq:spike_inner_conv}, and \eqref{eq:spike_WW_inv_conv} yields \eqref{eq:spike_ZGZ_limit}. Equation~\eqref{eq:spike_block_inv_limit} follows by substituting into \eqref{eq:spike_block_inv}.
\end{proof}

\paragraph{Step II. Resolvent-dependent terms.} 
We now analyze the second term in \eqref{eq:spike_Mj_def}, which involves the 
resolvent $\bR_0(z) = (\bY^\sT\bG\bY-z\bI)^{-1}$.

This term has a $2\times2$ block structure:
\begin{equation}\label{eq:spike_resolvent_term_blocks}
\begin{aligned}
&\begin{bmatrix}
\bTheta_{*;t}^\sT \\
\bZ^\sT\bG\bY
\end{bmatrix}
\bR_0(z)
\begin{bmatrix}
\bTheta_{*;t} & \bY^\sT\bG\bZ
\end{bmatrix} \\
&\quad= \begin{bmatrix}
\bTheta_{*;t}^\sT\bR_0(z)\bTheta_{*;t} & \bTheta_{*;t}^\sT\bR_0(z)\bY^\sT\bG\bZ \\
\bZ^\sT\bG\bY\bR_0(z)\bTheta_{*;t} & \bZ^\sT\bG\bY\bR_0(z)\bY^\sT\bG\bZ
\end{bmatrix}.
\end{aligned}
\end{equation}
We analyze each block separately. 

Recall from \eqref{eq:spike_YZ_def} that $\bY=\bP_{\bF_{0;t}}^{\perp}\bX_{\text{new}}\bP_{\bTheta_{*;t}}^{\perp}$. Hence
\begin{equation}\label{eq:spike_resolvent_factorization}
\bY^\sT\bG\bY 
= \bP_{\bTheta_{*;t}}^{\perp}\bX_{\text{new}}^\sT\bP_{\bF_{0;t}}^{\perp}\bG\bP_{\bF_{0;t}}^{\perp}\bX_{\text{new}}\bP_{\bTheta_{*;t}}^{\perp}.
\end{equation}
Perform a spectral decomposition on the central sandwich: since $\bP_{\bF_{0;t}}^{\perp}\bG\bP_{\bF_{0;t}}^{\perp}$ is symmetric, write
\begin{equation}\label{eq:spike_spectral_decomp}
\bP_{\bF_{0;t}}^{\perp}\bG\bP_{\bF_{0;t}}^{\perp} 
= \bO\widehat{\bG}\bO^\sT,
\end{equation}
where $\bO\in\mathbb{R}^{n\times n}$ is orthogonal and $\widehat{\bG}=\operatorname{diag}(\widehat{g}_1,\dots,\widehat{g}_n)$ is diagonal. 

Similarly, write $\bP_{\bTheta_{*;t}}^{\perp}=\bU\bU^\sT$ for some $\bU\in\mathbb{R}^{d\times (d-k-m-mt)}$ with $\bU^\sT\bU=\bI$. Here $\bU$ can be chosen as any orthonormal basis for the orthogonal complement of $\operatorname{span}(\bTheta_{*;t})$. Define the rotated data matrix
\begin{equation}\label{eq:spike_tilde_X}
\widetilde{\bX} := \bO^\sT\bX_{\text{new}}
\in \mathbb{R}^{n\times d}.
\end{equation}
Since $\bO$ is orthogonal and $\bX_{\text{new}}$ has i.i.d.\ $\normal(0,1)$ entries, $\widetilde{\bX}$ also has i.i.d.\ $\normal(0,1)$ entries (rotation-invariance of Gaussian). Substituting into \eqref{eq:spike_resolvent_factorization},
\begin{equation}\label{eq:spike_resolvent_sandwich}
\begin{aligned}
\bY^\sT\bG\bY - z\bI
&= \bU\bU^\sT \widetilde{\bX}^\sT\widehat{\bG}\widetilde{\bX}\bU\bU^\sT - z\bI 
&&\text{(by \eqref{eq:spike_resolvent_factorization}, \eqref{eq:spike_spectral_decomp}, \eqref{eq:spike_tilde_X})} \\
&=: \bU\bS\bU^\sT - z\bI,
\end{aligned}
\end{equation}
where 
\begin{equation}\label{eq:spike_S_def}
\bS := \bU^\sT \widetilde{\bX}^\sT\widehat{\bG}\widetilde{\bX}\bU
\in \mathbb{R}^{(d-k-m-mt)\times(d-k-m-mt)}.
\end{equation}
Therefore
\begin{equation}\label{eq:spike_projection_identity}
\bR_0(z) = (\bY^\sT\bG\bY - z\bI)^{-1}=(\bU\bS\bU^\sT - z\bI)^{-1}
= -\frac{1}{z}\bP_{\bTheta_{*;t}} + \bU(\bS - z\bI)^{-1}\bU^\sT.
\end{equation}
This decomposition is valid for $z\neq0$ and $z$ not in the spectrum of $\bS$.

Using \eqref{eq:spike_projection_identity}, we compute the $(1,1)$-block of \eqref{eq:spike_resolvent_term_blocks}. Since $\bTheta_{*;t}^\sT\bU=\boldsymbol{0}$ by construction (the columns of $\bU$ span the orthogonal complement of $\bTheta_{*;t}$), we have
\begin{equation}\label{eq:spike_11_block}
\begin{aligned}
\bTheta_{*;t}^\sT\bR_0(z)\bTheta_{*;t}
= -\frac{1}{z}\,\bTheta_{*;t}^\sT\bTheta_{*;t}.
\end{aligned}
\end{equation}
By Lemma~\ref{lemma_amp_feature}, $\bTheta_{*;t}^\sT\bTheta_{*;t}\xrightarrow{~p~}\mathbb{E}[O_tO_t^\sT]$, so
\begin{equation}\label{eq:spike_11_limit}
\bTheta_{*;t}^\sT\bR_0(z)\bTheta_{*;t} 
\xrightarrow{~p~} -\frac{1}{z}\,\mathbb{E}[O_tO_t^\sT].
\end{equation}

For the cross-blocks (off-diagonal), note that
\begin{equation}\label{eq:spike_cross_block_step1}
\begin{aligned}
\bTheta_{*;t}^\sT\bR_0(z)\bY^\sT\bG\bZ
&= \bTheta_{*;t}^\sT\!\left(-\frac{1}{z}\,\bP_{\bTheta_{*;t}} + \bU(\bS - z\bI)^{-1}\bU^\sT\right)\!\bY^\sT\bG\bZ
&&\text{(by \eqref{eq:spike_projection_identity})} \\
&= -\frac{1}{z}\,\bTheta_{*;t}^\sT\bP_{\bTheta_{*;t}}\bY^\sT\bG\bZ + \bTheta_{*;t}^\sT\bU(\bS - z\bI)^{-1}\bU^\sT\bY^\sT\bG\bZ
&&\\
& = {\boldsymbol 0}\, .
\end{aligned}
\end{equation}
In the last step we used the fact that $\bY=\bP_{\bF_{0;t}}^{\perp}\bX_{\text{new}}\bP_{\bTheta_{*;t}}^{\perp}$ implies $\bY^\sT\bP_{\bTheta_{*;t}}=\boldsymbol{0}$, and that  $\bTheta_{*;t}^\sT\bU=\boldsymbol{0}$. 

The $(2,2)$-block $\bB_{\sf{lower}}:=\bZ^\sT\bG\bY\bR_0(z)\bY^\sT\bG\bZ$ requires more care.
We characterize the asymptotics of $\bB_{\sf{lower}}$ as follows.

\begin{lemma}\label{lem:B_lower_convergence}
Under the assumptions of Theorem~\ref{theorem:main_result}, for $z\in \mathcal{D}$, we have
\begin{equation}\label{eq:B_lower_limit}
\begin{aligned}
&\bB_{\sf{lower}} \xrightarrow{~p~} B_{\sf{lower}}^\infty(z;t),\\
&B_{\sf{lower}}^\infty(z;t) = \left(\mathbb{E}[O_tO_t^\sT]\right)^{-1}
\begin{bmatrix}
\mathbb{E}\left[\dfrac{\alpha_t^j(z)G_t^{j2}}{\delta + G_t^j\alpha_t^j(z)} V_{*\leq}V_{*\leq}^\sT\right] & 0 \\[8pt]
0 & \widetilde{S}_{\mathrm{R}}^j(z;t)
\end{bmatrix}
\left(\mathbb{E}[O_tO_t^\sT]\right)^{-1}.
\end{aligned}
\end{equation}
Here $\widetilde{S}_{\mathrm{R}}^j(z;t) \in \mathbb{R}^{(k-r+m+mt)\times(k-r+m+mt)}$ is some deterministic matrix depending on expectations over the discrete-time DMFT random variables.
\end{lemma}

The proof of Lemma~\ref{lem:B_lower_convergence} is carried out in Sections~\ref{subsec:concentration} and~\ref{subsec:convergence}. Given Lemma \ref{lem:B_lower_convergence}, we can now prove Lemma \ref{lemma:limiting_spike_matrix}.

\begin{proof}[Proof of Lemma~\ref{lemma:limiting_spike_matrix}] 
Recall from \eqref{eq:spike_Mj_def} that the outlier matrix $\bM(z)$ is the sum of two terms. We analyze each term separately.

\emph{First term.} By Lemma~\ref{lem:spike_ZGZ_concentration}, the first term converges to
\begin{equation}
\begin{bmatrix}
\bZ^\sT\bG\bZ & \bI \\
\bI & \boldsymbol{0}
\end{bmatrix}^{-1}
\xrightarrow{~p~}
\begin{bmatrix}
0 & I \\
I & -\left(\mathbb{E}[O_tO_t^\sT]\right)^{-1}\mathbb{E}\left[G_t^j V_{*:t}V_{*:t}^\sT\right]\left(\mathbb{E}[O_tO_t^\sT]\right)^{-1}
\end{bmatrix}.
\end{equation}

\emph{Second term.} The second term in \eqref{eq:spike_Mj_def} has a $2\times 2$ block structure given by \eqref{eq:spike_resolvent_term_blocks}. By \eqref{eq:spike_11_limit}, the $(1,1)$-block converges to $-\frac{1}{z}\mathbb{E}[O_tO_t^\sT]$. By \eqref{eq:spike_cross_block_step1}, the cross-blocks equal to zero. By Lemma~\ref{lem:B_lower_convergence}, the $(2,2)$-block $\bB_{\sf{lower}}$ converges to $B_{\sf{lower}}^\infty(z;t)$.

Adding the two terms, we obtain
\begin{equation}
\bM_j(z;t) \xrightarrow{~p~}
\begin{bmatrix}
-\frac{1}{z}\mathbb{E}[O_tO_t^\sT] & I \\
I & -\left(\mathbb{E}[O_tO_t^\sT]\right)^{-1}\mathbb{E}\left[G_t^j V_{*:t}V_{*:t}^\sT\right]\left(\mathbb{E}[O_tO_t^\sT]\right)^{-1} + B_{\sf{lower}}^\infty(z;t)
\end{bmatrix}.
\end{equation}
To simplify the $(2,2)$-block, we use the block decomposition $V_{*:t} = \begin{bmatrix}V_{*\leq}^\sT & V_{*>:t}^\sT\end{bmatrix}^\sT$. By the second part of Lemma~\ref{lemma_amp_data} (cf.\ \eqref{eq:hard_orthogonality}), the hard directions $V_{*\leq} = P_{U_\mathrm{H}^*}V_*$ has no correlation with others:
\begin{equation}
\mathbb{E}\left[G_t^j V_{*\leq} V_{*>:t}^\sT\right] = 0.
\end{equation}
Therefore, $\mathbb{E}\left[G_t^j V_{*:t}V_{*:t}^\sT\right]$ is block diagonal:
\begin{equation}
\mathbb{E}\left[G_t^j V_{*:t}V_{*:t}^\sT\right] = \begin{bmatrix}
\mathbb{E}\left[G_t^j V_{*\leq}V_{*\leq}^\sT\right] & 0 \\
0 & \mathbb{E}\left[G_t^j V_{*>:t}V_{*>:t}^\sT\right]
\end{bmatrix}.
\end{equation}
Substituting the explicit form of $B_{\sf{lower}}^\infty(z;t)$ from Eq.~\eqref{eq:B_lower_limit} also using the block diagonal structure above, the $(2,2)$-block simplifies to
\begin{equation}
\left(\mathbb{E}[O_tO_t^\sT]\right)^{-1}
\begin{bmatrix}
\mathbb{E}\left[-\dfrac{\delta G_t^j}{\delta + G_t^j\alpha_t^j(z)} V_{*\leq}V_{*\leq}^\sT\right] & 0 \\
0 & S_{\mathrm{R}}^j(z;t)
\end{bmatrix}
\left(\mathbb{E}[O_tO_t^\sT]\right)^{-1},
\end{equation}
where $S_{\mathrm{R}}^j(z;t) := -\mathbb{E}\left[G_t^j V_{*>:t}V_{*>:t}^\sT\right] + \widetilde{S}_{\mathrm{R}}^j(z;t)$. This gives the limiting outlier matrix \eqref{eq:spike_Mj_infinity}.
\end{proof}

\subsection{Resolvent concentration via leave-one-out}\label{subsec:concentration}

$\bB_{\sf{lower}}=\bZ^\sT\bG\bY\bR_0(z)\bY^\sT\bG\bZ$ can be written as
\begin{equation}\label{eq:spike_22_block}
\begin{aligned}
\bB_{\sf{lower}}
&= \bZ^\sT\bG\bY\!\left(-\frac{1}{z}\bP_{\bTheta_{*;t}} + \bU(\bS - z\bI)^{-1}\bU^\sT\right)\!\bY^\sT\bG\bZ
&&\text{(by \eqref{eq:spike_projection_identity})} \\
&= \bZ^\sT\bG\bY\bU\left(\bS - z\bI\right)^{-1}\bU^\sT\bY^\sT\bG\bZ
&&\text{($\bY^\sT\bP_{\bTheta_{*;t}}=\boldsymbol{0}$)} \\
&= \bZ^\sT\bG\bP_{\bF_{0;t}}^{\perp}\bX_{\text{new}}\bU\left(\bS - z\bI\right)^{-1}\bU^\sT\bX_{\text{new}}^\sT\bP_{\bF_{0;t}}^{\perp}\bG\bZ.
\end{aligned}
\end{equation}
Define the Gaussian matrix
\begin{equation}\label{eq:spike_R_def}
\bR := \widetilde{\bX}\bU = \bO^\sT\bX_{\text{new}}\bU
\in \mathbb{R}^{n\times (d-k-m-mt)}.
\end{equation}
Since $\bU,\bO$ are measurable on $\cF_t$, the entries of $\bR$ are i.i.d.\ $\normal(0,1)$ and independent of $(\bX,\boldsymbol{\varepsilon})$. 
The Gaussian matrix $\bR$ should not be confused with the resolvent $\bR_0(z) = \left(\bY^\sT\bG\bY - z\bI\right)^{-1}$ defined in Eq.~\eqref{eq:spike_projection_identity}.
By Eq.~\eqref{eq:spike_S_def},
\begin{align}\label{eq:spike_S_R_relation}
&\bS = \bR^\sT\widehat{\bG}\bR,\\
\label{eq:spike_22_R_form}
&\bS=\bU^\sT\bX_{\text{new}}^\sT\bP_{\bF_{0;t}}^{\perp}\bG\bP_{\bF_{0;t}}^{\perp}\bX_{\text{new}}\bU.
\end{align}
Substituting this identity into Eq.~\eqref{eq:spike_22_block} and using the fact $\bP_{\bF_{0;t}}^{\perp} = \bP_{\bF_{0;t}}^{\perp}
\bO\bO^{\sT}$,
\begin{equation}\label{eq:spike_22_final}
\bB_{\sf{lower}}=\bZ^\sT\bG\bP_{\bF_{0;t}}^{\perp}\bO\bR\big(\bR^\sT\widehat{\bG}\bR - z\bI\big)^{-1}\bR^\sT\bO^\sT\bP_{\bF_{0;t}}^{\perp}\bG\bZ.
\end{equation}
We regard this as a function of the random matrix $\bR$ while conditioning on $(\bX,\boldsymbol{\varepsilon})$.

Denote
\begin{equation}\label{eq:conc_dtilde}
d_{\star}:= d - k - m - mt,
\end{equation}
the effective dimension after removing the conditioning subspace.
For the Stieltjes transform $\alpha_t^j(\cdot)$ of distribution $\mu_\infty^j(t)$, when the context is clear, we drop the time subscript $t$ and neuron index $j$, just denoting it by $\alpha(\cdot)$. 

The next lemma is the main result of this subsection.
We establish concentration of the resolvent bilinear form $\ba^\sT\bR^\sT\left(\bR\widehat{\bG}\bR^\sT - z\bI\right)^{-1}\bR\bb$ around its deterministic limit for random vectors $\ba, \bb$ that are measureable in $\cF_t$ (thus independent of $\bR$). Throughout, all statements and expectations are taken conditioning on $\{\ba, \bb, \widehat{\bG}\}$.
We note that the following result is not surprising. For instance, if $\widehat\bG$ is positive definite, the lemma can be derived via using its companion resolvent with the results from \cite{knowles2016anisotropiclocallawsrandom}.

\begin{lemma}\label{lemma:master_concentration}
    Consider $\ba, \bb \in \mathbb{R}^n$ such that $\|\ba\|, \|\bb\| = O_P(1/\sqrt{n})$ and are measurable in $\cF_t$. 
    
    Then for $z\in \cD$, we have 
    \begin{equation}\label{eq:conc_master_result}
    \abs{\ba^\sT\bR\left(\bR^\sT\widehat{\bG}\bR - z\bI\right)^{-1}\bR^\sT\bb
    - \sum_{i=1}^n a_i b_i \frac{d_{\star} \alpha(z)}{1+d_{\star} \widehat{g}_i\alpha(z)}}
    \xrightarrow{~p~} 0.
    \end{equation}
    where $\alpha(z)$ is the Stieltjes transform of distribution $\mu_\infty$ (cf.\ Lemma~\ref{lemma:bulk}).
    \end{lemma}

\par\addvspace{6pt}\noindent\textit{Proof of Lemma \ref{lemma:master_concentration}.}
We first state some basic facts about the spectrum of $\bR^\sT\widehat{\bG}\bR$.

\begin{lemma}\label{lem:spike_hatmu_conv}
Denote by $\widehat{\mu}_{n,d}(t)$ the empirical spectral distribution of $\bR^\sT\widehat{\bG}\bR$. Then
\begin{equation}\label{eq:spike_hatmu_conv}
\widehat{\mu}_{n,d}(t) \Rightarrow \mu_{\infty}(t)
\quad\text{in probability},
\end{equation}
where $\mu_{\infty}(t)$ is the limiting measure from Lemma~\ref{lemma:bulk}.
\end{lemma}

\begin{proof}
Recall from Eq.~\eqref{eq:spike_22_R_form} that \begin{equation}\label{eq:spike_hatmu_conv_proof_1}
\bR^\sT\widehat{\bG}\bR = \bU^\sT\bX_{\text{new}}^\sT\bP_{\bF_{0;t}}^{\perp}\bG\bP_{\bF_{0;t}}^{\perp}\bX_{\text{new}}\bU.
\end{equation} We compare this with $\bX_{\text{new}}^\sT\bG\bX_{\text{new}}$, whose limiting spectral distribution is $\mu_{\infty}(t)$ by Lemma~\ref{lemma:bulk}. 

Note that $\bP_{\bF_{0;t}}^{\perp}\bG\bP_{\bF_{0;t}}^{\perp} - \bG$ has rank at most $O(1)$ since $\bF_{0;t}\in\reals^{n\times mt}$. By Cauchy interlacing theorem (Lemma~\ref{lem:interlacing}), this rank-$O(1)$ perturbation does not affect the limiting spectral distribution. Similarly, the rotation by $\bU$ preserves eigenvalues. Hence the ESD of $\bR^\sT\widehat{\bG}\bR$ converges to $\mu_{\infty}(t)$.
\end{proof}

In addition, applying Lemma~\ref{lemma:smallest_eigenvalue_concentration} to $\bR^\sT\widehat{\bG}\bR$, we obtain $\lambda_{\min}(\bR^\sT\widehat{\bG}\bR) \xrightarrow{~p~} c(t)$.

\paragraph{Step I. Regularization.}

For convenience, we consider an arbitrary compact subset $\cK\subset\cD$ in this subsection and always assume that $z\in\cK$. For any such $\cK$, there exists a  positive distance 
\begin{equation}\label{eq:conc_epsilon_K}
\varepsilon_{\cK} := \frac{1}{2}\inf_{z\in \cK}\big(\min(c(t),0) - \Re(z)\big) > 0
\end{equation}
to the left edge. 
Without loss of generality, we assume that $\Im(z) \ge 0$; the case $\Im(z) < 0$ follows by a similar reasoning.
We further introduce an imaginary pertubation for technical convenience. Define the regularization parameter 
\begin{equation}\label{eq:conc_eta_n}
\eta_n := \frac{1}{\log n}.
\end{equation}
We emphasize that the concrete choice of $\eta_n$ does not matter as long as it converges to zero slowly enough.
The following lemma shows that the regularization introduces negligible error.

\begin{lemma}\label{lemma:distortion_trick}
For any $z\in\cK$ we have 
\begin{equation}\label{eq:conc_distortion_trick}
\abs{
\ba^\sT\bR\left(\bR^\sT\widehat{\bG}\bR - z\bI\right)^{-1} \bR^\sT\bb
-
\ba^\sT\bR\left(\bR^\sT\widehat{\bG}\bR - (z + i\eta_n)\bI\right)^{-1} \bR^\sT\bb
}
\xrightarrow{~p~} 0,
\end{equation}
for any sequence $\eta_n \to 0^+$.
\end{lemma}
This lemma is proven in Section~\ref{subsec:proof_distortion_trick}.

Denote the regularized resolvent
\begin{equation}\label{eq:conc_Q_def}
\bQ := \bQ(z+i\eta_n) = \left(\bR^\sT\widehat{\bG}\bR - (z+i\eta_n)\bI\right)^{-1} \in \mathbb{C}^{d_{\star}\times d_{\star}}.
\end{equation}
Consequently, we work with the regularized resolvent $\bQ$ in this subsection, and remove the regularization at the end.

\paragraph{Step II. Diagonal-cross decomposition and leave-one-out.}
We now decompose the bilinear form into diagonal and off-diagonal parts. 

Let $\br_i \in \mathbb{R}^{d_{\star}}$ denote the $i$-th row of $\bR$. Then we have
\begin{equation}\label{eq:conc_bilinear_expand}
\ba^\sT\bR\bQ\bR^\sT\bb
= \sum_{i=1}^n \sum_{j=1}^n a_i b_j \br_i^\sT \bQ \br_j
= \sum_{i=1}^n a_i b_i \br_i^\sT\bQ\br_i
+ \sum_{i\neq j} a_i b_j \br_i^\sT\bQ\br_j
:= S_{\mathsf{diag}} + S_{\mathsf{cross}}.
\end{equation}
The diagonal term $S_{\mathsf{diag}}$ will be shown to concentrate around its deterministic limit, whereas the cross term $S_{\mathsf{cross}}$ will be proved to vanish.

To handle both terms, we employ the leave-one-out decomposition. For each $i\in[n]$, define the leave-one-out resolvent
\begin{equation}\label{eq:conc_Q_minus_i}
\bQ_{-i}
:= \left(\bR^\sT\widehat{\bG}_{-i}\bR - (z + i\eta_n)\bI\right)^{-1},
\quad\text{where}\quad
\widehat{\bG}_{-i} := \widehat{\bG} - \widehat{g}_i \be_i \be_i^\sT
\end{equation}
and $\widehat{g}_i$ is the $i$-th diagonal entry of $\widehat{\bG}$.
By the Sherman-Morrison formula applied to the rank-one perturbation $\bR^\sT\widehat{\bG}\bR = \bR^\sT\widehat{\bG}_{-i}\bR + \widehat{g}_i\br_i\br_i^\sT$, we have
\begin{equation}\label{eq:conc_sherman_morrison}
\bQ
= \bQ_{-i}
- \frac{\widehat{g}_i\bQ_{-i}\br_i\br_i^\sT\bQ_{-i}}{1 + \widehat{g}_i\br_i^\sT\bQ_{-i}\br_i}.
\end{equation}
Applying \eqref{eq:conc_sherman_morrison} to the bilinear form $\br_i^\sT \bQ \br_j$, we obtain
\begin{equation}\label{eq:conc_cross_form}
\br_i^\sT \bQ \br_j
= \frac{\br_i^\sT \bQ_{-i} \br_j}{1 + \widehat{g}_i \br_i^\sT \bQ_{-i} \br_i}
\quad\text{for all }i,j\in [n].
\end{equation}

\paragraph{Step III. Cross-term vanishing.}
Define for each $i\in[n]$
\begin{equation}\label{eq:conc_T_i_def}
T_i := T_i(z+i\eta_n)
:= \frac{1}{n^2}\sum_{j\neq i} h_{ij}\br_i^\sT \bQ\br_j,
\quad\text{where}\quad
h_{ij} := n^2 a_i b_j \in \mathbb{R}.
\end{equation}
Then $S_{\mathsf{cross}} = \sum_{i=1}^n T_i$. To bound this, we estimate the conditional second moment $\mathbb{E}\left[\abs{S_{\mathsf{cross}}}^2\,\middle|\,\cF_t\right]$.

\begin{lemma}\label{lemma:estimation_cross_term_variance}
We have
\begin{equation}\label{eq:conc_cross_variance}
\mathbb{E}\left[\abs{S_{\mathsf{cross}}}^2\,\middle|\,\cF_t\right] 
\le \frac{2d_{\star} M_n C^2}{n^2\eta_n^4},
\quad\text{where}\quad
M_n := \frac{1}{n^2}\sum_{i,j} h_{ij}^2
= n^2\|\ba\|^2\|\bb\|^2.
\end{equation}
\end{lemma}
The proof of this statement is given Section~\ref{subsec:proof_cross_term_variance}.

Since $\|\ba\|, \|\bb\| = O_P(1/\sqrt{n})$ by our assumption, we have $M_n = O_P(1)$. Combining with $\eta_n = 1/\log n$ and $d_{\star} = O(n)$, the bound \eqref{eq:conc_cross_variance} yields
\begin{equation}\label{eq:conc_cross_vanish}
\mathbb{E}\left[\abs{S_{\mathsf{cross}}}^2\,\middle|\,\cF_t\right]
\le \frac{2d_{\star} M_n C^2}{n^2\eta_n^4}
= O_P\!\left(\frac{n \cdot (\log n)^4}{n^2}\right)
= O_P\!\left(\frac{(\log n)^4}{n}\right)
= o_P(1),
\end{equation}
and hence by Markov's inequality $S_{\mathsf{cross}} \xrightarrow{~p~} 0$.

\paragraph{Step IV. Diagonal-term concentration.}

To this end, we first prove the concentration of each diagonal term around its deterministic limit.
\begin{lemma}\label{lemma:concentration_quadratic_leave_out}
    For any $z\in\cK$, the following holds
    \begin{equation}\label{eq:conc_trace_local_law}
    \sup_{i\in[n]} \abs{(1/d_{\star})\br_i^\sT\bQ_{-i}\br_i - \alpha(z+i\eta_n)}
    \lesssim \frac{(\log n)^{19/2}}{\sqrt{n}}
    =: D_n,
    \end{equation}
    with probability at least $1-d_{\star}^{-C}$.
    \end{lemma}
    \begin{proof}
        We establish uniform concentration of $(1/d_{\star})\br_i^\sT\bQ_{-i}\br_i$ around the Stieltjes transform $\alpha(z+i\eta_n)$ via a two-step decomposition
        \begin{equation}\label{eq:concentration_decomposition}
        \begin{aligned}
        &\abs{(1/d_{\star})\br_i^\sT\bQ_{-i}\br_i - \alpha(z+i\eta_n)}
        \le \\
        &\qquad\qquad\abs{(1/d_{\star})\br_i^\sT\bQ_{-i}\br_i - (1/d_{\star})\operatorname{Tr}(\bQ_{-i})}
        +
        \abs{(1/d_{\star})\operatorname{Tr}(\bQ_{-i}) - \alpha(z+i\eta_n)}.
        \end{aligned}
        \end{equation}
        The first term is controlled via the Hanson-Wright concentration bound (Lemma~\ref{lemma:hanson_wright_statement}). Since $\br_i$ is independent of $\bQ_{-i}$,
using \eqref{eq:conc_hanson_wright} with $\bM = \bQ_{-i}$, and using the resolvent bounds $\|\bQ_{-i}\|_{\mathrm{op}} \le (1/\eta_n)$ and $\|\bQ_{-i}\|_F \le \sqrt{d_{\star}}/\eta_n$, we set $t = \alpha_n\sqrt{d_{\star}}/\eta_n$ with $\alpha_n = C\log n$ to obtain
\begin{equation}\label{eq:conc_quadratic_conc}
\mathbb{P}\!\left(\abs{\br_i^\sT\bQ_{-i}\br_i - \operatorname{Tr}(\bQ_{-i})} \ge \frac{C(\log n)\sqrt{d_{\star}}}{\eta_n}\right)
\le 2e^{-c(\log n)^2}
= o(n^{-C}).
\end{equation}
for some large enough universal constant $C>0$. Via union bounds, we have for all $i\in[n]$
        \begin{equation}\label{eq:hanson_wright_normalized}
        \abs{(1/d_{\star})\br_i^\sT \bQ_{-i}\br_i - (1/d_{\star})\operatorname{Tr}(\bQ_{-i})}
        \lesssim \frac{(\log n)^2}{\sqrt{n}}
        \end{equation}
with probability at least $1-O(n^{-C})$.

        For the second term, we establish concentration of the normalized trace $(1/d_{\star})\operatorname{Tr}(\bQ_{-i})$ around the Stieltjes transform uniformly over all $i \in [n]$ via invoking Lemma~\ref{lemma:local_law_adapted}.
        We specialize to our setting with $\eta_n = 1/\log n$ and choose $L$ a sufficiently large constant. Since $|z| \le C$ for $z \in \cK$, the error quantity \eqref{eq:error_quantity} becomes
        \begin{equation}\label{eq:error_specialized}
        \operatorname{E}_n
        = C'
        \frac{1+(\abs{z}^2+\eta_n^2)^2}{\eta_n^4}
        \left(
        L\sqrt{\frac{\log d_{\star}}{n}} + \frac{1}{n\eta_n}
        \right)
        \lesssim
        (\log n)^4 \cdot \left(\sqrt{\frac{\log n}{n}} + \frac{\log n}{n}\right)
        \lesssim \frac{(\log n)^{9/2}}{\sqrt{n}}.
        \end{equation}
        For $n$ sufficiently large, the condition \eqref{eq:technical_assumption} is automatically satisfied due to $\operatorname{E}_n = o(1)$. 
        Using Lemma~\ref{lemma:local_law_adapted} with probability $1-O(d_{\star}^{-C})$, we obtain for all $i\in[n]$
        \begin{equation}\label{eq:trace_deviation}
        \abs{(1/d_{\star})\operatorname{Tr}(\bQ_{-i}) - \alpha(z+i\eta_n)}
         \lesssim
        \frac{1+(\abs{z}^2+\eta_n^2)^2}{\eta_n^5}
        \operatorname{E}_n
        \lesssim \frac{(\log n)^4}{(\log n)^{-5}} \cdot \frac{(\log n)^{9/2}}{\sqrt{n}}
        \lesssim \frac{(\log n)^{19/2}}{\sqrt{n}}
        = D_n.
        \end{equation}
        By union bounds, this holds uniformly over $i\in[n]$ with probability at least $1-O(d_{\star}^{-C})$.
        Combining \eqref{eq:hanson_wright_normalized} and \eqref{eq:trace_deviation} via the triangle inequality decomposition \eqref{eq:concentration_decomposition}, we obtain
        \begin{equation}\label{eq:final_concentration}
        \abs{(1/d_{\star})\br_i^\sT \bQ_{-i}\br_i - \alpha(z+i\eta_n)}
        \lesssim \frac{(\log n)^2}{\sqrt{n}} + \frac{(\log n)^{19/2}}{\sqrt{n}}
        \lesssim \frac{(\log n)^{19/2}}{\sqrt{n}}
        = D_n
        \end{equation}
        uniformly over all $i\in[n]$ with probability at least $1-O(d_{\star}^{-C})$. This establishes the stated result.
        \end{proof}

We now combine the concentration bounds.
Using the leave-one-out identity \eqref{eq:conc_cross_form}, we write
\begin{equation}\label{eq:conc_diag_decompose}
\br_i^\sT\bQ\br_i
= \frac{\br_i^\sT\bQ_{-i}\br_i}{1 + \widehat{g}_i\br_i^\sT\bQ_{-i}\br_i}.
\end{equation}
Subtracting the target quantity $\tfrac{d_{\star} \alpha}{1+d_{\star} \widehat{g}_i\alpha}$, we obtain
\begin{equation}\label{eq:conc_diag_error}
\begin{aligned}
&\abs{
\frac{(1/d_{\star})\br_i^\sT\bQ_{-i}\br_i}{1 + \widehat{g}_i\br_i^\sT\bQ_{-i}\br_i}
-
\frac{\alpha}{1 + d_{\star} \widehat{g}_i\alpha}
}
\\&\quad=
\abs{
\frac{
(1/d_{\star})\br_i^\sT\bQ_{-i}\br_i\left(1+d_{\star} \widehat{g}_i\alpha\right)
- \alpha\left(1+\widehat{g}_i\br_i^\sT\bQ_{-i}\br_i\right)
}{
\left(1 + \widehat{g}_i\br_i^\sT\bQ_{-i}\br_i\right)
\left(1 + d_{\star} \widehat{g}_i\alpha\right)
}
}
&&\text{(algebraic expansion)}
\\&\quad\le
\frac{1}{\abs{1+\widehat{g}_i\br_i^\sT\bQ_{-i}\br_i}}
\frac{1}{\abs{1+d_{\star} \widehat{g}_i\alpha}}
\Big(
\abs{(1/d_{\star})\br_i^\sT\bQ_{-i}\br_i - \alpha}
\abs{1+d_{\star} \widehat{g}_i\alpha}
&&\text{(triangle ineq.)}
\\&\qquad\qquad\qquad\qquad\qquad\qquad\qquad\qquad\qquad\qquad
+ \abs{\alpha}
\abs{d_{\star} \widehat{g}_i\alpha - \widehat{g}_i\br_i^\sT\bQ_{-i}\br_i}
\Big).
\end{aligned}
\end{equation}

To control denominators, we invoke the following uniform bounds.

\begin{lemma}[Denominator Bounds]\label{lem:bounded-denominators}
For all $z\in\cK$, $\eta\in (0,1)$, there exists an absolute constant $C_{\varepsilon_{\cK}}>0$ such that
\begin{equation}\label{eq:conc_denom_bounds}
\abs{1 + d_{\star}\widehat{g}_i\alpha(z+i\eta)} \ge \frac{1}{C_{\varepsilon_{\cK}}},
\quad
\abs{1 + \widehat{g}_i\br_i^\sT\bQ_{-i}(z+i\eta)\br_i} \ge \frac{\eta}{C}
\end{equation}
uniformly and deterministically over all $i\in[n]$.
\end{lemma}

The first bound follows from the gap $\varepsilon_{\cK}$ between $z$ and the edge $c(t)$. The second bound is proven in Section~\ref{subsec:proof_denominator_bound}.

Substituting the bounds from Lemma~\ref{lem:bounded-denominators} and Lemma~\ref{lemma:concentration_quadratic_leave_out} into \eqref{eq:conc_diag_error}, and using $|\alpha| \le C$ (since we are working in the compact subset $\cK$), $|\widehat{g}_i| \le C/n$, $d_{\star} = O(n)$, we obtain
\begin{equation}\label{eq:conc_diag_pointwise}
\abs{
\frac{(1/d_{\star})\br_i^\sT\bQ_{-i}\br_i}{1 + \widehat{g}_i\br_i^\sT\bQ_{-i}\br_i}
-
\frac{\alpha}{1 + d_{\star}\widehat{g}_i\alpha}
}
\le \frac{C^4 C_{\varepsilon_{\cK}}}{\eta_n}D_n
= O_P\!\left(\frac{(\log n)^{21/2}}{\sqrt{n}}\right)
= o_P(1).
\end{equation}
Multiplying by $d_{\star}$ and summing over $i\in[n]$ with weights $a_i b_i$, and applying the triangle inequality, we obtain the following bound
\begin{equation}\label{eq:conc_diag_final}
\begin{aligned}
&\abs{\sum_{i=1}^n a_i b_i \br_i^\sT\bQ\br_i
- \sum_{i=1}^n a_i b_i \frac{d_{\star} \alpha(z+i\eta_n)}{1+d_{\star} \widehat{g}_i\alpha(z+i\eta_n)}}
\\&\quad\le
\sum_{i=1}^n |a_i b_i| \cdot \frac{C^4 C_{\varepsilon_{\cK}} d_{\star} D_n}{\eta_n}
&&
\\&\quad\le
\|\ba\|\|\bb\| \cdot \frac{C^4 C_{\varepsilon_{\cK}} d_{\star} D_n}{\eta_n}
&&\text{(Cauchy-Schwarz)}
\\&\quad=
O_P\!\left(\frac{1}{n} \cdot \frac{n \cdot \frac{(\log n)^{19/2}}{\sqrt{n}}}{\frac{1}{\log n}}\right)
= O_P\!\left(\frac{(\log n)^{21/2}}{\sqrt{n}}\right)
= o_P(1).
&&\text{($\|\ba\|\|\bb\|=O_P(1/n)$)}
\end{aligned}
\end{equation}

Last, we remove the distortion $i\eta_n$ in the resolvent. Using Lipschitz continuity of $\alpha(\cdot)$ in $\cK$ (standard property for Stieltjes transforms), we have
\begin{equation}\label{eq:conc_lipschitz}
\abs{\alpha(z+i\eta_n) - \alpha(z)}
\le C_{\varepsilon_{\cK}}\eta_n
\end{equation}
for all $z\in\cK$. Repeating the denominator bounds and algebraic expansions as in \eqref{eq:conc_diag_error}, we obtain
\begin{equation}\label{eq:conc_eta_remove}
\begin{aligned}
&\abs{
\sum_{i=1}^n a_i b_i \frac{d_{\star}\alpha(z+i\eta_n)}{1+d_{\star}\widehat{g}_i\alpha(z+i\eta_n)}
- \sum_{i=1}^n a_i b_i \frac{d_{\star}\alpha(z)}{1+d_{\star}\widehat{g}_i\alpha(z)}
}
\\&\quad\le
d_{\star}\!\left(\sum_{i=1}^n |a_i b_i|\right)
C_{\varepsilon_{\cK}}^2
\abs{\alpha(z+i\eta_n)(1+d_{\star} \widehat{g}_i\alpha(z))
- \alpha(z)(1+d_{\star} \widehat{g}_i\alpha(z+i\eta_n))}
&&\text{(denom. bounds)}
\\&\quad\le
C^2C_{\varepsilon_{\cK}}^3 \eta_n d_{\star}\|\ba\|\|\bb\|
&&\text{(Lipschitz \eqref{eq:conc_lipschitz})}
\\&\quad=
O_P\!\left(\frac{1}{\log n} \cdot n \cdot \frac{1}{n}\right)
= O_P\!\left(\frac{1}{\log n}\right)
= o_P(1).
&&\text{($\|\ba\|\|\bb\| = O_P(1/n)$)}
\end{aligned}
\end{equation}
Combining \eqref{eq:conc_diag_final} and \eqref{eq:conc_eta_remove} yields the desired diagonal-term concentration
\begin{equation}\label{eq:conc_diag_unregularized}
\abs{\sum_{i=1}^n a_i b_i \br_i^\sT\bQ\br_i
- \sum_{i=1}^n a_i b_i \frac{d_{\star} \alpha(z)}{1+d_{\star} \widehat{g}_i\alpha(z)}}
\xrightarrow{~p~} 0.
\end{equation}

Combining the cross-term vanishing \eqref{eq:conc_cross_vanish} and the diagonal-term concentration \eqref{eq:conc_diag_unregularized}, and also Lemma~\ref{lemma:distortion_trick} to remove the regularization, we obtain the main result (Lemma \ref{lemma:master_concentration}) of this subsection. 
\hfill\qed\par\addvspace{6pt}

\subsection{Convergence of the outlier matrix}\label{subsec:convergence}

This subsection completes the proof of Lemma~\ref{lem:B_lower_convergence} by establishing the convergence of the intermediate form obtained from Lemma~\ref{lemma:master_concentration} (Section~\ref{subsec:concentration}).

We continue with the notation from Sections~\ref{subsec:spike_basic} and~\ref{subsec:concentration}. 
Recall from Eq. \eqref{eq:spike_22_final} that
\begin{equation}\label{eq:conv_hard_core}
    \bB_{\sf{lower}}=\bZ^\sT\bG\bP_{\bF_{0;t}}^{\perp}\bO\bR\left(\bR^\sT\widehat{\bG}\bR - z\bI\right)^{-1}\bR^\sT\bO^\sT\bP_{\bF_{0;t}}^{\perp}\bG\bZ,
\end{equation}
and define $\widetilde f(s;z):=\frac{d_{\star}\alpha(z)}{1+(d_{\star}/n)s\alpha(z)}$. 
Lemma \ref{lemma:master_concentration} implies that conditioning on $(\bX, \boldsymbol{\varepsilon})$, we have
\begin{equation}\label{eq:conv_first_concentration}
    \abs{\bZ^\sT\bG\bP_{\bF_{0;t}}^{\perp}\bO\bR\left(\bR^\sT\widehat{\bG}\bR - z\bI\right)^{-1}\bR^\sT\bO^\sT\bP_{\bF_{0;t}}^{\perp}\bG\bZ - \bZ^\sT\bG\bP_{\bF_{0;t}}^{\perp}\bO \widetilde f\left(n\widehat{\bG};z\right) \bO^\sT\bP_{\bF_{0;t}}^{\perp}\bG\bZ} \xrightarrow{~p~} 0.
\end{equation}

The main result of this subsection is the following convergence of the second term inside the absolute value in Eq. \eqref{eq:conv_first_concentration} to the deterministic limit $B_{\sf{lower}}^\infty(z;t)$ defined in Lemma~\ref{lem:B_lower_convergence}.
\begin{lemma}\label{lem:intermediate_to_Blower}
For $z\in \mathcal{D}$, we have
\begin{equation}\label{eq:conv_intermediate_limit}
\bZ^\sT\bG\bP_{\bF_{0;t}}^{\perp}\bO \widetilde f\left(n\widehat{\bG};z\right) \bO^\sT\bP_{\bF_{0;t}}^{\perp}\bG\bZ \xrightarrow{~p~} B_{\sf{lower}}^\infty(z;t),
\end{equation}
where $B_{\sf{lower}}^\infty(z;t)$ is defined in \eqref{eq:B_lower_limit}.
\end{lemma}
The proof of Lemma~\ref{lem:intermediate_to_Blower} occupies the remainder of this subsection. Combining \eqref{eq:conv_first_concentration} and Lemma~\ref{lem:intermediate_to_Blower} immediately yields Lemma~\ref{lem:B_lower_convergence}.

\par\addvspace{6pt}\noindent\textit{Proof of Lemma \ref{lem:intermediate_to_Blower}.}
Throughout this subsection, for convenience, we work with $z$ in an arbitrary fixed compact subset $\cK \subset \cD$ as in Section~\ref{subsec:concentration}. Constants and bounds may depend on $\cK$, but are uniform over $z \in \cK$ and $n, d$ sufficiently large.
We proceed in two steps. In Step~I we reduce the intermediate form to a simpler one leveraging the specific form of $\widetilde f(s;z)$. In Step~II we establish convergence of the resulting expression via the discrete-time DMFT lemmas (Lemmas~\ref{lemma_amp_data}, \ref{lemma_amp_feature}).

\paragraph{Step I. Replacement of $\widetilde f(s;z)$ with $f(s;z)$.}

Define
\begin{equation}\label{eq:conv_f_function}
f(s;z) := \frac{\delta\alpha(z)}{\delta + s\alpha(z)}.
\end{equation}
Recall that for $\ba, \bb$, we have the following representation
\begin{equation}\label{eq:conv_f_diag_form}
\sum_{i=1}^n a_i b_i f(n\widehat{g}_i;z) = \ba^\sT f\left(n\widehat{\bG};z\right)\bb.
\end{equation} To relate the left-hand side term in \eqref{eq:conv_intermediate_limit} to \eqref{eq:conv_f_diag_form}, we approximate $d_{\star}$ by $n/\delta$ as follows.

\begin{lemma}\label{lemma:dtilde_approximation}
Let $\ba, \bb \in \mathbb{R}^n$ with $\|\ba\|, \|\bb\| = O_P(1/\sqrt{n})$. Then
\begin{equation}\label{eq:conv_dtilde_d_approx}
\sum_{i=1}^n a_i b_i \frac{\alpha(z)}{1 + d_{\star}\widehat{g}_i\alpha(z)}
= \ba^\sT f\left(n\widehat{\bG};z\right)\bb + o_P(1/d).
\end{equation}
\end{lemma}
We prove this technical lemma in Section~\ref{subsec:proof_conv_dtilde_d_approx}. 
With this lemma, we can work on the form $d\cdot\bZ^\sT\bG\bP_{\bF_{0;t}}^{\perp}\bO f\left(n\widehat{\bG};z\right) \bO^\sT\bP_{\bF_{0;t}}^{\perp}\bG\bZ$ instead of the more complicated one $\bZ^\sT\bG\bP_{\bF_{0;t}}^{\perp}\bO \widetilde f\left(n\widehat{\bG};z\right) \bO^\sT\bP_{\bF_{0;t}}^{\perp}\bG\bZ$ from now on, since the error is $o_P(1)$.

Define
\begin{equation}\label{eq:conv_sz_def}
    s_z := -\delta/\alpha(z).
\end{equation}
Since $f(s;z) = \delta\alpha(z)\cdot\left(\delta+s\alpha(z)\right)^{-1} = -\delta(s_z - s)^{-1}$, we obtain
\begin{equation}\label{eq:conv_func_calc_applied}
    f\left(n\widehat{\bG};z\right) = -\delta\left(s_z\bI - n\widehat{\bG}\right)^{-1}.
\end{equation}
Define the rescaled matrix $\widetilde{\bG} := n\bG = \operatorname{diag}(g_1, \dots, g_n)$, so that $\widetilde{\bG}$ has $O(1)$ diagonal entries. 
Recall $\bZ = \bX\bTheta_{*;t}(\bTheta_{*;t}^\sT\bTheta_{*;t})^{-1}$. Expanding $\bZ$ and using $\bG = \widetilde{\bG}/n$, we have
\begin{equation}\label{eq:conv_hard_rescaled}
\begin{aligned}
&d\cdot\bZ^\sT\bG\bP_{\bF_{0;t}}^{\perp}\bO f\left(n\widehat{\bG};z\right) \bO^\sT\bP_{\bF_{0;t}}^{\perp}\bG\bZ \\
&\quad= \left(\bTheta_{*;t}^\sT\bTheta_{*;t}\right)^{-1} \cdot \frac{d}{n^2}\bTheta_{*;t}^\sT\bX^\sT\widetilde{\bG}\bP_{\bF_{0;t}}^{\perp}\bO f\left(n\widehat{\bG};z\right)\bO^\sT\bP_{\bF_{0;t}}^{\perp}\widetilde{\bG}\bX\bTheta_{*;t} \cdot \left(\bTheta_{*;t}^\sT\bTheta_{*;t}\right)^{-1}.
\end{aligned}
\end{equation}
Since $\plim_{n,d\to\infty}\left(\bTheta_{*;t}^\sT\bTheta_{*;t}\right)^{-1} = \left(\mathbb{E}[O_tO_t^\sT]\right)^{-1}$, we can focus on the core term
\begin{equation}\label{eq:conv_core_sandwich_term}
\bB_{\mathsf{core}} := \frac{d}{n^2}\bTheta_{*;t}^\sT\bX^\sT\widetilde{\bG}\bP_{\bF_{0;t}}^{\perp}\bO f\left(n\widehat{\bG};z\right)\bO^\sT\bP_{\bF_{0;t}}^{\perp}\widetilde{\bG}\bX\bTheta_{*;t}.
\end{equation}
and add the $\left(\mathbb{E}[O_tO_t^\sT]\right)^{-1}$ back in the end.
By \eqref{eq:spike_spectral_decomp}, $\bO\widehat{\bG}\bO^\sT = \bP_{\bF_{0;t}}^{\perp}\bG\bP_{\bF_{0;t}}^{\perp}$, so
\begin{equation}\label{eq:conv_conjugation_identity}
\bO f\left(n\widehat{\bG};z\right)\bO^\sT = f\left(n\bP_{\bF_{0;t}}^{\perp}\bG\bP_{\bF_{0;t}}^{\perp};z\right) = f\left(\bP_{\bF_{0;t}}^{\perp}\widetilde{\bG}\bP_{\bF_{0;t}}^{\perp};z\right).
\end{equation}
Combining \eqref{eq:conv_conjugation_identity} with the identity \eqref{eq:conv_func_calc_applied}, using $d/n \to 1/\delta$, we obtain
\begin{equation}\label{eq:conv_target_simplified}
\bB_{\mathsf{core}}
= -\frac{1}{n}\bTheta_{*;t}^\sT\bX^\sT\widetilde{\bG}\bP_{\bF_{0;t}}^{\perp}\left(s_z\bI - \bP_{\bF_{0;t}}^{\perp}\widetilde{\bG}\bP_{\bF_{0;t}}^{\perp}\right)^{-1}\bP_{\bF_{0;t}}^{\perp}\widetilde{\bG}\bX\bTheta_{*;t} + o_P(1).
\end{equation}
Therefore it suffices to analyze the asymptotics of the right-hand side of \eqref{eq:conv_target_simplified}.

\paragraph{Step II. Convergence at $s=s_z$.}

We analyze the expression
\begin{equation}\label{eq:conv_integrand_def}
\bK_n(s;z)
:= \frac{1}{n}\bTheta_{*;t}^\sT\bX^\sT\widetilde{\bG}\bP_{\bF_{0;t}}^{\perp}\left(s\bI - \bP_{\bF_{0;t}}^{\perp}\widetilde{\bG}\bP_{\bF_{0;t}}^{\perp}\right)^{-1}\bP_{\bF_{0;t}}^{\perp}\widetilde{\bG}\bX\bTheta_{*;t}.
\end{equation}
\begin{lemma}\label{lem:pointwise_convergence}
For $s=s_z$, we have
\begin{equation}\label{eq:conv_pointwise_limit}
\bK_n(s_z;z)
\xrightarrow{~p~}
K_\infty(s_z;z)
:= \begin{bmatrix}
\mathbb{E}\left[\dfrac{G_t^2}{s_z - G_t} V_{*\leq}V_{*\leq}^\sT\right] & 0 \\[6pt]
0 & C_\infty(s_z;z)
\end{bmatrix} \in \mathbb{R}^{(k+m+mt) \times (k+m+mt)},
\end{equation}
where $C_\infty(s_z;z) \in \mathbb{R}^{(k-r+m+mt) \times (k-r+m+mt)}$ is a deterministic matrix depending on expectations over certain functions of the DMFT random variables.
\end{lemma}

The remainder of this step is devoted to the proof of Lemma~\ref{lem:pointwise_convergence}.
\par\addvspace{6pt}\noindent\textit{Proof of Lemma \ref{lem:pointwise_convergence}.}
 Throughout this proof, we write $s=s_z$ and define
\begin{equation}\label{eq:conv_def_B}
\bB := s_z\bI - \widetilde{\bG}.
\end{equation}
By Lemma~\ref{lem:bounded-denominators}, which provides a uniform lower bound on $|1 + d_{\star}\widehat{g}_i\alpha(z)|$, we have $\|\bB^{-1}\|_{\mathrm{op}} = O(1)$.

By the Woodbury matrix identity, we have
\begin{equation}\label{eq:conv_woodbury}
\left(s\bI - \bP_{\bF_{0;t}}^{\perp}\widetilde{\bG}\bP_{\bF_{0;t}}^{\perp}\right)^{-1}
= \bB^{-1} - \bB^{-1}\bU\big(\bI + \bK_F\bU^\sT\bB^{-1}\bU\big)^{-1}\bK_F\bU^\sT\bB^{-1},
\end{equation}
where
\begin{equation}\label{eq:conv_U_K}
    \begin{aligned}
&\bU := [\widetilde{\bG}\bF_{0;t},\ \bF_{0;t}] \in \mathbb{R}^{n \times 2mt},
\quad\\
&\bK_F := \begin{bmatrix}
0 & (\bF_{0;t}^\sT\bF_{0;t})^{-1} \\
(\bF_{0;t}^\sT\bF_{0;t})^{-1} & -(\bF_{0;t}^\sT\bF_{0;t})^{-1}(\bF_{0;t}^\sT\widetilde{\bG}\bF_{0;t})(\bF_{0;t}^\sT\bF_{0;t})^{-1}
\end{bmatrix} \in \mathbb{R}^{2mt \times 2mt}.
\end{aligned}
\end{equation}
Substituting \eqref{eq:conv_woodbury} into \eqref{eq:conv_integrand_def}, we obtain
\begin{equation}\label{eq:conv_two_terms}
\bK_n(s;z) = \bK_n^{(1)}(s;z) - \bK_n^{(2)}(s;z),
\end{equation}
where
\begin{equation}\label{eq:conv_K1_K2_def}
\begin{aligned}
\bK_n^{(1)}(s;z)
&:= \frac{1}{n}\bTheta_{*;t}^\sT\bX^\sT\widetilde{\bG}\bP_{\bF_{0;t}}^{\perp}\bB^{-1}\bP_{\bF_{0;t}}^{\perp}\widetilde{\bG}\bX\bTheta_{*;t}, \\
\bK_n^{(2)}(s;z)
&:= \frac{1}{n}\bTheta_{*;t}^\sT\bX^\sT\widetilde{\bG}\bP_{\bF_{0;t}}^{\perp}\bB^{-1}\bU\big(\bI + \bK_F\bU^\sT\bB^{-1}\bU\big)^{-1}\bK_F\bU^\sT\bB^{-1}\bP_{\bF_{0;t}}^{\perp}\widetilde{\bG}\bX\bTheta_{*;t}.
\end{aligned}
\end{equation}
Expand $\bK_n^{(1)}$ using $\bP_{\bF_{0;t}}^{\perp} = \bI - \bP_{\bF_{0;t}}$:
\begin{equation}\label{eq:conv_K1_expand}
\begin{aligned}
\bK_n^{(1)}(s;z)
&= \frac{1}{n}\bTheta_{*;t}^\sT\bX^\sT\widetilde{\bG}\left(\bI - \bP_{\bF_{0;t}}\right)\bB^{-1}\left(\bI - \bP_{\bF_{0;t}}\right)\widetilde{\bG}\bX\bTheta_{*;t} \\
&= \frac{1}{n}\bTheta_{*;t}^\sT\bX^\sT\widetilde{\bG}\bB^{-1}\widetilde{\bG}\bX\bTheta_{*;t}
- \frac{1}{n}\bTheta_{*;t}^\sT\bX^\sT\widetilde{\bG}\bP_{\bF_{0;t}}\bB^{-1}\widetilde{\bG}\bX\bTheta_{*;t} \\
&\quad - \frac{1}{n}\bTheta_{*;t}^\sT\bX^\sT\widetilde{\bG}\bB^{-1}\bP_{\bF_{0;t}}\widetilde{\bG}\bX\bTheta_{*;t}
+ \frac{1}{n}\bTheta_{*;t}^\sT\bX^\sT\widetilde{\bG}\bP_{\bF_{0;t}}\bB^{-1}\bP_{\bF_{0;t}}\widetilde{\bG}\bX\bTheta_{*;t} \\
&=: \bT_1 - \bT_2 - \bT_3 + \bT_4.
\end{aligned}
\end{equation}
Similarly, expand $\bK_n^{(2)}$ using $\bP_{\bF_{0;t}}^{\perp} = \bI - \bP_{\bF_{0;t}}$ on both sides:
\begin{equation}\label{eq:conv_K2_expand}
\begin{aligned}
\bK_n^{(2)}(s;z)
&= \frac{1}{n}\bTheta_{*;t}^\sT\bX^\sT\widetilde{\bG}\left(\bI - \bP_{\bF_{0;t}}\right)\bB^{-1}\bU\big(\bI + \bK_F\bU^\sT\bB^{-1}\bU\big)^{-1}\bK_F\bU^\sT\bB^{-1}\left(\bI - \bP_{\bF_{0;t}}\right)\widetilde{\bG}\bX\bTheta_{*;t} \\
&= \frac{1}{n}\bTheta_{*;t}^\sT\bX^\sT\widetilde{\bG}\bB^{-1}\bU\big(\bI + \bK_F\bU^\sT\bB^{-1}\bU\big)^{-1}\bK_F\bU^\sT\bB^{-1}\widetilde{\bG}\bX\bTheta_{*;t} \\
&\quad - \frac{1}{n}\bTheta_{*;t}^\sT\bX^\sT\widetilde{\bG}\bP_{\bF_{0;t}}\bB^{-1}\bU\big(\bI + \bK_F\bU^\sT\bB^{-1}\bU\big)^{-1}\bK_F\bU^\sT\bB^{-1}\widetilde{\bG}\bX\bTheta_{*;t} \\
&\quad - \frac{1}{n}\bTheta_{*;t}^\sT\bX^\sT\widetilde{\bG}\bB^{-1}\bU\big(\bI + \bK_F\bU^\sT\bB^{-1}\bU\big)^{-1}\bK_F\bU^\sT\bB^{-1}\bP_{\bF_{0;t}}\widetilde{\bG}\bX\bTheta_{*;t} \\
&\quad + \frac{1}{n}\bTheta_{*;t}^\sT\bX^\sT\widetilde{\bG}\bP_{\bF_{0;t}}\bB^{-1}\bU\big(\bI + \bK_F\bU^\sT\bB^{-1}\bU\big)^{-1}\bK_F\bU^\sT\bB^{-1}\bP_{\bF_{0;t}}\widetilde{\bG}\bX\bTheta_{*;t} \\
&=: \bT_5 - \bT_6 - \bT_7 + \bT_8.
\end{aligned}
\end{equation}
Thus, $\bK_n(s;z) = \sum_{\circ=1}^8 c_{\circ}\bT_{\circ}$ where $c_{\circ} \in \{-1, +1\}$ are appropriate signs.

The matrix $\bK_n(s;z)$ inherits a block structure.
We partition $\bTheta_{*;t} = [\bTheta_{*;t,\mathrm{H}} \ \bTheta_{*;t,\mathrm{R}}]$ where $\bTheta_{*;t,\mathrm{H}}=\bTheta_{*\mathrm{H}} \in \mathbb{R}^{d \times r}$ contains the first $r$ columns (hard directions $\btheta_{*1}, \dots, \btheta_{*r}$) and $\bTheta_{*;t,\mathrm{R}} \in \mathbb{R}^{d \times (k-r+m+mt)}$ contains the remaining. 
Following the structure of $\bTheta_{*;t}$, we write $\bK_n(s;z)$ as
\begin{equation}\label{eq:conv_block_structure}
\bK_n(s;z) = \begin{bmatrix}
[\bK_n]_{\mathrm{HH}} & [\bK_n]_{\mathrm{HR}} \\
[\bK_n]_{\mathrm{RH}} & [\bK_n]_{\mathrm{RR}}
\end{bmatrix},
\end{equation}
where we denote the hard-hard block as $[\bK_n]_{\mathrm{HH}} \in \mathbb{R}^{r \times r}$, the cross blocks as $[\bK_n]_{\mathrm{HR}}, [\bK_n]_{\mathrm{RH}}$, and the rest-rest block as $[\bK_n]_{\mathrm{RR}} \in \mathbb{R}^{(k-r+m+mt) \times (k-r+m+mt)}$.

We first focus on the hard-hard block and the cross blocks.
\begin{lemma}\label{lemma:hard_cross_vanishing}

 For any $\circ \in \{2, 3, \dots, 8\}$, we have
\begin{equation}\label{eq:conv_vanishing_terms}
[\bT_{\circ}]_{\mathrm{HH}} \xrightarrow{~p~} 0_{r \times r}, \quad
[\bT_{\circ}]_{\mathrm{HR}} \xrightarrow{~p~} 0_{r \times (k-r+m+mt)}, \quad
[\bT_{\circ}]_{\mathrm{RH}} \xrightarrow{~p~} 0_{(k-r+m+mt) \times r}.
\end{equation}
\end{lemma}
The proof is deferred to the end of this subsection.

For the term $\bT_1$, recall that $\bB = s\bI - \widetilde{\bG}$ is diagonal with entries $s - g_i$. Therefore, $\widetilde{\bG}\bB^{-1}\widetilde{\bG}$ is also diagonal with entries $g_i^2/(s - g_i)$ and we have
\begin{equation}\label{eq:conv_T1_full}
\bT_1 = \frac{1}{n}\bTheta_{*;t}^\sT\bX^\sT\widetilde{\bG}\bB^{-1}\widetilde{\bG}\bX\bTheta_{*;t}
= \frac{1}{n}\sum_{i=1}^n \frac{g_i^2}{s - g_i} (\bTheta_{*;t}^\sT\bx_i)(\bTheta_{*;t}^\sT\bx_i)^\sT.
\end{equation}
By Lemma~\ref{lemma_amp_data}, we have
\begin{equation}\label{eq:conv_T1_limit}
\bT_1 \xrightarrow{~p~} \mathbb{E}\left[\frac{G_t^2}{s - G_t} O_t O_t^\sT\right]
= \begin{bmatrix}
\mathbb{E}\left[\dfrac{G_t^2}{s - G_t} V_{*\leq}V_{*\leq}^\sT\right] & 0 \\[6pt]
0 & \mathbb{E}\left[\dfrac{G_t^2}{s - G_t} V_{*>:t}V_{*>:t}^\sT\right]
\end{bmatrix},
\end{equation}
where the cross terms vanish because hard directions $V_{*\leq}$ are uncorrelated with any function over the remaining coordinates $V_{*>:t} := \begin{bmatrix}V_{*>}^\sT & V(0)^\sT & \ldots & V(t)^\sT\end{bmatrix}^\sT$ (cf.\ Lemma~\ref{lemma_amp_data}).

For the rest-rest block, each $\bT_\circ$ do not vanish in general and converges to its own deterministic limit. By Lemmas~\ref{lemma_amp_data} and~\ref{lemma_amp_feature}, the full rest-rest block converges to a deterministic matrix:
\begin{equation}\label{eq:conv_RR_limit}
[\bK_n]_{\mathrm{RR}}(s;z) 
\xrightarrow{~p~}
C_\infty(s;z) \in \mathbb{R}^{(k-r+m+mt) \times (k-r+m+mt)}.
\end{equation}

Combining the limit of $\bT_1$ in \eqref{eq:conv_T1_limit} with Lemma~\ref{lemma:hard_cross_vanishing}, and the rest-rest block \eqref{eq:conv_RR_limit}, we obtain the convergence stated in Lemma~\ref{lem:pointwise_convergence}. \hfill\qed\par\addvspace{6pt}

\paragraph{Concluding.}

By \eqref{eq:conv_target_simplified} and the definition \eqref{eq:conv_integrand_def}, $\bB_{\sf core}=-\bK_n(s_z;z)+o_P(1)$. Lemma~\ref{lem:pointwise_convergence} implies
\begin{equation}\label{eq:conv_core_limit}
    -\bK_n(s_z;z) \xrightarrow{~p~} -K_\infty(s_z;z).
\end{equation}
Define
$\widetilde S_{\mathrm{R}}(z;t) := -C_\infty(s_z;z)$.

Since $s_z=-\delta/\alpha(z)$, we have
\begin{equation}\label{eq:conv_hard_block_simplify}
    -\frac{G_t^2}{s_z - G_t} = \frac{\alpha(z)G_t^2}{\delta + G_t\alpha(z)}.
\end{equation}
Combining Lemma~\ref{lemma:dtilde_approximation}, \eqref{eq:conv_hard_rescaled}, \eqref{eq:conv_core_limit}, \eqref{eq:conv_hard_block_simplify}, and recalling the normalization factor $\left(\mathbb{E}[O_tO_t^\sT]\right)^{-1}$ on both sides, we obtain
\begin{equation}
\begin{aligned}
&\bZ^\sT\bG\bP_{\bF_{0;t}}^{\perp}\bO \widetilde f\left(n\widehat{\bG};z\right) \bO^\sT\bP_{\bF_{0;t}}^{\perp}\bG\bZ \\
&\quad\xrightarrow{~p~} 
\left(\mathbb{E}[O_tO_t^\sT]\right)^{-1}
\begin{bmatrix}
\mathbb{E}\left[\dfrac{\alpha(z)G_t^2}{\delta + G_t\alpha(z)} V_{*\leq}V_{*\leq}^\sT\right] & 0 \\
0 & \widetilde{S}_{\mathrm{R}}^j(z;t)
\end{bmatrix}
\left(\mathbb{E}[O_tO_t^\sT]\right)^{-1},
\end{aligned}
\end{equation}
which is exactly $B_{\sf{lower}}^\infty(z;t)$ as defined in \eqref{eq:B_lower_limit}. This completes the proof of Lemma~\ref{lem:intermediate_to_Blower}.
\hfill\qed\par\addvspace{6pt}

\begin{proof}[Proof of Lemma~\ref{lemma:hard_cross_vanishing}]
    Fix $\circ \in \{2, 3, \dots, 8\}$ and consider $[\bT_{\circ}]_{pq}$ where $p \in [r]$. Here $[\bT_{\circ}]_{pq}$ denotes the $(p,q)$-th entry of $\bT_{\circ}$. The case $q \in [r]$ follows analogously.
    
    By \eqref{eq:conv_K1_expand} and \eqref{eq:conv_K2_expand}, the term $\bT_{\circ}$ is a product of some matrices and contains at least one $\bP_{\bF_{0;t}}$. Since $\bP_{\bF_{0;t}}$ projects onto $\operatorname{span}(\bF_{0;t})$, writing $\bF_{0;t} = [\bf_1, \dots, \bf_{mt}] \in \mathbb{R}^{n \times mt}$, we have
    \begin{equation}\label{eq:conv_PQ_decomposition}
    \bP_{\bF_{0;t}} = \bF_{0;t}(\bF_{0;t}^\sT\bF_{0;t})^{-1}\bF_{0;t}^\sT = \sum_{j=1}^{mt} \bf_j(\bF_{0;t}^\sT\bF_{0;t})^{-1}_{j\cdot}\bF_{0;t}^\sT,
    \end{equation}
    where the sum contains $mt = O(1)$ terms. Substituting \eqref{eq:conv_PQ_decomposition} into $\bT_{\circ}$, the entry $[\bT_{\circ}]_{pq}$ decomposes into a sum of $O(1)$ terms of the form
    \begin{equation}\label{eq:conv_product_structure}
    \left(\frac{1}{n}\sum_{i=1}^n \btheta_{*p}^\sT\bx_i \cdot \rho_i\right) \times \left(\text{other terms}\right) + o_P(1),
    \end{equation}
    where $\rho_i = \rho\left(\bTheta(0)^\sT\bx_i, \dots, \bTheta(t)^\sT\bx_i,y_i\right)$ for some bounded Lipschitz function $\rho(\cdot)$ (using the explicit form of each $\bf_j$, namely the formula of the gradient at each iteration).
    
    Since $\btheta_{*p}$ is a hard direction, Lemma~\ref{lemma_amp_data} gives
    \begin{equation}\label{eq:conv_quadratic_limit}
    \frac{1}{n}\sum_{i=1}^n \btheta_{*p}^\sT\bx_i \cdot \rho_i
    \xrightarrow{~p~}
    \mathbb{E}\left[\rho(V(0), \dots, V(t), V_{*>},h(V_{*}, \varepsilon))V_{*p}\right]=0,
    \end{equation}
because $V_{*p}$ is uncorrelated with the other random variable in the product.

    The `other terms' in \eqref{eq:conv_product_structure} are $O_P(1)$ by Lemmas~\ref{lemma_amp_data} and \ref{lemma_amp_feature}, hence the entire product vanishes in probability. Summing over the $O(1)$ terms yields $[\bT_{\circ}]_{pq} \xrightarrow{~p~} 0$.
    \end{proof}

    \subsection{Eigenvalue localization via determinant zeros}\label{section:eigenvalues}

    In Section~\ref{subsec:spike_basic}, we established that $\bM_j(z;t) \xrightarrow{~p~} M_j^\infty(z;t)$ for $z \in \cD$. 
    
    This subsection translates this convergence into a statement about convergence of Hessian eigenvalues. The main result is the following.
    
    \begin{lemma}\label{lem:eigenvalue_localization}
    We have the following two statements:
    \begin{enumerate}
    \item[$(i)$] Suppose $\det(M_j^\infty(z;t)) = 0$ has a solution $z^* \in \cD$ of multiplicity $k \ge 1$. Then there exist random eigenvalues $z_{n,1}, \dots, z_{n,k}$ of $\bH_j(t)$ such that
    \begin{equation}\label{eq:eig_localization}
    \max_{1 \le \ell \le k}|z_{n,\ell} - z^*| \xrightarrow{~p~} 0.
    \end{equation}
    \item[$(ii)$] For any compact $\cK \subset \cD$, if $\det(M_j^\infty(z;t)) \neq 0$ for all $z \in \cK$, then with probability $1 - o(1)$, $\bH_j(t)$ has no eigenvalues in $\cK$.
    \end{enumerate}
    \end{lemma}
    
    Therefore, Hessian eigenvalues outside the support of the LSD concentrate near solutions of the determinant equation $\det(M_j^\infty(z;t)) = 0$. The proof of Lemma~\ref{lem:eigenvalue_localization} applies Rouch\'e's theorem to the determinant equation; we present the details below.
    
    Throughout this subsection, we work on compact subsets $\cK \subset \cD$.
    
    \begin{lemma}\label{lemma:det_lipschitz}
    For any compact $\cK \subset \cD$ and $z, z' \in \cK$,
    \begin{equation}\label{eq:eig_det_lipschitz}
    \abs{\det(\bM_j(z;t)) - \det(\bM_j(z';t))}
    \le L_n |z - z'|
    \end{equation}
    where $L_n = O_P(1)$ uniformly over $n$. 
    \end{lemma}
    We prove this lemma in Section~\ref{subsec:proof_det_lipschitz}.

    Next, we establish uniform convergence of the determinant.
    
    \begin{lemma}\label{lemma:det_uniform}
    For any compact $\cK \subset \cD$,
    \begin{equation}\label{eq:eig_det_uniform}
    \sup_{z \in \cK}\abs{\det(\bM_j(z;t)) - \det(M_j^\infty(z;t))}
    \xrightarrow{~p~} 0.
    \end{equation}
    \end{lemma}
    
    \begin{proof}
    By the Lipschitz bound \eqref{eq:eig_det_lipschitz} and the pointwise convergence from Lemma~\ref{lemma:limiting_spike_matrix}, we apply Lemma~\ref{lemma:change_limit_first}: equicontinuity combined with pointwise convergence implies uniform convergence.
    \end{proof}

    We now apply Rouch\'e's theorem to relate zeros of the finite-sample determinant to those of the limiting determinant.

    \begin{proof}[Proof of Lemma~\ref{lem:eigenvalue_localization}]
    We apply Lemma~\ref{lemma:rouche_zeros} and Lemma~\ref{lem:rouche_no_zeros} with $f_n(z) = \det(\bM_j(z;t))$, $f(z) = \det(M_j^\infty(z;t))$, and $D = \cD$. 
    
    Part $(i)$ follows from Lemma~\ref{lemma:rouche_zeros}, using the fact that isolated eigenvalues of $\bH_j(t)$ outside the support of the limiting spectral distribution correspond to zeros of $\det(\bM_j(z;t))$ in $\cD$. Part $(ii)$ follows similarly from Lemma~\ref{lem:rouche_no_zeros}.
    \end{proof}
    
    We now analyze the limiting determinant equation using its block structure.
    
    \begin{lemma}\label{lemma:limiting_det_factorization}
    $\det(M_j^\infty(z;t)) = 0$ holds if and only if one of the following two equations holds:
    \begin{align}
    &\det\left(-zI_r + \mathbb{E}\left[\frac{\delta G_t^j}{\delta + G_t^j\alpha_t^j(z)}V_{*\leq}V_{*\leq}^\sT\right]\right) = 0,
    \label{eq:eig_hard_det} \\
    &\det\left(z\left(\mathbb{E}\left[\Theta_{R;t}\Theta_{R;t}^\sT\right]\right)^{-1} + \left(\mathbb{E}\left[\Theta_{R;t}\Theta_{R;t}^\sT\right]\right)^{-1}S_{\mathrm{R}}^j(z;t)\left(\mathbb{E}\left[\Theta_{R;t}\Theta_{R;t}^\sT\right]\right)^{-1}\right) = 0.
    \label{eq:eig_rest_det}
    \end{align}
    \end{lemma}
    In the second equation, \begin{equation}\label{eq:eig_ThetaR_definition}\Theta_{R;t} := [\Theta_{*,>}^\sT, \Theta(0)^\sT, \ldots, \Theta(t)^\sT]^\sT \in \mathbb{R}^{k-r+m+mt},\end{equation} where $\Theta_{*,>} \in \mathbb{R}^{k-r}$ denotes the last $k-r$ coordinates of $\Theta_*$. We note that the first equation is our master equation \eqref{equa_main} in the main theorem.
    
    \begin{proof}[Proof of Lemma \ref{lemma:limiting_det_factorization}]
    From Lemma~\ref{lemma:limiting_spike_matrix}, the limiting outlier matrix has the form
    \begin{align}
    &M_j^\infty(z;t)
    = \begin{bmatrix}
    -\frac{1}{z}\mathbb{E}\big[O_tO_t^\sT\big] & I \\
    I & \left(\mathbb{E}\big[O_tO_t^\sT\big]\right)^{-1}S(z;t)\left(\mathbb{E}\big[O_tO_t^\sT\big]\right)^{-1}
    \end{bmatrix},
    \label{eq:eig_Mj_original}\\
    &S(z;t) = \begin{bmatrix}
    \mathbb{E}\left[-\frac{\delta G_t^j}{\delta + G_t^j\alpha_t^j(z)}V_{*\leq}V_{*\leq}^\sT\right] & 0^\sT \\
    0 & S_{\mathrm{R}}(z;t)
    \end{bmatrix}.
    \label{eq:eig_S_definition}
    \end{align}

    We apply the Schur complement formula to Eq. \eqref{eq:eig_Mj_original}. Since $\mathbb{E}\left[O_tO_t^\sT\right]$ is positive definite, we have
    \begin{equation}\label{eq:eig_schur_step1}
    \begin{aligned}
    &\det(M_j^\infty(z;t))=\\
    & \quad\det\left(-\frac{1}{z}\mathbb{E}\left[O_tO_t^\sT\right]\right) \cdot \det\left(\left(\mathbb{E}\left[O_tO_t^\sT\right]\right)^{-1}S(z;t)\left(\mathbb{E}\left[O_tO_t^\sT\right]\right)^{-1} + z\left(\mathbb{E}\left[O_tO_t^\sT\right]\right)^{-1}\right).
    \end{aligned}
    \end{equation}
    The first factor is nonzero since $\mathbb{E}\left[O_tO_t^\sT\right]$ is invertible.

    For the second, by Lemma~\ref{lemma_amp_feature}, $\mathbb{E}\left[O_tO_t^\sT\right]$ is block-diagonal:
    \begin{equation}\label{eq:eig_Ot_block}
    \mathbb{E}\left[O_tO_t^\sT\right] = \begin{bmatrix}
    I_r & 0 \\
    0 & \mathbb{E}\left[\Theta_{R;t}\Theta_{R;t}^\sT\right]
    \end{bmatrix},
    \end{equation}
    where the $r \times r$ block corresponds to hard directions and the $\left(k-r+m+mt\right) \times \left(k-r+m+mt\right)$ block corresponds to rest directions. Using \eqref{eq:eig_S_definition} and \eqref{eq:eig_Ot_block}, we expand the second as
    \begin{equation}\label{eq:eig_second_factor_expanded}
    \begin{aligned}
    &\left(\mathbb{E}\left[O_tO_t^\sT\right]\right)^{-1}S(z;t)\left(\mathbb{E}\left[O_tO_t^\sT\right]\right)^{-1} + z\left(\mathbb{E}\left[O_tO_t^\sT\right]\right)^{-1} \\
    &\quad= \begin{bmatrix}
    I_r & 0 \\
    0 & \left(\mathbb{E}\left[\Theta_{R;t}\Theta_{R;t}^\sT\right]\right)^{-1}
    \end{bmatrix}
    \begin{bmatrix}
    -\mathbb{E}\left[\frac{\delta G_t^j}{\delta + G_t^j\alpha_t^j(z)}V_{*\leq}V_{*\leq}^\sT\right] & 0^\sT \\
    0 & S_{\mathrm{R}}(z;t)
    \end{bmatrix}
    \begin{bmatrix}
    I_r & 0 \\
    0 & \left(\mathbb{E}\left[\Theta_{R;t}\Theta_{R;t}^\sT\right]\right)^{-1}
    \end{bmatrix} \\
    &\qquad+ z\begin{bmatrix}
    I_r & 0 \\
    0 & \left(\mathbb{E}\left[\Theta_{R;t}\Theta_{R;t}^\sT\right]\right)^{-1}
    \end{bmatrix} \\
    &\quad= \begin{bmatrix}
    -\mathbb{E}\left[\frac{\delta G_t^j}{\delta + G_t^j\alpha_t^j(z)}V_{*\leq}V_{*\leq}^\sT\right] + zI_r & 0^\sT \\
    0 & \left(\mathbb{E}\left[\Theta_{R;t}\Theta_{R;t}^\sT\right]\right)^{-1}S_{\mathrm{R}}(z;t)\left(\mathbb{E}\left[\Theta_{R;t}\Theta_{R;t}^\sT\right]\right)^{-1} + z\left(\mathbb{E}\left[\Theta_{R;t}\Theta_{R;t}^\sT\right]\right)^{-1}
    \end{bmatrix}.
    \end{aligned}
    \end{equation}
    This matrix is block-diagonal. Therefore, $\det(M_j^\infty(z;t)) = 0$ if and only if
    \begin{equation}\label{eq:eig_product}
    \begin{aligned}
    &\det\left(-\mathbb{E}\left[\frac{\delta G_t^{j}}{\delta + G_t^j\alpha_t^j(z)}V_{*\leq}V_{*\leq}^\sT\right] + zI_r\right) \\
    &\qquad\qquad \cdot \det\left(\left(\mathbb{E}\left[\Theta_{R;t}\Theta_{R;t}^\sT\right]\right)^{-1}S_{\mathrm{R}}(z;t)\left(\mathbb{E}\left[\Theta_{R;t}\Theta_{R;t}^\sT\right]\right)^{-1} + z\left(\mathbb{E}\left[\Theta_{R;t}\Theta_{R;t}^\sT\right]\right)^{-1}\right) = 0.
    \end{aligned}
    \end{equation}
    This product equals zero if and only if one of the two equals zero, and the conclusion follows.
    \end{proof}

    \subsection{Weak recovery of hard directions from eigenvectors}\label{section:eigenvectors}

    This subsection analyzes the correlation between Hessian eigenvectors corresponding to outliers and the hard directions $\bTheta_{*\mathrm{H}} = [\btheta_1^*, \dots, \btheta_r^*] \in \mathbb{R}^{d \times r}$. The main result is the following.
    
    \begin{lemma}\label{lemma:eigenvector_correlation}
    Let $z^* \in \cD$ with $\det(M_j^\infty(z^*;t)) = 0$ and multiplicity $\beta \ge 1$. Let $U = \begin{bmatrix}U_{\leq}^\sT & U_{>}^\sT\end{bmatrix}^\sT \in \mathbb{R}^{2(k+m+mt) \times \beta}$ be a basis for $\ker(M_j^\infty(z^*;t))$, where $U_{\leq} \in \mathbb{R}^{(k+m+mt) \times \beta}$ corresponds to the upper block and $U_{>} \in \mathbb{R}^{(k+m+mt) \times \beta}$ to the lower block. Let $U_{\leq,\mathrm{H}} \in \mathbb{R}^{r \times \beta}$ denote the first $r$ rows of $U_{\leq}$.
    
    Consider the eigenvalues $\{z_{n,\ell}\}_{\ell=1}^\beta$ of $\bH_j(t)$ converging to $z^*$ (cf.\ Lemma~\ref{lem:eigenvalue_localization}), with orthonormal eigenvectors $\{\boldsymbol{\xi}_{n,\ell}\}_{\ell=1}^\beta$. Then
    \begin{equation}\label{eq:evec_main_result}
    \sum_{\ell=1}^\beta \|\bTheta_{*\mathrm{H}}^\sT\boldsymbol{\xi}_{n,\ell}\|^2
    \xrightarrow{~p~}
    \operatorname{Tr}\left(U_{\leq,\mathrm{H}}\left(U_{\leq}^\sT\left(\mathbb{E}[O_tO_t^\sT] + S'(z^*;t)\right)U_{\leq}\right)^{-1}U_{\leq,\mathrm{H}}^\sT\right),
    \end{equation}
    where $S'(z^*;t) := \frac{\partial}{\partial z}S(z;t)\big|_{z=z^*}$.
    \end{lemma}
    
    By Lemma~\ref{lemma:limiting_det_factorization}, the determinant equation $\det(M_j^\infty(z;t)) = 0$ factors into hard and rest parts. If $z^*$ arises from the rest block (i.e., from \eqref{eq:eig_rest_det}), then $U_{\leq,\mathrm{H}} = 0$ and the correlation \eqref{eq:evec_main_result} equals zero.
    Instead, if $z^*$ arises from the hard block (i.e., from \eqref{eq:eig_hard_det}), then by the block-diagonal structure of $\mathbb{E}[O_tO_t^\sT]$ and $S'(z^*;t)$, the formula \eqref{eq:evec_main_result} simplifies to
    \begin{equation}\label{eq:evec_hard_only}
    \sum_{\ell=1}^\beta \|\bTheta_{*\mathrm{H}}^\sT\boldsymbol{\xi}_{n,\ell}\|^2
    \xrightarrow{~p~}
    \operatorname{Tr}\left(U_{\leq,\mathrm{H}}\left(U_{\leq,\mathrm{H}}^\sT\left(I_r + S'_{\mathrm{H}}(z^*;t)\right)U_{\leq,\mathrm{H}}\right)^{-1}U_{\leq,\mathrm{H}}^\sT\right),
    \end{equation}
    which is strictly positive and depends only on the hard block. Therefore, in order to have positive correlation, the corresponding eigenvalue must arise from the equation of the hard block.
    
    The proof of Lemma~\ref{lemma:eigenvector_correlation} occupies the remainder of this subsection.

    \par\addvspace{6pt}\noindent\textit{Proof of Lemma \ref{lemma:eigenvector_correlation}.}
    Let $z^* \in \cD$ be a solution of $\det(M_j^\infty(z;t)) = 0$ with multiplicity $\beta \ge 1$. By Lemma~\ref{lem:eigenvalue_localization}, there exist random eigenvalues $\{z_{n,\ell}\}_{\ell=1}^\beta$ of $\bH_j(t)$ such that $\max_{\ell \in [\beta]}|z_{n,\ell} - z^*| \xrightarrow{~p~} 0$. Let $\{\boldsymbol{\xi}_{n,\ell}\}_{\ell=1}^\beta$ denote the corresponding orthonormal eigenvectors.
    
    We measure the correlation via the following. Choose a contour $\gamma$ centered at $z^*$ with radius small enough that $\gamma \subset \cD$ and $\gamma$ encloses $z^*$ but no other solutions of $\det(M_j^\infty(z;t)) = 0$ and no points in the support of the LSD. Define
    \begin{equation}\label{eq:evec_projector}
    \bP_*^{n,d} := -\frac{1}{2\pi i}\oint_{\gamma}(\bH_j(t) - z\bI)^{-1}\de z.
    \end{equation}
    With probability $1 - o(1)$, the eigenvalues $\{z_{n,\ell}\}_{\ell=1}^\beta$ are the only eigenvalues of $\bH_j(t)$ inside $\gamma$, hence
    \begin{equation}\label{eq:evec_projector_sum}
    \bP_*^{n,d} = \sum_{\ell=1}^\beta \boldsymbol{\xi}_{n,\ell}\boldsymbol{\xi}_{n,\ell}^\sT.
    \end{equation}
    The total squared correlation is
    \begin{equation}\label{eq:evec_correlation}
    \sum_{\ell=1}^\beta \|\bTheta_{*\mathrm{H}}^\sT\boldsymbol{\xi}_{n,\ell}\|^2
    = \operatorname{Tr}\left(\bTheta_{*\mathrm{H}}^\sT\bP_*^{n,d}\bTheta_{*\mathrm{H}}\right)
    = -\frac{1}{2\pi i}\oint_{\gamma}\operatorname{Tr}\left(\bTheta_{*\mathrm{H}}^\sT(\bH_j(t) - z\bI)^{-1}\bTheta_{*\mathrm{H}}\right)\de z.
    \end{equation}
    
    From Section~\ref{subsec:spike_basic}, we have the following resolvent decomposition:
    \begin{equation}\label{eq:evec_resolvent_decomp}
    \begin{aligned}
    (\bH_j(t) - z\bI)^{-1}
    &= \bR_0(z) - \bR_0(z)\begin{bmatrix}\bTheta_{*;t} & \bY^\sT\bG\bZ\end{bmatrix}
    \bM_j(z;t)^{-1}
    \begin{bmatrix}\bTheta_{*;t}^\sT \\ \bZ^\sT\bG\bY\end{bmatrix}
    \bR_0(z),
    \end{aligned}
    \end{equation}
    where $\bR_0(z) := (\bY^\sT\bG\bY - z\bI)^{-1}$ is the resolvent and $\bM_j(z;t)$ is the outlier matrix.
    Substituting \eqref{eq:evec_resolvent_decomp} into \eqref{eq:evec_correlation}, we obtain
    \begin{equation}\label{eq:evec_two_terms}
    \begin{aligned}
    &\sum_{\ell=1}^\beta \|\bTheta_{*\mathrm{H}}^\sT\boldsymbol{\xi}_{n,\ell}\|^2
    = -\frac{1}{2\pi i}\oint_{\gamma}\operatorname{Tr}\left(\bTheta_{*\mathrm{H}}^\sT\bR_0(z)\bTheta_{*\mathrm{H}}\right)\de z \\
    &\quad + \frac{1}{2\pi i}\oint_{\gamma}\operatorname{Tr}\left(\bTheta_{*\mathrm{H}}^\sT\bR_0(z)
    \begin{bmatrix}\bTheta_{*;t} & \bY^\sT\bG\bZ\end{bmatrix}
    \bM_j(z;t)^{-1}
    \begin{bmatrix}\bTheta_{*;t}^\sT \\ \bZ^\sT\bG\bY\end{bmatrix}
    \bR_0(z)\bTheta_{*\mathrm{H}}\right)\de z.
    \end{aligned}
    \end{equation}
    The contour $\gamma$ avoids the spectrum of $\bY^\sT\bG\bY$ with probability $1 - o(1)$ due to Lemma~\ref{lemma:smallest_eigenvalue_concentration}, hence $\bR_0(z) = (\bY^\sT\bG\bY - z\bI)^{-1}$ is analytic on and inside $\gamma$ conditioned on the event. Therefore,
    \begin{equation}\label{eq:evec_bulk_vanishes}
    \oint_{\gamma}\operatorname{Tr}\left(\bTheta_{*\mathrm{H}}^\sT\bR_0(z)\bTheta_{*\mathrm{H}}\right)\de z = 0.
    \end{equation}
    Thus,
    \begin{equation}\label{eq:evec_contour_spike}
    \sum_{\ell=1}^\beta \|\bTheta_{*\mathrm{H}}^\sT\boldsymbol{\xi}_{n,\ell}\|^2
    = \frac{1}{2\pi i}\oint_{\gamma}\operatorname{Tr}\left(\bTheta_{*\mathrm{H}}^\sT\bR_0(z)
    \begin{bmatrix}\bTheta_{*;t} & \bY^\sT\bG\bZ\end{bmatrix}
    \bM_j(z;t)^{-1}
    \begin{bmatrix}\bTheta_{*;t}^\sT \\ \bZ^\sT\bG\bY\end{bmatrix}
    \bR_0(z)\bTheta_{*\mathrm{H}}\right)\de z.
    \end{equation}
    
    Next, we establish uniform convergence of the entire integrand in \eqref{eq:evec_contour_spike}.
    
    \begin{lemma}\label{lemma:uniform_integrand}
    Define
    \begin{equation}\label{eq:evec_Fn_def}
    F_{n,d}(z) := \bTheta_{*\mathrm{H}}^\sT\bR_0(z)\begin{bmatrix}\bTheta_{*;t} & \bY^\sT\bG\bZ\end{bmatrix}
    \bM_j(z;t)^{-1}
    \begin{bmatrix}\bTheta_{*;t}^\sT \\ \bZ^\sT\bG\bY\end{bmatrix}
    \bR_0(z)\bTheta_{*\mathrm{H}}.
    \end{equation}
    and the limiting matrix
    \begin{equation}\label{eq:evec_Finfty_def}
    F_\infty(z) := A_\infty(z)^\sT M_j^\infty(z;t)^{-1} A_\infty(z), \quad \text{where} \quad A_\infty(z) := -\frac{1}{z}\begin{bmatrix} I_r \\ 0 \end{bmatrix} \in \mathbb{R}^{2(k+m+mt) \times r}.
    \end{equation}
    Then
    \begin{equation}\label{eq:evec_uniform_F}
    \sup_{z \in \gamma}\|F_{n,d}(z) - F_\infty(z)\|_{\mathrm{op}}
    \xrightarrow{~p~} 0.
    \end{equation}
    \end{lemma}
    We prove this lemma in Section~\ref{subsec:proof_uniform_integrand}.
    Next, as a consequence of Lemma~\ref{lemma:uniform_integrand},
    \begin{equation}\label{eq:evec_contour_limit}
    \oint_{\gamma}F_{n,d}(z)\de z
    \xrightarrow{~p~}
    \oint_{\gamma}F_\infty(z)\de z.
    \end{equation}
    It therefore suffices to compute the deterministic integral $\oint_{\gamma}F_\infty(z)\de z$ by residue calculus.
    The only pole inside $\gamma$ arises from $M_j^\infty(z;t)^{-1}$ at $z = z^*$.
    
    By standard analytic perturbation theory, since $\det(M_j^\infty(z;t))$ has a zero of order $\beta$ at $z^*$ and $U^\sT\frac{\partial}{\partial z}M_j^\infty(z^*;t)U$ is positive definite on $\ker(M_j^\infty(z;t))$, the inverse $M_j^\infty(z;t)^{-1}$ has a Laurent expansion near $z^*$ with principal part
    \begin{equation}\label{eq:evec_laurent}
    M_j^\infty(z;t)^{-1}
    = \frac{U\left(U^\sT\frac{\partial}{\partial z}M_j^\infty(z^*;t)U\right)^{-1}U^\sT}{z - z^*}
    + \mathcal{R}(z),
    \end{equation}
    where $\mathcal{R}(z)$ is analytic around $z^*$. Therefore, by the residue theorem,
    \begin{equation}\label{eq:evec_residue_integral}
    \frac{1}{2\pi i}\oint_{\gamma}F_\infty(z)\de z
    = \frac{1}{2\pi i}\oint_{\gamma}\operatorname{Tr}\left(A_\infty(z)^\sT M_j^\infty(z;t)^{-1}A_\infty(z)\right)\de z
    = \operatorname{Tr}\left(A_\infty(z^*)^\sT  U\Lambda^{-1}U^\sT A_\infty(z^*)\right),
    \end{equation}
    where we abbreviate $\Lambda := U^\sT\frac{\partial}{\partial z}M_j^\infty(z^*;t)U$.

    Note that $A_\infty(z^*)$ extracts only the first $r$ rows. Therefore,
    \begin{equation}\label{eq:evec_A_U_product}
    A_\infty(z^*)^\sT U
    = -\frac{1}{z^*}[I_r \ 0]
    \begin{bmatrix}
    U_{\leq} \\
    U_{>}
    \end{bmatrix}
    = -\frac{1}{z^*}U_{\leq,\mathrm{H}}
    \in \mathbb{R}^{r \times \beta}.
    \end{equation}
    Therefore,
    \begin{equation}\label{eq:evec_residue_simplified}
    \operatorname{Tr}\left(A_\infty(z^*)^\sT U\Lambda^{-1}U^\sT A_\infty(z^*)\right)
    = \frac{1}{(z^*)^2}\operatorname{Tr}\left(U_{\leq,\mathrm{H}}^\sT\Lambda^{-1}U_{\leq,\mathrm{H}}\right).
    \end{equation}
    
    From the block structure of $M_j^\infty(z;t)$ in \eqref{eq:spike_Mj_infinity}, the kernel satisfies
    \begin{equation}\label{eq:evec_kernel_relation}
    U_{>} = \frac{1}{z^*}\mathbb{E}[O_tO_t^\sT]U_{\leq}.
    \end{equation}
    Differentiating $M_j^\infty(z;t)$ with respect to $z$ gives
    \begin{equation}\label{eq:evec_Mj_derivative}
    \frac{\partial}{\partial z}M_j^\infty(z;t)
    = \begin{bmatrix}
    \frac{1}{z^2}\mathbb{E}[O_tO_t^\sT] & 0 \\
    0 & (\mathbb{E}[O_tO_t^\sT])^{-1}S'(z;t)(\mathbb{E}[O_tO_t^\sT])^{-1}
    \end{bmatrix},
    \end{equation}
    where
    \begin{equation}\label{eq:evec_Sprime}
    S'(z;t) := \frac{\partial}{\partial z}S(z;t)
    = \begin{bmatrix}
    S'_{\mathrm{H}}(z;t) & 0 \\
    0 & S'_{\mathrm{R}}(z;t)
    \end{bmatrix},
    \quad
    S'_{\mathrm{H}}(z;t)
    := \mathbb{E}\left[\frac{\delta G_t^{j2} \alpha_t^{j'}(z)}{(\delta + G_t^j\alpha_t^j(z))^2}V_{*\leq}V_{*\leq}^\sT\right],
    \end{equation}
    and $\alpha_t^{j'}(z) := \frac{\partial}{\partial z}\alpha_t^j(z)$ is the derivative of the Stieltjes transform.
    
    For $u \in \ker(M_j^\infty(z^*;t))$ with $u = \begin{bmatrix}u_{\leq}^\sT & u_{>}^\sT\end{bmatrix}^\sT$, we compute
    \begin{equation}\label{eq:evec_quadratic_form}
    \begin{aligned}
    u^\sT\frac{\partial}{\partial z}M_j^\infty(z^*;t)u
    &= \frac{1}{(z^*)^2}u_{\leq}^\sT\mathbb{E}[O_tO_t^\sT]u_{\leq}
    + u_{\leq}^\sT(\mathbb{E}[O_tO_t^\sT])^{-1}S'(z^*;t)(\mathbb{E}[O_tO_t^\sT])^{-1}u_{>}.
    \end{aligned}
    \end{equation}
    Using \eqref{eq:evec_kernel_relation},
    \begin{equation}\label{eq:evec_quadratic_simplified}
    \begin{aligned}
    u^\sT\frac{\partial}{\partial z}M_j^\infty(z^*;t)u
    &= \frac{1}{(z^*)^2}u_{\leq}^\sT\mathbb{E}[O_tO_t^\sT]u_{\leq}
    + \frac{1}{(z^*)^2}u_{\leq}^\sT(\mathbb{E}[O_tO_t^\sT])^{-1}S'(z^*;t)u_{\leq} \\
    &= \frac{1}{(z^*)^2}u_{\leq}^\sT\left(\mathbb{E}[O_tO_t^\sT] + S'(z^*;t)\right)u_{\leq}.
    \end{aligned}
    \end{equation}
    Same calculation works for $U$. Substituting $\Lambda = \frac{1}{(z^*)^2}U_{\leq}^\sT(\mathbb{E}[O_tO_t^\sT] + S'(z^*;t))U_{\leq}$ from \eqref{eq:evec_quadratic_simplified}, the $(z^*)^2$ factors cancel:
    \begin{equation}\label{eq:evec_final}
    \frac{1}{2\pi i}\oint_{\gamma}F_\infty(z)\de z
    = \operatorname{Tr}\left(U_{\leq,\mathrm{H}}^\sT\left(U_{\leq}^\sT\left(\mathbb{E}[O_tO_t^\sT] + S'(z^*;t)\right)U_{\leq}\right)^{-1}U_{\leq,\mathrm{H}}\right).
    \end{equation}
    
    Combining \eqref{eq:evec_contour_limit} and \eqref{eq:evec_final} yields the formula \eqref{eq:evec_main_result}, completing the proof of Lemma~\ref{lemma:eigenvector_correlation}.
    \qed

\subsection{Proof of Theorem \ref{theorem:main_result}: Finite time spectral phase transition}

Theorem \ref{theorem:main_result} is a direct consequence of results in Sections~\ref{section:eigenvalues} and~\ref{section:eigenvectors}.

The statements of outlier eigenvalues follow from Lemma~\ref{lem:eigenvalue_localization}. The weak recovery formulas in both regimes are established in Lemma~\ref{lemma:eigenvector_correlation} and its subsequent discussion. 

\subsection{Proof of Theorem \ref{theorem:main_result_limit}: Large time limit}\label{subsec:proof_limit}

This subsection establishes the results for $t \to \infty$ via leveraging the finite-time results of Theorem~\ref{theorem:main_result}.

For technical convenience, we work on arbitrary compact subsets $\cK\subset\cD^\infty := \{z \in \mathbb{C}:\ \Re(z) < \min(0, c_j(\infty))\}$, where $c_j(\infty)$ denotes the left edge of the limiting spectrum as $t \to \infty$. 

\paragraph{Step I. Uniform convergence of determinants.}For notational convenience, we define the determinant functions
\begin{align}
f_t(z)
&:= \det\left(-zI_r + \mathbb{E}\left[\frac{\delta G_t^{j2}}{\delta + G_t^j\alpha_t^j(z)}V_{*\leq}V_{*\leq}^\sT\right]\right)
\label{eq:limit_ft_def}\\
\text{and}\quad
f_\infty(z)
&:= \det\left(-zI_r + \mathbb{E}\left[\frac{\delta G_\infty^{j2}}{\delta + G_\infty^j\alpha_\infty^j(z)}V_{*\leq}V_{*\leq}^\sT\right]\right),
\label{eq:limit_finfty_def}
\end{align}
where $G_t^j := g(V(t), h(V_*, \varepsilon); j)$, $G_\infty^j := g(V(\infty), h(V_*, \varepsilon); j)$, and $\alpha_t^j(z)$, $\alpha_\infty^j(z)$ are the Stieltjes transforms of the finite-time and infinite-time measures $\mu_\infty^j(t)$, $\mu_\infty^j(\infty)$ with left edges $c_j(t)$, $c_j(\infty)$.
By construction, the outlier equations \eqref{equa_main} and \eqref{equa_main_limit} from Theorems~\ref{theorem:main_result} and~\ref{theorem:main_result_limit} correspond to $f_t(z) = 0$ and $f_\infty(z) = 0$.

\begin{lemma}\label{lemma:ft_to_finfty}
For any compact $\cK \subset \cD^\infty$, we have $\sup_{z \in \cK}|f_t(z) - f_\infty(z)| \to 0$ as $t \to \infty$.
\end{lemma}
The proof is deferred to the end of this subsection.

\paragraph{Step II. Regime $\delta > \delta_j^*(\infty)$.}
By definition \eqref{equa_threshold_limit}, the equation $f_\infty(z) = 0$ has at least one negative solution $z < c_j(\infty)$. Let $z_1^*, \dots, z_p^* < c_j(\infty)$ denote the distinct negative solutions with multiplicities $k_1, \dots, k_p \ge 1$. Since $f_\infty$ is analytic (as the determinant of a matrix with analytic entries), its zeros are isolated. 

For each $i \in \{1, \dots, p\}$, choose a radius $r_i > 0$ small enough that the open disk
\begin{equation}\label{eq:limit_disk_def}
\mathsf{D}_i := \{z \in \mathbb{C}:\ |z - z_i^*| < r_i\}
\end{equation}
satisfies $\overline{\mathsf{D}}_i \subset \cD^\infty$, the disks $\{\mathsf{D}_i\}_{i=1}^p$ are pairwise disjoint, and $f_\infty$ has exactly $k_i$ zeros (counted with multiplicity) in $\mathsf{D}_i$ and no zeros on the boundary $\gamma_i := \partial\mathsf{D}_i$.
By Lemma~\ref{lemma:ft_to_finfty}, we have
\begin{equation}\label{eq:limit_ft_boundary_conv}
\sup_{z \in \gamma_i}|f_t(z) - f_\infty(z)| \to 0
\quad\text{as }t \to \infty.
\end{equation}
Define $m_i := \min_{z \in \gamma_i}|f_\infty(z)| > 0$. For all sufficiently large $t$, we have $\sup_{z \in \gamma_i}|f_t(z) - f_\infty(z)| < m_i/2$ by \eqref{eq:limit_ft_boundary_conv}. Therefore, for all $z \in \gamma_i$,
\begin{equation}\label{eq:limit_rouche_bound}
|f_t(z) - f_\infty(z)| < m_i/2 \le |f_\infty(z)|.
\end{equation}

By Rouch\'e's theorem, $f_t$ and $f_\infty$ have the same number of zeros (counted with multiplicity) inside $\gamma_i$. Since $f_\infty$ has exactly $k_i$ zeros in $\mathsf{D}_i$, so does $f_t$ for all large $t$.
Denote the zeros of $f_t$ in $\mathsf{D}_i$ as $\{z_i^{*(t,q)}\}_{q=1}^{k_i}$ (counted with multiplicity). To show convergence $z_i^{*(t,q)} \to z_i^*$, fix $\varepsilon > 0$ and choose $0 < \rho < r_i$ with $\rho < \varepsilon$. By the same Rouch\'e argument applied to the circle $\{z:\ |z - z_i^*| = \rho\}$, for all large $t$, $f_t$ has exactly $k_i$ zeros in $\{z:\ |z - z_i^*| < \rho\}$ and exactly $k_i$ zeros in $\mathsf{D}_i$. Therefore, all zeros of $f_t$ in $\mathsf{D}_i$ lie in $\{z:\ |z - z_i^*| < \rho\}$. Since $\varepsilon$ is arbitrary, we obtain the convergence for $t \to \infty$.

Next, by Theorem~\ref{theorem:main_result}, for each $q \in [k_i]$, there exist random eigenvalues $z_{n,q}^{(t)}$ of $\bH_j(t)$ such that
\begin{align}
\plim_{n, d \to \infty}\max_{q \in [k_i]}|z_{n,q}^{(t)} - z_i^{*(t,q)}| &= 0.
\label{eq:limit_Hess_eig_finite_t}
\intertext{Combining them via the triangle inequality, we obtain}
\lim_{t \to \infty}\plim_{n, d \to \infty}\max_{q \in [k_i]}|z_{n,q}^{(t)} - z_i^*| &= 0,
\label{eq:limit_final_eig}
\end{align}
which is exactly the eigenvalue convergence statement in Theorem~\ref{theorem:main_result_limit}.

\paragraph{Step III. Eigenspace convergence.}
Fix $i \in \{1, \dots, p\}$ and consider a large enough time $t$. We use the same notation as in Step II.

Let $\left\{\boldsymbol{\xi}_{n,\ell}^{(t)}\right\}_{\ell=1}^{k_i}$ denote orthonormal eigenvectors of $\bH_j(t)$ corresponding to eigenvalues inside $\gamma_i=\partial\mathsf{D}_i$. Denote the distinct zeros of $f_t(z)$ in $\mathsf{D}_i$ as $\left\{z_i^{(t,r)}\right\}_{r=1}^{q_i(t)}$ with multiplicities $m_{i,r}(t)$ satisfying $\sum_{r}m_{i,r}(t) = k_i$. 

For each $r$, let $U_{\leq,\mathrm{H}}^{(t,i,r)} \in \mathbb{R}^{r \times m_{i,r}(t)}$ be an orthonormal basis for the kernel
\begin{equation}\label{eq:limit_kernel_finite_t}
\ker\left(-z_i^{(t,r)}I_r + \mathbb{E}\left[\frac{\delta G_t^{j2}}{\delta + G_t^j\alpha_t^j(z_i^{(t,r)})}V_{*\leq}V_{*\leq}^\sT\right]\right).
\end{equation}
By Theorem~\ref{theorem:main_result}, for fixed $t$,
\begin{equation}\label{eq:limit_evec_finite_t}
\sum_{\ell=1}^{k_i}\|\bTheta_{*\mathrm{H}}^\sT\boldsymbol{\xi}_{n,\ell}^{(t)}\|^2
\xrightarrow{~p~}
\sum_{r=1}^{q_i(t)}\operatorname{Tr}\!\left(U_{\leq,\mathrm{H}}^{(t,i,r)}\left(U_{\leq,\mathrm{H}}^{(t,i,r)\sT}\big(I_r + S'_\mathrm{H}(t; z_i^{(t,r)})\big)U_{\leq,\mathrm{H}}^{(t,i,r)}\right)^{-1}U_{\leq,\mathrm{H}}^{(t,i,r)\sT}\right)
\end{equation}
as $n, d \to \infty$, where
\begin{equation}\label{eq:limit_Sprime_def}
S'_\mathrm{H}(t; z)
:= \mathbb{E}\left[\frac{\delta G_t^{j2}\alpha_t^{j'}(z)}{(\delta + G_t^j\alpha_t^j(z))^2}V_{*\leq}V_{*\leq}^\sT\right]
\end{equation}
and $\alpha_t^{j'}(z) := \frac{\partial}{\partial z}\alpha_t^j(z)$.

We now take $t \to \infty$. By Assumption~\ref{assumption:saddle_2}, we have $G_t^j \rightarrow G_\infty^j$ in $L^2$ (also in Wasserstein-2 sense), $\alpha_t^j \to \alpha_\infty^j$, and $\alpha_t^{j'} \to \alpha_\infty^{j'}$ uniformly on $\cK$.

Since $z_i^{(t,r)} \to z_i^*$ by Step II, dominated convergence yields
\begin{equation}\label{eq:limit_Sprime_conv}
S'_\mathrm{H}(t; z_i^{(t,r)}) \to S'_\mathrm{H}(\infty; z_i^*)
\quad\text{as }t \to \infty,
\end{equation}
where $S'_\mathrm{H}(\infty; z) := \mathbb{E}\left[\frac{\delta G_\infty^{j2}\alpha_\infty^{j'}(z)}{(\delta + G_\infty^j\alpha_\infty^j(z))^2}V_{*\leq}V_{*\leq}^\sT\right]$. Define $B_\infty^i := I_r + S'_\mathrm{H}(\infty; z_i^*)$.

Concatenating the bases, we form
\begin{equation}\label{eq:limit_Uti_def}
U^{(t,i)} := [U_{\leq,\mathrm{H}}^{(t,i,1)}\cdots U_{\leq,\mathrm{H}}^{(t,i,q_i(t))}] \in \mathbb{R}^{r \times k_i}, \quad \mathcal{S}_t^i:=\operatorname{span}(U^{(t,i)}).
\end{equation}
Since $A_t(\cdot)\to A_\infty(\cdot)$ uniformly on $\cK$ (Step I) and $z_i^{(t,r)}\to z_i^*$ (Step II), we have $A_t(z_i^{(t,r)})\to A_\infty(z_i^*)$ for each $r$. By the Davis-Kahan theorem (applied to the eigenspace of $A_\infty(z_i^*)$, which has dimension $k_i$), the corresponding $k_i$-dimensional subspaces $\mathcal{S}_t^i$ converge as $t\to\infty$ with the limit
\begin{equation}\label{eq:limit_Sinfty_def}
\mathcal{S}_\infty^i := \ker\left(-z_i^*I_r + \mathbb{E}\left[\frac{\delta G_\infty^{j2}}{\delta + G_\infty^j\alpha_\infty^j(z_i^*)}V_{*\leq}V_{*\leq}^\sT\right]\right).
\end{equation}
Let $U_\infty^i \in \mathbb{R}^{r \times k_i}$ be an orthonormal basis of $\mathcal{S}_\infty^i$.

Since the trace formula $U \mapsto \operatorname{Tr}(U(U^\sT B_\infty^iU)^{-1}U^\sT)$ is continuous on the Stiefel manifold and basis-invariant, we obtain
\begin{equation}\label{eq:limit_final_evec}
\lim_{t \to \infty}\plim_{n, d \to \infty}\sum_{\ell=1}^{k_i}\|\bTheta_{*\mathrm{H}}^\sT\boldsymbol{\xi}_{n,\ell}^{(t)}\|^2
= \operatorname{Tr}\!\left(U_\infty^i\big(U_\infty^{i\sT}B_\infty^iU_\infty^i\big)^{-1}U_\infty^{i\sT}\right),
\end{equation}
which matches the eigenvector correlation formula in Theorem~\ref{theorem:main_result_limit}.

\paragraph{Step IV. Regime $\delta < \delta_j^*(\infty)$.}

By definition \eqref{equa_threshold_limit}, $f_\infty(z) = 0$ has no negative solutions in the region $z < c_j(\infty)$. Since $f_\infty$ is continuous and nonvanishing on the compact set $\cK$,
\begin{equation}\label{eq:limit_m_bound}
\beta := \inf_{z \in \cK}|f_\infty(z)| > 0.
\end{equation}
By Lemma~\ref{lemma:ft_to_finfty}, for all sufficiently large $t$, we have $\sup_{z \in \cK}|f_t(z) - f_\infty(z)| < m/2$. 

Therefore, for all $z \in \cK$,
\begin{equation}\label{eq:limit_ft_lower}
|f_t(z)| \ge |f_\infty(z)| - |f_t(z) - f_\infty(z)| \ge \beta - \beta/2 = \beta/2 > 0,
\end{equation}
which means $f_t(z) \neq 0$ for all $z \in \cK$. By Theorem~\ref{theorem:main_result} (the case $\delta < \delta_j^*(t)$), for each large $t$, with probability $1 - o(1)$, the Hessian $\bH_j(t)$ has no eigenvalues in $\cK$ arising from the hard block. The conclusion follows since $\cK$ is arbitrary.

The conclusion for eigenvectors also follows from the finite-time result. Taking $t$ large enough and noting that the choice of $\cK$ is arbitrary (as long as $\cK \subset \cD^\infty$), we obtain
\begin{equation}\label{eq:limit_no_corr}
\plim_{n, d \to \infty}\|\bTheta_{*\mathrm{H}}^\sT\boldsymbol{\xi}_p\|^2 = 0
\end{equation}
for each fixed $p \ge 1$ and all large enough $t$, where $\boldsymbol{\xi}_p$ is the eigenvector associated with the $p$-th smallest eigenvalue of $\bH_j(t)$. This completes the proof of Theorem~\ref{theorem:main_result_limit}.\qed

\begin{proof}[Proof of Lemma~\ref{lemma:ft_to_finfty}]
Note that $V(t)\to V(\infty)$ in $L^2$ by Assumption~\ref{assumption:saddle_2}. Under Assumption~\ref{assumption_regularity}, the map $v\mapsto g(v,h(V_*,\varepsilon);j)$ is bounded and Lipschitz, therefore also $G_t^j\to G_\infty^j$ in $L^2$ (in particular, $G_t^j\Rightarrow G_\infty^j$). This yields $\mu_\infty^j(t)\Rightarrow \mu_\infty^j(\infty)$ and $c_j(t)\to c_j(\infty)$ (which is the statement of Lemma~\ref{lemma:limiting_measure}). 
    Since $\cK\subset \cD^\infty$ is at strictly positive distance from $\supp(\mu_\infty^j(\infty))$, for all sufficiently large $t$ we have $\cK \subset \cD^t := \{z:\ \Re(z) < \min(0, c_j(t))\}$. Therefore,
    \begin{equation}\label{eq:limit_stieltjes_conv}
    \sup_{z \in \cK}|\alpha_t^j(z) - \alpha_\infty^j(z)| \to 0
    \quad\text{as }t \to \infty,
    \end{equation}
    Define
    \begin{equation}\label{eq:limit_At_def}
    A_t(z) := -zI_r + \mathbb{E}\left[\frac{\delta G_t^{j2}}{\delta + G_t^j\alpha_t^j(z)}V_{*\leq}V_{*\leq}^\sT\right]
    \in \mathbb{R}^{r \times r},
    \end{equation}
    and similarly $A_\infty(z)$ with $t$ replaced by $\infty$. By the determinant Lipschitz bound, 
    \begin{equation}\label{eq:limit_det_lipschitz}
    |f_t(z) - f_\infty(z)|
    \lesssim \left(\|A_t(z)\|_{\mathrm{op}}^{r-1} + \|A_\infty(z)\|_{\mathrm{op}}^{r-1}\right)\|A_t(z) - A_\infty(z)\|_{\mathrm{op}}.
    \end{equation}
    From the definition of $\cK$, we have the uniform boundedness of $G_t^j$, $G_\infty^j$, $\alpha_t^j(z)$, $\alpha_\infty^j(z)$ on $\cK$ for large enough $t$, and also the uniform boundedness of the operator norms $\|A_t(z)\|_{\mathrm{op}}$ and $\|A_\infty(z)\|_{\mathrm{op}}$ over $z \in \cK$ and large enough $t$. Therefore, it suffices to show $\sup_{z \in \cK}\|A_t(z) - A_\infty(z)\|_{\mathrm{op}} \to 0$.
    
    Define $\phi(g, \alpha) := \frac{\delta g^2}{\delta + g\alpha}$, such that
    \begin{equation}\label{eq:limit_At_diff}
    A_t(z) - A_\infty(z)
    = \mathbb{E}\left[\big(\phi(G_t^j, \alpha_t^j(z)) - \phi(G_\infty^j, \alpha_\infty^j(z))\big)V_{*\leq}V_{*\leq}^\sT\right].
    \end{equation}
    For fixed $z$, we have $G_t^j \rightarrow G_\infty^j$ in $L^2$ and $\alpha_t^j(z) \to \alpha_\infty^j(z)$. The function $\phi$ is Lipschitz continuous and absolutely bounded in $\cK$, and $\|V_{*\leq}\|^2$ is integrable. By dominated convergence,
    \begin{equation}\label{eq:limit_pointwise}
    \|A_t(z) - A_\infty(z)\|_{\mathrm{op}} \to 0
    \quad\text{for each fixed }z \in \cK.
    \end{equation}
    
    To obtain uniform convergence over $z \in \cK$, note that by \eqref{eq:limit_stieltjes_conv}, the family $\{\alpha_t^j(\cdot)\}_{t \ge 1}$ is equicontinuous on $\cK$ for large enough $t$. Combined with the Lipschitz continuity of $\phi(\cdot)$ in $\alpha$ uniformly over the possible values of $g$, the functions $z \mapsto A_t(z)$ are also equicontinuous for large enough $t$. Therefore, pointwise convergence \eqref{eq:limit_pointwise} implies
    \begin{equation}\label{eq:limit_uniform}
    \sup_{z \in \cK}\|A_t(z) - A_\infty(z)\|_{\mathrm{op}} \to 0
    \quad\text{as }t \to \infty,
    \end{equation}
    which yields the result by \eqref{eq:limit_det_lipschitz}.
    \end{proof}

\section{Proofs of supporting lemmas}\label{sec:technical_lemmas}

This appendix provides proofs of the supporting technical lemmas used in Appendix \ref{sec:proof_result_iii}.

\subsection{Proof of Lemma \ref{lemma:distortion_trick}}\label{subsec:proof_distortion_trick}

\begin{proof}

First, define the high-probability event
\begin{equation}\label{eq:invertibility_event}
\mathcal{E}
:= 
\bigl\{\Re(z) \le \lambda_{\min}\left(\bR\widehat{\bG}\bR^\sT\right) - \varepsilon_{\cK} \bigr\},
\end{equation}
where $\varepsilon_{\cK}>0$ is the gap defined in the beginning of Section~\ref{subsec:concentration}. 

By Lemma~\ref{lemma:smallest_eigenvalue_concentration}, we have $\mathbb{P}(\mathcal{E})\to 1$ as $n,d\to\infty$. On the event $\mathcal{E}$, both resolvents are well-defined and satisfy the operator-norm bounds
\begin{equation}\label{eq:resolvent_op_bounds}
\bigl\|(\bR\widehat{\bG}\bR^\sT - z\bI)^{-1}\bigr\|_{\mathrm{op}}
\le \frac{1}{\varepsilon_{\cK}}
\quad\text{and}\quad
\bigl\|(\bR\widehat{\bG}\bR^\sT - (z+i\eta_n)\bI)^{-1}\bigr\|_{\mathrm{op}}
\le \frac{1}{\varepsilon_{\cK}}.
\end{equation}

Second, define the bilinear forms
\begin{align}
S(z)
&:= \ba^\sT \bR^\sT (\bR\widehat{\bG}\bR^\sT - z\bI)^{-1}\bR \bb
\nonumber\\
\text{and}\quad
S(z+i\eta_n)
&:= \ba^\sT \bR^\sT (\bR\widehat{\bG}\bR^\sT - (z+i\eta_n)\bI)^{-1}\bR \bb.
\label{eq:bilinear_forms}
\end{align}
Using the fact that $\bA^{-1}-\bB^{-1} = \bB^{-1}(\bB-\bA)\bA^{-1}$ with $\bA=\bR\widehat{\bG}\bR^\sT - z\bI$ and $\bB=\bR\widehat{\bG}\bR^\sT - (z+i\eta_n)\bI$, we obtain 
\begin{equation}\label{eq:perturbation_bound_step1}
\begin{aligned}
&\abs{S(z)-S(z+i\eta_n)}= \eta_n \abs{\ba^\sT \bR^\sT
(\bR\widehat{\bG}\bR^\sT - z\bI)^{-1}
(\bR\widehat{\bG}\bR^\sT - (z+i\eta_n)\bI)^{-1}
\bR \bb }
&&
\\[3pt]
&\qquad\le \eta_n \|\ba\|\|\bb\|\|\bR\|_{\mathrm{op}}^2
\bigl\|(\bR\widehat{\bG}\bR^\sT - z\bI)^{-1}\bigr\|_{\mathrm{op}}
\bigl\|(\bR\widehat{\bG}\bR^\sT - (z+i\eta_n)\bI)^{-1}\bigr\|_{\mathrm{op}}
&&
\\[3pt]
&\qquad\le \frac{\eta_n}{\varepsilon_{\cK}^2}\|\ba\|\|\bb\|\|\bR\|_{\mathrm{op}}^2.
&&\text{(bound \eqref{eq:resolvent_op_bounds})}
\end{aligned}
\end{equation}

Since $\bR\in\mathbb{R}^{n\times d_{\star}}$ has i.i.d. $\normal(0,1)$ entries (where we recall that $d_{\star}=d-k-m-mt$), there exists a constant $C$ such that, with probability $1-o(1)$
\cite{vershynin2018high}
\begin{equation}\label{eq:R_op_norm}
\|\bR\|_{\mathrm{op}} \le C\left(\sqrt{n}+\sqrt{d}\right)
\end{equation}
for some absolute constant $C>0$.  
Conditioning on the events $\mathcal{E}$ and \eqref{eq:R_op_norm} (which together hold with probability $1-o(1)$), and substituting into \eqref{eq:perturbation_bound_step1}, we obtain
\begin{equation}\label{eq:final_perturbation}
\abs{S(z)-S(z+i\eta_n)}
\ \le\
\frac{C'\eta_n(\sqrt{n}+\sqrt{d})^2}{\varepsilon_{\cK}^2}\|\ba\|\|\bb\|,
\end{equation}
for some universal constant $C'>0$ independent of $n,d$.

To conclude, we analyze the right-hand side of \eqref{eq:final_perturbation}. By the assumption of Lemma~\ref{lemma:distortion_trick}, we have $\|\ba\|, \|\bb\| = O_P(1/\sqrt{n})$, which implies $\|\ba\|\|\bb\| = O_P(1/n)$. In addition, we have $(\sqrt{n}+\sqrt{d})^2 = O(n)$.
Combining these observations, the bound \eqref{eq:final_perturbation} becomes
\begin{equation}\label{eq:final_scaling}
\abs{S(z)-S(z+i\eta_n)}
\le \frac{C'\eta_n \cdot O(n)}{\varepsilon_{\cK}^2} \cdot O_P(1/n)
= O_P(\eta_n).
\end{equation}
Since $\eta_n = 1/\log n \to 0$ as $n\to\infty$, we conclude $S(z+i\eta_n) - S(z) \xrightarrow{~p~} 0$. The convergence holds unconditionally since both events $\mathcal{E}$ and \eqref{eq:R_op_norm} occur with probability $1-o(1)$.
\end{proof}
\subsection{Proof of Lemma \ref{lemma:estimation_cross_term_variance}}\label{subsec:proof_cross_term_variance}

\begin{proof}
We establish the bound via decomposing the conditional second moment into diagonal and off-diagonal terms, then invoking Gaussian moment calculations and the leave-one-out technique. Throughout, we use the resolvent bound $\|\bQ_{-i}\|_{\mathrm{op}} \le (1/\eta_n)$ (since $\Im(z+i\eta_n) \geq \eta_n$).

First, recall from Section~\ref{subsec:concentration} that
\begin{equation}\label{eq:cross_term_def_reminder}
S_{\mathsf{cross}} = \sum_{i=1}^n T_i,
\quad\text{where}\quad
T_i := \frac{1}{n^2}\sum_{j\neq i} h_{ij}\br_i^\sT \bQ\br_j,
\quad
h_{ij} := n^2 a_i b_j.
\end{equation}
The second moment can be expanded as 
\begin{equation}\label{eq:second_moment_expansion}
\mathbb{E}\!\left[\abs{S_{\mathsf{cross}}}^2\,\middle|\,\cF_t\right]
= \mathbb{E}\left[S_{\mathsf{cross}}\overline{S_{\mathsf{cross}}}\,\middle|\,\cF_t\right]
= \sum_{i=1}^n \mathbb{E}\left[T_i\overline{T_i}\,\middle|\,\cF_t\right]
+ \sum_{i\neq j}\mathbb{E}\left[T_i\overline{T_j}\,\middle|\,\cF_t\right].
\end{equation}
We bound the two terms respectively.

For the off-diagonal terms, for $i\neq j$, we expand the product
\begin{equation}\label{eq:cross_off_diagonal}
\begin{aligned}
\mathbb{E}\left[T_i\overline{T_j}\,\middle|\,\cF_t\right]
&= \frac{1}{n^4}
\mathbb{E}\left[
\Bigl(\sum_{k\neq i} h_{ik}\br_i^\sT \bQ\br_k\Bigr)
\Bigl(\sum_{l\neq j} h_{jl}\overline{\br_j^\sT \bQ\br_l}\Bigr)\,\middle|\,\cF_t\right]
&&
\\
&= \frac{1}{n^4}
\mathbb{E}\left[
h_{ij}^2(\br_i^\sT \bQ\br_j)
\overline{(\br_i^\sT \bQ\br_j)}
\,\middle|\,\cF_t\right]
&&\text{(only $(k,l)=(j,i)$ survives)}
\end{aligned}
\end{equation}
where all other cross-indexed terms have zero expectations.

Conditioning on $\{\br_k\}_{k\neq i}$ (which determines both $\bQ_{-i}$ and $\br_j$ since $j\neq i$), and applying the leave-one-out formula $\br_i^\sT \bQ\br_j = \tfrac{\br_i^\sT \bQ_{-i}\br_j}{1+\widehat{g}_i\br_i^\sT \bQ_{-i}\br_i}$, we obtain
\begin{equation}\label{eq:off_diagonal_bound}
\begin{aligned}
\mathbb{E}\left[h_{ij}^2 \abs{\br_i^\sT \bQ \br_j}^2\,\middle|\,\cF_t\right]
&= \mathbb{E}\left[h_{ij}^2 \mathbb{E}\!\left[\abs{\br_i^\sT \bQ \br_j}^2\,\middle|\,\{\br_k\}_{k\neq i},\cF_t\right]\right]
&&
\\
&= \mathbb{E}\left[h_{ij}^2 \mathbb{E}\!\left[\abs{\frac{\br_i^\sT \bQ_{-i}\br_j}{1+\widehat{g}_i\br_i^\sT \bQ_{-i}\br_i}}^2\,\middle|\,\{\br_k\}_{k\neq i},\cF_t\right]\right]
&&\text{(leave-one-out)}
\\
& \leq \frac{C^2}{\eta_n^2}\mathbb{E}\left[h_{ij}^2\,\mathbb{E}\!\left[\abs{\br_i^\sT \bQ_{-i}\br_j}^2\,\middle|\,\{\br_k\}_{k\neq i},\cF_t\right]\right]    
&&\text{(Lemma~\ref{lem:bounded-denominators})}
\\
&= \frac{C^2}{\eta_n^2}\mathbb{E}\left[h_{ij}^2\,\mathbb{E}\!\left[\|\bQ_{-i}\br_j\|^2\,\middle|\,\{\br_k\}_{k\neq i},\cF_t\right]\right]
&&\text{($\br_i\sim\normal(\boldsymbol{0},\bI_{d_{\star}})$ indep.)}
\\
&= \frac{C^2}{\eta_n^2}\mathbb{E}\!\left[h_{ij}^2\|\bQ_{-i}\br_j\|^2\,\middle|\,\cF_t\right]
&&
\\
& \leq \frac{C^2}{\eta_n^2}\mathbb{E}\!\left[h_{ij}^2\|\bQ_{-i}\|_{\mathrm{op}}^2\|\br_j\|^2\,\middle|\,\cF_t\right]\leq \frac{d_{\star} C^2}{\eta_n^4}h_{ij}^2,
&&\text{($\|\bQ_{-i}\|_{\mathrm{op}}\le(1/\eta_n)$)}
\end{aligned}
\end{equation}
where $d_{\star}=d-k-m-mt$. Summing over all $i\neq j$ and defining $M_n := (1/n^2)\sum_{i,j}h_{ij}^2 = n^2\|\ba\|^2\|\bb\|^2$, we obtain the off-diagonal bound
\begin{equation}\label{eq:off_diagonal_sum}
\sum_{i\neq j}\mathbb{E}\left[T_i\overline{T_j}\,\middle|\,\cF_t\right]
= \frac{1}{n^4}\sum_{i\neq j} \mathbb{E}\left[h_{ij}^2 \abs{\br_i^\sT \bQ \br_j}^2\,\middle|\,\cF_t\right]
\leq \frac{d_{\star} M_nC^2}{n^2\eta_n^4}.
\end{equation}

For the diagonal terms, we write $T_i$ using the leave-one-out formula as
\begin{equation}\label{eq:T_i_leave_one_out}
T_i
= \frac{1}{n^2} \frac{1}{1+\widehat{g}_i \br_i^\sT \bQ_{-i}\br_i}\br_i^\sT \left(\bQ_{-i}\sum_{j\neq i} h_{ij}\br_j\right).
\end{equation}
Taking the expectation for $\abs{T_i}^2$ and conditioning on $\{\br_j\}_{j\neq i}$, we obtain
\begin{equation}\label{eq:diagonal_bound}
\begin{aligned}
\mathbb{E}\!\left[\abs{T_i}^2\,\middle|\,\cF_t\right]
&= \frac{1}{n^4}\mathbb{E}\!\left[\abs{\frac{1}{1+\widehat{g}_i \br_i^\sT \bQ_{-i}\br_i}\br_i^\sT \left(\bQ_{-i}\sum_{j\neq i} h_{ij}\br_j\right)}^2\,\middle|\,\cF_t\right]
&&
\\
&\leq \frac{C^2}{n^4\eta_n^2}\mathbb{E}\!\left[\abs{\br_i^\sT \left(\bQ_{-i}\sum_{j\neq i} h_{ij}\br_j\right)}^2\,\middle|\,\cF_t\right]
&&\text{(Lemma~\ref{lem:bounded-denominators})}
\\
& \leq  \frac{C^2}{n^4\eta_n^2}\mathbb{E}\!\left[\left\|\bQ_{-i}\sum_{j\neq i} h_{ij}\br_j\right\|^2\,\middle|\,\cF_t\right]
\leq \frac{C^2}{n^4\eta_n^4}\mathbb{E}\!\left[\left\|\sum_{j\neq i} h_{ij}\br_j\right\|^2\,\middle|\,\cF_t\right]
&&\text{($\|\bQ_{-i}\|\le(1/\eta_n)$)}
\\
&= \frac{C^2}{n^4\eta_n^4}\sum_{j\neq i} h_{ij}^2\mathbb{E}\!\left[\|\br_j\|^2\,\middle|\,\cF_t\right]
= \frac{d_{\star} C^2}{n^4\eta_n^4}\sum_{j\neq i} h_{ij}^2. &&\text{(independence)}
\end{aligned}
\end{equation}
Summing over all $i\in[n]$, we obtain
\begin{equation}\label{eq:diagonal_sum}
\sum_{i=1}^n \mathbb{E}\!\left[\abs{T_i}^2\,\middle|\,\cF_t\right]
\leq \sum_{i=1}^n \frac{d_{\star} C^2}{n^4\eta_n^4}\sum_{j\neq i} h_{ij}^2
\leq \frac{d_{\star} M_nC^2}{n^2\eta_n^4}.
\end{equation}

Finally, combining the diagonal bound \eqref{eq:diagonal_sum} and the off-diagonal bound \eqref{eq:off_diagonal_sum} into the equation \eqref{eq:second_moment_expansion}, we conclude
\begin{equation}\label{eq:final_variance_bound}
\mathbb{E}\!\left[\abs{S_{\mathsf{cross}}}^2\,\middle|\,\cF_t\right]
\leq \frac{d_{\star} M_nC^2}{n^2\eta_n^4}
+ \frac{d_{\star} M_nC^2}{n^2\eta_n^4}
= \frac{2d_{\star} M_nC^2}{n^2\eta_n^4}.
\end{equation}
This completes the proof.
\end{proof}

\subsection{Proof of Lemma \ref{lem:bounded-denominators}}\label{subsec:proof_denominator_bound}

\begin{proof}
Define the companion resolvent
\begin{equation}\label{eq:companion_resolvent}
\underline{\bQ}(z)
:= \left(\widehat{\bG}\bR^\sT\bR - z\bI_n\right)^{-1}
\in \mathbb{R}^{n\times n}.
\end{equation}
By the Woodbury matrix identity (or direct verification), we have the following resolvent relation
\begin{equation}\label{eq:companion_identity}
\underline{\bQ}(z)
= -\frac{1}{z}\bI_n
+ \frac{1}{z}\widehat{\bG}\bR^\sT \bQ(z)\bR,
\quad\text{where}\quad
\bQ(z) := \left(\bR^\sT\widehat{\bG}\bR - z\bI_{d_{\star}}\right)^{-1}.
\end{equation}
This formula couples the two resolvents acting on spaces of different dimensions.

Taking the $(i,i)$ diagonal entry of both sides of \eqref{eq:companion_identity}, we obtain
\begin{equation}\label{eq:diagonal_companion}
(\underline{\bQ}(z))_{ii}
= -\frac{1}{z}
+ \frac{\widehat{g}_i}{z}\br_i^\sT \bQ(z)\br_i,
\end{equation}
where $\br_i\in\mathbb{R}^{d_{\star}}$ is the $i$-th row of $\bR$ and $\widehat{g}_i$ is the $i$-th diagonal entry of $\widehat{\bG}$. We then invoke the leave-one-out formula,
\begin{equation}\label{eq:leave_one_out_substitution}
\br_i^\sT \bQ(z)\br_i
= \frac{\br_i^\sT \bQ_{-i}(z)\br_i}{1+\widehat{g}_i \br_i^\sT \bQ_{-i}(z)\br_i},
\end{equation}
and substituting into \eqref{eq:diagonal_companion}, we obtain
\begin{equation}\label{eq:denom_extraction}
(\underline{\bQ}(z))_{ii}
= -\frac{1}{z}
\frac{1}{1+\widehat{g}_i \br_i^\sT \bQ_{-i}(z)\br_i}.
\end{equation}
Rearranging yields
\begin{equation}\label{eq:denominator_formula}
1+\widehat{g}_i \br_i^\sT \bQ_{-i}(z)\br_i
= -\frac{1}{z(\underline{\bQ}(z))_{ii}}.
\end{equation}

For $z+i\eta$ with $\eta>0$, the matrix $\widehat{\bG}\bR^\sT\bR - (z+i\eta)\bI_n$ has eigenvalues bounded away from zero by $\eta$. Therefore, the following resolvent bound holds
\begin{equation}\label{eq:companion_op_bound}
\|\underline{\bQ}(z+i\eta)\|_{\mathrm{op}}
= \left\|\left(\widehat{\bG}\bR^\sT\bR - (z+i\eta)\bI_n\right)^{-1}\right\|_{\mathrm{op}}
\le \frac{1}{\eta}.
\end{equation}
From \eqref{eq:denominator_formula}, we have
\begin{equation}\label{eq:abs_denominator}
\begin{aligned}
\abs{1+\widehat{g}_i \br_i^\sT \bQ_{-i}(z+i\eta)\br_i}
&= \abs{\frac{1}{(z+i\eta)\,(\underline{\bQ}(z+i\eta))_{ii}}}
&&
\\
&\ge \frac{1}{\|z+i\eta\|}
\frac{1}{\|\underline{\bQ}(z+i\eta)\|_{\mathrm{op}}}
&&
\\
&\ge \frac{1}{C'}\eta
&&\text{(bound \eqref{eq:companion_op_bound} and $|z+i\eta|\le C'$)}
\end{aligned}
\end{equation}
for some constant $C'>0$ depending only on the magnitude of $z$. Since the considered domain of $z$ is bounded, $C'$ is an absolute constant.
\end{proof}

\subsection{Proof of Lemma \ref{lemma:dtilde_approximation}}\label{subsec:proof_conv_dtilde_d_approx}
\begin{proof}
    Write $\Delta := n/\delta - d_{\star}$. Since $n/\delta = \delta_n d/\delta = d + o(d)$ and $d_{\star} = d - C_0$, we have $|\Delta| = o(d)$.
    In addition, we have
    \begin{equation}\label{eq:conv_main_diff}
    \begin{aligned}
    \frac{\alpha(z)}{1 + d_{\star}\widehat{g}_i\alpha(z)} - \frac{\delta\alpha(z)}{\delta + n\widehat{g}_i\alpha(z)}
    &= \frac{\Delta \cdot \widehat{g}_i\alpha^2(z)}{(1 + d_{\star}\widehat{g}_i\alpha(z))(1 + n\widehat{g}_i\alpha(z)/\delta)}.
    \end{aligned}
    \end{equation}
    By Lemma~\ref{lem:bounded-denominators}, there exists $c_{\cK} > 0$ such that $|1 + d_{\star}\widehat{g}_i\alpha(z)| \ge c_{\cK}$ and $|\delta + n\widehat{g}_i\alpha(z)| \ge \delta c_{\cK}$ uniformly for $z \in \cK$ and $i\in [n]$. Since $|\widehat{g}_i| \le C/n$, $|\alpha(z)| \le C_\cK$, and $|\Delta| = o(d)$, we obtain from \eqref{eq:conv_main_diff}
    \begin{equation}\label{eq:conv_pointwise_bound}
    \left|\frac{\alpha(z)}{1 + d_{\star}\widehat{g}_i\alpha(z)} - \frac{\delta\alpha(z)}{\delta + n\widehat{g}_i\alpha(z)}\right|
    \le \frac{C_\cK \cdot |\Delta| \cdot (C/n) \cdot C_\cK}{c_{\cK}^2}
    = o(1).
    \end{equation}
    Summing over $i\in[n]$ and using $\|\ba\|\|\bb\| = O_P(1/n)$ we get
    \begin{equation}\label{eq:conv_total_error}
    \begin{aligned}
    &\left|\sum_{i=1}^n a_i b_i \left(\frac{\alpha(z)}{1 + d_{\star}\widehat{g}_i\alpha(z)} - \frac{\delta\alpha(z)}{\delta + n\widehat{g}_i\alpha(z)}\right)\right| \\
    &\quad\quad\quad\quad\le o(1) \sum_{i=1}^n |a_i b_i|
    \le o(1) \cdot \|\ba\| \|\bb\|  = o_P(1/d)
    &&\text{(Cauchy-Schwarz)} 
    \end{aligned}
    \end{equation}
    which establishes \eqref{eq:conv_dtilde_d_approx}.
\end{proof}

\subsection{Proof of Lemma \ref{lemma:det_lipschitz}}\label{subsec:proof_det_lipschitz}
\begin{proof}
    Recall the outlier matrix from Section~\ref{subsec:spike_basic}:
    \begin{equation}\label{eq:eig_Mj_structure}
    \bM_j(z;t)
    = \begin{bmatrix}
    \bZ^\sT\bG\bZ & \bI \\
    \bI & \boldsymbol{0}
    \end{bmatrix}^{-1}
    + \begin{bmatrix}
    \bTheta_{*;t}^\sT \\
    \bZ^\sT\bG\bY
    \end{bmatrix}
    \bR_0(z)
    \begin{bmatrix}
    \bTheta_{*;t} & \bY^\sT\bG\bZ
    \end{bmatrix},
    \end{equation}
    where $\bZ = \bX\bTheta_{*;t}(\bTheta_{*;t}^\sT\bTheta_{*;t})^{-1}$, $\bY = \bP_{\bF_{0;t}}^{\perp}\bX_{\text{new}}\bP_{\bTheta_{*;t}}^{\perp}$, $\bG = \bG_j(t)$, and $\bTheta_{*;t}, \bF_{0;t}$ are as defined in Section~\ref{subsec:spike_basic}. The resolvent $\bR_0(z) := (\bY^\sT\bG\bY - z\bI)^{-1}$ was introduced in Section~\ref{subsec:spike_basic}.
    
    By the resolvent identity, for $z, z' \in \cK$,
    \begin{equation}\label{eq:eig_resolvent_identity}
    \bR_0(z) - \bR_0(z')
    = (z' - z)\bR_0(z)\bR_0(z').
    \end{equation}
    From the Lemma~\ref{lemma:smallest_eigenvalue_concentration}, for $z \in \cK$ we have $\min_{i\in[n]}|\lambda_i(\bY^\sT\bG\bY) - z| \ge \varepsilon_{\cK}$ with probability $1-o(1)$, where $\varepsilon_{\cK} > 0$ depends on $\cK$ but is uniform over $n,d$. Therefore,
    \begin{equation}\label{eq:eig_resolvent_bound}
    \sup_{z \in \cK}\|\bR_0(z)\|_{\mathrm{op}}
    \le \frac{1}{\varepsilon_{\cK}}
    \quad\text{with probability }1 - o(1).
    \end{equation}
    Combining \eqref{eq:eig_resolvent_identity} and \eqref{eq:eig_resolvent_bound}, we obtain
    \begin{equation}\label{eq:eig_resolvent_diff}
    \|\bR_0(z) - \bR_0(z')\|_{\mathrm{op}}
    \le \frac{|z - z'|}{\varepsilon_{\cK}^2}
    \quad\text{with probability }1 - o(1).
    \end{equation}
    
    Define the block matrices
    \begin{equation}\label{eq:eig_blocks}
    \bA := \begin{bmatrix}
    \bTheta_{*;t}^\sT \\
    \bZ^\sT\bG\bY^\sT
    \end{bmatrix},
    \quad
    \bB := \begin{bmatrix}
    \bTheta_{*;t} & \bY\bG\bZ
    \end{bmatrix},
    \quad
    \bC := \begin{bmatrix}
    \bZ^\sT\bG\bZ & \bI \\
    \bI & \boldsymbol{0}
    \end{bmatrix}^{-1}.
    \end{equation}
    From \eqref{eq:eig_Mj_structure}, the $z$-dependence enters only through $\bR_0(z)$:
    \begin{equation}\label{eq:eig_Mj_diff}
    \bM_j(z;t) - \bM_j(z';t)
    = \bA(\bR_0(z) - \bR_0(z'))\bB.
    \end{equation}
    By submultiplicativity and \eqref{eq:eig_resolvent_diff},
    \begin{equation}\label{eq:eig_Mj_op_diff}
    \|\bM_j(z;t) - \bM_j(z';t)\|_{\mathrm{op}}
    \le \|\bA\|_{\mathrm{op}}\|\bR_0(z) - \bR_0(z')\|_{\mathrm{op}}\|\bB\|_{\mathrm{op}}
    \le \frac{|z - z'|}{\varepsilon_{\cK}^2}\|\bA\|_{\mathrm{op}}\|\bB\|_{\mathrm{op}}.
    \end{equation}
    By Lemma~\ref{lemma_amp_data} and the results established in Section~\ref{subsec:spike_basic},
    \begin{equation}\label{eq:eig_AB_bounds}
    \|\bA\|_{\mathrm{op}} = O_P(1),
    \quad
    \|\bB\|_{\mathrm{op}} = O_P(1),
    \quad
    \sup_{z \in \cK}\|\bM_j(z;t)\|_{\mathrm{op}} = O_P(1).
    \end{equation}
    Therefore, from \eqref{eq:eig_Mj_op_diff},
    \begin{equation}\label{eq:eig_Mj_Lipschitz_op}
    \|\bM_j(z;t) - \bM_j(z';t)\|_{\mathrm{op}}
    \le C_n |z - z'|
    \quad\text{where }C_n = O_P(1).
    \end{equation}
    
    Finally, invoke the standard determinant Lipschitz bound: for matrices $\bX, \bY \in \mathbb{R}^{p \times p}$,
    \begin{equation}\label{eq:eig_det_Lipschitz_general}
    |\det(\bX) - \det(\bY)|
    \lesssim (\|\bX\|_{\mathrm{op}}^{p-1} + \|\bY\|_{\mathrm{op}}^{p-1})\|\bX - \bY\|_{\mathrm{op}}.
    \end{equation}
    Here for $\bM_j(z;t)$ the dimension $p=O(1)$. Applying \eqref{eq:eig_det_Lipschitz_general} with $\bX = \bM_j(z;t)$ and $\bY = \bM_j(z';t)$, and using \eqref{eq:eig_AB_bounds} and \eqref{eq:eig_Mj_Lipschitz_op}, we obtain \eqref{eq:eig_det_lipschitz}.
    \end{proof}

\subsection{Proof of Lemma \ref{lemma:uniform_integrand}}\label{subsec:proof_uniform_integrand}

We first establish that $\bM_j(z;t)^{-1}$ converges to $M_j^\infty(z;t)^{-1}$ 
uniformly over $z\in \gamma$.
\begin{lemma}\label{lemma:uniform_inverse_convergence}
Choose the contour $\gamma$ such that $\inf_{z \in \gamma}|\det(M_j^\infty(z;t))| > 0$. Then
\begin{equation}\label{eq:evec_uniform_inverse}
\sup_{z \in \gamma}\|\bM_j(z;t)^{-1} - M_j^\infty(z;t)^{-1}\|_{\mathrm{op}}
\xrightarrow{~p~} 0.
\end{equation}
\end{lemma}

\begin{proof}
By Lemma~\ref{lemma:limiting_spike_matrix}, we have $\sup_{z \in \gamma}\|\bM_j(z;t) - M_j^\infty(z;t)\|_{\mathrm{op}} \xrightarrow{~p~} 0$. Since $\inf_{z \in \gamma}|\det(M_j^\infty(z;t))| > 0$ and $\det(M_j^\infty(z;t))$ is continuous on the compact set $\gamma$, there exists $c_\gamma > 0$ such that $\sup_{z \in \gamma}\|M_j^\infty(z;t)^{-1}\|_{\mathrm{op}} \le c_\gamma^{-1}$. By uniform convergence, with probability $1 - o(1)$, we have $\inf_{z \in \gamma}|\det(\bM_j(z;t))| \ge c_\gamma/2 > 0$, hence $\sup_{z \in \gamma}\|\bM_j(z;t)^{-1}\|_{\mathrm{op}} \le C_\gamma$ for some $C_\gamma < \infty$, and large enough $n$.

Using the identity $\bM_j(z;t)^{-1} - M_j^\infty(z;t)^{-1} = \bM_j(z;t)^{-1}(M_j^\infty(z;t) - \bM_j(z;t))M_j^\infty(z;t)^{-1}$ and submultiplicativity,
\begin{equation}\label{eq:evec_inverse_diff}
\begin{aligned}
\|\bM_j(z;t)^{-1} - M_j^\infty(z;t)^{-1}\|_{\mathrm{op}}
&\le \|\bM_j(z;t)^{-1}\|_{\mathrm{op}}\|M_j^\infty(z;t) - \bM_j(z;t)\|_{\mathrm{op}}\|M_j^\infty(z;t)^{-1}\|_{\mathrm{op}} \\
&\le C_\gamma \cdot \|\bM_j(z;t) - M_j^\infty(z;t)\|_{\mathrm{op}} \cdot c_\gamma^{-1},
\end{aligned}
\end{equation}
which converges to zero uniformly in probability over $z \in \gamma$ by Lemma~\ref{lemma:limiting_spike_matrix}.
\end{proof}

\begin{proof}[Proof of Lemma~\ref{lemma:uniform_integrand}]
    We decompose $F_{n,d}(z)$ as a product of three parts:
    \begin{equation}\label{eq:evec_F_decomposition}
    F_{n,d}(z) = \bA(z)^\sT \cdot \bM_j(z;t)^{-1} \cdot \bA(z),
    \end{equation}
    where
    \begin{equation}\label{eq:evec_A_def}
    \bA(z) := \begin{bmatrix}\bTheta_{*;t} & \bY^\sT\bG\bZ\end{bmatrix}^\sT \bR_0(z)\bTheta_{*\mathrm{H}}
    = \begin{bmatrix}
    \bTheta_{*;t}^\sT\bR_0(z)\bTheta_{*\mathrm{H}} \\
    \bZ^\sT\bG\bY\bR_0(z)\bTheta_{*\mathrm{H}}
    \end{bmatrix}
    \in \mathbb{R}^{2(k+m+mt) \times r}.
    \end{equation}
    We establish uniform convergence with respect to $z \in \gamma$ for $\bA(z)$.

First, by Section~\ref{subsec:spike_basic}, specifically \eqref{eq:spike_11_block}, we have
    \begin{equation}\label{eq:evec_R0_on_W}
    \bTheta_{*;t}^\sT\bR_0(z)\bTheta_{*;t} = \bTheta_{*;t}^\sT(\bY^\sT\bG\bY - z\bI)^{-1}\bTheta_{*;t}
    = -\frac{1}{z}\bTheta_{*;t}^\sT\bTheta_{*;t}.
    \end{equation}
    By Lemma~\ref{lemma_amp_feature}, $\bTheta_{*\mathrm{H}}^\sT\bTheta_{*;t} \xrightarrow{~p~} [I_r \ 0]$. Therefore,
    \begin{equation}\label{eq:evec_upper_convergence}
    \bTheta_{*;t}^\sT\bR_0(z)\bTheta_{*\mathrm{H}}
    = -\frac{1}{z}\bTheta_{*;t}^\sT\bTheta_{*\mathrm{H}}
    \xrightarrow{~p~} -\frac{1}{z}[I_r \ 0]^\sT
    \end{equation}
    uniformly over $z \in \gamma$ (since $\gamma$ is compact and bounded away from $z = 0$).
    
    Second, recall that $\bZ^\sT\bG\bY\bR_0(z)\bTheta_{*;t} = 0$. Since $\bTheta_{*\mathrm{H}}$ consists of the first $r$ columns of $\bTheta_{*;t}$, we have immediately
    \begin{equation}\label{eq:evec_lower_limit}
    \bZ^\sT\bG\bY\bR_0(z)\bTheta_{*\mathrm{H}} = 0
    \end{equation}
    for all $z \in \gamma$.
    
    Combining \eqref{eq:evec_upper_convergence} and \eqref{eq:evec_lower_limit}, we obtain
    \begin{equation}\label{eq:evec_A_convergence}
    \sup_{z \in \gamma}\left\|\bA(z) - A_\infty(z)\right\|_{\mathrm{op}}
    \xrightarrow{~p~} 0,
    \end{equation}
    where
    \begin{equation}\label{eq:evec_A_infinity}
    A_\infty(z) := -\frac{1}{z}\begin{bmatrix} I_r \\ 0 \end{bmatrix} \in \mathbb{R}^{2(k+m+mt) \times r}.
    \end{equation}
    By Lemma~\ref{lemma:uniform_inverse_convergence}, we have
    \begin{equation}\label{eq:evec_M_convergence_recall}
    \sup_{z \in \gamma}\|\bM_j(z;t)^{-1} - M_j^\infty(z;t)^{-1}\|_{\mathrm{op}}
    \xrightarrow{~p~} 0.
    \end{equation}
    Moreover, $\bA(z)$, $\bM_j(z;t)^{-1}$, and their limits are uniformly $O_P(1)$ on $\gamma$.
    
    From the decomposition
    \begin{equation}\label{eq:evec_F_difference}
    \begin{aligned}
    F_{n,d}(z) - F_\infty(z)
    &= \bA(z)^\sT\bM_j(z;t)^{-1}\bA(z) - A_\infty(z)^\sT M_j^\infty(z;t)^{-1}A_\infty(z) \\
    &= (\bA(z) - A_\infty(z))^\sT\bM_j(z;t)^{-1}\bA(z) \\
    &+ A_\infty(z)^\sT(\bM_j(z;t)^{-1} - M_j^\infty(z;t)^{-1})\bA(z) \\
    &+ A_\infty(z)^\sT M_j^\infty(z;t)^{-1}(\bA(z) - A_\infty(z)),
    \end{aligned}
    \end{equation}
    we obtain by submultiplicativity and \eqref{eq:evec_A_convergence}, \eqref{eq:evec_M_convergence_recall} that
    \begin{equation}\label{eq:evec_F_uniform_final}
    \sup_{z \in \gamma}\|F_{n,d}(z) - F_\infty(z)\|_{\mathrm{op}}
    \xrightarrow{~p~} 0,
    \end{equation}
    which completes the proof of \eqref{eq:evec_uniform_F}.
    \end{proof}

\section{Additional experiments}\label{sec:additional_simulation}
\subsection{Results for networks with moderately large width}

Numerically verifying our results and thresholds for large width networks is computationally prohibitive. Nevertheless, we provide numerical simulations for moderately large width networks to give insights into the phase transition phenomena.

We extend our experiments to multi-neuron networks. For these experiments, we still use $\GeLU$ activation, Huber loss with parameter $M=1$, and fixed learning rate $\eta=1.5$. We consider widths $m\in\{1,2,3,5,10\}$ with $a_j=1$ and $b_j=0$ for all neurons $j\in[m]$. The correlation is defined as the maximum over all $m$ neurons:
\begin{equation}
\rho(t) := \max_{j\in[m]}\left|\frac{\btheta_j(t)^\sT\btheta_*}{\|\btheta_j(t)\|}\right|.
\end{equation}
Figure~\ref{fig:phase_transition_m5} presents the success rate and the final correlation for $m=5$ across varying dimensions, exhibiting sharp phase transitions around $\delta=3.5$.

We further investigate the proportional scaling regime where the network width grows proportionally with dimension, maintaining a fixed ratio $m/d = 0.01$. In these experiments, we consider pairs $(m,d) \in \{(4,400), (6,600), (8,800), (10,1000), (12,1200)\}$, with the same experimental setup as before: $\GeLU$ activation, Huber loss with parameter $M=1$, learning rate $\eta=1.5$, and all biases set to zero.
Figure~\ref{fig:proportional_width} shows the phase transition behavior under proportional width scaling. Interestingly, we observe that as both $m$ and $d$ increase proportionally, the phase transition threshold shifts a bit to the right.

Figure~\ref{fig:phase_transition_large_width} (left panel) shows success rate versus $\delta$ for different widths at fixed dimension $d=1000$. As width increases, the phase transition threshold changes. Whether the threshold continues to decrease as width increases remains an open question due to computational constraints that limit us to moderately large widths and dimensions.
Figure~\ref{fig:phase_transition_large_width} (right panel) shows the final correlation for the same widths at fixed dimension $d=1000$.

\begin{figure}[!t]
  \centering
  \includegraphics[width=0.49\textwidth]{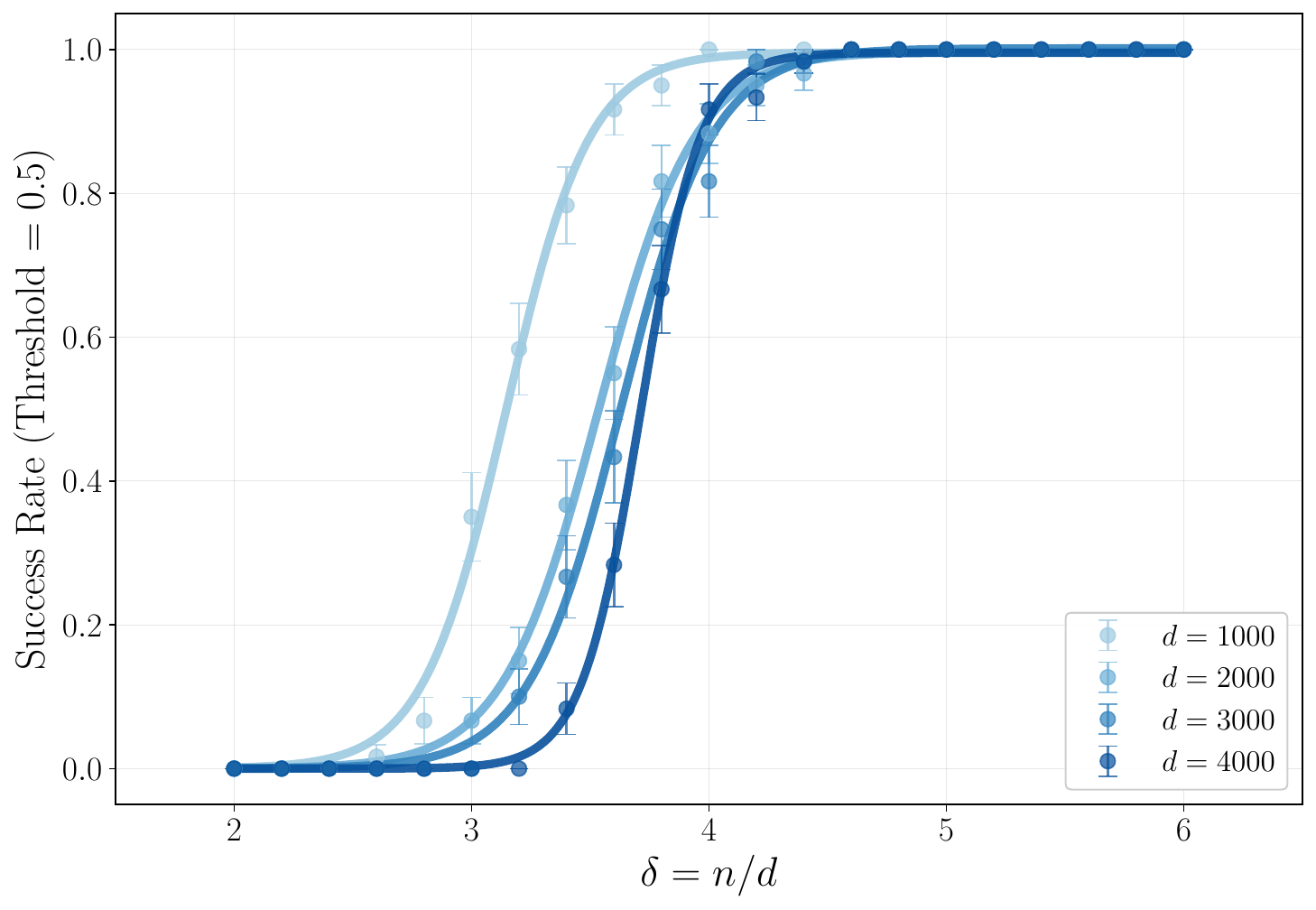}
  \hfill
  \includegraphics[width=0.49\textwidth]{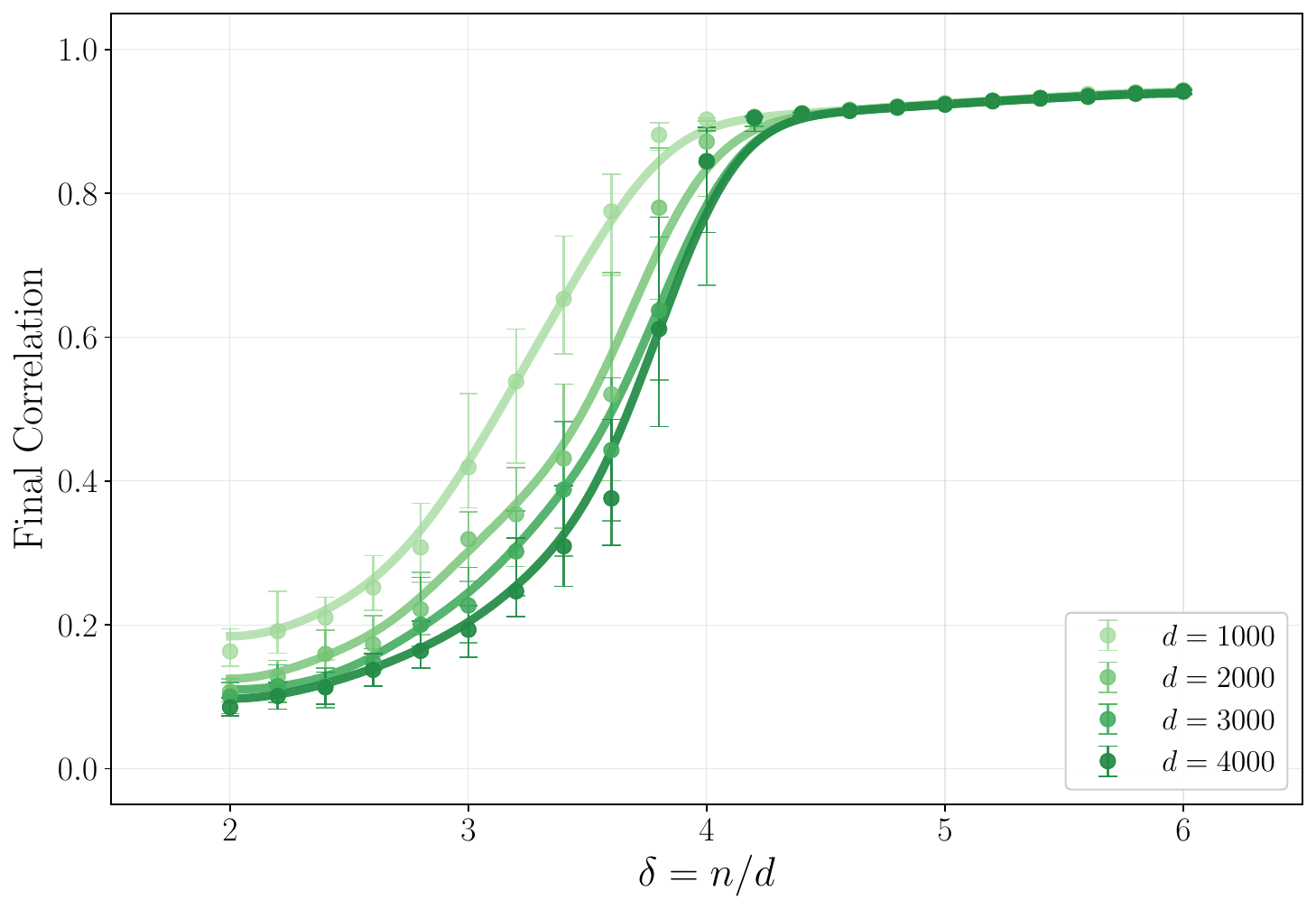}
  \vspace{-0.1cm}
  \caption{Phase transitions for moderately large width learners with $\GeLU$ activation. Left: Success rate for $m=5$, showing sharp transitions around $\delta \approx 3.5$. Right: Final correlation.}
  \label{fig:phase_transition_m5}
\end{figure}

\begin{figure}[!t]
  \centering
  \includegraphics[width=0.49\textwidth]{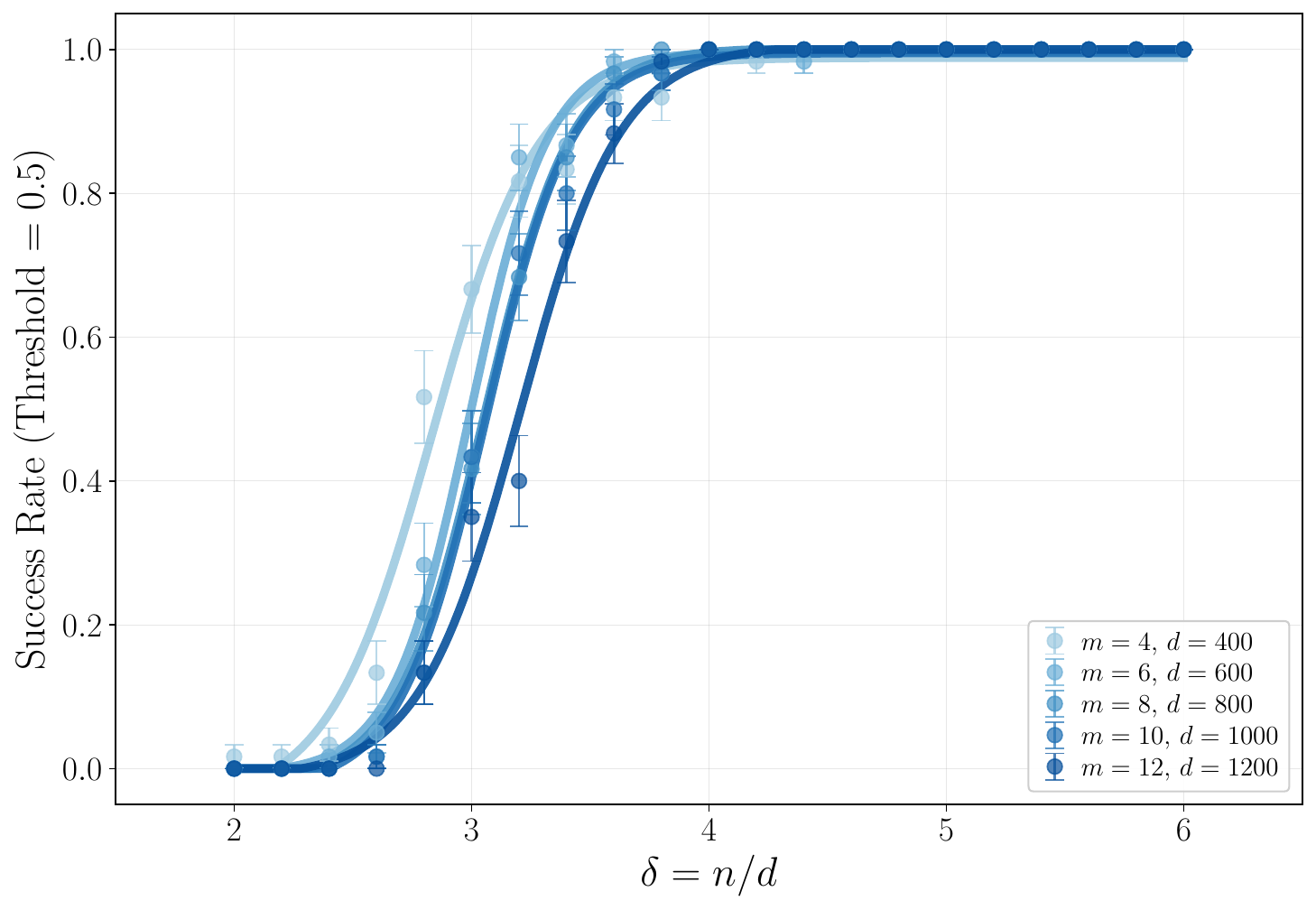}
  \hfill
  \includegraphics[width=0.49\textwidth]{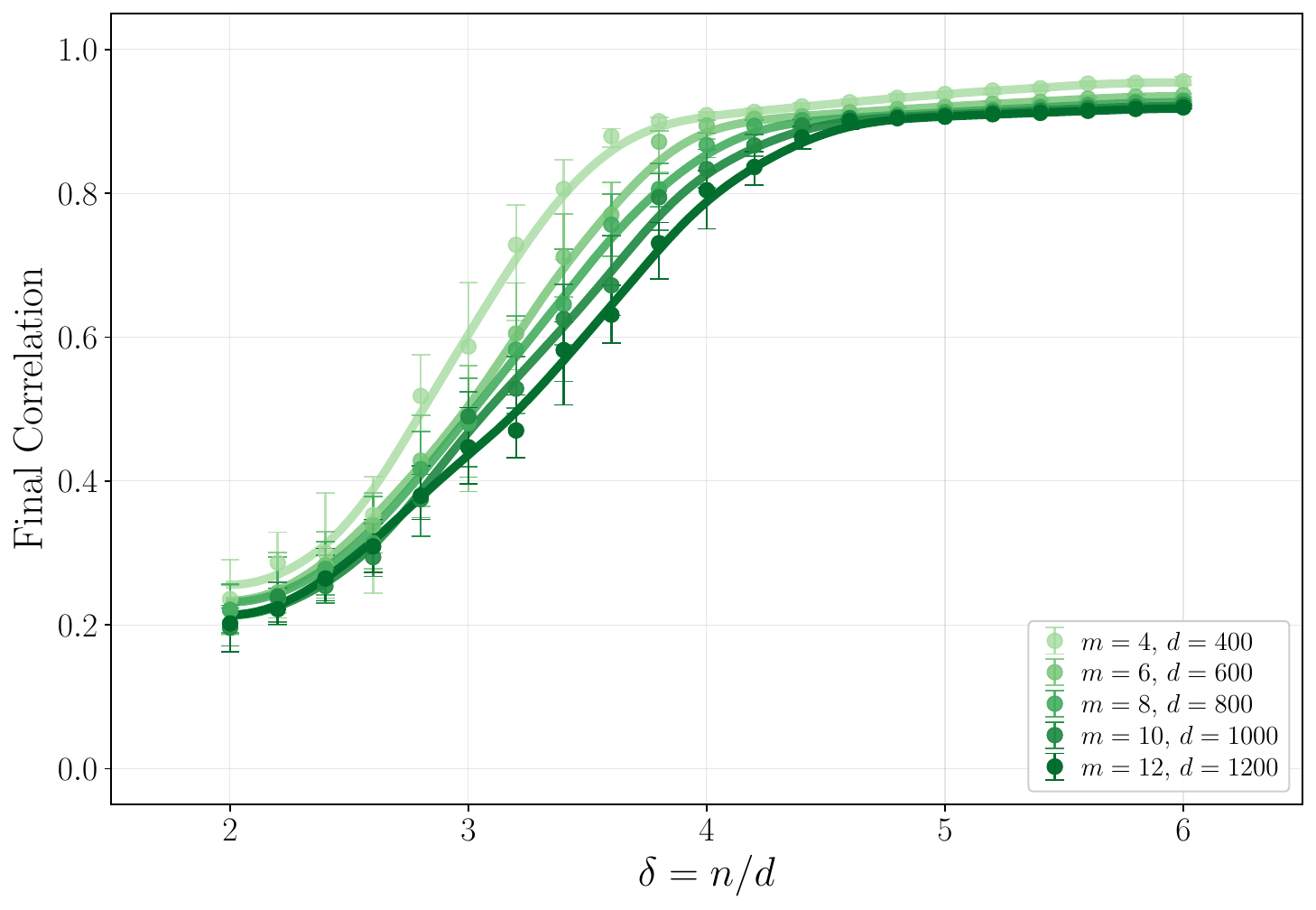}
  \vspace{-0.1cm}
  \caption{Phase transitions for proportional width scaling with $m/d = 0.01$. Left: Success rate (threshold $0.5$). Right: Final correlation (median with 30th/70th percentiles).}
  \label{fig:proportional_width}
\end{figure}
\begin{figure}[!t]
  \centering
  \includegraphics[width=0.49\textwidth]{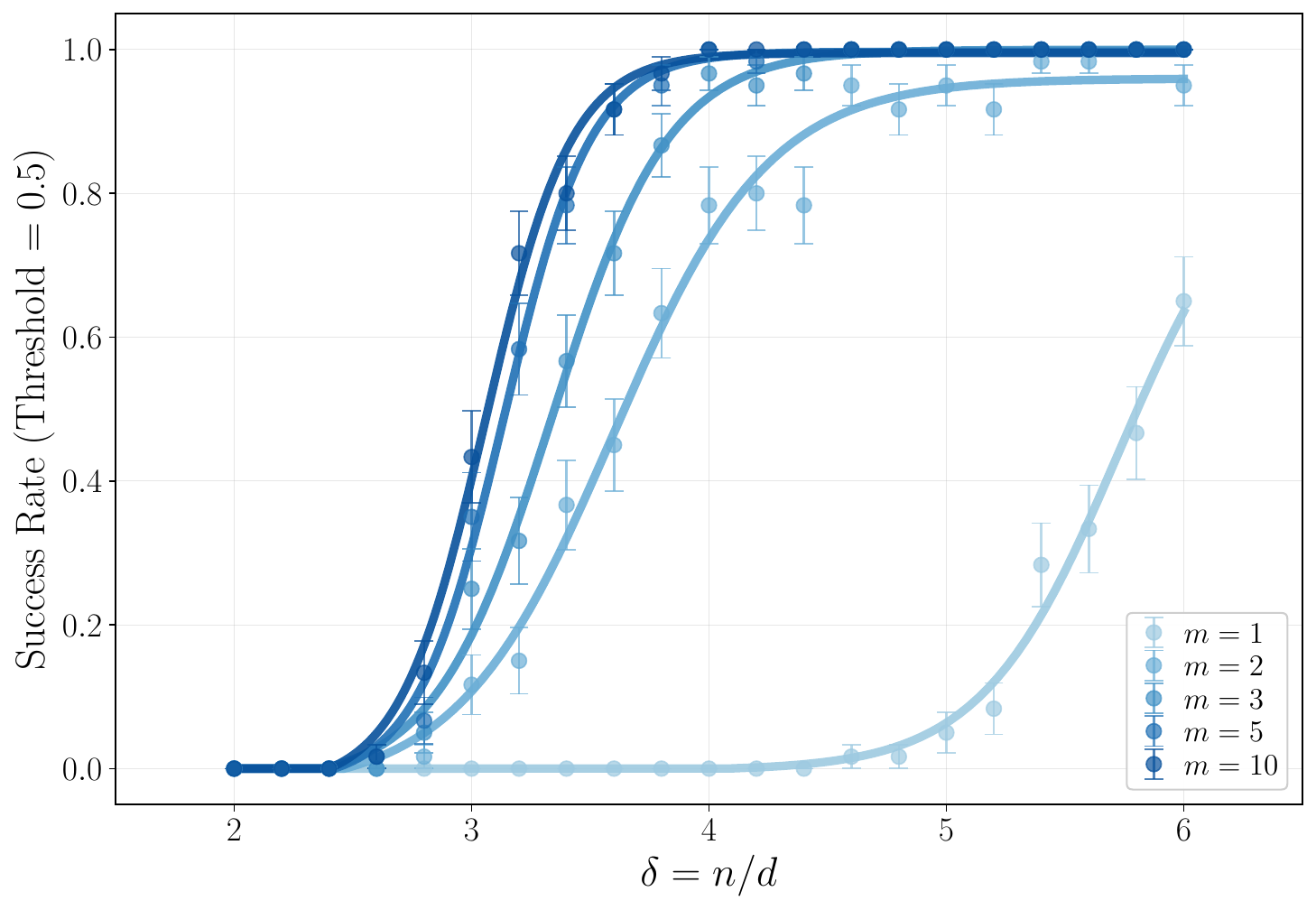}
  \hfill
  \includegraphics[width=0.49\textwidth]{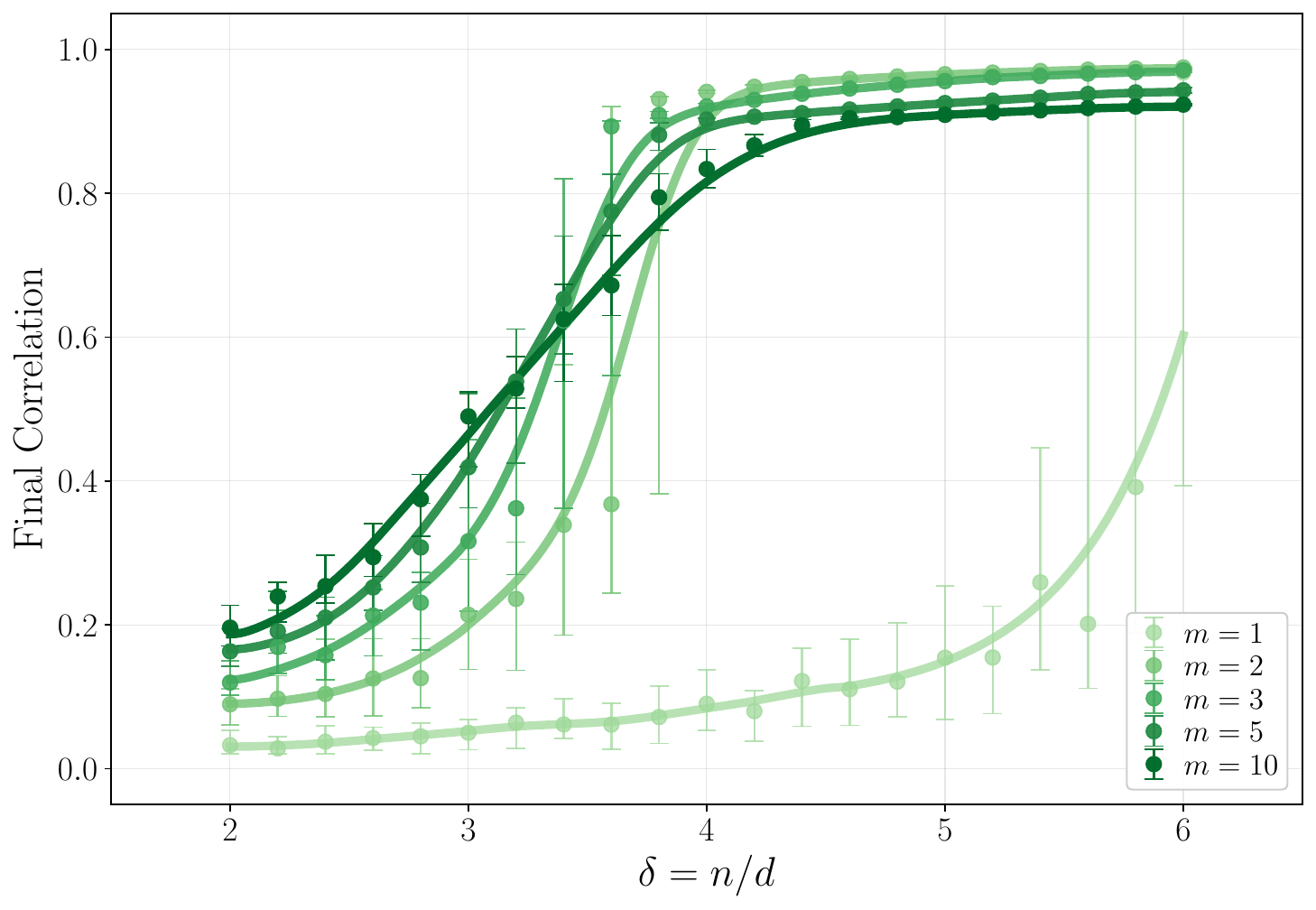}
  \vspace{-0.1cm}
  \caption{Phase transitions for $\GeLU$ activation with varying width at dimension $d=1000$. Left: Success rate with sigmoid fits showing how phase transitions change with increasing width. Right: Final correlation (median with 30th/70th percentiles) for varying widths.}
  \label{fig:phase_transition_large_width}
\end{figure}

\subsection{Grokking across activation functions}

Figure~\ref{fig:grokking_quad_relu} demonstrates that the same grokking 
phenomenon discussed in the main text extends to other activation functions. We show learning dynamics for $\mathsf{Quad}$ (left, $\delta=10$, $\eta=0.15$) and $\mathsf{ReLU}$ (right, $\delta=17.5$, $\eta=0.5$), both with $d=5000$. Both experiments exhibit the two-stage pattern: an initial phase where the misalignment $1-\rho(t)$ remains close to $1$ with a constant gap between training and test losses, followed by a rapid transition where the network learns the signal direction and achieves near-perfect generalization.

\begin{figure}[!t]
  \centering
  \begin{minipage}{0.48\textwidth}
    \centering
    \includegraphics[width=\textwidth]{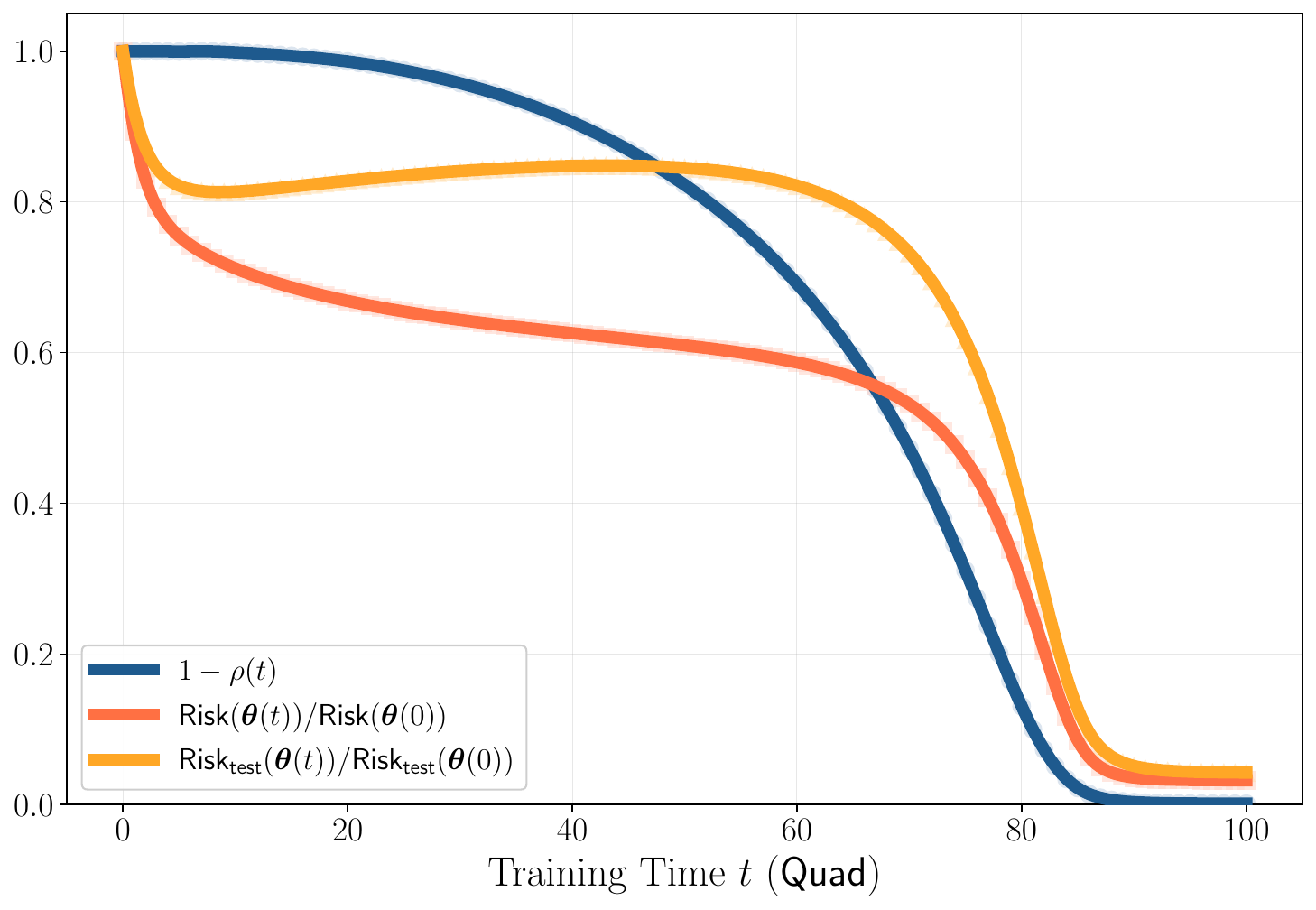}
  \end{minipage}
  \begin{minipage}{0.48\textwidth}
    \centering
    \includegraphics[width=\textwidth]{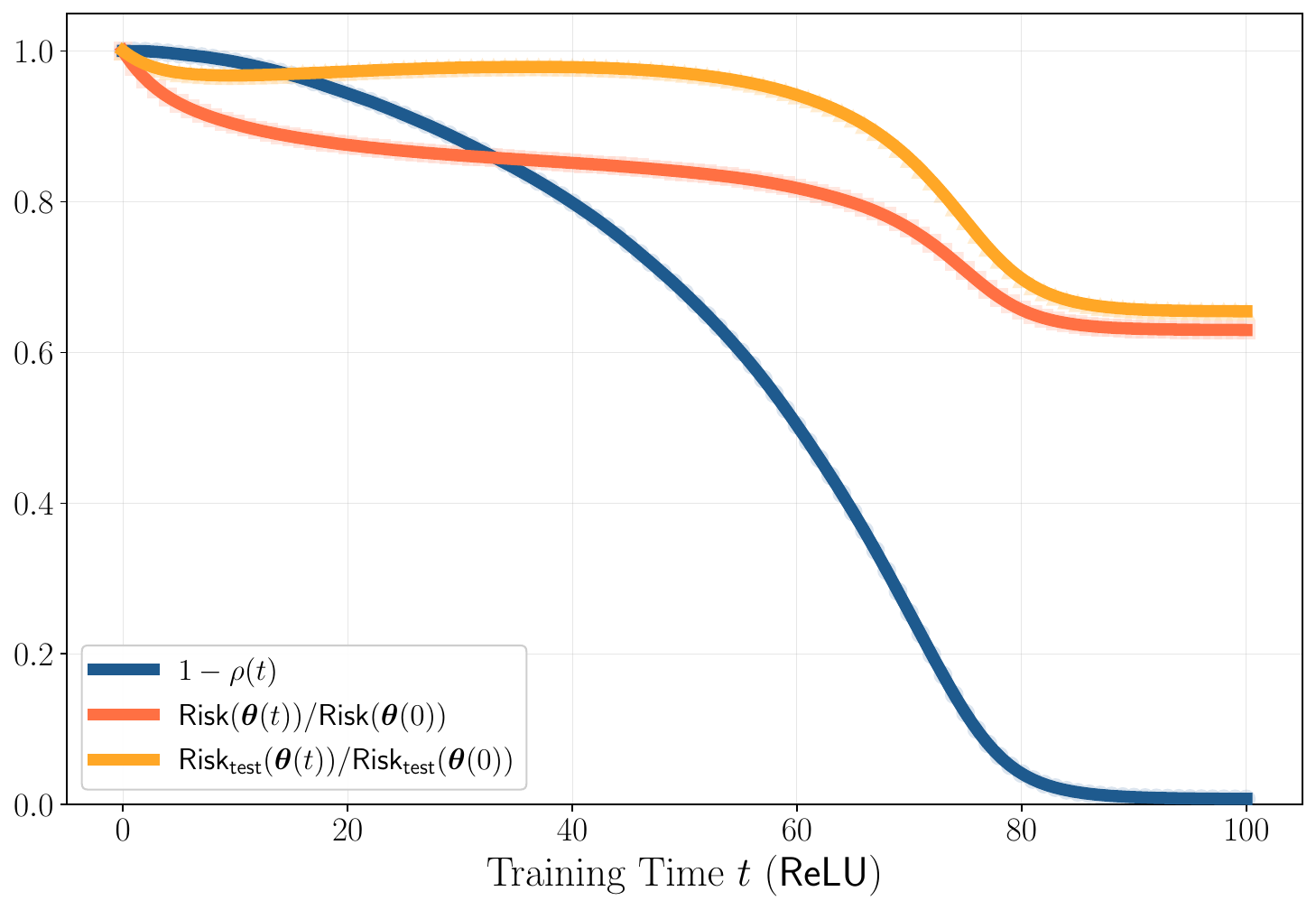}
  \end{minipage}
  \caption{Grokking dynamics for different activations. Left: $\mathsf{Quad}$ with $d=5000$, $\delta=10$, $\eta=0.15$. Right: $\mathsf{ReLU}$ with $d=5000$, $\delta=17.5$, $\eta=0.5$.}
  \label{fig:grokking_quad_relu}
\end{figure}

\end{document}